\documentclass[a4paper,12pt]{report}
\usepackage{utdiss}

\usepackage{amsfonts}           
\usepackage{amscd}              
\usepackage{amsmath}            
\usepackage{amsthm}             
\usepackage{amssymb}
\usepackage{hyperref,nameref,cleveref}           
\hypersetup{
    pdftoolbar=true,        
    pdfmenubar=true,        
    pdffitwindow=false,     
    pdfstartview={FitH},    
    pdftitle={title},    
    pdfnewwindow=true,      
    colorlinks=true,
    linkcolor=blue,
    citecolor=blue,
    filecolor=blue,
    urlcolor=blue,
}
\usepackage{graphicx}           
\usepackage{lmodern}            
\usepackage{makeidx}            

\usepackage{blindtext}          
\usepackage{layout}             
\usepackage{url}                
\usepackage{verbatim}           

\theoremstyle{plain}
\newtheorem{theorem}{Theorem}[chapter]
\newtheorem{corollary}[theorem]{Corollary}
\newtheorem{lemma}[theorem]{Lemma}
\newtheorem{proposition}[theorem]{Proposition}

\theoremstyle{definition}
\newtheorem{definition}{Definition}[chapter]

\theoremstyle{remark}
\newtheorem{remark}{Remark}[chapter]

\numberwithin{equation}{chapter}
\crefformat{equation}{(#2#1#3)}

\newtheorem{example}{Example}
\newtheorem{assumption}{Assumption}

\newcommand{\C}{\mathbb{C}} 

\newcommand{\E}{\mathbb{E}}

\newcommand{\II}{\mathbb{I}}

\newcommand{\N}{\mathbb{N}} 
\newcommand{\PP}{\mathbb{P}} 
\newcommand{\R}{\mathbb{R}} 
\newcommand{\SSS}{\mathbb{S}}

\newcommand{\bb}{\mathbf{b}}

\newcommand{\eb}{\mathbf{e}}

\newcommand{\gb}{\mathbf{g}}

\newcommand{\pb}{\mathbf{p}}
\newcommand{\qb}{\mathbf{q}}
\newcommand{\rb}{\mathbf{r}}
\renewcommand{\sb}{\mathbf{s}}

\newcommand{\ub}{\mathbf{u}}
\newcommand{\vb}{\mathbf{v}}
\newcommand{\wb}{\mathbf{w}}
\newcommand{\xb}{\mathbf{x}}
\newcommand{\yb}{\mathbf{y}}
\newcommand{\zb}{\mathbf{z}}

\newcommand{\Ab}{\mathbf{A}}
\newcommand{\Bb}{\mathbf{B}}
\newcommand{\Cb}{\mathbf{C}}

\newcommand{\Eb}{\mathbf{E}}

\newcommand{\Gb}{\mathbf{G}}

\newcommand{\Ib}{\mathbf{I}}

\newcommand{\Lb}{\mathbf{L}}
\newcommand{\Mb}{\mathbf{M}}
\newcommand{\Nb}{\mathbf{N}}

\newcommand{\Pb}{\mathbf{P}}
\newcommand{\Qb}{\mathbf{Q}}
\newcommand{\Rb}{\mathbf{R}}
\newcommand{\Sb}{\mathbf{S}}
\newcommand{\Tb}{\mathbf{T}}
\newcommand{\Ub}{\mathbf{U}}
\newcommand{\Vb}{\mathbf{V}}
\newcommand{\Wb}{\mathbf{W}}
\newcommand{\Xb}{\mathbf{X}}
\newcommand{\Yb}{\mathbf{Y}}
\newcommand{\Zb}{\mathbf{Z}}

\newcommand{\Acal}{\mathcal{A}}
\newcommand{\Bcal}{\mathcal{B}}
\newcommand{\Ccal}{\mathcal{C}}
\newcommand{\Dcal}{\mathcal{D}}
\newcommand{\Ecal}{\mathcal{E}}
\newcommand{\Fcal}{\mathcal{F}}

\newcommand{\Hcal}{\mathcal{H}}
\newcommand{\Ical}{\mathcal{I}}

\newcommand{\Ncal}{\mathcal{N}}

\newcommand{\Scal}{{\mathcal{S}}}

\newcommand{\Ucal}{\mathcal{U}}
\newcommand{\Vcal}{\mathcal{V}}
\newcommand{\Wcal}{\mathcal{W}}
\newcommand{\Xcal}{\mathcal{X}}
\newcommand{\Ycal}{\mathcal{Y}}
\newcommand{\Zcal}{\mathcal{Z}}

\newcommand{\Gfrak}{\mathfrak{G}}

\newcommand{\Rfrak}{\mathfrak{R}}

\newcommand{\eps}{\epsilon}

\newcommand{\betab}{\boldsymbol{\beta}}

\newcommand{\deltab}{\boldsymbol{\delta}}
\newcommand{\epsb}{\boldsymbol{\eps}}

\newcommand{\zetab}{\boldsymbol{\zeta}}

\newcommand{\thetab}{\boldsymbol{\theta}}
\newcommand{\vthetab}{\boldsymbol{\vartheta}}

\newcommand{\nub}{\boldsymbol{\nu}}

\newcommand{\sigmab}{\boldsymbol{\sigma}}

\newcommand{\Gammab}{\boldsymbol{\Gamma}}
\newcommand{\Deltab}{\boldsymbol{\Delta}}

\newcommand{\Pib}{\boldsymbol{\Pi}}

\newcommand{\Sigmab}{\boldsymbol{\Sigma}}

\newcommand{\Phib}{\boldsymbol{\Phi}}

\newcommand{\Omegab}{\boldsymbol{\Omega}}

\newcommand{\dfeq}{\triangleq}
\newcommand{\aleq}{\preccurlyeq}
\newcommand{\ageq}{\succcurlyeq}
\newcommand{\pinv}{\dagger}
\newcommand*{\poly}{\mathop{\mathrm{poly}}}

\newcommand*{\argmin}{\mathop{\mathrm{argmin}}}
\newcommand*{\argmax}{\mathop{\mathrm{argmax}}}

\newcommand{\sgn}{\mathop{\mathrm{sign}}}
\newcommand{\tr}{\mathop{\mathrm{tr}}}
\newcommand{\diag}{\mathop{\mathrm{diag}}}
\newcommand{\rank}{\mathop{\mathrm{rank}}}
\newcommand{\range}{\mathop{\mathrm{Range}}}

\newcommand{\nnz}{\mathop{\mathrm{nnz}}}

\newcommand{\spn}{\mathop{\mathrm{span}}}
\newcommand{\row}{\mathop{\mathrm{row}}}
\newcommand{\col}{\mathop{\mathrm{col}}}
\newcommand{\nucnorm}[1]{{\left\vert\kern-0.25ex\left\vert\kern-0.25ex\left\vert #1 
    \right\vert\kern-0.25ex\right\vert\kern-0.25ex\right\vert}}

\newcommand{\wh}[1]{\widehat{#1}}
\newcommand{\wt}[1]{\widetilde{#1}}

\newcommand{\eg}{\emph{e.g.}\xspace}
\newcommand{\ie}{\emph{i.e.}\xspace}
\newcommand{\iid}{\emph{i.i.d.}\xspace}
\newcommand{\cf}{\emph{cf.}\xspace}

\newcommand{\whp}{\emph{w.h.p.}\xspace}

\newcommand{\rbr}[1]{\left(#1\right)}
\newcommand{\sbr}[1]{\left[#1\right]}
\newcommand{\cbr}[1]{\left\{#1\right\}}
\newcommand{\nbr}[1]{\left\|#1\right\|}
\newcommand{\abbr}[1]{\left\vert#1\right\vert}
\newcommand{\abr}[1]{\langle#1\rangle}
\newcommand{\nnbr}[1]{{\left\vert\kern-0.25ex\left\vert\kern-0.25ex\left\vert #1 \right\vert\kern-0.25ex\right\vert\kern-0.25ex\right\vert}}

\newcommand{\rsep}[2]{\left(#1 \middle\vert #2 \right)}

\newcommand{\ssepp}[2]{\left[#1 ~\middle\vert~ #2 \right]}
\newcommand{\csepp}[2]{\left\{#1 ~\middle\vert~ #2 \right\}}

\newcommand{\cpar}[1]{\{#1\}}
\newcommand{\bpar}[1]{\Big( #1 \Big)}

\newcommand{\norm}[1]{\|#1\|}

\newcommand{\bmat}[1]{\begin{bmatrix} #1 \end{bmatrix}}

\usepackage{xcolor}
\definecolor{commentcolor}{RGB}{110,154,155}   

\definecolor{cyan}{rgb}{0.0, 1.0, 1.0}
\definecolor{magenta}{rgb}{0.79, 0.08, 0.48}
\definecolor{cssgreen}{rgb}{0.0, 0.5, 0.0}
\newcommand{\bs}[1]{\boldsymbol{#1}}
\renewcommand{\b}{\textbf}
\renewcommand{\t}[1]{\text{#1}}

\newcommand{\cid}[2]{\wh{#1}_{*,#2}}
\newcommand{\rid}[2]{\wh{#1}_{#2,*}}
\newcommand{\cur}[2]{\wt{#1}_{#2}}
\newcommand{\tsid}[2]{\wh{#1}_{#2}}
\newcommand{\oc}[1]{{#1^{\perp}}}
\newcommand{\pobx}{\Pb_{X}}
\newcommand{\pobc}{\Pb_{C}}
\usepackage{xspace}
\newcommand{\pqr}{CPQR\xspace}
\newcommand{\plu}{LUPP\xspace}
\newcommand{\ortho}{\mathop{\mathrm{ortho}}}
\usepackage{tikz} 
\usetikzlibrary{trees, arrows.meta} 

\newcommand{\EER}{\t{EER}}
\newcommand{\dacop}{T^{dac}_{\widetilde \Acal, \Xb}}
\newcommand{\dau}{d_{\textit{aug}}}

\newcommand{\Aemp}{\Acal}

\newcommand{\hgdacfin}{\wh{h}^{dac}}

\newcommand{\nbh}{\mathit{NB}}
\newcommand{\Cnull}{C_{\Ncal}}
\newcommand{\herm}{da-erm}
\newcommand{\Rndac}{\Rfrak}
\newcommand{\covtr}{\Sigmab_\Xb}
\newcommand{\covaug}{\Sigmab_{\Deltab}}
\newcommand{\covaugwt}{\Sigmab_{\wt\Deltab}}
\newcommand{\covall}{\Sigmab_{\wt\Acal\rbr{\Xb}}}
\newcommand{\covs}{\Sigmab_{\Sb}}
\newcommand{\projnull}{\Pb^{\perp}_{\Deltab}}
\newcommand{\projrg}{\Pb_{\Deltab}}

\newcommand{\projAX}{\Pb_{\wt\Acal\rbr{\Xb}}}
\newcommand{\constmis}{C_{\textit{mis}}}
\newcommand{\Pgt}{P}

\newcommand{\Hred}{\Hcal_{\textit{dac}}}
\newcommand{\kernel}{\mathop{\mathrm{Null}}}

\newcommand{\dsc}{\textit{DSC}}
\newcommand{\iloc}[2]{{#1}_{[#2]}}
\newcommand{\iffun}[1]{\II\cbr{#1}}

\newcommand{\wdac}{\textit{WAC}}

\newcommand{\Fnb}[1]{\Fcal_{\theta^*}\rbr{#1}}
\newcommand{\lambdac}{\lambda_{\textit{AC}}}
\newcommand{\lossce}{\ell_{\textit{CE}}}
\newcommand{\lossdice}{\ell_{\textit{DICE}}}
\newcommand{\regdac}{\ell_{\textit{AC}}}
\newcommand{\ours}{\textit{AdaWAC}\xspace}
\newcommand{\best}[1]{\textbf{#1}}

\usepackage{eucal}

\usepackage{import}
\usepackage{algorithm}
\usepackage[noend]{algorithmic} 
\usepackage[english]{babel} 
\usepackage{bm}
\usepackage{booktabs}    
\usepackage{url}
\usepackage{comment} 
\usepackage[inline]{enumitem}
\usepackage{fancyhdr}
\usepackage[T1]{fontenc}    
\usepackage{fullpage}
\usepackage{graphicx}
\usepackage[utf8]{inputenc} 
\usepackage{latexsym}
\usepackage{listings} 
\usepackage{mathtools}
\usepackage{multicol,multirow}
\usepackage[square,sort,comma,numbers]{natbib}
\usepackage{bibentry}
\usepackage{nicefrac}  
\usepackage[parfill]{parskip}
\usepackage{subcaption}
\usepackage{times}
\usepackage{titlesec} 
\usepackage{thmtools}
\usepackage{upquote}        
\usepackage{wrapfig}
\usepackage{xspace}
\usepackage{adjustbox}
\usepackage{indentfirst}
\author{Yijun Dong}                        
\address{ydong@utexas.edu}                    
\title{Randomized Dimension Reduction with Statistical Guarantees}      

\supervisor[Per-Gunnar Martinsson]{Rachel Ward}
\committeemembers
     [Joseph Kileel]
     [George Biros]
     {Yuji Nakatsukasa}                       

\previousdegrees{B.S., M.S.}
\graduationmonth{August}
\graduationyear{2023}
\typist{the author}

\makeindex

\begin{document} 

\copyrightpage

\commcertpage
\titlepage




\begin{acknowledgments} 
Words are insufficient to convey my appreciation to my advisors, Prof. Per-Gunnar Martinsson and Prof. Rachel Ward, for their guidance and support throughout my Ph.D. journey. Their knowledge and insights have been lighthouses that navigate me in the fascinating realms of numerical linear algebra and machine learning. They have provided invaluable advice for both my research and my career, beyond the scope of this thesis. They endow me with the horizon to broaden my view and explore different research areas. Their passion and devotion to research encourage me to follow and try pursuing an academic career.

In addition to my advisors, I have had the fortune to collaborate with and learn from some great applied mathematicians and computer scientists during my doctoral research, including Shuo Yang, Prof. Sujay Sanghavi, Prof. Inderjit Dhillon, Prof. Qi Lei, Yuege Xie, Prof. Yuji Nakatsukasa, Kevin Miller, Kate Pearce, and Chao Chen. Works presented in this thesis and beyond could not have been finished without their efforts. More importantly, their meticulousness and diligence motivate me as a researcher; while their erudition helps diversify my sight from various aspects.

In particular, I would like to give my profound thanks to Prof. Qi Lei, who has been an inspiring mentor, as well as a supportive friend, since the beginning of my graduate study, who introduced me to the splendid field of statistical learning theory, and who constantly influences me with her vision and attitude towards research.

Meanwhile, I am sincerely grateful for the constructive suggestions and generous help of Prof. Joseph Kileel, Prof. George Biros, and Prof. Yuji Nakatsukasa, together with my advisors, who kindly serve on my thesis committee. Especially, I would like to thank Prof. Yuji Nakatsukasa for the insightful discussions and guidance on several projects in numerical linear algebra.

I truly appreciate the experience in Prof. Martinsson's and Prof. Ward's research groups, both of which involve brilliant researchers and collaborative environments where I have learned millions. I am thankful to everyone in both groups, especially my collaborators among them: Yuege Xie, Kevin Miller, Kate Pearce, and Chao Chen; along with Anna Yesypenko, Ke Chen, Bowei Wu, Heather Wilber, Ruhui Jin, Amelia Henriksen, Xiaoxia Wu, and more for their enlightening advice at different stages of my graduate study. It is also my fortune to be a part of the broader Oden community, learning from its incredible variety of research directions while enjoying its inclusiveness.

Despite the fleeting time, my four years as an undergraduate student at Emory University before graduate school were irreplaceable for my professional and personal development. Looking back, I am genuinely grateful to my undergraduate advisors Prof. Effrosyni Seitaridou and Prof. Eric Weeks for patiently unveiling a corner of the research world to a curious undergraduate, as well as for encouraging me to follow my interests while pursuing doctoral study along a different direction. I was also fortunate enough to make some cherished friendships during my undergraduate, among which I am truly thankful to Xiaoyi Zhang whose wisdom and attitude on life have always inspired and motivated me since our paths first crossed. 

Finally, I would like to give my wholehearted gratitude to my parents. I owe them for being mostly away from home in the past nine years and probably more years to come, for the absence during their struggles in the COVID-19 pandemic, and everything. Their unconditional love and support have never been diminished by distance.

\end{acknowledgments}

\utabstract 
\indent
Large models and enormous data are essential driving forces of the unprecedented successes achieved by modern algorithms, especially in scientific computing and machine learning. Nevertheless, the growing dimensionality and model complexity, as well as the non-negligible workload of data pre-processing, also bring formidable costs to such successes in both computation and data aggregation. As the deceleration of Moore's Law slackens the cost reduction of computation from the hardware level, fast heuristics for expensive classical routines and efficient algorithms for exploiting limited data are increasingly indispensable for pushing the limit of algorithm potency. This thesis explores some of such algorithms for fast execution and efficient data utilization. 
\begin{enumerate}[nosep]
    \item From the \textbf{computational efficiency} perspective, we design and analyze fast randomized low-rank decomposition algorithms for large matrices based on ``matrix sketching'', which can be regarded as a dimension reduction strategy in the data space. These include the \textbf{randomized pivoting-based interpolative and CUR decomposition} discussed in \Cref{ch:cur} and the \textbf{randomized subspace approximations} discussed in \Cref{ch:svra}.
    \item From the \textbf{sample efficiency} perspective, we focus on learning algorithms with various incorporations of data augmentation that improve generalization and distributional robustness provably. Specifically, \Cref{ch:dac} presents a sample complexity analysis for \textbf{data augmentation consistency regularization} where we view sample efficiency from the lens of dimension reduction in the function space. 
    Then in \Cref{ch:adawac}, we introduce an \textbf{adaptively weighted data augmentation consistency regularization} algorithm for distributionally robust optimization with applications in medical image segmentation.
\end{enumerate}


\longtocentry                          
\tableofcontents
\listoftables
\listoffigures

\nobibliography*
\chapter{Overview}

\section{Computational Efficiency: Randomized Low-rank Decompositions}

Low-rank decompositions are dimension reduction techniques that unveil latent low-dimensional structures in large matrices, which are ubiquitous in various applications like principal component analysis and spectral clustering. However, to compute common matrix decompositions (\eg, SVD) of an $m \times n$ matrix, classical deterministic algorithms generally scale as $O(mn^2)$, making them untenable for large-scale problems. 

As a remedy, the ``matrix sketching'' framework~\citep{halko2011finding} embeds high-dimensional matrices into random low-dimensional subspaces via fast linear transforms, commonly known as randomized linear embeddings or (fast) Johnson-Lindenstrauss transforms. Some popular choices include Gaussian random matrices~\citep{indyk1998}, subsampled random trigonometric transforms~\citep{woolfe2008}, and sparse embeddings like count sketch~\citep{meng2013} and sparse sign matrices~\citep{clarkson2017}.
After such dimension reduction through randomized linear embedding, classical matrix decomposition algorithms can be executed efficiently, and low-rank approximations can be reconstructed without much compromise in accuracy. 

\Cref{ch:cur} and \Cref{ch:svra} of this thesis explore the potency of randomization with ``matrix sketching'' in two classical low-rank decomposition problems, namely the matrix skeletonization and the randomized subspace approximation.

\subsection{Randomized Pivoting-based Matrix Skeletonization}
Given a matrix $\Ab \in \R^{m \times n}$, the matrix skeletonization problem (\ie, interpolative decomposition (ID) and CUR decomposition) solves for low-rank ``natural bases'' formed by the original columns (and/or rows) of $\Ab$. Precisely, the goal is to identify column (or row) skeletons $\Cb = \Ab\rbr{:,J_s} \in \R^{m \times k}$ (or $\Rb = \Ab\rbr{I_s,:} \in \R^{k \times n}$, in MATLAB notation) indexed by $J_s \subset [n]$ (or $I_s \subset [m]$ where $\abbr{J_s} = \abbr{I_s} = k$) that serve as good bases,
\begin{align*}
    \nbr{\Ab - \Cb\Cb^\pinv \Ab}_F \le \rbr{1 + \eps} \min_{J_s \subset [n]} \nbr{\Ab - \Ab\rbr{:,J_s}\Ab\rbr{:,J_s}^\pinv \Ab}_F.
\end{align*}
Despite the NP-hardness~\citep{civril2013exponential} of identifying the nearly optimal skeleton selections like the row/column subset with the maximum spanning volume~\citep{goreinov1997}, there exists fast heuristics~\citep{mahoney2009,sorensen2014,voronin2017,derezinski2020dpp} that enjoys statistical guarantees and/or practical successes. 
In particular, randomization via ``matrix sketching'' plays a critical role in many of these heuristics~\citep{voronin2017,chen2020,dong2021simpler}. 

Randomized pivoting-based matrix skeletonization~\citep{sorensen2014,voronin2017,dong2021simpler} is a class of such fast heuristics widely used in scientific computing, whose general framework consists of two stages: 
\begin{enumerate}[label=(\roman*)]
    \item dimension reduction via sketching (\eg, constructing row sketch $\Xb = \Gammab \Ab \in \R^{l \times n}$ via a Gaussian random matrix $\Gammab$ with \iid~entries $\Gamma_{ij} \sim \Ncal\rbr{0,l^{-1}}$ and $l \ll m$) and
    \item greedy skeleton selection via pivoting on the reduced matrix sketch (\eg, applying LU with partial pivoting on $\Xb^\top$ and taking the first $\abbr{J_s}$ pivots as the column skeletonizations).
\end{enumerate}
In \Cref{ch:cur}, we first surveyed and compared different options for the two stages, \eg, sketching/randomized SVD~\citep{halko2011finding} for the dimension reduction stage and column pivoted QR (CPQR)/LU with partial pivoting (LUPP) for the pivoting stage.
Motivated by the systematic comparison, we then proposed a novel combination of sketching and LU with partial pivoting (LUPP) for efficient randomized matrix skeletonization. 
Compared to column pivoted QR (CPQR) commonly used in the existing algorithms, LUPP enjoys superior empirical efficiency and parallelizability~\citep{grigori2011calu, solomonik2011communication} while compromising the rank-revealing guarantees. Fortunately, for matrix skeletonization, such a trade-off between efficiency and rank-revealing property can be avoided via sketching.
In particular, we demonstrated that, instead of relying on rank-revealing properties of the pivoting scheme, \textit{the simple combination of sketching and LUPP exploits the spectrum-preserving capability of sketching and achieves considerable acceleration without compromising accuracy}.

\subsection{Randomized Subspace Approximations}
{Theoretical underpinnings of randomized subspace approximation} is another key aspect of analyzing randomized low-rank decompositions. A low-rank decomposition can be viewed as a bilinear combination of the associated low-rank bases for the column and row spaces that encapsulate key information for a wide range of tasks (\eg, canonical component analysis and leverage score sampling). Specifically, for a truncated SVD
\begin{align*}
    \underset{m \times n}{\Ab} \approx \Ab_k = \underset{m \times k}{\Ub_k}~~\underset{k \times k}{\Sigmab_k}~~\underset{k \times n}{\Vb_k^\top}, \quad
    \Ub_k^\top \Ub_k = \Vb_k^\top \Vb_k = \Ib_k, \quad
    \Sigmab_k = \diag\rbr{\sigma_1,\dots,\sigma_k},
\end{align*}
the left and right leading singular vectors $\Ub_k$ and $\Vb_k$ can be approximated efficiently via {randomized subspace approximations}.
In the basic version, randomized subspace approximations leverage the appealing property of randomized linear embeddings (\eg, a Gaussian embedding $\Omegab \in \R^{n \times l}$ with $\iid$ entries $\Omega_{ij} \sim \Ncal\rbr{0, l^{-1}}$) that, with moderate oversampling $l = k + O(1)$, sketching (with $q$ power iterations) $\Yb = (\Ab\Ab^\top)^q\Ab\Omegab$ captures the leading singular vectors $\Ub_k$ with high probability. Under orthonormalization at each iteration (commonly known as {randomized subspace iteration}~\citep[Algorithm 4.4]{halko2011finding}), $\range\rbr{\Yb}$ provides a numerical stable estimate for the leading singular subspace.

Canonical angles~\citep{golub2013} (formally defined in \Cref{def:canonical-angles}) are commonly used to quantify the difference between two subspaces of the same space, which provides natural error measures for the randomized subspace approximations. 
For instance, given arbitrary full-rank matrices $\Ub \in \R^{d \times l}$ and $\Vb \in \R^{d \times k}$ (assuming $k \leq l \leq d$ without loss of generality), the $k$ canonical angles $\angle\rbr{\Ub,\Vb} \triangleq \angle\rbr{\range(\Ub),\range(\Vb)}$ between their corresponding range subspaces in $\R^d$ are given by the spectra of $\Qb_{\Ub}^\top\Qb_{\Vb} \in \R^{l \times k}$ where the columns of $\Qb_{\Ub} \in \R^{d \times l}$ and $\Qb_{\Vb} \in \R^{d \times k}$ consist of orthonormal bases of $\Ub$ and $\Vb$, respectively. 
Precisely, for each $i \in [k]$, $\cos\angle_i\rbr{\Ub,\Vb} = \sigma_i\rbr{\Qb_{\Ub}^\top\Qb_{\Vb}}$, or equivalently, $\sin\angle_i\rbr{\Ub, \Vb} = \sigma_{k-i+1}\rbr{\rbr{\Ib-\Qb_{\Ub}\Qb_{\Ub}^\top}\Qb_{\Vb}}$ (\cf \cite{bjorck1973numerical} Section 3).

In \Cref{ch:svra}, we extended the existing analysis on the accuracy of singular vectors approximated by the randomized subspace iteration, in terms of the canonical angles $\angle\rbr{\Ub_k, \Yb}$ between the true and approximated leading singular subspaces $\range\rbr{\Ub_k}$ and $\range\rbr{\Yb}$. 
By casting a computational efficiency view on the bounds and estimates of canonical angles, we provided \textit{a set of prior probabilistic bounds that is not only asymptotically tight but also computable in linear time}. Moreover, we derived \textit{unbiased prior estimates}, along with \textit{residual-based posterior bounds}, of canonical angles that can be evaluated efficiently, while further demonstrating the empirical effectiveness of these bounds and estimates with numerical evidence.

\section{Sample Efficiency: Data Augmentation for Better Generalization}

Modern machine learning models, especially deep learning models, require substantially large amounts of samples for training. However, data collection and human annotation often come with non-negligible costs in practice. Therefore, \emph{sample efficiency} and \emph{generalization} are critical properties of learning algorithms. 
In the most basic setting, a learning algorithm is designed for recovering some unknown ground truths (\eg, descriptions of images) via sampling from some unknown distributions (\eg, images from the Internet with descriptions). The goal is to learn a prediction function (\eg, image captioning) that well approximates the ground truth by providing accurate predictions beyond the training samples (\eg, (in-distribution) generalization to unseen testing samples from the same distribution or out-of-distribution generalization to testing samples from related but different distributions), with as few training samples as possible (\ie, sample efficiency).

Since the seminal work \cite{krizhevsky2012imagenet}, \emph{data augmentation} has been a ubiquitous ingredient in many state-of-the-art machine learning algorithms~\citep{simard1998transformation,simonyan2014very,he2016deep,cubuk2019autoaugment,kuchnik2018efficient}. It started from simple transformations on samples (\eg, (random) perturbations, distortions, scales, crops, rotations, and horizontal flips on images) that roughly preserve the semantic information. More sophisticated variants were subsequently designed; a non-exhaustive list includes Mixup \citep{zhang2017mixup}, Cutout \citep{devries2017improved}, and Cutmix \citep{yun2019cutmix}.
Despite the known capability of improving generalization and sample efficiency empirically, the theoretical understanding of how data augmentation works remain limited due to the wide variety of domain-specific designs~\citep{sennrich2015improving,zhang2017mixup} and algorithmic choices of utilizing data augmentations~\citep{krizhevsky2012imagenet,sohn2020fixmatch,he2020momentum}. 
The classical wisdom interprets and leverages data augmentation as a natural expansion of the original training samples~\citep{krizhevsky2012imagenet, simard1998transformation, cubuk2019autoaugment, simonyan2014very, he2016deep}. 
However, simply enlargening the training set is not sufficient to explain the unprecedented successes (in comparison to the classical wisdom) recently achieved by an alternative line of data-augmentation-based learning algorithms --- \emph{data augmentation consistency regularization}~\citep{bachman2014learning, laine2016temporal, sajjadi2016regularization, sohn2020fixmatch, berthelot2021adamatch} --- that encourages similar predictions among the original sample and its augmentations.

\Cref{ch:dac} and \Cref{ch:adawac} of this thesis discuss the sample efficiency of data augmentation consistency regularization from the dimension reduction aspect, along with an application in medical image segmentation.

\subsection{Sample Efficiency of Data Augmentation Consistency Regularization}
In efforts to interpret the effect of different algorithmic choices on utilizing data augmentations, \Cref{ch:dac} conducts \emph{apple-to-apple comparisons} between two popular data-augmentation-based algorithms --- the empirical risk minimization on the augmented training set (DA-ERM) and the data augmentation consistency regularization (DAC) --- in the \emph{supervised} setting. 

Concretely, given a ground truth distribution $P: \Xcal \times \Ycal \to [0,1]$ and a well-specified function class $\Hcal \subseteq \csepp{h = f_h \circ \phi_h}{\phi_h: \Xcal \to \Wcal,~ f_h: \Wcal \to \Ycal}$ (\eg, a class of neural networks with a sufficiently expressive hidden-layer representation function $\phi_h\rbr{\cdot}$) such that for a given loss function $\ell: \Ycal \times \Ycal \to \R_{\ge 0}$, $h^* \dfeq \argmin_{h \in \Hcal} \E_{\rbr{\xb,y} \sim P} \sbr{\ell\rbr{h(\xb),y}} \in \Hcal$, let $\rbr{\xb_i, y_i}_{i \in [N]} \sim P\rbr{\xb, y}^N$ be a set of $N$ training samples drawn $\iid$ from $P$ and 
\begin{align*}
    \wt\Acal(\Xb) = \sbr{\xb_{1}; \cdots; \xb_{N}; \xb_{1,1}; \cdots; \xb_{N,1}; \cdots; \xb_{1, \alpha}; \cdots; \xb_{N, \alpha}} \in \Xcal^{(1+\alpha) N}
\end{align*}
be the features of its (random) augmentation (\ie, $\alpha \in \N$ additional augmentations per sample).
Considering the basic version of data augmentation which preserves the labels of original samples, DA-ERM directly includes the augmented samples in the training set and learns via empirical risk minimization (ERM),
\begin{align*}
    \wh{h}^{\herm} = \argmin_{h\in \Hcal} \sum_{i=1}^N \ell(h(\xb_i), y_i) + \sum_{i=1}^N\sum_{j=1}^{\alpha} \ell(h(\xb_{i,j}), y_i).
\end{align*}
Instead, DAC regularization rewards $h \in \Hcal$ that provides similar representations $\phi_h$ among data augmentations of the same sample,
\begin{align*}
    \wh{h}^{dac} = \argmin_{h\in\Hcal}\sum_{i=1}^{N}l(h(\xb_i), y_i) + \underbrace{\lambda\sum_{i=1}^N\sum_{j=1}^{\alpha} \varrho\rbr{\phi_h(\xb_i), \phi_h(\xb_{i, j})}}_{\textit{DAC regularization}},
\end{align*}
where $\varrho: \Wcal \times \Wcal \to \R_{\ge 0}$ is a metric associated with the metric space $\Wcal$ (\eg, the Euclidean distance).

Previous works~\citep{chen2020group,mei2021learning,bietti2021sample} generally view augmentations as groups endowed with Haar measures which inevitably assumed access to augmentations over the population. 
By contrast, we investigate a more realistic circumstance where both the training data and their augmentations are presented as finite random samples, $\rbr{\Xb, \yb} \sim P\rbr{\xb,y}^N$ and $\wt{\Acal}\rbr{\Xb} \in \Xcal^{(1+\alpha)N}$, by considering the DAC regularization as a reduction in the complexity of the function class $\Hcal$ (\eg, a dimension reduction in the linear regression setting). 
Apart from the well-known semi-supervised learning capability of DAC, in \Cref{ch:dac}, we demonstrated the \emph{intrinsic efficiency of DAC over DA-ERM in utilizing both samples and their augmentations}, even without unlabeled data, by establishing separations of sample complexities between DAC and DA-ERM in various settings, including different function classes (\eg, linear regression and neural networks), different data augmentations, as well as in-distribution and out-of-distribution generalization.

\subsection{Adaptively Weighted Data Augmentation Consistency Regularization}

Grounding the theoretical insight on data augmentation consistency regularization as guidance for algorithm design, in \Cref{ch:adawac}, we explore the \emph{combination of data augmentation consistency regularization} and \emph{sample reweighting} for a distributionally robust optimization setting commonly encountered in \emph{medical image segmentation}.

\emph{Concept shift} is a prevailing problem in natural tasks like medical image segmentation where samples usually come from different subpopulations with variant correlations between features and labels. A common type of concept shift in medical image segmentation is the ``information imbalance'' between the \emph{label-sparse} samples with few (if any) segmentation labels and the \emph{label-dense} ones with plentiful labeled pixels.
Existing distributionally robust algorithms have focused on adaptively truncating/down-weighting the ``less informative'' (\ie, label-sparse) samples.  
To exploit data features of label-sparse samples more efficiently, in \Cref{ch:adawac}, we propose an adaptively weighted online optimization algorithm --- \ours --- to incorporate data augmentation consistency regularization in sample reweighting. As a simplified overview, when learning from a function class parameterized by $\theta \in \Theta$, \ours introduces a set of trainable weights $\betab \in [0,1]^n$ to balance the supervised (cross-entropy) loss $\lossce\rbr{\theta;\rbr{\xb_i,\yb_i}}$ and unsupervised consistency regularization $\regdac\rbr{\theta;\xb_i, A_{i,1}, A_{i,2}}$\footnote{Here, $A_{i,1}$ and $A_{i,2}$ denote the (random) augmentations associated with $\xb_i$ for each $i \in [n]$ (refer to \Cref{sec:problem_setup_adawac} for formal definitions).} of each sample $\rbr{\xb_i, \yb_i}$ separately:
\begin{align*}
\begin{split}
    &\min_{\theta \in \Theta}~ \max_{\betab \in [0,1]^n}~
    \frac{1}{n} \sum_{i=1}^n \beta_{i} \cdot \lossce\rbr{\theta;\rbr{\xb_i,\yb_i}} + (1-\beta_{i}) \cdot \regdac\rbr{\theta;\xb_i,A_{i,1},A_{i,2}}. 
\end{split}
\end{align*}
At the saddle point $\rbr{\wh\theta, \wh\betab}$ of the underlying objective, the weights assign label-dense samples to the supervised loss (\ie, $\wh\beta_{i}= 1$) and label-sparse samples to the unsupervised consistency regularization (\ie, $\wh\beta_{i}= 0$).
We provide a convergence guarantee by recasting the optimization as online mirror descent on a saddle point problem. Our empirical results further demonstrate that \ours not only enhances the segmentation performance and sample efficiency but also improves the robustness to concept shift on various medical image segmentation tasks with different UNet-style backbones.

Overall, this thesis is based on my works \cite{dong2021simpler,dong2022efficient,yang2022sample,dong2022adawac}. Beyond the score of this thesis, some related topics that I have been working on include knowledge distillation~\citep{dong2023cluster} and blockwise adaptive sampling~\citep{dong2023robust}

\chapter{Randomized Pivoting Algorithms for CUR and Interpolative Decompositions}\label{ch:cur}

\subsection*{Abstract}
Matrix skeletonizations like the interpolative and CUR decompositions provide a framework for low-rank approximation in which subsets of a given matrix's columns and/or rows are selected to form approximate spanning sets for its column and/or row space.
Such decompositions that rely on ``natural'' bases have several advantages over traditional low-rank decompositions with orthonormal bases, including preserving properties like sparsity or non-negativity, maintaining semantic information in data, and reducing storage requirements.
Matrix skeletonizations can be computed using classical deterministic algorithms such as column-pivoted QR, which work well for small-scale problems in practice, but suffer from slow execution as the dimension increases and can be vulnerable to adversarial inputs.
More recently, randomized pivoting schemes have attracted much attention, as they have proven capable of accelerating practical speed, scale well with dimensionality, and sometimes also lead to better theoretical guarantees. 
This manuscript provides a comparative study of various randomized pivoting-based matrix skeletonization algorithms that leverage classical pivoting schemes as building blocks.
We propose a general framework that encapsulates the common structure of these randomized pivoting-based algorithms and provides an a-posteriori-estimable error bound for the framework. Additionally, we propose a novel concretization of the general framework and numerically demonstrate its superior empirical efficiency.\footnote{This chapter is based on the following published journal paper: \\
\bibentry{dong2021simpler}~\citep{dong2021simpler}.}

\section{Introduction}

The problem of computing a low-rank approximation to a matrix is a classical one that has drawn increasing attention due to its importance in the analysis of large data sets.
At the core of low-rank matrix approximation is the task of constructing bases that approximately span the column and/or row spaces of a given matrix.
This manuscript investigates algorithms for low-rank matrix approximations with ``natural bases'' of the column and row spaces -- bases formed by selecting subsets of the actual columns and rows of the matrix.
To be precise, given an $m \times n$ matrix $\Ab$ and a target rank $k < \min(m,n)$, we seek to determine an $m \times k$ matrix $\Cb$ holding $k$ of the columns of $\Ab$, and a $k \times n$ matrix $\Zb$ such that
\begin{equation}
\label{eq:colID}
\begin{array}{ccccccc}
\Ab &\approx& \Cb & \Zb.\\
m\times n && m\times k & k\times n
\end{array}
\end{equation}
We let $J_{\rm s}$ denote the index vector of length $k$ that
identifies the $k$ chosen columns, so that, in MATLAB notation,
\begin{equation}
\label{eq:Js}
\Cb = \Ab(\colon,J_{\rm s}).
\end{equation}
If we additionally identify an index vector $I_{\rm s}$ that marks a subset of the rows that forms an approximate basis for the row space of $\Ab$, we can then form the ``CUR'' decomposition
\begin{equation}
\label{eq:CUR}
\begin{array}{ccccccc}
\Ab & \approx & \Cb & \Ub & \Rb,\\
m\times n && m\times k & k\times k & k\times n
\end{array}
\end{equation}
where $\Ub$ is a $k\times k$ matrix, and
\begin{equation}
\label{eq:Is}
\Rb = \Ab(I_{\rm s},\colon).
\end{equation}
The decomposition (\ref{eq:CUR}) is also known as a ``matrix skeleton'' \cite{goreinov1997} approximation (hence  the subscript ``s'' for ``skeleton'' in $I_{\rm s}$ and $J_{\rm s}$).
Matrix decompositions of the form (\ref{eq:colID}) or (\ref{eq:CUR}) possess several compelling properties:
(i) Identifying $I_{\rm s}$ and/or $J_{\rm s}$ is often helpful in data interpretation.
(ii) The decompositions (\ref{eq:colID}) and (\ref{eq:CUR}) preserve important properties of the matrix $\Ab$. For instance, if $\Ab$ is sparse/non-negative, then $\Cb$ and $\Rb$ are also sparse/non-negative.
(iii) The decompositions (\ref{eq:colID}) and (\ref{eq:CUR}) are often memory efficient.
In particular, when the entries of $\Ab$ itself are available, or can inexpensively be computed or retrieved, then once $J_{\rm s}$ and $I_{\rm s}$ have been determined, there is no need to store $\Cb$ and $\Rb$ explicitly.

Deterministic techniques for identifying close-to-optimal index vectors $I_{\rm s}$ and $J_{\rm s}$ are well established. Greedy algorithms
such as the classical column pivoted QR (\pqr) \cite[Sec.~5.4.1]{golub2013},
and variations of LU with complete pivoting \cite{trefethen1997, zhao2005} 
often work well in practice. There also exist specialized pivoting schemes that come with strong theoretical performance guarantees \cite{gu1996}.

While effective for smaller dense matrices, classical techniques based on
pivoting become computationally inefficient as the matrix sizes grow. The
difficulty is that a global update to the matrix is in general required
before the next pivot element can be selected. The situation becomes
particularly dire for sparse matrices, as each update tends to create
substantial fill-in of zero entries.

To better handle large matrices, and huge sparse matrices in particular,
a number of algorithms based on
\textit{randomized sketching} have been proposed in recent years.
The idea is to extract a ``sketch'' $\Xb$ of the matrix that is
far smaller than the original matrix, yet contains enough information
that the index vectors $I_{\rm s}$ and/or $J_{\rm s}$ can be determined
by using the information in $\Xb$ alone. Examples include:

\begin{enumerate}
\item \textit{Discrete empirical interpolation method (DEIM):}
The sketching step consists of computing approximations to the dominant left and right
singular vectors of $\Ab$, for instance using the
randomized SVD (RSVD) \cite{halko2011finding,Liberty20167}.
Then a greedy pivoting-based scheme is used to pick the index sets $I_{\rm s}$ and $J_{\rm s}$
\cite{drmac2016qdeim,sorensen2014}.
\item \textit{Leverage score sampling:}
Again, the procedures start by computing approximations to the dominant left and right singular vectors of $\Ab$ through a randomized scheme. Then these approximations are used to compute probability distributions on the row and/or column indices, from which a random subset of columns and/or rows is sampled.
\item \textit{Pivoting on a random sketch:}
With a random matrix $\Gammab \in \mathbb{R}^{k\times m}$ drawn from some appropriate distribution, a sketch of $\Ab$ is formed via $\Xb = \Gammab\Ab$. Then, a classical pivoting strategy such as the \pqr is applied on $\Xb$ to identify a spanning set of columns.
\end{enumerate}

The existing literature \cite{anderson2017,chen2020,cohen2014,derezinski2020improved,derezinski2020dpp,drineas2012,drmac2016qdeim,gu1996,mahoney2009,sorensen2014,voronin2017} presents compelling evidence in support of each of these frameworks, in the form of mathematical theory and/or empirical numerical experiments.

The objective of the present manuscript is to organize different strategies and to conduct a systematic comparison, with a focus on their empirical accuracy and efficiency.
In particular, we compare different strategies for extracting a random sketch, such as techniques based on
Gaussian random matrices \cite{halko2011finding,indyk1998,martinsson2020,woodruff2015},
random fast transforms \cite{boutsidis2013srtt,halko2011finding,martinsson2020, rokhlin2008,tropp2011,woolfe2008}, and
random sparse embeddings \cite{clarkson2017,martinsson2020,meng2013, nelson2013,tropp2017a,woodruff2015}. We also compare different pivoting strategies
such as pivoted QR \cite{golub2013,voronin2017}
versus pivoted LU \cite{chen2020,sorensen2014}. Finally, we compare how well sampling-based schemes perform in relation to pivoting-based schemes.

In addition to providing a comparison of existing methods, the manuscript proposes a general framework that encapsulates the common structure shared by some popular randomized pivoting-based algorithms and presents an a-posteriori-estimable error bound for the framework. Moreover, the manuscript introduces a novel concretization of the general framework that is faster in execution than the schemes of \cite{sorensen2014, voronin2017} while picking equally close-to-optimal skeletons in practice.
In its most basic version, our simplified
method for finding a subset of $k$ columns of $\Ab$ works as follows:

\vspace{2mm}

\begin{center}
\begin{minipage}{0.9\textwidth}
\textit{Sketching step:} Draw $\Gammab \in \mathbb{R}^{k\times m}$
from a Gaussian distribution and form $\Xb = \Gammab\Ab$.

\vspace{2mm}

\textit{Pivoting step:} Perform a \textit{partially pivoted} LU decomposition of
$\Xb^{\top} \in \R^{n \times k}$. Collect the chosen pivot indices in the index vector $J_{\rm s}$.
(Since partially pivoted LU picks rows of $\Xb^{*}$, $J_{\rm s}$ indicates 
columns of $\Xb$.)
\end{minipage}
\end{center}

\vspace{2mm}

What is particularly interesting about this process is that while the
LU factorization with partial pivoting (\plu) is \textit{not}
rank revealing for a general matrix $\Ab$, the randomized mixing done in the sketching step makes \plu excel at picking spanning columns. Furthermore, the randomness introduced by sketching empirically serves as a remedy for the vulnerability of classical pivoting schemes like \plu to adversarial inputs (\eg, the Kahan matrix \cite{kahan1966}).
The scheme can be accelerated further by incorporating a structured
random embedding $\Gammab$. Alternatively, its accuracy can be
enhanced by incorporating one of two steps of power iteration when
building the sample matrix $\Xb$.

The manuscript is organized as follows:
\Cref{sec:background} provides a brief overview of the interpolative and CUR decompositions (\Cref{subsec:cur}), along with some essential building blocks of the randomized pivoting algorithms, including randomized linear embeddings (\Cref{subsec:embedding}), randomized low-rank SVD (\Cref{subsec:rsvd}), and matrix decompositions with pivoting (\Cref{subsec:pivoting_mat_decomp}).
\Cref{sec:exist_cur_algo} reviews existing algorithms for matrix skeletonizations (\Cref{subsec:sampling_cur}, \Cref{subsec:pivoting_cur}), and introduces a general framework that encapsulates the structures of some randomized pivoting-based algorithms.
In \Cref{sec:cur_rand_lupp}, we propose a novel concretization of the general framework, and provide an a-posteriori-estimable bound for the associated low-rank approximation error.
With the numerical results in \Cref{sec:cur_experiments}, we first compare the efficiency of various choices for the two building blocks in the general framework: randomized linear embeddings (\Cref{subsec:exp_embedding}) and matrix decompositions with pivoting (\Cref{subsec:exp_pivot}). Then, we demonstrate the empirical advantages of the proposed algorithm by investigating the accuracy and efficiency of assorted randomized skeleton selection algorithms for the CUR decomposition (\Cref{subsec:exp_cur}).

\section{Background}\label{sec:background}
We first introduce some closely related low-rank matrix decompositions that rely on ``natural'' bases, including the CUR decomposition, and the column, row, and two-sided interpolative decompositions (ID) in \Cref{subsec:cur}.
\Cref{subsec:embedding} describes techniques for computing randomized sketches of matrices, based on which \Cref{subsec:rsvd} discusses the randomized construction of low-rank SVD. \Cref{subsec:pivoting_mat_decomp} describes how these can be used to construct matrix decompositions.
While introducing the background, we include proofs of some well-established facts that provide key ideas but are hard to extract from the context of relevant references.

\subsection{Notation}\label{subsec:notation}
Let $\Ab \in \R^{m \times n}$ be an arbitrary given matrix of rank $r \le \min\cbr{m,n}$, whose SVD is given by
\begin{align*}
    \Ab = \underset{m \times r}{\Ub_{A}}\ \underset{r \times r}{\Sigmab_{A}}\ \underset{r \times n}{\Vb_{A}^{\top}} = \sbr{\ub_{A,1},\dots,\ub_{A,r}} \diag\rbr{\sigma_{A,1},\dots,\sigma_{A,r}} \sbr{\vb_{A,1},\dots,\vb_{A,r}}^{\top}
\end{align*}
such that for any rank parameter $k \leq r$, we denote $\Ub_{A,k} \triangleq \sbr{\ub_{A,1},\dots,\ub_{A,k}}$ and $\Vb_{A,k} \triangleq \sbr{\vb_{A,1},\dots,\vb_{A,k}}$ as the orthonormal bases of the dimension-$k$ leading left and right singular subspaces of $\Ab$, while $\oc{\Ub_{A,k}} \triangleq\sbr{\ub_{A,k+1},\dots,\ub_{A,r}}$ and $\oc{\Vb_{A,k}} \triangleq\sbr{\vb_{A,k+1},\dots,\vb_{A,r}}$ as the orthonormal bases of the respective orthogonal complements. The diagonal submatrices consisting of the spectrum, $\Sigmab_{A,k} \triangleq \diag\rbr{\sigma_{A,1}, \dots, \sigma_{A,k}}$ and $\oc{\Sigmab_{A,k}} \triangleq \diag\rbr{\sigma_{A,k+1}, \dots, \sigma_{A,r}}$, follow analogously.
We denote $\Ab_k \triangleq \Ub_{A,k} \Sigmab_{A,k} \Vb_{A,k}^{\top}$ as the rank-$k$ truncated SVD that minimizes rank-$k$ approximation error of $\Ab$ \cite{eckart1936}.
Furthermore, we denote the spectrum of $\Ab$, $\sigma\rbr{\Ab}$, as a $r \times r$ diagonal matrix, while for each $i=1,\dots,r$, let $\sigma_i(\Ab)$ be the $i$-th singular value of $\Ab$.

For the QR factorization, given an arbitrary rectangular matrix 
$\Mb \in \R^{d \times l}$ with full column rank ($d \geq l$), let $\Mb = \sbr{\Qb_{M},\oc{\Qb_{M}}} \sbr{\Rb_{M};\b{0}}$ be a full QR factorization of $\Mb$ such that $\Qb_{M} \in \R^{d \times l}$ and $\oc{\Qb_{M}} \in \R^{d \times (d-l)}$ consist of orthonormal bases of the subspace spanned by the columns of $\Mb$ and its orthogonal complement. We denote $\ortho: \csepp{\Mb \in \R^{d \times l}}{\rank\rbr{\Mb} = l} \to \R^{d \times l}$ ($d \ge l$) as a map that identifies an orthonormal basis (not necessarily unique) for $\Mb$, $\ortho(\Mb) = \Qb_M$.

We adopt MATLAB notation for matrices throughout this work. Unless specified otherwise ($\eg$, with subscripts), we use $\nbr{\cdot}$ to represent either the spectral norm or the Frobenius norm ($\ie$, holding simultaneously for both norms).

\subsection{Interpolative and CUR decompositions} \label{subsec:cur}
We first recall the definitions of the interpolative and CUR decompositions of a given $m\times n$ real matrix $\Ab$.
After providing the basic definitions, we discuss first how well it is theoretically possible to do
low-rank approximation under the constraint that natural bases must be used.
We then briefly describe the further suboptimality incurred by standard algorithms.

\subsubsection{Basic definitions}
We consider low-rank approximations for $\Ab$ with column and/or row subsets as bases. Given an arbitrary linearly independent column subset $\Cb = \Ab\rbr{:,J_s}$ ($J_s \subset [n]$), 
the rank-$\abbr{J_s}$ column ID of $\Ab$ with respect to column skeletons $J_s$ can be formulated as,
\begin{align}\label{eq:def_column_id}
    \cid{\Ab}{J_s} \triangleq \Cb \Cb ^{\pinv} \Ab,
\end{align}
where $\Cb \Cb ^{\pinv}$ is the orthogonal projector onto the spanning subspace of column skeletons.
Analogously, given any linearly independent row subset $\Rb = \Ab\rbr{I_s,:}$ ($I_s \subset [m]$), the rank-$\abbr{I_s}$ column ID of $\Ab$ with respect to row skeletons $I_s$ takes the form
\begin{align}\label{eq:def_row_id}
    \rid{\Ab}{I_s} \triangleq \Ab \Rb^{\pinv} \Rb,
\end{align}
where $\Rb^{\pinv} \Rb$ is the orthogonal projector onto the span of row skeletons.
While with both column and row skeletons, we can construct low-rank approximations for $\Ab$ in two forms -- the two-sided ID and CUR decomposition: with $\abbr{I_s}=\abbr{J_s}$, let $\Sb \triangleq \Ab\rbr{I_s,J_s}$ be an invertible two-sided skeleton of $\Ab$ such that
\begin{align}
    \label{eq:def_two_sided_id}
    &\t{Two-sided ID:}
    &&\tsid{\Ab}{I_s,J_s} \triangleq  \rbr{\Cb \Sb^{-1}} \Sb \rbr{\Cb^{\pinv} \Ab}
    \\
    \label{eq:def_cur}
    &\t{CUR decomposition:}
    &&\cur{\Ab}{I_s,J_s} \triangleq  \Cb \rbr{\Cb^\pinv \Ab \Rb^\pinv} \Rb
\end{align}
where in the exact arithmetic, since $\Sb^{-1} \Sb = \Ib$, the two-sided ID is equivalent to the column ID characterized by $\Cb$, \ie, $\tsid{\Ab}{I_s,J_s} = \cid{\Ab}{J_s}$. Nevertheless, the two-sided ID $\tsid{\Ab}{I_s,J_s}$ and CUR decomposition $\cur{\Ab}{I_s,J_s}$ differ in both suboptimality and conditioning.

\begin{remark}[Suboptimality of ID versus CUR]\label{remark:id_cur_suboptimality}
For any given column and row skeletons $\Cb$ and $\Rb$, 
\begin{align}\label{eq:id_cur_suboptimality}
    \nbr{\Ab - \Cb\Cb^\pinv\Ab} \le \nbr{\Ab - \Cb\Cb^\pinv\Ab\Rb^\pinv\Rb} \le \rbr{\nbr{\Ab - \Cb\Cb^\pinv\Ab}^2 + \nbr{\Ab - \Ab\Rb^\pinv\Rb}^2}^{\frac{1}{2}}.
\end{align}
\end{remark}

\begin{proof}[Rationale for \Cref{remark:id_cur_suboptimality}]
We observe the simple orthogonal decomposition 
\begin{align*}
    \Ab - \Cb\Cb^\pinv\Ab\Rb^\pinv\Rb = \rbr{\Ib_m - \Cb\Cb^\pinv}\Ab + \Cb\Cb^\pinv\rbr{\Ab - \Ab\Rb^\pinv\Rb}
\end{align*}
where $\rbr{\Ib_m - \Cb\Cb^\pinv}$ and $\Cb\Cb^\pinv$ are orthogonal projectors. With the Frobenius norm,
\begin{align*}
    \nbr{\Ab - \Cb\Cb^\pinv\Ab\Rb^\pinv\Rb}_F^2 
    = &\nbr{\Ab - \Cb\Cb^\pinv\Ab}_F^2 + \nbr{\Cb\Cb^\pinv\rbr{\Ab - \Ab\Rb^\pinv\Rb}}_F^2 
    \\
    \le &\nbr{\Ab - \Cb\Cb^\pinv\Ab}_F^2 + \nbr{\Ab - \Ab\Rb^\pinv\Rb}_F^2,
\end{align*}    
while with the spectral norm
\begin{align*}
    \nbr{\Ab - \Cb\Cb^\pinv\Ab\Rb^\pinv\Rb}_2^2  
    = \max_{\nbr{\vb}_2 \le 1} \nbr{\rbr{\Ab - \Cb\Cb^\pinv\Ab}\vb}_2^2 + \nbr{\Cb\Cb^\pinv\rbr{\Ab - \Ab\Rb^\pinv\Rb} \vb}_2^2
\end{align*}    
where 
\begin{align*}
    \max_{\nbr{\vb}_2 \le 1} \nbr{\rbr{\Ab - \Cb\Cb^\pinv\Ab}\vb}_2^2 = 
    \nbr{\Ab - \Cb\Cb^\pinv\Ab}_2^2
\end{align*}
and 
\begin{align*}
    &\max_{\nbr{\vb}_2 \le 1} \nbr{\rbr{\Ab - \Cb\Cb^\pinv\Ab}\vb}_2^2 + \nbr{\Cb\Cb^\pinv\rbr{\Ab - \Ab\Rb^\pinv\Rb} \vb}_2^2 \\
    \le &\max_{\nbr{\vb}_2 \le 1} \nbr{\rbr{\Ab - \Cb\Cb^\pinv\Ab}\vb}_2^2 + \max_{\nbr{\vb}_2 \le 1} \nbr{\Cb\Cb^\pinv\rbr{\Ab - \Ab\Rb^\pinv\Rb} \vb}_2^2 \\ 
    \le &\nbr{\Ab - \Cb\Cb^\pinv\Ab}_2^2 + \nbr{\Ab - \Ab\Rb^\pinv\Rb}_2^2.
\end{align*}
\end{proof}

\begin{remark}[Conditioning of ID versus CUR]\label{remark:id_cur_stability}
The construction of CUR decomposition tends to be more ill-conditioned than that of two-sided ID. 
Precisely, for properly selected column and row skeletons $J_s$ and $I_s$, the corresponding skeletons $\Sb$, $\Cb$, and $\Rb$ share similar spectrum decay as $\Ab$, which is usually ill-conditioned in the context.
In the CUR decomposition, both the bases $\Cb$, $\Rb$ and the small matrix $\Cb^\pinv \Ab \Rb^\pinv$ in the middle tend to suffer from large condition numbers as that of $\Ab$. 
In contrast, the only potentially ill-conditioned component in the two-sided ID is $\Sb$ ($\ie$, despite being expressed in $\Sb^{-1}$ and $\Cb^\pinv$, $\rbr{\Cb \Sb^{-1}}$ and $\rbr{\Cb^\pinv \Ab}$ in \Cref{eq:def_two_sided_id} are well-conditioned, and can be evaluated without direct inversions). 
\end{remark}

\begin{remark}[Stable CUR]\label{remark:stable_cur}
Numerically, the stable construction of a CUR decomposition $\cur{\Ab}{I_s,J_s}$ can be conducted via (unpivoted) QR factorization of $\Cb$ and $\Rb$ (\cite{anderson2015spectral}, Algorithm 2): let $\Qb_C \in \R^{m \times \abbr{J_s}}$ and $\Qb_R \in \R^{n \times \abbr{I_s}}$ be matrices from the QR whose columns form orthonormal bases for $\Cb$ and $\Rb^{\top}$, respectively, then
\begin{align}\label{eq:cur_stable}
    \cur{\Ab}{I_s,J_s} = \Qb_C \rbr{\Qb_C^{\top} \Ab \Qb_R} \Qb_R^{\top}.
\end{align}
\end{remark}    



\subsubsection{Notion of suboptimality}
\label{sec:cssp}

Both interpolative and CUR decompositions share the common goal of identifying proper column and/or row skeletons for $\Ab$ whose column and/or row spaces are well covered by the respective spans of these skeletons.
Without loss of generality, we consider the column skeleton selection problem: for a given rank $k < r$, we aim to find a proper column subset, $\Cb = \Ab\rbr{:,J_s}$ ($J_s \subset [n]$, $\abbr{J_s} = k$), such that
\begin{align}\label{eq:def_cssp_goal_close_to_opt_err}
    \nbr{\Ab - \cid{\Ab}{J_s}} \leq \phi\rbr{k,m,n} \nbr{\Ab - \Ab_k}
\end{align}
where common choices of the norm $\nbr{\cdot}$ include the spectral norm $\nbr{\cdot}_2$ and Frobenius norm $\nbr{\cdot}_F$; $\phi\rbr{k,m,n}$ is a function with $\phi\rbr{k,m,n} \geq 1$ for all $k,m,n$, and depends on the choice of $\nbr{\cdot}$; and we recall that $\Ab_k \triangleq \Ub_{A,k} \Sigmab_{A,k} \Vb_{A,k}^{\top}$ yields the optimal rank-$k$ approximation error. Meanwhile, similar low-rank approximation error bounds are desired for the row ID $\rid{\Ab}{I_s}$, two-sided ID $\tsid{\Ab}{I_s,J_s}$, and CUR decomposition $\cur{\Ab}{I_s,J_s}$.

\subsubsection{Suboptimality of matrix skeletonization algorithms}
The suboptimality of column subset selection, as well as the corresponding ID and CUR decomposition, has been widely studied in a variety of literature.

Specifically, with $\Sb = \Ab(I_s,J_s)$ ($\abbr{I_s}=\abbr{J_s}=k$) being the maximal-volume submatrix in $\Ab$, the corresponding CUR decomposition (called pseudoskeleton component in the original paper) satisfies \Cref{eq:def_cssp_goal_close_to_opt_err} in $\nbr{\cdot}_2$ with $\phi = O\rbr{\sqrt{k}\rbr{\sqrt{m}+\sqrt{n}}}$ \cite{goreinov1997}. However, it was also pointed out in \cite{goreinov1997} that skeletons associated with the maximal-volume submatrix are not guaranteed to minimize the low-rank approximation error in \Cref{eq:def_cssp_goal_close_to_opt_err}.
Moreover, identification of the maximal-volume submatrix is known to be NP-hard \cite{civril2013exponential,civril2009}.

Nevertheless, from the pivoting perspective, the existence of rank-$k$ column IDs with $\phi = \sqrt{1+k(n-k)}$ can be shown constructively via the strong rank-revealing QR factorization \cite{gu1996}, which further provides a polynomial-time relaxation for constructing IDs with $\phi = O\rbr{ \sqrt{k(n-k)}}$.
From the sampling perspective, the existence of a rank-$k$ column ID with $\phi=\sqrt{(k+1)(n-k)}$ for $\nbr{\cdot}_2$ and $\phi=\sqrt{1+k}$ for $\nbr{\cdot}_F$ can be shown by upper bounding the expectation of $\nbr{\Ab - \cid{\Ab}{J_s}}$ for volume sampling \cite{deshpande2006}. Later on, polynomial-time algorithms were proposed for selecting such column skeletons \cite{deshpande2010, cortinovis2019lowrank}.

Furthermore, it was unveiled in the recent work \cite{derezinski2020improved} that the suboptimality factor $\phi(k,m,n)$ can exhibit a multiple-descent trend with respect to $k$ where depending on spectrum decay, $\phi(k,m,n)$ can be as tight as $\phi = O\rbr{k^{1/4}}$ for small $k$s; while for larger $k$s that fall in certain intervals, $\phi = \Omega(\sqrt{k})$ \cite{derezinski2020improved}.

\subsection{Randomized linear embeddings}\label{subsec:embedding}

For a given matrix $\Ab_k \in \R^{m \times n}$ of rank $k \leq \min\rbr{m,n}$ (typically we consider $k \ll \min(m,n)$), and a distortion parameter $\epsilon \in (0,1)$, a linear map $\Gammab: \R^m \to \R^l$ ($\ie$, $\Gammab \in \R^{l \times m}$, typically we consider $l \ll m$ for embeddings) is called an \textit{$\ell_2$ linear embedding} of $\Ab_k$ with distortion $\epsilon$ if
\begin{align}\label{eq:def_l2_rand_linear_embedding}
    (1 - \epsilon) \norm{\Ab_k\xb}_2 \leq \norm{\Gammab\Ab_k\xb}_2 \leq (1 + \epsilon) \norm{\Ab_k\xb}_2 \quad \forall\ \xb \in \R^n.
\end{align}
A distribution $\Scal$ over linear maps $\R^m \to \R^l$ (or equivalently, over $\R^{l \times m}$) generates \textit{randomized oblivious $\ell_2$ linear embeddings} (abbreviated as \textit{randomized linear embeddings}) if over $\Gammab \sim \Scal$, \Cref{eq:def_l2_rand_linear_embedding} holds for all $\Ab_k$ with at least constant probability. Given $\Ab_k$ and a randomized linear embedding $\Gammab \sim \Scal$, $\Gammab\Ab_k$ provides a (row) sketch of $\Ab_k$, and the process of forming $\Gammab\Ab_k$ is known as sketching \cite{woodruff2015, martinsson2020}.

Randomized linear embeddings are closely related to various concepts like the Johnson-Lindenstrauss lemma and the restricted isometry property, and are studied in a broad scope of literature. Some popular randomized linear embeddings (\cf \cite{martinsson2020} Section 8, 9) include:
\begin{enumerate}
    \item Gaussian embeddings: $\Gammab \in \R^{l \times m}$ with $\iid$ Gaussian entries drawn from $\Ncal(0, 1/l)$ \cite{indyk1998, halko2011finding, woodruff2015, martinsson2020};
    \item subsampled randomized trigonometric transforms (SRTT):
    \begin{align*}
        \Gammab = \sqrt{\frac{m}{l}} \Pib_{m \to l} \Tb \Phib \Pib_{m \to m}
    \end{align*}
    where
    $\Pib_{m \to l} \in \R^{l \times m}$ is a uniformly random selection of $l$ out of $m$ rows;
    $\Tb$ is an $m \times m$ unitary trigonometric transform ($\eg$, discrete Hartley transform for $\R$, and discrete Fourier transform for $\C$);
    $\Phib \triangleq \diag\rbr{\varphi_1,\dots,\varphi_m}$ with $\iid$ Rademacher random variables $\cbr{\varphi_i}_{i \in [m]}$ flips signs randomly; and
    $\Pib_{m \to m}$ is a random permutation \cite{woolfe2008, halko2011finding, rokhlin2008, tropp2011, boutsidis2013srtt, martinsson2020}; and
    \item sparse sign matrices: $\Gammab = \sqrt{\frac{m}{\zeta}} \sbr{\sb_1,\dots,\sb_m}$ for some $2 \leq \zeta \leq l$, with $\iid$ $\zeta$-sparse columns $\cbr{\sb_j \in \R^l}_{j \in [m]}$ constructed such that each $\sb_j$ is filled with $\zeta$ independent Rademacher random variables at uniformly random coordinates \cite{meng2013, nelson2013, woodruff2015, clarkson2017, tropp2017a, martinsson2020}.
\end{enumerate}
\Cref{tab:embedding_summary} summarizes lower bounds on $l$s that provide theoretical guarantee for \Cref{eq:def_l2_rand_linear_embedding}, along with asymptotic complexities of sketching, denoted as $T_s(l,\Ab_k)$, for these randomized linear embeddings.
In spite of the weaker guarantees for structured randomized embeddings ($\ie$, SRTTs and sparse sign matrices) in the theory by a logarithmic factor, from the empirical perspective, $l = \Omega\rbr{k/\epsilon^2}$ is usually sufficient for all the embeddings in \Cref{tab:embedding_summary} when considering tasks such as constructing randomized rangefinders (which we subsequently leverage for fast skeleton selection). For instance, it was suggested in \cite{halko2011finding, martinsson2020} to take $l=k+\Omega(1)$ ($\eg$, $l=k+10$) for Gaussian embeddings, $l=\Omega\rbr{k}$ for SRTTs, and $l=\Omega\rbr{k}$, $\zeta = \min\rbr{l,8}$ for sparse sign matrices in practice \cite{tropp2019streaming}.

\begin{table}
    \centering
    \caption{Lower bounds of $l$s that provide theoretical guarantee for \Cref{eq:def_l2_rand_linear_embedding}, and asymptotic complexities of sketching, $T_s(l,\Ab_k)$, for some common randomized linear embeddings.}
    \label{tab:embedding_summary}
    \begin{tabular}{c|c|c}
    \toprule
    Randomized linear embedding & Theoretical best dimension reduction & $T_s(l,\Ab_k)$ \\
    \hline
    Gaussian embedding & $l = \Omega\rbr{k/\epsilon^2}$ & $O(\nnz(\Ab_k) l)$ \\
    SRTT & $l = \Omega\rbr{k \log k / \epsilon^2}$ & $O(mn \log l)$ \\
    Sparse sign matrix & $l = \Omega\rbr{k \log k / \epsilon^2}$, $\zeta = \Omega\rbr{\log k / \epsilon}$ & $O(\nnz(\Ab_k) \zeta)$\\
    \bottomrule
    \end{tabular}
\end{table}

\subsection{Randomized rangefinder and low-rank SVD}\label{subsec:rsvd}
Given $\Ab \in \R^{m \times n}$, the randomized rangefinder problem aims to construct a matrix $\Xb \in \R^{l \times n}$ such that the row space of $\Xb$ aligns well with the leading right singular subspace of $\Ab$ \cite{martinsson2020}: $\nbr{\Ab - \Ab \Xb^{\pinv} \Xb}$ is sufficiently small for some unitary invariant norm $\nbr{\cdot}$ ($\eg$, $\nbr{\cdot}_2$ or $\nbr{\cdot}_F$). 
When $\Xb$ admits full row rank, we call $\Xb$ a rank-$l$ row basis approximator of $\Ab$. The well-known optimality result from \cite{eckart1936} demonstrated that, for a fixed rank $k$, the optimal rank-$k$ row basis approximator of $\Ab$ is given by its leading $k$ right singular vectors: $\nbr{\Ab - \Ab \Vb_{A,k} \Vb_{A,k}^{\top}}_F^2 = \nbr{\Ab - \Ab_k}_F^2 = \sum_{i=k+1}^{\min\cbr{m,n}} \sigma_i\rbr{\Ab}^2$.

A row sketch $\Xb = \Gammab \Ab$ generated by some proper randomized linear embedding $\Gammab$ is known to serve as a good solution for the randomized rangefinder problem with high probability. For instance, with the Gaussian embedding, a small constant oversampling $l-k \geq 4$ is sufficient for a good approximation \cite{halko2011finding}:
\begin{align}\label{eq:rand_rangefinder_error_bound}
    \E \sbr{\nbr{\Ab - \Ab \Xb^{\pinv} \Xb}_F^2} \leq \frac{l-1}{l-k-1} \nbr{\Ab - \Ab_k}_F^2,
\end{align}
and moreover, $\nbr{\Ab - \Ab \Xb^{\pinv} \Xb}_F^2 \lesssim l (l-k) \log\rbr{l-k} \nbr{\Ab - \Ab_k}_F^2$ with high probability.
Similar guarantees hold for spectral norm (\cite{halko2011finding}, Section 10).
The randomized rangefinder error depends on the spectral decay of $\Ab$, and can be aggravated by a flat spectrum. In this scenario, power iterations (with proper orthogonalization, \cite{halko2011finding}), as well as Krylov and block Krylov subspace iterations (\cite{musco2015randomized}), may be incorporated after the initial sketching as a remedy. For example, with a randomized linear embedding $\Omegab$ of size $l \times n$, a row basis approximator with $q$ power iterations ($q \geq 1$) is given by
\begin{align}\label{eq:def_power_iter}
    \Xb = \Omegab \rbr{\Ab^{\top} \Ab}^{q}
,\end{align}
and takes $\Xb$ takes $O\rbr{T_s(l,\Ab) + (2q-1) \nnz(\Ab) l}$ operations to construct. However, such plain power iteration in \Cref{eq:def_power_iter} is numerically unstable and can lead to large errors when $\Ab$ is ill-conditioned and $q > 1$. For a stable construction, orthogonalization can be applied at each iteration:
\begin{align}\label{eq:ortho_power_iter}
    & \Yb^{(1)} = \Ab \Omegab^{\top} \nonumber
    \\
    & \Yb^{(i)} = \ortho\rbr{\Ab \ortho\rbr{\Ab^{\top} \Yb^{(i-1)}}} \ \forall\ i = 2,\dots,q\ (\t{if}\ q > 1) \nonumber
    \\
    & \Xb = \ortho\rbr{\Yb^{(q)}}^{\top} \Ab
\end{align}
with an additional cost of $O\rbr{q(m+n)l^2}$ overall.

In addition, with a proper $l$ that does not exceed the exact rank of $\Ab$, the row sketch $\Xb \in \R^{l \times n}$ has a full row rank almost surely.
Precisely, recall $r=\rank(\Ab) \le \min\cbr{m,n}$ from \Cref{subsec:notation}:
\begin{remark}\label{lemma:sketch_for_rand_rangefinder}
    For a Gaussian embedding $\Gammab \in \R^{l \times m}$ with \iid entries from $\Ncal\rbr{0,1/l}$ and $l \leq r$, the row sketch $\Xb = \Gammab \Ab$ has full row rank almost surely.
\end{remark}
\begin{proof}[Rationale for \Cref{lemma:sketch_for_rand_rangefinder}]
    Recall the reduced SVD of $\Ab$, $\Ab = \Ub_\Ab \Sigmab_\Ab \Vb_\Ab^\top$. Given $l \le r$, it is sufficient to show that $\rank\rbr{\Gammab \Ub_\Ab}=l$. Since $\Ub_{\Ab}$ consists of orthonormal columns, by the rotation invariance of Gaussian distribution, $\Gammab \Ub_\Ab \in \R^{l \times r}$ can also be viewed as a Gaussian random matrix with \iid entries from $\Ncal\rbr{0,1/l}$. 
    Since all square submatrices of a Gaussian random matrix are invertible almost surely \cite{davidson2001local,Rudelson2008invertibility}, we have $\rank\rbr{\Gammab \Ub_\Ab}=l$ almost surely.
\end{proof}

A low-rank row basis approximator $\Xb$ can be subsequently leveraged to construct a randomized rank-$l$ SVD.
Assuming $l$ is properly chosen such that $\Xb$ has full row rank, let $\Qb_X \in \R^{n \times l}$ be an orthonormal basis for the row space of $\Xb$.
The exact SVD of the smaller matrix $\Ab \Qb_X \in \R^{m \times l}$,
\begin{align}\label{eq:rsvd_procedures}
    \sbr{\underset{m \times l}{\wh \Ub_A}, \underset{l \times l}{\wh \Sigmab_A}, \underset{l \times l}{\wt \Vb_A}} = \t{svd} \rbr{\underbrace{\Ab \Qb_X}_{m \times l}, \t{`econ'}},
    \quad
    \underset{n \times l}{\wh\Vb_A} = \Qb_X \wt\Vb_A,
\end{align}
can be evaluated efficiently in $O\rbr{ml^2}$ operations (and $O\rbr{nl^2}$ additional operations for constructing $\wh\Vb_A$) such that $\Ab \approx \Ab \Xb^{\pinv} \Xb = \wh\Ub_A \wh\Sigmab_A \wh\Vb_A^{\top}$ \cite{halko2011finding}.

\subsection{Matrix decompositions with pivoting}\label{subsec:pivoting_mat_decomp}
We next briefly survey how pivoted QR and LU decompositions can be leveraged to
resolve the matrix skeleton selection problem.
In this section, $\Xb \in \R^{l \times n}$ denotes a matrix of full row rank
(that will typically arise as a ``row space approximator'').
Let $\Xb^{(t)} \in \R^{l \times n}$ be the resulted matrix after the $t$-th step of pivoting and matrix updating, so that $\Xb^{(0)} = \Xb$.

\subsubsection{Column pivoted QR (\pqr)}
Applying the \pqr to $\Xb$ gives:
\begin{align}\label{eq:def_QRCP}
    \Xb\ \underset{n \times n}{\Pib_n} = \Xb\ \underset{n \times l}{\left[\Pib_{n,1}\right.},\ \underset{n \times (n-l)}{\left.\Pib_{n,2}\right]} =  \underset{l \times l}{\Qb_l}\ \underset{l \times n}{\Rb^{QR}} = \Qb_l\ \underset{l \times l}{\left[\Rb^{QR}_{1}\right.},\ \underset{l \times (n-l)}{\left.\Rb^{QR}_{2}\right]},
\end{align}
where $\Qb_l$ is an orthogonal matrix; $\Rb^{QR}_{1}$ is upper triangular; and $\Pib_n \in \R^{n \times n}$ is a column permutation.
QR decompositions rank-$1$ update the active submatrix at each step for orthogonalization ($\eg$, \cite{householder1958hhqr}, \cite{trefethen1997}, Algorithm 10.1).
For each $t=0,\dots,l-2$, at the $(t+1)$-th step, the \pqr searches the entire active submatrix $\Xb^{(t)}\rbr{t+1:l, t+1:n}$ for the $(t+1)$-th column pivot with the maximal $\ell_2$-norm:
\begin{align*} 
    j_{t+1} = \argmax_{t+1 \leq j \leq n} \nbr{\Xb^{(t)}\rbr{t+1:l, j}}_2.
\end{align*}
As illustrated in  \cite{gu1996}, \pqr satisfies $\max_{i,j}\abbr{\rbr{\rbr{\Rb^{QR}_{1}}^{-1} \Rb^{QR}_{2}}_{ij}} \leq 2^{l-i}$; while the upper bound is tight with the classical Kahan matrix \cite{kahan1966}. Nevertheless, these adversarial inputs are scarce and sensitive to perturbations. The empirical success of \pqr also suggests that exponential growth with respect to $l$ almost never occurs in practice \cite{trefethen1990}.
Meanwhile, there exist more sophisticated variations of \pqr, like the rank-revealing \cite{chan1987, hong1992} and strong rank-revealing QR \cite{gu1996}, guaranteeing that $\max_{i,j} \abbr{\rbr{\rbr{\Rb^{QR}_{1}}^{-1} \Rb^{QR}_{2}}_{ij}}$ is upper bounded by some low-degree polynomial in $l$, but coming with higher complexities as trade-off.

\subsubsection{LU with partial pivoting (\plu)}
Applying the \plu to $\Xb^\top$ yields:
\begin{align}\label{eq:def_LUPP}
    \Xb\ \underset{n \times n}{\Pib_n} = \Xb\ \underset{n \times l}{\left[\Pib_{n,1}\right.},\ \underset{n \times (n-l)}{\left.\Pib_{n,2}\right]} =  \underset{l \times l}{\Lb_l}\ \underset{l \times n}{\Rb^{LU}} = \Lb_l\ \underset{l \times l}{\left[\Rb^{LU}_{1}\right.},\ \underset{l \times (n-l)}{\left.\Rb^{LU}_{2}\right]},
\end{align}
where $\Lb_l$ is lower triangular; $\Rb^{LU}_{1}$ is upper triangular; $\Rb_1^{LU}\rbr{i,i} = 1$ and $\abbr{\Rb^{LU}\rbr{i,j}} \leq 1$ for all $i \in [l]$, $i \leq j \leq n$; and $\Pib_n \in \R^{n \times n}$ is a column permutation.
LU decompositions update active submatrices via Shur complements ($\eg$, \cite{trefethen1997}, Algorithm 21.1): for $t=0,\dots,l-2$,
\begin{align*}
    \Xb^{(t+1)}\rbr{t+2:l,t+2:n} = &\Xb^{(t)}\rbr{t+2:l,t+2:n} \\
    - &\Xb^{(t)}\rbr{t+2:l,t} \Xb^{(t)}\rbr{t,t+2:n}/\Xb^{(t)}\rbr{t,t}.
\end{align*}
At the $(t+1)$-th step, the \plu on $\Xb^\top$ searches only the $(t+1)$-th row in the active submatrix and pivots
\begin{align*}
    j_{t+1} = \argmax_{t+1 \leq j \leq n} \abbr{\Xb^{(t)}\rbr{t+1, j}},
\end{align*}
such that $\Rb^{LU}\rbr{i,j} = \Xb^{(i-1)}\rbr{i, j} / \Xb^{(i)}\rbr{i, i}$ for all $i \in [l]$, $i+1 \leq j \leq n$ (except for $\Rb^{LU}\rbr{i,j_i} = \Xb^{(i-1)}\rbr{i, i} / \Xb^{(i)}\rbr{i, i}$), and therefore $\abbr{\Rb^{LU}\rbr{i,j}} \leq 1$.

Analogous to \pqr, the pivoting strategy of \plu leads to a loose, exponential upper bound:
\begin{remark}\label{lemma:LUPP_entry_bound}
The \plu in \Cref{eq:def_LUPP} satisfies that
\begin{align*}
    \max_{i,j} \abbr{\rbr{\rbr{\Rb^{LU}_{1}}^{-1} \Rb^{LU}_{2}}_{ij}} \leq 2^{l-i},
\end{align*}
where the upper bound is tight, for instance, when $\Rb^{LU}_{1}\rbr{i,j}=-1$ for all $i \in [l-1]$, $i+1 \leq j \leq l$ and $\Rb^{LU}_{2}\rbr{i,j} = 1$ for all $i \in [l]$, $j \in [n-l]$ ($\ie$, a Kahan-type matrix \cite{kahan1966, peters1975}). 
\end{remark}

\begin{proof}[Rationale for \Cref{lemma:LUPP_entry_bound}]
In reminiscence of the exponential worse-case growth factor of Gaussian elimination with partial pivoting (\eg, \cite{golub2013} Section 3.4.5), we start by observing the following recursive relations: for all $j=1,\dots,n-l$ and $i=l-1,\dots,1$,
\begin{align*}
    & \rbr{\rbr{\Rb^{LU}_{1}}^{-1} \Rb^{LU}_{2}}_{lj} = \Rb^{LU}_{2}\rbr{l,j}
    \\
    & \rbr{\rbr{\Rb^{LU}_{1}}^{-1} \Rb^{LU}_{2}}_{ij} = \Rb^{LU}_{2}\rbr{i,j} - \sum_{\iota=i+1}^l \Rb^{LU}_{1}\rbr{i,\iota} \rbr{\rbr{\Rb^{LU}_{1}}^{-1} \Rb^{LU}_{2}}_{\iota,j} 
\end{align*}
given $\Rb_1^{LU}\rbr{i,i} = 1$. Then both the upper bound and the adversarial examples of Kahan-type matrices follow from the fact that $\abbr{\Rb^{LU}\rbr{i,j}} \leq 1$ for all $i \in [l]$, $i \leq j \leq n$.
\end{proof}

In addition to the exponential worst-case scenario in \Cref{lemma:LUPP_entry_bound}, \plu is also vulnerable to rank deficiency since it only views one row for each pivoting step (in contrast to \pqr which searches the entire active submatrix). The advantage of the \plu type pivoting scheme is its superior empirical efficiency and parallelizability \cite{geist1988lu, kurzak2007, grigori2011calu, solomonik2011communication}.
Fortunately, as with \pqr, adversarial inputs for \plu are sensitive to perturbations ($\eg$, flip the signs of random off-diagonal entries in $\Rb_1^{LU}$), and are rarely encountered in practice.

\plu can be further stabilized with randomization \cite{pan2015rand, pan2017num, trefethen1990}. In terms of the worse-case exponential entry-wise bound in \Cref{lemma:LUPP_entry_bound}, the average-case growth factors of \plu on random matrices drawn from a variety of distributions ($\eg$, the Gaussian distribution, uniform distributions, Rademacher distribution, symmetry / Toeplitz matrices with Gaussian entries, and orthogonal matrices following Haar measure) were investigated in \cite{trefethen1990} where it was conjectured that the growth factor increases sublinearly with respect to the problem size in average cases.

\begin{remark}[Conjectured in \cite{trefethen1990}]\label{remark:lupp_stable_in_practice}
With randomized preprocessing like sketching, \plu is robust to adversarial inputs in practice, with $\max_{i,j} \abbr{\rbr{\rbr{\Rb^{LU}_{1}}^{-1} \Rb^{LU}_{2}}_{ij}} = O(l)$ in average cases.
\end{remark}

Some common alternatives to partial pivoting for LU decompositions include (adaptive) cross approximations \cite{Bebendorf2000ApproximationOB, tyrtyshnikov2000, zhao2005}, and complete pivoting.
Specifically, complete pivoting is a more robust ($\eg$, to rank deficiency) alternative to partial pivoting that searches the entire active submatrix, and permutes rows and columns simultaneously. Despite lacking theoretical guarantees for the plain complete pivoting, like for QR decompositions, there exists modified complete pivoting strategies for LU that come with better rank-revealing guarantees \cite{pan2000rrlu, gu2003srrlu,anderson2017, chen2020}, but higher computational cost as a trade-off.

\section{Summary of existing algorithms}\label{sec:exist_cur_algo}

A vast assortment of algorithms for interpolative and CUR decompositions have been proposed and analyzed in the past decades \cite{deshpande2006, deshpande2010, drineas2008relative, mahoney2009,  bien2010cur, wang2012cur, cohen2014, woodruff2014optimalCUR, boutsidis2014, sorensen2014, anderson2015spectral, aizenbud2016randomized, voronin2017, anderson2017, shabat2018randomized, cortinovis2019lowrank, chen2020, derezinski2020improved}. From the skeleton selection perspective, these algorithms broadly fall into two categories:
\begin{enumerate}
    \item sampling-based methods that draw matrix skeletons (directly, adaptively, or iteratively) from some proper distributions, and
    \item pivoting-based methods that pick matrix skeletons greedily by constructing low-rank matrix decompositions with pivoting.
\end{enumerate}
In this section, we discuss existing algorithms for matrix skeletonizations, with a focus on algorithms based on randomized linear embeddings and matrix decompositions with pivoting.

\subsection{Sampling based skeleton selection}\label{subsec:sampling_cur}
The idea of skeleton selection via sampling is closely related to various topics including graph sparsification \cite{batson2008} and volume sampling \cite{deshpande2006}.
Concerning volume sampling, adaptive sampling strategies \cite{deshpande2010, anderson2015spectral, cortinovis2019lowrank, derezinski2020improved} that lead to matrix skeletons with close-to-optimal error guarantees are reviewed in \Cref{sec:cssp}.
Meanwhile, leverage score sampling for constructing CUR decompositions, as well as efficient estimations for the leverage scores, are extensively studied in \cite{drineas2008relative, mahoney2009, bien2010cur, drineas2012}.
Furthermore, some sophisticated variations of sampling-based skeleton selection algorithms were proposed in \cite{woodruff2014optimalCUR, boutsidis2014, cohen2014, wang2012cur} where iterative sampling and/or combinations of different sampling schemes were incorporated.

\subsection{Skeleton selection via deterministic pivoting}\label{subsec:pivoting_cur}
Greedy algorithms based on column and row pivoting can also be used for matrix skeletonizations. For instance, with proper rank-revealing pivoting like the strong rank-revealing QR proposed in \cite{gu1996}, a rank-$k$ ($k < r$) column ID can be constructed with the first $k$ column pivots
\begin{align*}
    \Ab \underset{n \times l}{\left[\Pib_{n,1}\right.},\ \underset{n \times (n-l)}{\left.\Pib_{n,2}\right]} 
    = &\underset{m \times k}{\left[\Qb_{A,1}\right.},\ \underset{m \times (r-k)}{\left.\Qb_{A,2}\right]} \bmat{\Rb_{A,11} & \Rb_{A,12} \\ \b{0} & \Rb_{A,22}} \\
    \approx &\rbr{\Ab \Pib_{n,1}} \sbr{\Ib_k,\ \Rb_{A,11}^{-1} \Rb_{A,12}}
\end{align*}
where $\Pib_n = \sbr{\Pib_{n,1}, \Pib_{n,2}}$ is a permutation of columns; and $\Rb_{A,11}$ and $\Rb_{A,22}$ are non-singular and upper triangular. $\Cb = \rbr{\Ab \Pib_{n,1}}$ are the selected column skeletons that satisfies $\nbr{\Ab - \Cb \Cb^{\pinv} \Ab} = \nbr{\Rb_{A,22}} \lesssim \sqrt{k(n-k)} \nbr{\Ab - \Ab_k}$. 

As a more affordable alternative to the rank-revealing pivoting, the \pqr discussed in \Cref{subsec:pivoting_mat_decomp} also works well for skeleton selection in practice \cite{voronin2017}, despite the weaker theoretical guarantee due to the known existence of adversarial inputs ($\ie$, $\nbr{\Ab - \Cb \Cb^{\pinv} \Ab} \lesssim 2^k \nbr{\Ab - \Ab_k}$).

In addition to the QR-based pivoting schemes, (randomized) LU-based pivoting algorithms with rank-revealing guarantees \cite{pan2000rrlu, gu2003srrlu, aizenbud2016randomized, anderson2017, shabat2018randomized, chen2020} can also be leveraged for greedy matrix skeleton selection (as discussed in \Cref{subsec:pivoting_mat_decomp}). 
Alternatively, the DEIM skeleton selection algorithm \cite{sorensen2014, drmac2016qdeim} relaxes the rank-revealing requirements on pivoting schemes by applying \plu on the leading singular vectors of $\Ab$.

\subsection{Randomized pivoting-based skeleton selection}\label{subsec:rand_pivot_cur}
In comparison to the sampling-based skeleton selection, the deterministic pivoting-based skeleton selection methods suffer from two major drawbacks. First, pivoting is usually unaffordable for large-scale problems in common modern applications. Second, classical pivoting schemes like the \pqr and \plu are vulnerable to antagonistic inputs.
Fortunately, randomized pre-processing with sketching provides remedies to both problems:
\begin{enumerate}
    \item Faster execution speed is attained by executing classical pivoting schemes on a sketch
    $\Xb = \Gammab\Ab \in \R^{l \times n}$, for some randomized embedding $\Gammab$, instead on
    $\Ab$ directly.
    \item With randomization, classical pivoting schemes like the \pqr and \plu 
    are robust to adversarial inputs in practice (\Cref{remark:lupp_stable_in_practice}, \cite{trefethen1990}).
\end{enumerate}
\Cref{algo:sketch_pivot_CUR_general} describes a general framework for randomized pivoting-based skeleton selection. Grounding this framework down with different combinations of row basis approximators and pivoting schemes, it was proposed in \cite{voronin2017} to take $\Xb = \Gammab \Ab$ as a row sketch and apply \pqr to $\Xb$ for column skeleton selection.
Alternatively, the DEIM skeleton selection algorithm proposed in \cite{sorensen2014} can be accelerated by taking $\Xb$ as an approximation of the leading-$l$ right singular vectors of $\Ab$ (\Cref{eq:rsvd_procedures}), where \plu is applied for skeleton selection.

\begin{algorithm}[ht]
\caption{Randomized pivoting-based skeleton selection: a general framework}\label{algo:sketch_pivot_CUR_general}
\begin{algorithmic}[1]
\REQUIRE $\Ab \in \R^{m \times n}$ of rank $r$, rank $l \leq r$ (typically $l \ll \min(m,n)$).
\ENSURE Column and/or row skeleton indices, $J_s \subset [n]$ and/or $I_s \subset [m]$, $\abbr{J_s} = \abbr{I_s} = l$.

\STATE Draw an oblivious $\ell_2$-embedding $\Gammab \in \R^{l \times m}$.
\STATE Construct a row basis approximator $\Xb \in R^{l \times n}$ via sketching with $\Gammab$.\\
$\eg$, $\Xb$ can be 1) a row sketch or 2) approximations of right singular vectors.
\STATE Perform column-wise pivoting on $\Xb$. Let $J_s$ index the $l$ column pivots.
\STATE Perform row-wise pivoting on $\Cb = \Ab(:,J_s)$. Let $I_s$ index the $l$ row pivots.
\end{algorithmic}
\end{algorithm}

First, we recall from \Cref{lemma:sketch_for_rand_rangefinder} that when taking $\Xb$ as a row sketch, with Gaussian embeddings, $\Gammab$ and $\Xb = \Gammab \Ab$ are both full row rank with probability $1$. 
Moreover, when taking $\Xb$ as an approximation of right singular vectors constructed with a row basis approximator consisting of $l$ linearly independent rows, $\Xb$ also admits full row rank.

Second, when both column and row skeletons are inquired, \Cref{algo:sketch_pivot_CUR_general} selects the column skeletons first with randomized pivoting and subsequently identifies the row skeletons by pivoting on the selected columns. 
With $\Xb$ being full row rank (almost surely when $\Gammab$ is a Gaussian embedding), the column skeletons $\Cb$ are linearly independent. Therefore, the row-wise skeletonization of $\Cb$ is exact, without introducing additional errors. That is, the two-sided ID constructed by \Cref{algo:sketch_pivot_CUR_general} is equal to the associated column ID in exact arithmetic, $\tsid{\Ab}{I_s,J_s} = \cid{\Ab}{J_s}$.

\section{A simple but effective modification: \plu on sketches}\label{sec:cur_rand_lupp}
Inspired by the idea of pivoting on sketches \cite{voronin2017} and the remarkably competitive performance of \plu when applied to leading singular vectors \cite{sorensen2014}, we propose a simple but effective modification -- applying \plu directly to a sketch of $\Ab$. In terms of the general framework in \Cref{algo:sketch_pivot_CUR_general}, this corresponds to taking $\Xb$ as a row sketch of $\Ab$, and then selecting skeletons via \plu on $\Xb$ and $\Cb$, as summarized in \Cref{algo:rand-id-lupp}.

\begin{algorithm}[ht]
\caption{Randomized \plu skeleton selection}
\label{algo:rand-id-lupp}
\begin{algorithmic}[1]
\REQUIRE $\Ab \in \R^{m \times n}$ of rank $r$, rank $l \leq r$ (typically $l \ll \min(m,n)$).
\ENSURE Column and/or row skeleton indices, $J_s \subset [n]$ and/or $I_s \subset [m]$, $\abbr{J_s} = \abbr{I_s} = l$.

\STATE Draw an oblivious $\ell_2$-embedding $\Gammab \in \R^{l \times m}$.
\STATE Construct a row sketch $\Xb = \Gammab \Ab$.
\STATE Perform \plu on $\Xb^\top$. Let $J_s$ index the $l$ column pivots of $\Xb$ (\ie, row pivots of $\Xb^\top$).
\STATE Perform \plu on $\Cb = \Ab(:,J_s)$. Let $I_s$ index the $l$ row pivots.
\end{algorithmic}
\end{algorithm}

Comparing to pivoting with \pqr \cite{voronin2017}, \Cref{algo:rand-id-lupp} with \plu is empirically faster, as discussed in \Cref{subsec:pivoting_mat_decomp}, and illustrated in \Cref{fig:pivot-ps-time}.
Meanwhile, assuming that the true SVD of $\Ab$ is unavailable, in comparison to pivoting on the approximated leading singular vectors \cite{sorensen2014} from \Cref{eq:rsvd_procedures}, \Cref{algo:rand-id-lupp} saves the effort of constructing randomized SVD which takes $O\rbr{\nnz(\Ab)l + (m+n)l^2}$ additional operations. Additionally, with randomization, the stability of \plu conjectured in \cite{trefethen1990} (\Cref{remark:lupp_stable_in_practice}) applies, and \Cref{algo:rand-id-lupp} effectively circumvents the potential vulnerability of \plu to adversarial inputs in practice.
A formal error analysis of \Cref{algo:sketch_pivot_CUR_general} in general reflects these points:

\begin{theorem}[Column skeleton selection by pivoting on a row basis approximator]\label{thm:pivoting_on_rangeapprox_error_bound}
Given a row basis approximator $\Xb \in \R^{l \times n}$ ($l \leq r$) of $\Ab$ that admits full row rank, let $\Pib_n \in \R^{n \times n}$ be the resulted permutation after applying some proper column pivoting scheme on $\Xb$ that identifies $l$ linearly independent column pivots: for the $(l,n-l)$ column-wise partition $\Xb \Pib_n = \Xb \sbr{\Pib_{n,1}, \Pib_{n,2}} = \sbr{\Xb_{1}, \Xb_{2}}$, the first $l$ column pivots $\Xb_1 = \Xb\Pib_{n,1} \in \R^{l \times l}$ admits full column rank. 
Moreover, the rank-$l$ column ID $\cid{\Ab}{J_s} = \Cb \Cb^{\pinv} \Ab$, with linearly independent column skeletons $\Cb = \Ab \Pib_{n,1}$, satisfies that
\begin{align}\label{eq:pivoting_on_rangeapprox_error_bound}
    \nbr{\Ab - \Cb \Cb^{\pinv} \Ab} \leq \eta \nbr{\Ab - \Ab \Xb^{\pinv} \Xb},
\end{align}
where $\eta \leq \sqrt{1 + \nbr{\Xb_{1}^{\pinv} \Xb_{2}}_2^2}$, and $\nbr{\cdot}$ represents the spectral or Frobenius norm.
\end{theorem}

\Cref{thm:pivoting_on_rangeapprox_error_bound} states that when selecting column skeletons by pivoting on a row basis approximator, the low-rank approximation error of the resulting column ID is upper bounded by that of the associated row basis approximator up to a factor $\eta > 1$ that can be computed a posteriori efficiently in $O\rbr{l^2(n-l)}$ operations.
\Cref{eq:pivoting_on_rangeapprox_error_bound} essentially decouples the error from the row basis approximation with $\Xb$ ($\nbr{\Ab - \Ab \Xb^{\pinv} \Xb}$ corresponding to Line 1 and 2 of \Cref{algo:sketch_pivot_CUR_general}, as reviewed in \Cref{subsec:rsvd}) and that from the skeleton selection by pivoting on $\Xb$ ($\eta$ corresponding to Line 3 and 4 of \Cref{algo:sketch_pivot_CUR_general}).

Now we ground \Cref{thm:pivoting_on_rangeapprox_error_bound} with different choices of row basis approximation and pivoting strategies:
\begin{enumerate}
    \item With \Cref{algo:rand-id-lupp}, $\nbr{\Ab - \Ab \Xb^{\pinv} \Xb}$ is the randomized rangefinder error (\Cref{eq:rand_rangefinder_error_bound}, \cite{halko2011finding} Section 10), and $\eta \leq \sqrt{1+\nbr{\rbr{\Rb_1^{LU}}^{-1} \Rb_2^{LU}}_2^2}$ (recall \Cref{eq:def_LUPP}). Although in the worst-case scenario (where the entry-wise upper bound in \Cref{lemma:LUPP_entry_bound} is tight), $\eta = \Theta\rbr{2^l \sqrt{n-l}}$, with a randomized row sketch $\Xb$, assuming the stability of \plu conjectured in \cite{trefethen1990} holds (\Cref{remark:lupp_stable_in_practice}), $\eta = O\rbr{l^{3/2} \sqrt{n-l}}$.
    \item Skeleton selection with \pqr on row sketches ($\ie$, randomized \pqr proposed in \cite{voronin2017}) shares the same error bound as \Cref{algo:rand-id-lupp} ($\ie$, analogous arguments hold for $\nbr{\rbr{\Rb_1^{QR}}^{-1} \Rb_2^{QR}}_2^2$).
    \item When applying \plu on the true leading singular vectors ($\ie$, DEIM proposed in \cite{sorensen2014}, assuming that the true SVD is available), $\nbr{\Ab - \Ab \Xb^{\pinv} \Xb} = \nbr{\Ab - \Ab_l}$, but without randomization, LUPP is vulnerable to adversarial inputs which can lead to $\eta = \Theta\rbr{2^l \sqrt{n-l}}$ in the worse case.
    \item When applying \plu on approximations of leading singular vectors (constructed via \Cref{eq:rsvd_procedures}, $\ie$, randomized DEIM suggested in \cite{sorensen2014}), $\nbr{\Ab - \Ab \Xb^{\pinv} \Xb}$ corresponds to the randomized rangefinder error with power itertions (\cite{halko2011finding} Corollary 10.10), while $\eta$ follows the analogous analysis as for \Cref{algo:rand-id-lupp}.
\end{enumerate}

\begin{proof}[Proof of \Cref{thm:pivoting_on_rangeapprox_error_bound}]
We start by defining two oblique projectors
\begin{align*}
    \underset{n \times n}{\pobx} \triangleq \Pib_{n,1} \rbr{\Xb \Pib_{n,1}}^{\pinv} \Xb,
    \quad
    \underset{n \times n}{\pobc} \triangleq \Pib_{n,1} \rbr{\Cb^{\top} \Cb}^{\pinv} \Cb^{\top} \Ab,
\end{align*}
and observe that, since $\Cb$ consists of linearly independent columns, $\rbr{\Cb^{\top} \Cb}^{\pinv} \Cb^{\top} \Ab \Pib_{n,1} = \Ib_l$, and
\begin{align*}
    \pobc\pobx = \Pib_{n,1} \rbr{\Cb^{\top} \Cb}^{\pinv} \Cb^{\top} \Ab \Pib_{n,1} \rbr{\Xb \Pib_{n,1}}^{\pinv} \Xb = \pobx.
\end{align*}
With $\pobc$, we can express the column ID as
\begin{align*}
    \cid{\Ab}{J_s} = \Cb \Cb^{\pinv} \Ab = \Ab \rbr{\Pib_{n,1} \rbr{\Cb^{\top} \Cb}^{\pinv} \Cb^{\top} \Ab} = \Ab \pobc,
\end{align*}
Therefore, the low-rank approximation error of $\cid{\Ab}{J_s}$ satisfies
\begin{align*}
    \nbr{\Ab - \Cb \Cb^{\pinv} \Ab} = & \nbr{\Ab \rbr{\Ib - \pobc}}
    \\
    = & \nbr{\Ab \rbr{\Ib_n - \pobc} \rbr{\Ib_n - \pobx}}
    \\
    = & \nbr{\rbr{\Ib_m - \Cb \Cb^{\pinv}} \Ab \rbr{\Ib_n - \pobx}}
    \\
    \leq & \nbr{\Ib_m - \Cb \Cb^{\pinv}}_2 \nbr{\Ab \rbr{\Ib_n - \pobx}},
\end{align*}
where $\nbr{\Ib_m - \Cb \Cb^{\pinv}}_2 = 1$, and since $\Xb \pobx = \Xb_1 \Xb_1^{\pinv} \Xb = \Xb$ with $\Xb_1$ being full-rank,
\begin{align*}
    \nbr{\Ab \rbr{\Ib_n - \pobx}} = \nbr{\Ab \rbr{\Ib_n - \Xb^{\pinv} \Xb} \rbr{\Ib_n - \pobx}} = \nbr{\Ib_n - \pobx}_2 \nbr{\Ab \rbr{\Ib_n - \Xb^{\pinv} \Xb}}.
\end{align*}
As a result, we have
\begin{align*}
    \nbr{\Ab - \Cb \Cb^{\pinv} \Ab} \leq \nbr{\Ib_n - \pobx}_2 \nbr{\Ab \rbr{\Ib_n - \Xb^{\pinv} \Xb}},
\end{align*}
and it is sufficient to show that $\eta \triangleq \nbr{\Ib_n - \pobx}_2 \leq \sqrt{1 + \nbr{\Xb_{1}^{\pinv} \Xb_{2}}_2^2}$.
Indeed,
\begin{align*}
    \Ib_n - \pobx 
    = &\Pib_n^{\top} \Pib_n - \Pib_n^{\top} \Pib_{n,1} \rbr{\Xb \Pib_{n,1}}^{\pinv} \Xb \Pib_n \\
    = &\Ib_n - \bmat{\Ib_l \\ \b0} \Xb_1^{\pinv} \bmat{\Xb_1 \Xb_2} = \bmat{\b0 & - \Xb_1^{\pinv} \Xb_2 \\ \b0 & \Ib_{n-l}}
\end{align*}
such that $\eta = \nbr{\sbr{- \Xb_1^{\pinv} \Xb_2; \Ib_{n-l}}}_2 \leq \sqrt{1 + \nbr{\Xb_{1}^{\pinv} \Xb_{2}}_2^2}$.
\end{proof}

Here, the proof of \Cref{thm:pivoting_on_rangeapprox_error_bound} is reminiscent of \cite{sorensen2014}, while it generalizes the result for fixed right leading singular vectors to any proper row basis approximators of $\Ab$ ($\eg$, $\wh\Vb_A$ in \Cref{eq:rsvd_procedures}, or simply a row sketch). The generalization of \Cref{thm:pivoting_on_rangeapprox_error_bound} leads to a factor $\eta$ that is efficiently computable a posteriori, which can serve as an empirical replacement of the exponential upper bound induced by the scarce adversarial inputs.

In addition to the empirical efficiency and robustness discussed above, \Cref{algo:rand-id-lupp} has another potential advantage: the skeleton selection algorithm can be easily adapted to the streaming setting. 
The streaming setting considers $\Ab$ as a data stream that can only be accessed as a sequence of snapshots. Each snapshot of $\Ab$ can be viewed only once, and the storage of the entire matrix $\Ab$ is infeasible \cite{tropp2017a, tropp2019streaming, martinsson2020}.

\begin{algorithm}[ht]
\caption{Streaming \plu/\pqr skeleton selection}\label{algo:sketch_pivot_CUR_online}
\begin{algorithmic}[1]
\REQUIRE $\Ab \in \R^{m \times n}$ of rank $r$, rank $l \leq r$ (typically $l \ll \min(m,n)$).
\ENSURE Column and/or row skeleton indices, $J_s \subset [n]$ and/or $I_s \subset [m]$, $\abbr{J_s} = \abbr{I_s} = l$.

\STATE Draw independent oblivious $\ell_2$-embeddings $\Gammab \in \R^{l \times m}$ and $\Omegab \in \R^{l \times n}$.
\STATE Construct row and column sketches, $\Xb = \Gammab \Ab$ and $\Yb = \Ab \Omegab^{\top}$, in a single pass through $\Ab$.
\STATE Perform column-wise pivoting (\plu on $\Xb^\top$ or \pqr on $\Xb$). Let $J_s$ index the $l$ column pivots.
\STATE Perform row-wise pivoting (\plu on $\Yb$ or \pqr on $\Yb^\top$). Let $I_s$ index the $l$ row pivots.
\end{algorithmic}
\end{algorithm}

\begin{remark}
When only the column and/or row skeleton \textit{indices} are required (and not the explicit construction of the corresponding interpolative or CUR decomposition), \Cref{algo:rand-id-lupp} can be adapted to the streaming setting (as shown in \Cref{algo:sketch_pivot_CUR_online}) by sketching both sides of $\Ab$ independently in a single pass, and pivoting on the resulting column and row sketches.
Moreover, with the column and row skeletons $J_s$ and $I_s$ from \Cref{algo:sketch_pivot_CUR_online}, \Cref{thm:pivoting_on_rangeapprox_error_bound} and its row-wise analog together, along with \Cref{eq:id_cur_suboptimality}, imply that
\begin{align*}
    \nbr{\Ab - \cur{\Ab}{I_s,J_s}} \leq \eta_X \nbr{\Ab - \Ab \Xb^{\pinv} \Xb} + \eta_Y \nbr{\Ab - \Yb \Yb^{\pinv} \Ab},
\end{align*}
where $\eta_X$ and $\eta_Y$ are small in practice with pivoting on randomized sketches and have efficiently a posteriori computable upper bounds given $\Xb$ and $\Yb$, as discussed previously. $\nbr{\Ab - \Ab \Xb^{\pinv} \Xb}_F^2$ and $\nbr{\Ab - \Yb \Yb^{\pinv} \Ab}$ are the randomized rangefinder errors with well-established upper bounds (\cite{halko2011finding} Section 10).
\end{remark}

We point out that, although the column and row skeleton selection can be conducted in a streaming fashion, the explicit stable construction of ID or CUR requires two additional passes through $\Ab$: one pass for retrieving the skeletons $\Cb$ and/or $\Rb$, and the other pass to construct $\Cb^\pinv \Ab \Rb^\pinv$ for CUR, or $\Cb^\pinv \Ab$, $\Ab \Rb^\pinv$ for IDs.
In practice, for efficient estimations of the ID or CUR when revisiting $\Ab$ is expensive, it is possible to circumvent the second pass through $\Ab$ with compromise on accuracy and stability, albeit the inevitability of the first pass for skeleton retrieval. 

Precisely for the ID, $\Cb^\pinv \Ab$ (in \Cref{eq:def_column_id}) or $\Ab \Rb^\pinv$ (in \Cref{eq:def_row_id}) can be estimated without revisiting $\Ab$ leveraging the associated row and column sketches: 
\begin{align*}
    \Cb^\pinv \Ab \approx \Xb_1^\pinv \Xb,
    \quad 
    \Ab \Rb^\pinv \approx \Yb \Yb_1^\pinv,    
\end{align*}    
where $\Xb_1 = \Xb\rbr{:,J_s}$ and $\Yb_1 = \Yb\rbr{I_s,:}$ are the $l$ column and row pivots in $\Xb$ and $\Yb$, respectively.
Meanwhile for the CUR, by retrieving the skeletons $\Sb = \Ab\rbr{I_s,J_s}$, $\Cb=\Ab\rbr{:,J_s}$ and $\Rb=\Ab\rbr{I_s,:}$, we can construct a CUR decomposition $\Cb \Sb^{-1} \Rb$, despite the compromise on both accuracy and stability.

\section{Numerical experiments}\label{sec:cur_experiments}

In this section, we study the empirical performance of various randomized skeleton selection algorithms.
Starting with the randomized pivoting-based algorithms, we investigate the efficiency of two major components of \Cref{algo:sketch_pivot_CUR_general}: 
(1) the sketching step for row basis approximator construction, and 
(2) the pivoting step for greedy skeleton selection. 
Then we explore the suboptimality (in terms of low-rank approximation errors of the resulting CUR decompositions $\nbr{\Ab - \cur{\Ab}{I_s,J_s}}$), as well as the efficiency (in terms of empirical run time), of different randomized skeleton selection algorithms.

We conduct all the experiments, except for those in \Cref{fig:sketch-ps-time} on the efficiency of sketching, in MATLAB R2020a. In the implementation, the computationally dominant processes, including the sketching, \plu, \pqr, and SVD, are performed by the MATLAB built-in functions.
The experiments in \Cref{fig:sketch-ps-time} are conducted in Julia Version 1.5.3 with the JuliaMatrices/LowRankApprox.jl package \cite{lowrankjl2020}.

\subsection{Computational speeds of different embeddings}
\label{subsec:exp_embedding}
Here, we compare the empirical efficiency of constructing sketches with some common randomized embeddings listed in \Cref{tab:embedding_summary}. We consider applying an embedding $\Gammab$ of size $l \times m$ to a matrix $\Ab$ of size $m \times n$, which can be interpreted as embedding $n$ vectors in an ambient space $\R^m$ to a lower dimensional space $\R^l$. We scale the experiments with respect to the ambient dimension $m$, at several different embedding dimensions $l$, with a fixed number of repetitions $n=1000$.
Figure \ref{fig:sketch-ps-time} suggests that, with proper implementation, the sparse sign matrices are more efficient than the Gaussian embeddings and the SRTTs, especially for large-scale problems. The SRTTs outperform Gaussian embeddings in terms of efficiency, and such an advantage can be amplified as $l$ increases. These observations align with the asymptotic complexity in \Cref{tab:embedding_summary}. While we also observe that, with MATLAB default implementation, the Gaussian embeddings usually enjoy matching efficiency as sparse sign matrices for moderate-size problems, and are more efficient than SRTTs.

\begin{figure}[ht]
    \centering
    \begin{subfigure}{0.32\textwidth}
    \centering
    \includegraphics[width=\linewidth]{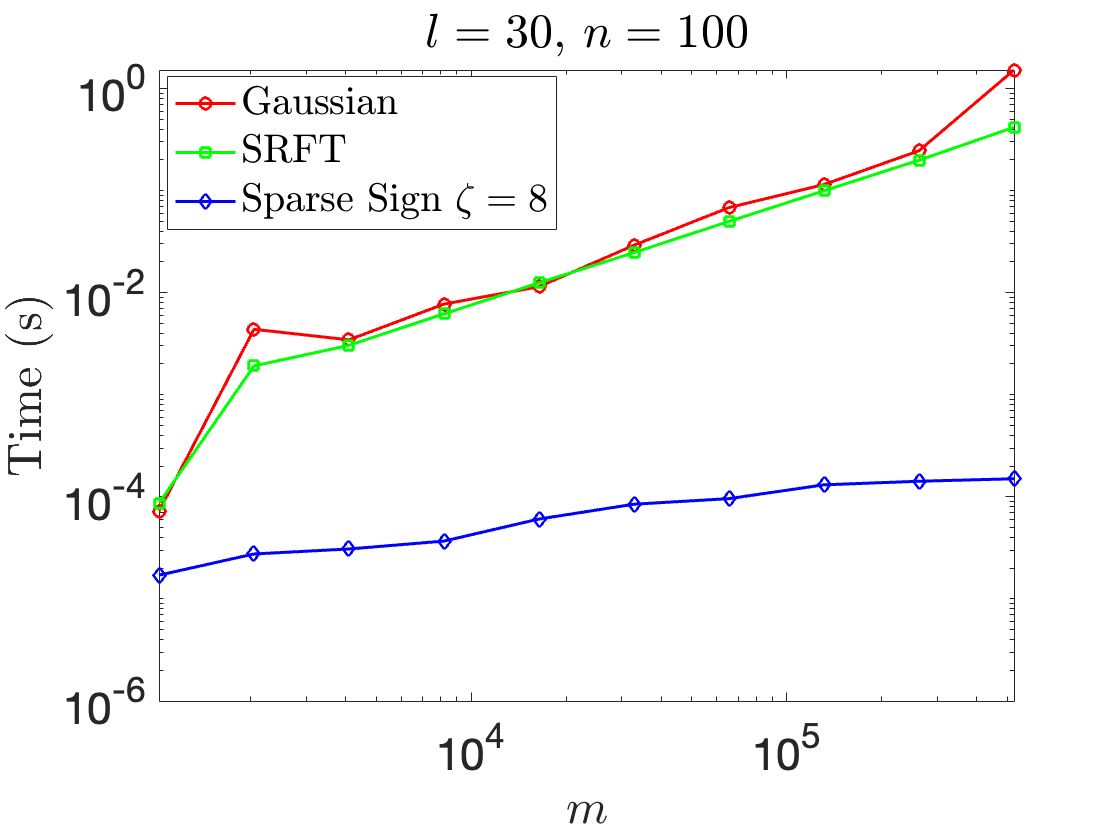}
    \end{subfigure}
    \begin{subfigure}{0.32\textwidth}
    \centering
    \includegraphics[width=\linewidth]{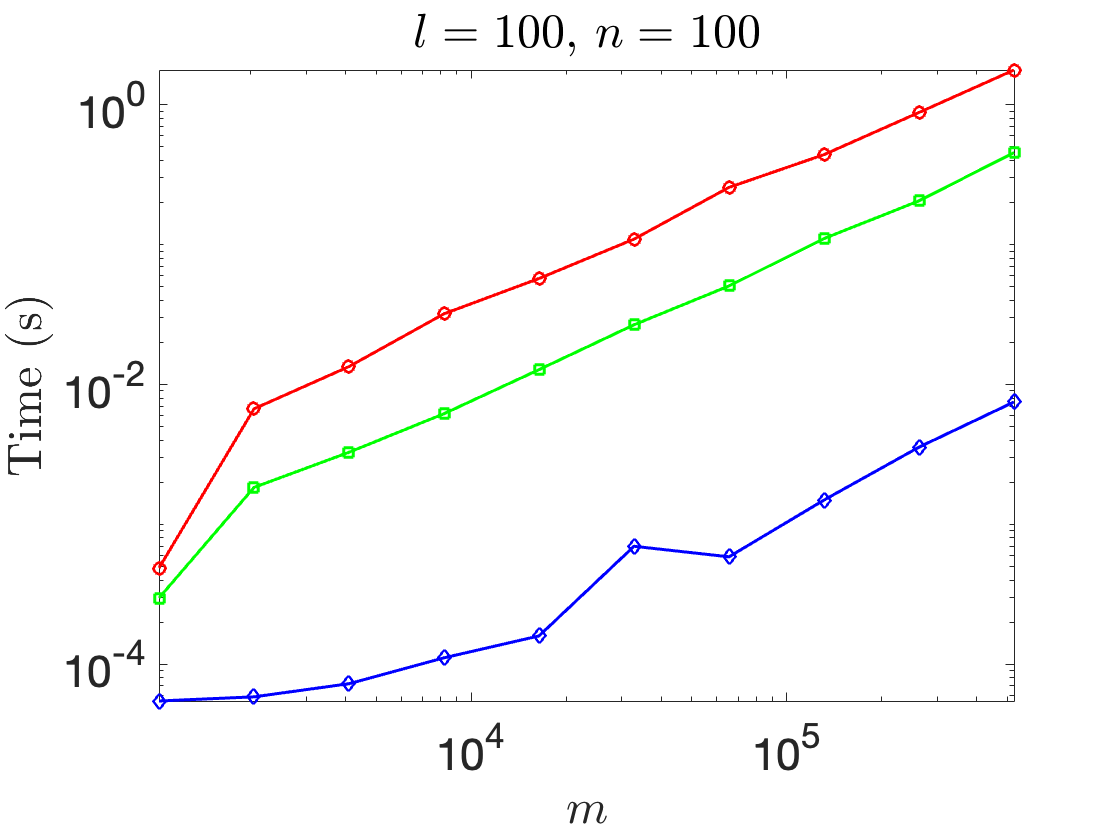}
    \end{subfigure}
    \begin{subfigure}{0.32\textwidth}
    \centering
    \includegraphics[width=\linewidth]{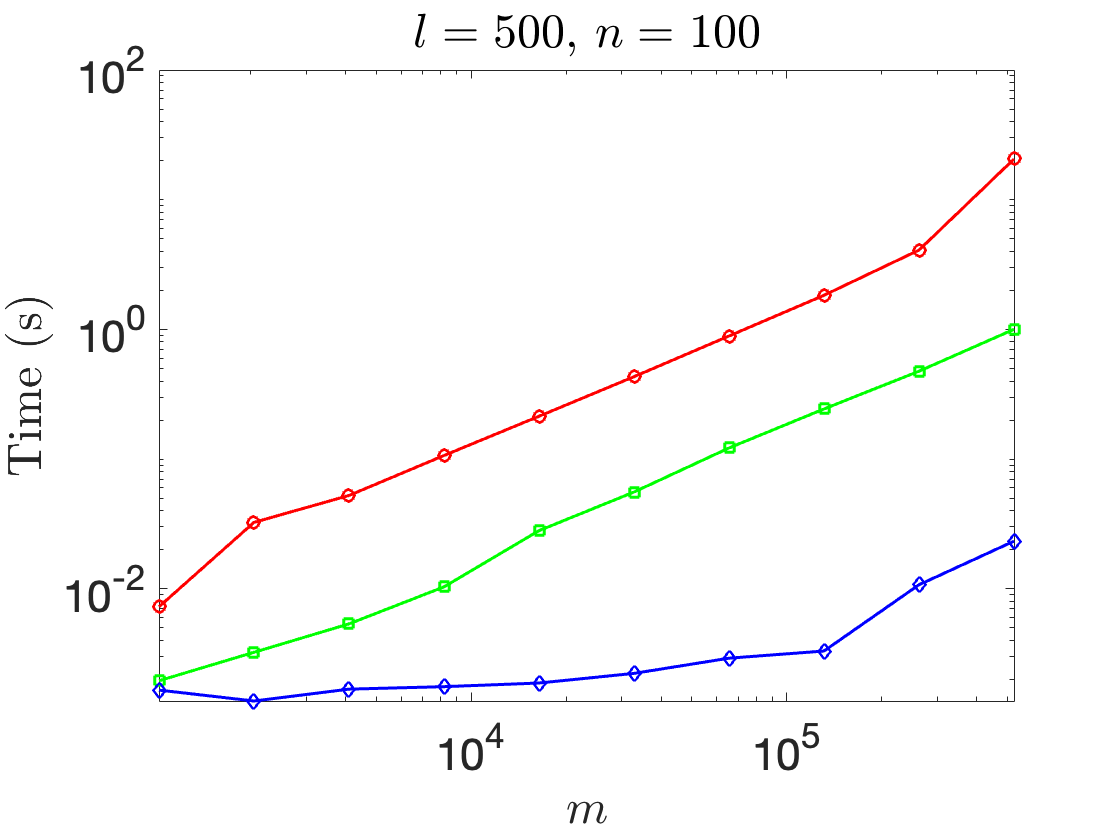}
    \end{subfigure}
    \caption{Run time of applying different randomized embeddings $\bs{\Gamma} \in \R^{l \times m}$ to some dense matrices of size $m \times n$, scaled with respect to the ambience dimension $m$, with different embedding dimension $l$, and a fixed number of embeddings $n=100$.}
    \label{fig:sketch-ps-time}
\end{figure}

\subsection{Computational speeds of different pivoting schemes}\label{subsec:exp_pivot}
Given a sketch of $\Ab$, we isolate different pivoting schemes in \Cref{algo:sketch_pivot_CUR_general} and compare their run time as the problem size $n$ increases.
Specifically, the \plu and \pqr pivot directly on the given row sketch $\Xb = \Gammab \Ab \in \R^{l \times n}$, while the DEIM involves one additional power iteration with orthogonalization (\Cref{eq:ortho_power_iter}) before applying the \plu ($\ie$, with a given column sketch $\Yb = \Ab \Omegab \in \R^{m \times l}$, for DEIM, we first construct an orthonormal basis $\Qb_Y \in \R^{m \times l}$ for columns of the sketch, and then we compute the reduced SVD for $\Qb_Y^{\top} \Ab \in l \times n$, and finally we column-wisely pivot on the resulting right singular vectors of size $l \times n$).
\begin{figure}[!ht]
    \centering
    \begin{subfigure}{0.32\textwidth}
    \centering
    \includegraphics[width=\linewidth]{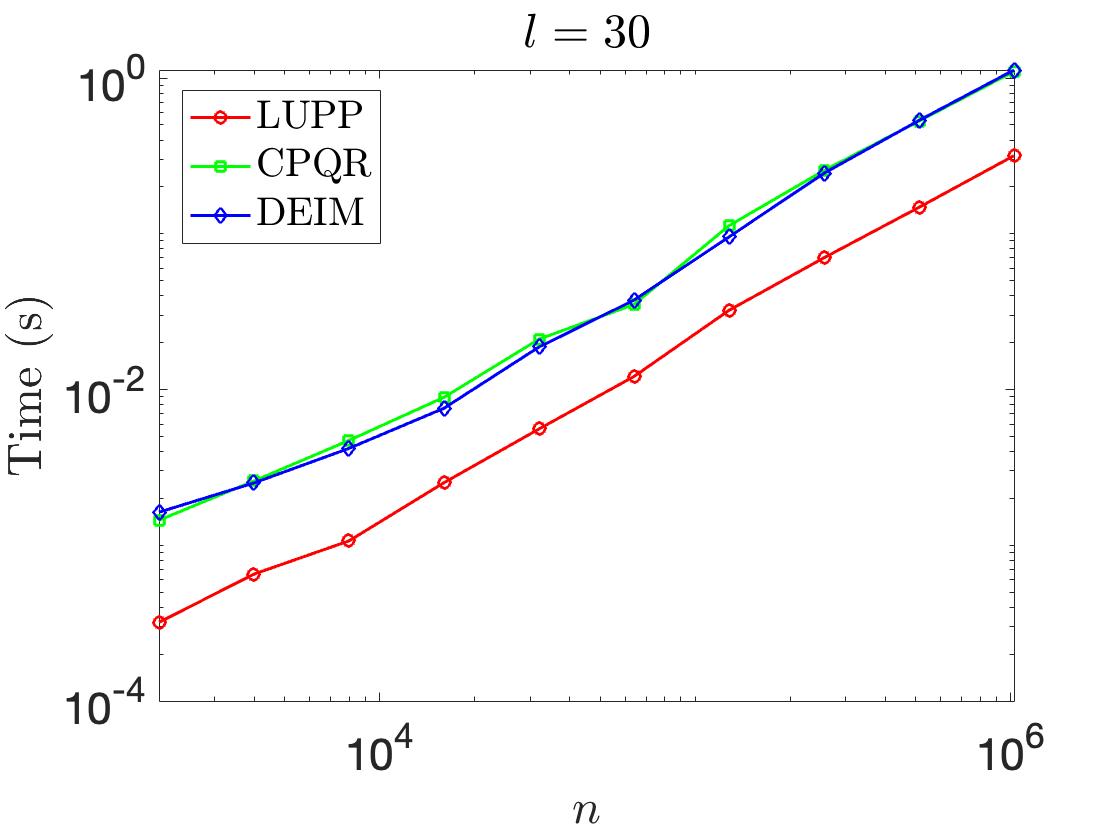}
    \end{subfigure}
    \begin{subfigure}{0.32\textwidth}
    \centering
    \includegraphics[width=\linewidth]{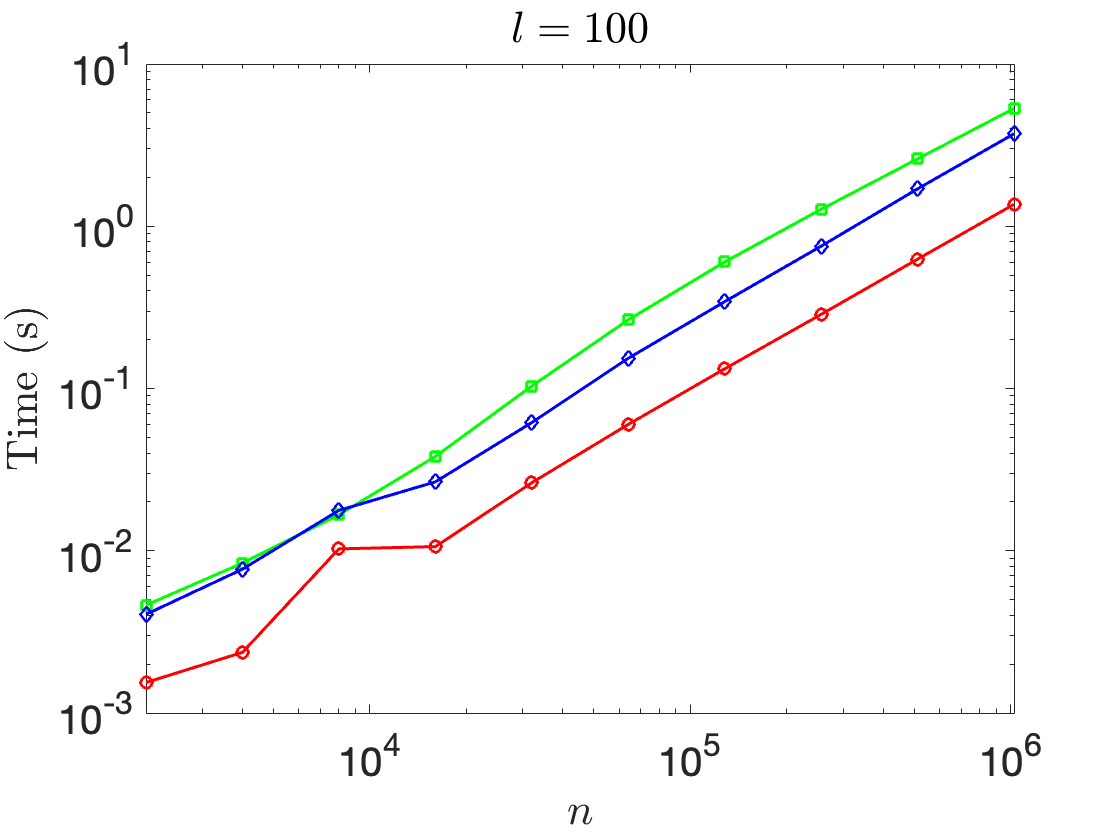}
    \end{subfigure}
    \begin{subfigure}{0.32\textwidth}
    \centering
    \includegraphics[width=\linewidth]{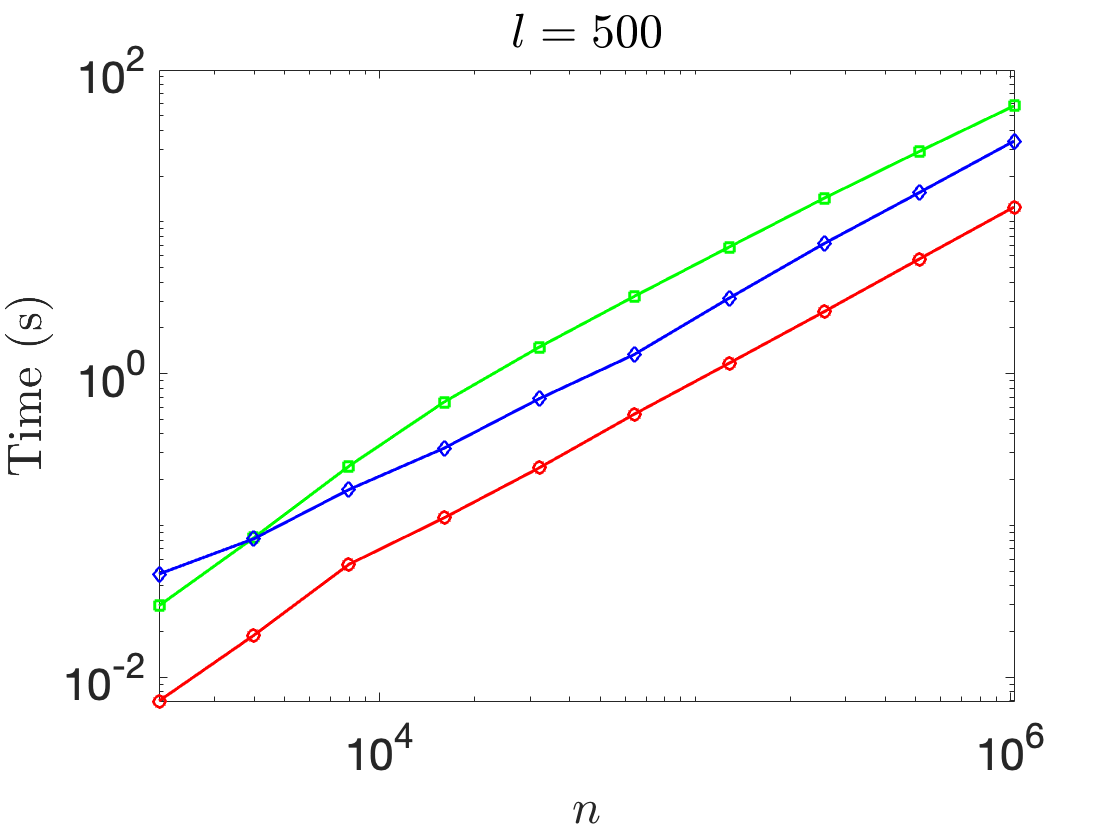}
    \end{subfigure}
    \caption{Run time of different pivoting schemes, scaled with respect to the problem size $n$, with different embedding dimension $l$.}
    \label{fig:pivot-ps-time}
\end{figure}
In Figure \ref{fig:pivot-ps-time}, we observe a considerable run time advantage of the \plu over the \pqr and DEIM, especially when $l$ is large. (Additionally, we see that DEIM slightly outperforms CPQR, which is perhaps surprising, given the substantially larger number of flops required by DEIM.)

\subsection{Randomized skeleton selection algorithms: accuracy and efficiency}\label{subsec:exp_cur}
As we move from measuring speed to measuring the precision of revealing the numerical rank of a matrix, 
the choice of test matrix becomes important. 
We consider four different classes of test matrices, including some synthetic random matrices with 
different spectral patterns, as well as some empirical datasets, as summarized below:
\begin{enumerate}
    \item \texttt{large}: a full-rank $4,282 \times 8,617$ sparse matrix with $20,635$ nonzero entries from the SuiteSparse matrix collection, generated by a linear programming problem sequence \cite{large}. 
    \item \texttt{YaleFace64x64}: a full-rank $165 \times 4096$ dense matrix, consisting of $165$ face images each of size $64 \times 64$. The flattened image vectors are centered and normalized such that the average image vector is zero, and the entries are bounded within $[-1,1]$.
    \item \texttt{MNIST} training set consists of $60,000$ images of hand-written digits from $0$ to $9$. Each image is of size $28 \times 28$. The images are flattened and normalized to form a full-rank matrix of size $N \times d$ where $N$ is the number of images and $d = 784$ is the size of the flattened images, with entries bounded in $[0,1]$. The nonzero entries take approximately $20\%$ of the matrix for both the training and the testing sets.
    \item Random \textit{sparse non-negative (SNN)} matrices are synthetic random sparse matrices used in \cite{voronin2017, sorensen2014} for testing skeleton selection algorithms. Given $s_1 \geq \dots \geq s_r > 0$, a random SNN matrix $\Ab$ of size $m \times n$ takes the form,
    \begin{equation}
        \label{eq:cur_snn_def}
        \Ab = \text{SNN}\bpar{\cpar{s_i}_{i=1}^r;\ m,n}:= \sum_{i=1}^r s_i \xb_i \yb_i^T
    \end{equation}
    where $\xb_i \in \R^m$, $\yb_i \in \R^{n}$, $i \in [r]$ are random sparse vectors with non-negative entries.
    In the experiments, we use two random SNN matrices of distinct sizes:
    \begin{enumerate}[label=(\roman*)]
        \item \texttt{SNN1e3} is a $1000 \times 1000$ SNN matrix with $r = 1000$, $s_i = \frac{2}{i}$ for $i=1,\dots,100$, and $s_i = \frac{1}{i}$ for $i=101,\dots,1000$;
        \item \texttt{SNN1e6} is a $10^6 \times 10^6$ SNN matrix with $r = 400$, $s_i = \frac{2}{i}$ for $i=1,\dots,100$, and $s_i = \frac{1}{i}$ for $i=101,\dots,400$.
    \end{enumerate}
\end{enumerate}

Scaled with respect to the approximation ranks $k$, we compare the accuracy and efficiency of the following randomized CUR algorithms:
\begin{enumerate}
    \item Rand-\plu (and Rand-\plu-1piter):
    \Cref{algo:sketch_pivot_CUR_general} with $\Xb = \Gammab \Ab$ being a row sketch (or with one plain power iteration as in \Cref{eq:def_power_iter}), and pivoting with \plu;
    \item Rand-\pqr (and Rand-\pqr-1piter): \Cref{algo:sketch_pivot_CUR_general} with $\Xb = \Gammab \Ab$ being a row sketch (or with one power iteration as in \Cref{eq:def_power_iter}), and pivoting with \pqr \cite{voronin2017};
    \item RSVD-DEIM: \Cref{algo:sketch_pivot_CUR_general} with $\Xb$ being an approximation of leading-$k$ right singular vectors (\Cref{eq:rsvd_procedures}), and pivoting with \plu \cite{sorensen2014};
    \item RSVD-LS: Skeleton sampling based on approximated leverage scores \cite{mahoney2009} from a rank-$k$ SVD approximation (\Cref{eq:rsvd_procedures});
    \item SRCUR: Spectrum-revealing CUR decomposition proposed in \cite{chen2020}.
\end{enumerate}
The asymptotic complexities of the first three randomized pivoting-based skeleton selection algorithms based on \Cref{algo:sketch_pivot_CUR_general} are summarized in \Cref{tab:complexity_rand_pivot}.
\begin{table}[!h]
    \centering
    \caption{Asymptotic complexities of various randomized pivoting-based skeleton selection algorithms based on \Cref{algo:sketch_pivot_CUR_general}.}
    \label{tab:complexity_rand_pivot}
    \begin{tabular}{c|c|c}
    \hline
        Algorithm & Row basis approximator construction (Line 1,2) & Pivoting (Line 3) \\
    \hline
        Rand-LUPP & $O(T_s(l,\Ab))$ & $O(n l^2)$ \\
        Rand-LUPP-1piter & $O(T_s(l,\Ab) + \nnz(\Ab) l)$ & $O(n l^2)$ \\
    \hline
        Rand-CPQR & $O(T_s(l,\Ab))$ & $O(n l^2)$ \\
        Rand-CPQR-1piter & $O(T_s(l,\Ab) + \nnz(\Ab) l)$ & $O(n l^2)$ \\
    \hline
        RSVD-DEIM & $O\rbr{T_s(l,\Ab) + (m+n)l^2 + \nnz(\Ab) l}$ & $O(nl^2)$ \\
    \hline
    \end{tabular}
\end{table}

For consistency, we use Gaussian embeddings for sketching throughout the experiments.
With the selected column and row skeletons, we leverage the stable construction in \Cref{eq:cur_stable} to form the corresponding CUR decompositions $\cur{\Ab}{I_s,J_s}$.
Although oversampling ($\ie$, $l > k$) is necessary for multiplicative error bounds with respect to the optimal rank-$k$ approximation error (\Cref{eq:rand_rangefinder_error_bound}, \Cref{thm:pivoting_on_rangeapprox_error_bound}), since oversampling can be interpreted as a shift of curves along the axis of the approximation rank, for the comparison purpose, we simply treat $l=k$, and compare the rank-$k$ approximation errors of the CUR decompositions against the optimal rank-$k$ approximation error $\norm{\Ab-\Ab_k}$.

\begin{figure}[!ht]
    \centering
    \begin{subfigure}{0.32\textwidth}
    \centering
    \includegraphics[width=\linewidth]{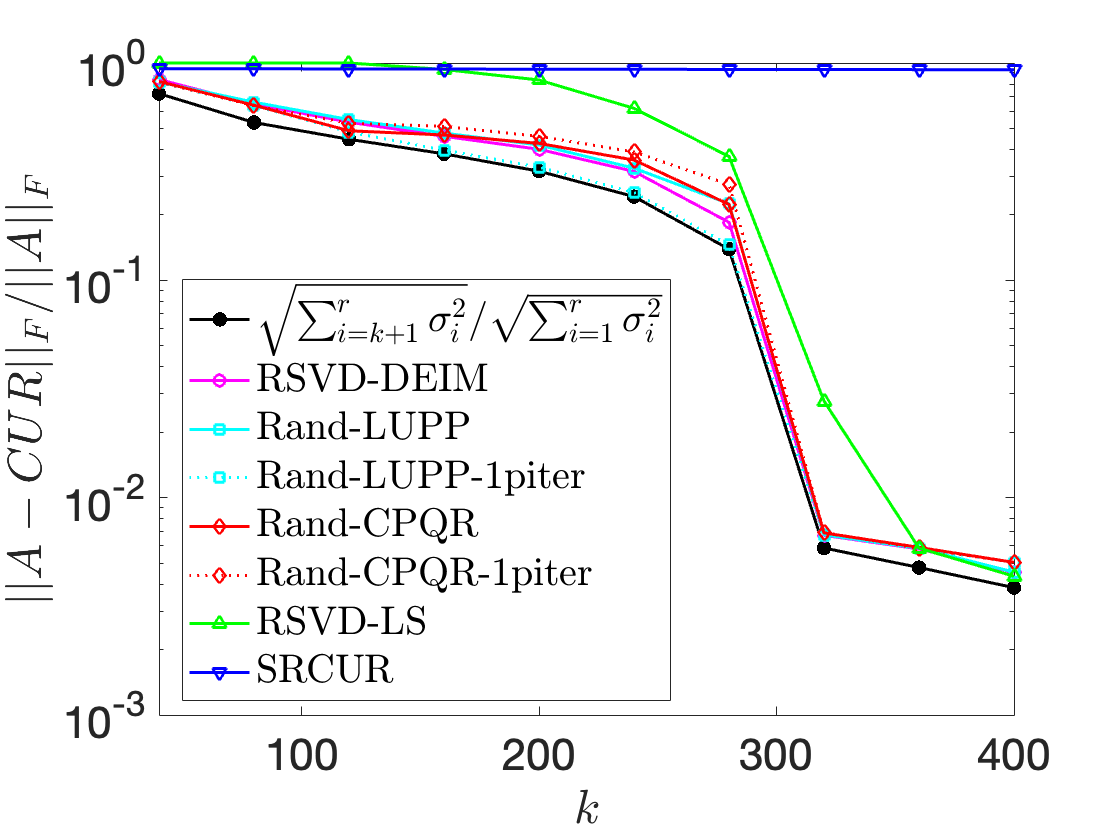}
    \caption{Frobenius norm error.}
    \label{fig:rand-errfro-rank_large}
    \end{subfigure}
    \begin{subfigure}{0.32\textwidth}
    \centering
    \includegraphics[width=\linewidth]{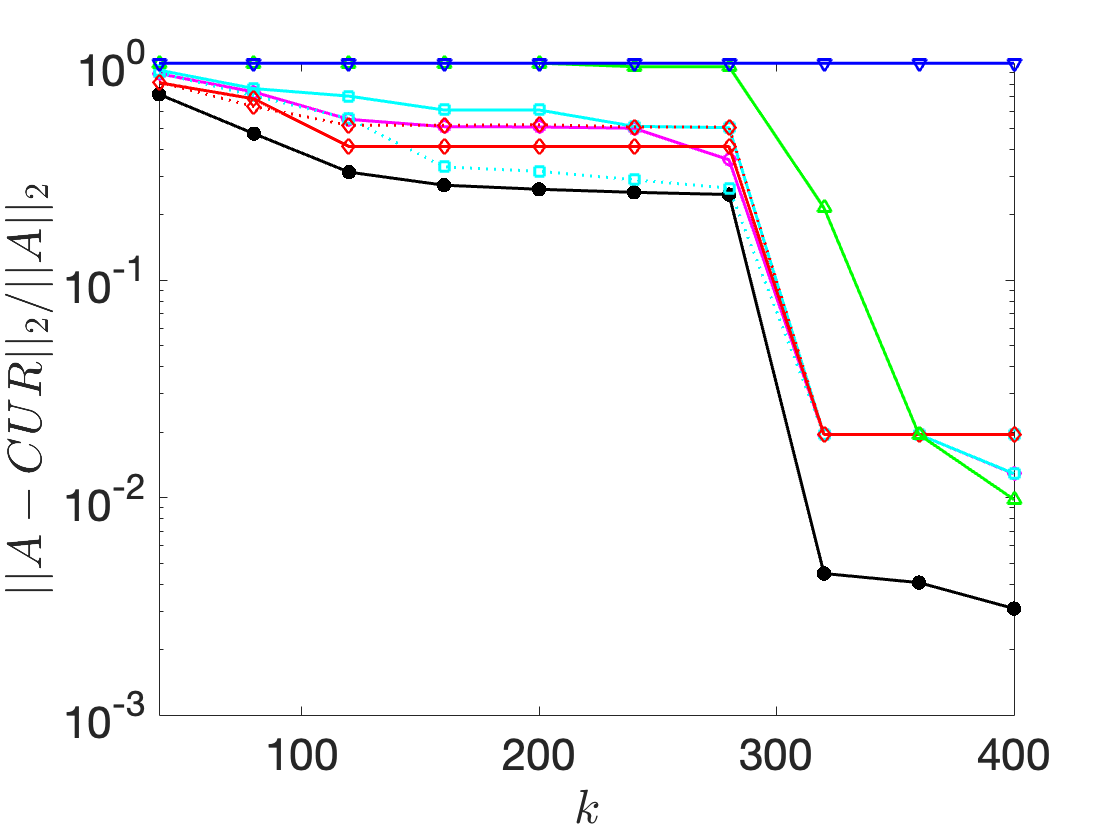}
    \caption{Spectral norm error.}
    \label{fig:rand-err2-rank_large}
    \end{subfigure}
    \begin{subfigure}{0.32\textwidth}
    \centering
    \includegraphics[width=\linewidth]{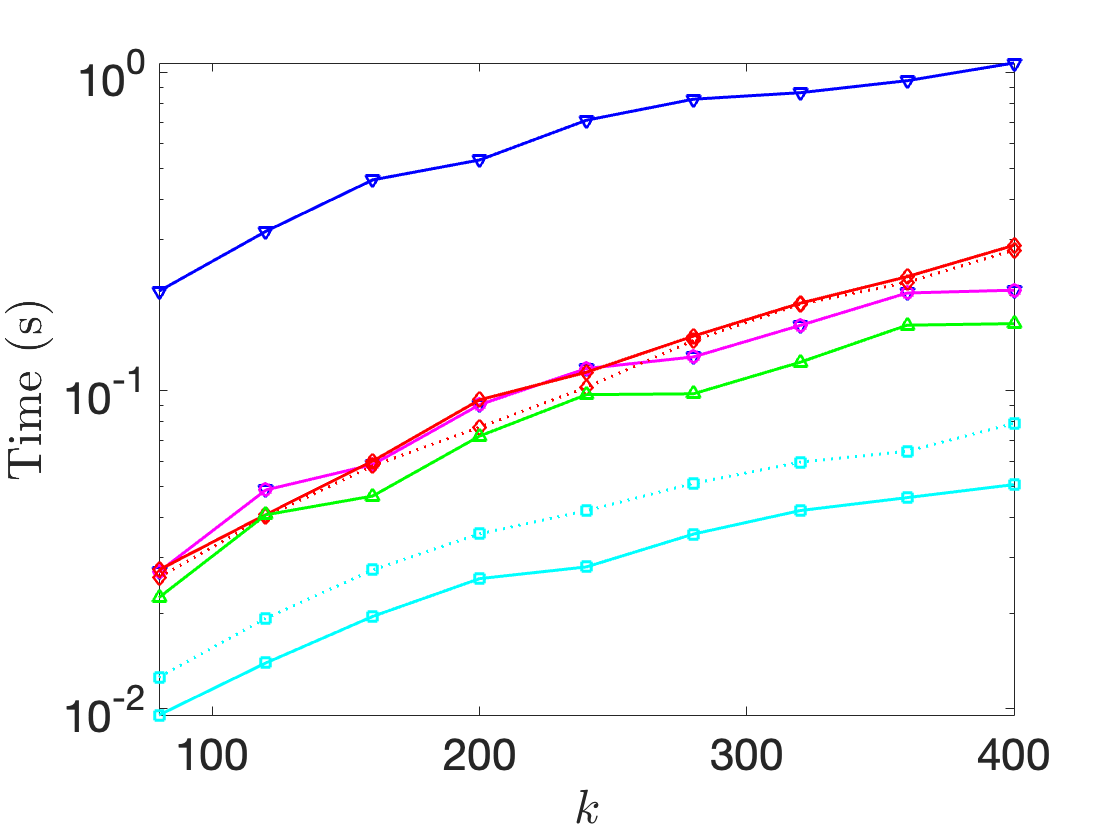}
    \caption{Runtime.}
    \label{fig:rand-time-rank_large}
    \end{subfigure}
    \caption{Relative error and run time of randomized skeleton selection on the \texttt{large} data set.}
    \label{fig:rand-err-rank_large}
\end{figure}

\begin{figure}[!ht]
    \centering
    \begin{subfigure}{0.32\textwidth}
    \centering
    \includegraphics[width=\linewidth]{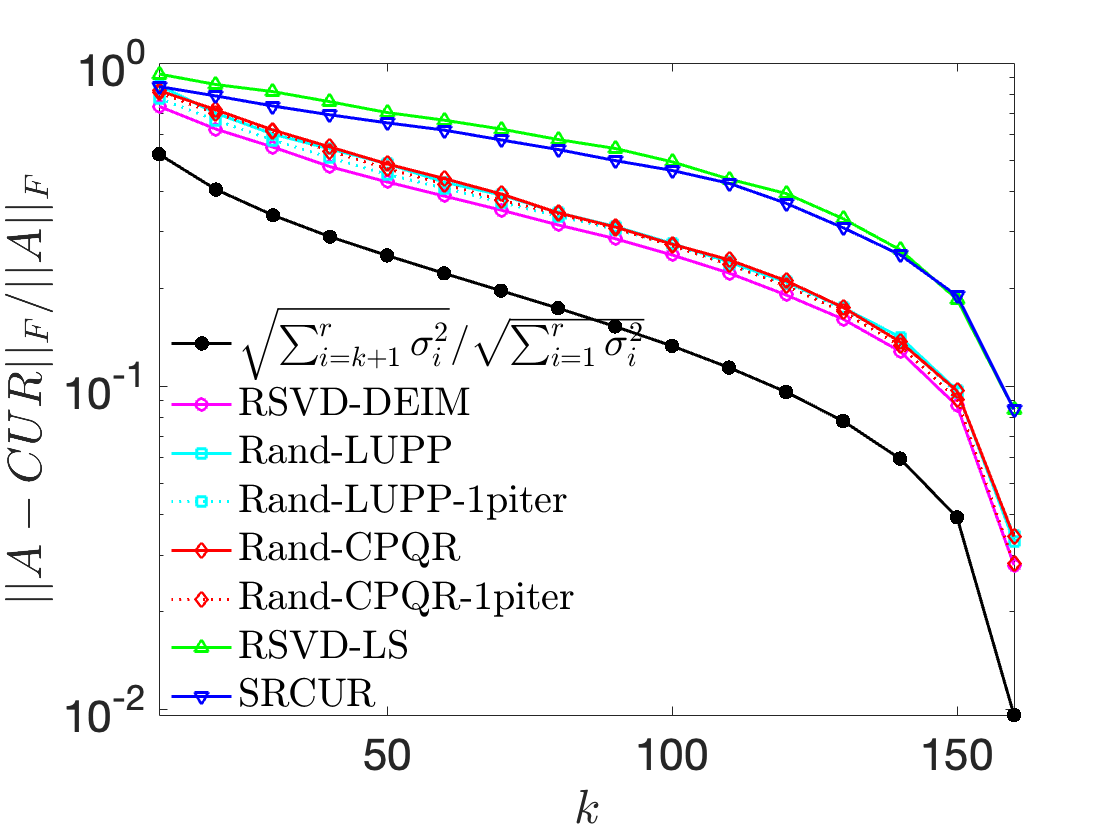}
    \caption{Frobenius norm error.}
    \label{fig:rand-errfro-rank_yaleface-64x64}
    \end{subfigure}
    \begin{subfigure}{0.32\textwidth}
    \centering
    \includegraphics[width=\linewidth]{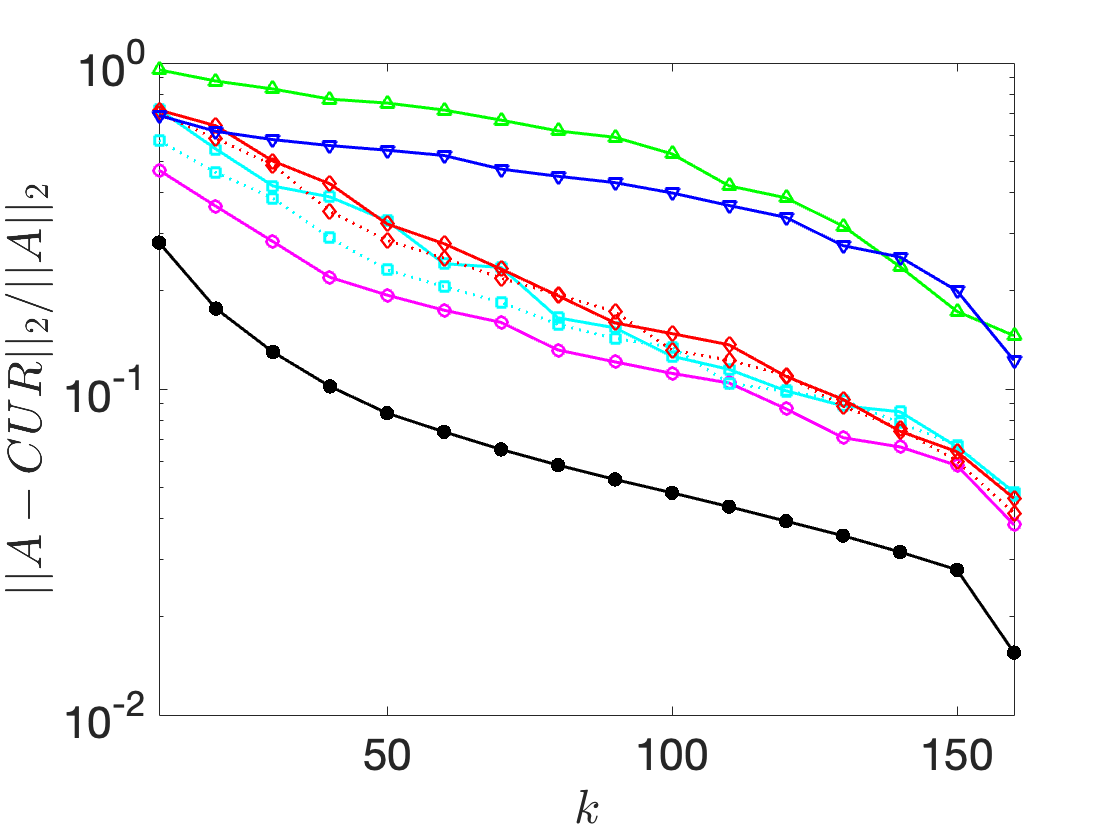}
    \caption{Spectral norm error.}
    \label{fig:rand-err2-rank_yaleface-64x64}
    \end{subfigure}
    \begin{subfigure}{0.32\textwidth}
    \centering
    \includegraphics[width=\linewidth]{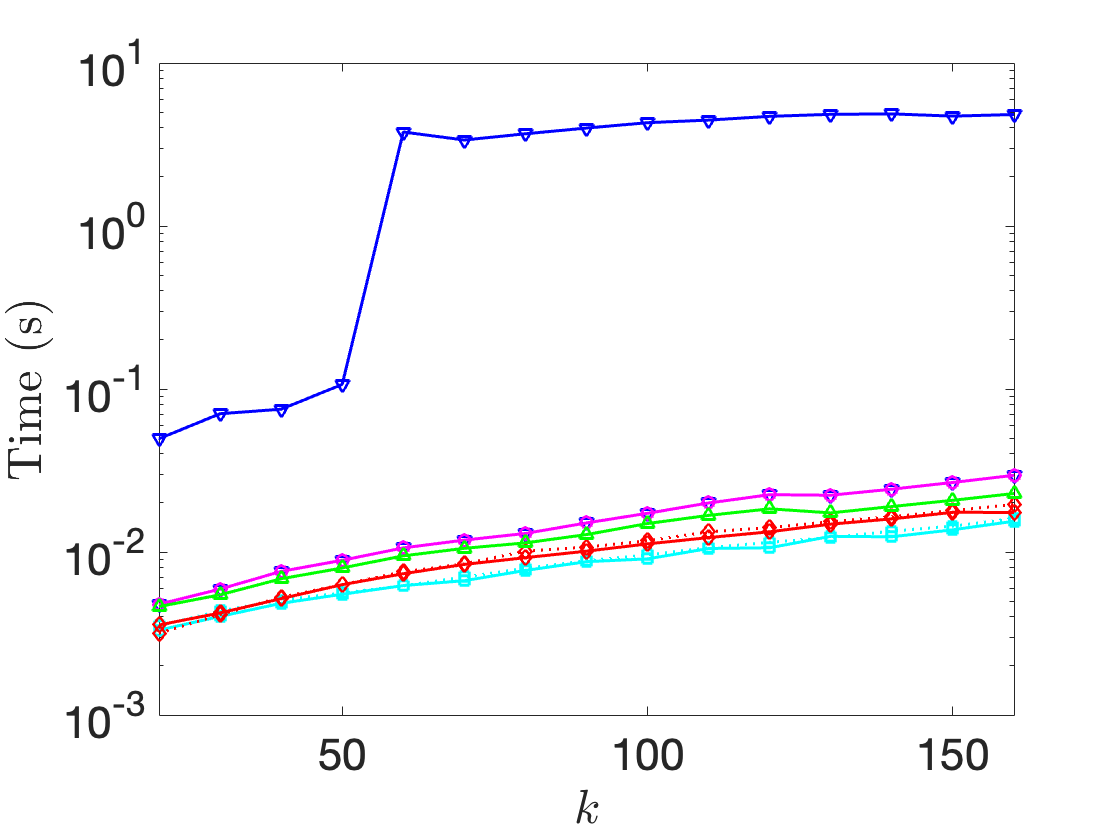}
    \caption{Runtime.}
    \label{fig:rand-time-rank_yaleface-64x64}
    \end{subfigure}
    \caption{Relative error and run time of randomized skeleton selection on the \texttt{YaleFace64x64} data set.}
    \label{fig:rand-err-rank_yaleface-64x64}
\end{figure}

\begin{figure}[!ht]
    \centering
    
    \begin{subfigure}{0.32\textwidth}
    \centering
    \includegraphics[width=\linewidth]{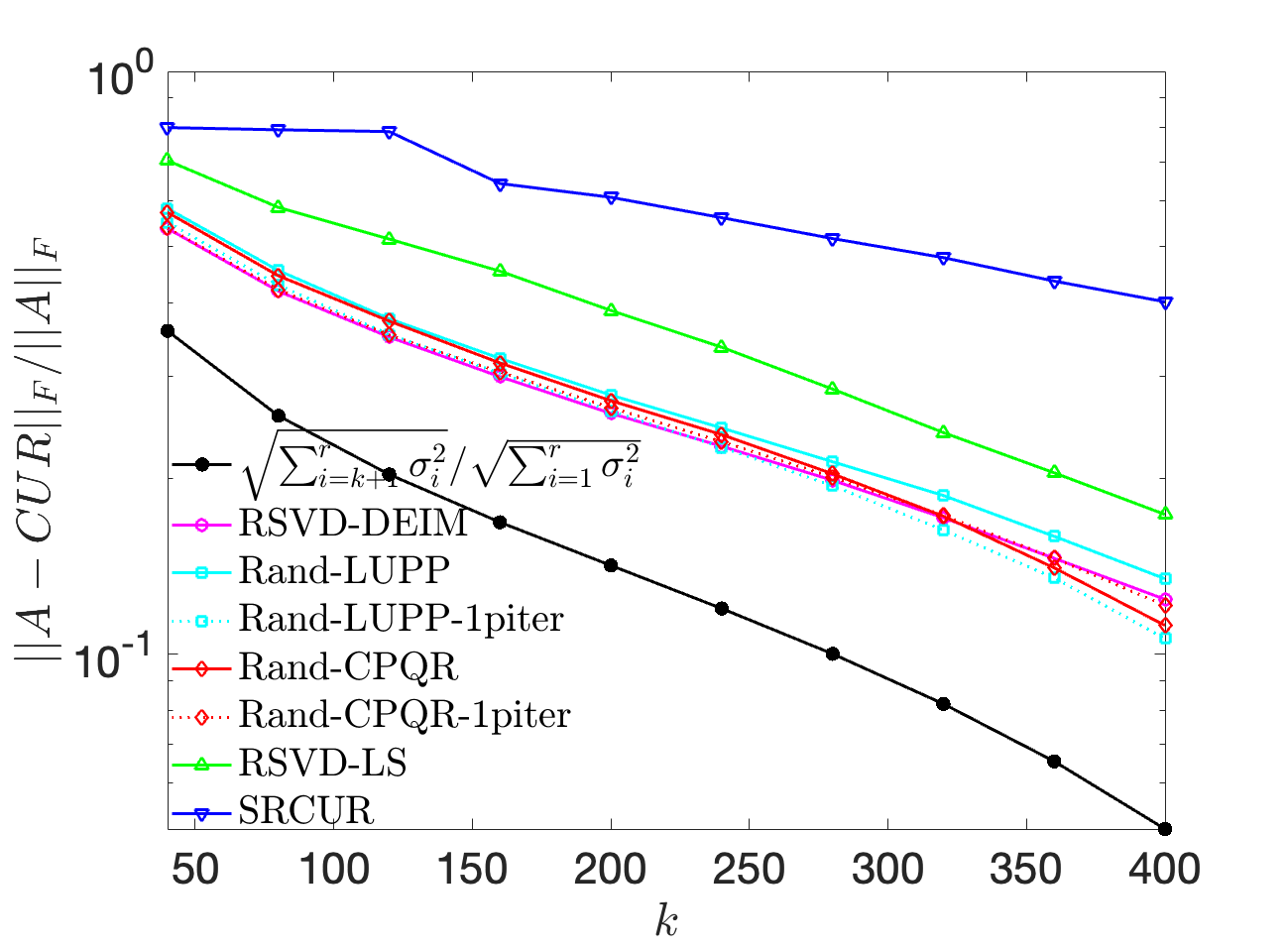}
    \caption{Frobenius norm error.}
    \label{fig:rand-errfro-rank_mnist-train}
    \end{subfigure}
    \begin{subfigure}{0.32\textwidth}
    \centering
    \includegraphics[width=\linewidth]{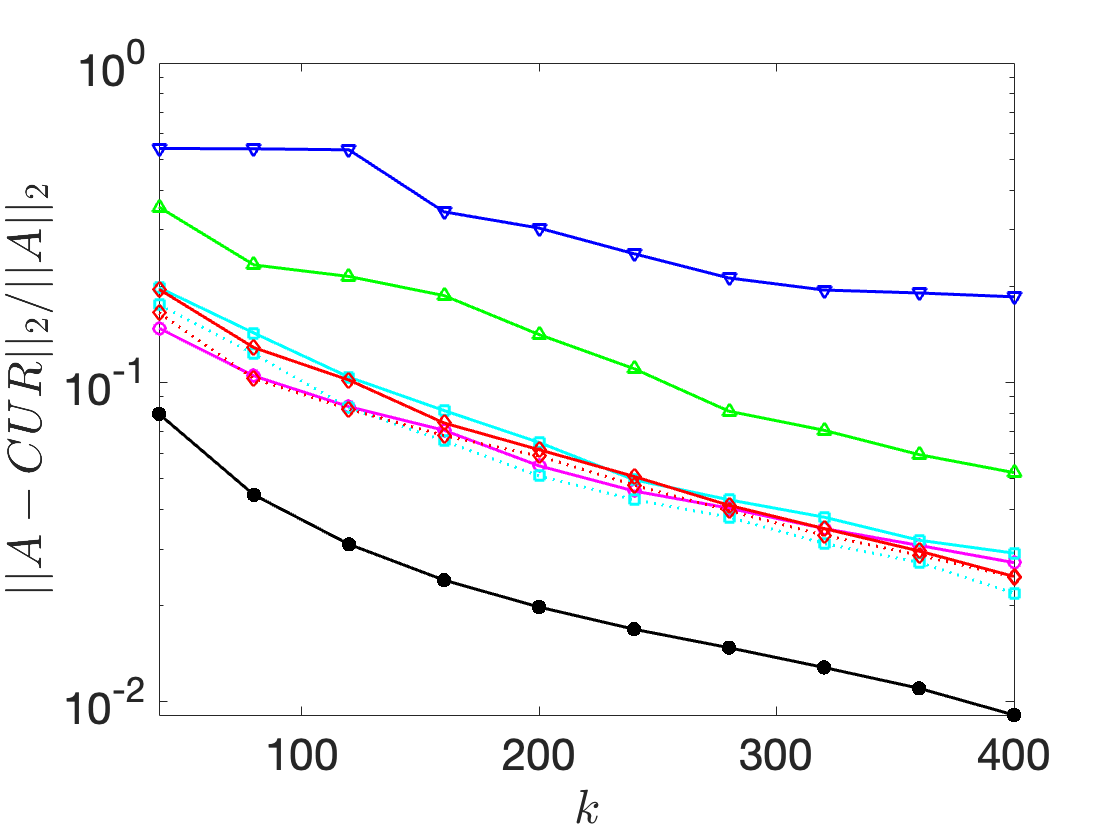}
    \caption{Spectral norm error.}
    \label{fig:rand-err2-rank_mnist-train}
    \end{subfigure}
    \begin{subfigure}{0.32\textwidth}
    \centering
    \includegraphics[width=\linewidth]{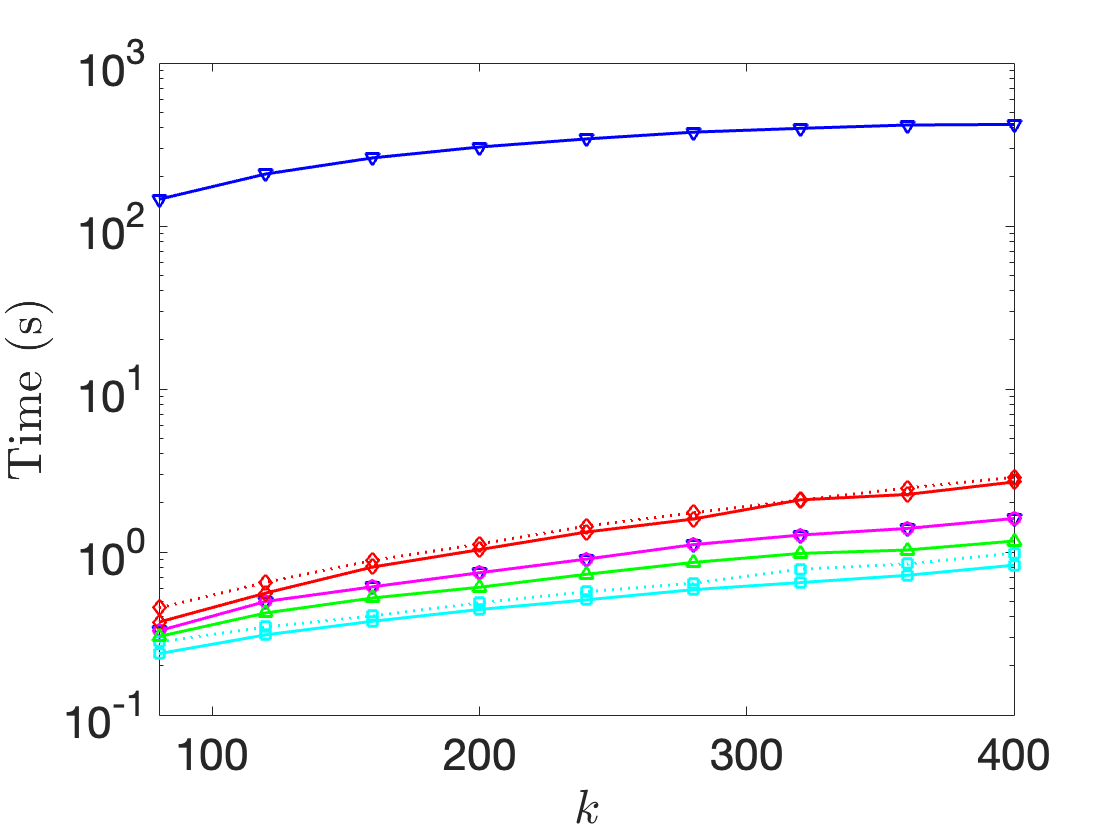}
    \caption{Runtime.}
    \label{fig:rand-time-rank_mnist-train}
    \end{subfigure}
    
    \caption{Relative error and run time of randomized skeleton selection on the training set of MNIST.}
    \label{fig:rand-err-rank_mnist-train}
\end{figure}

\begin{figure}[!ht]
    \centering
    
    \begin{subfigure}{0.32\textwidth}
    \centering
    \includegraphics[width=\linewidth]{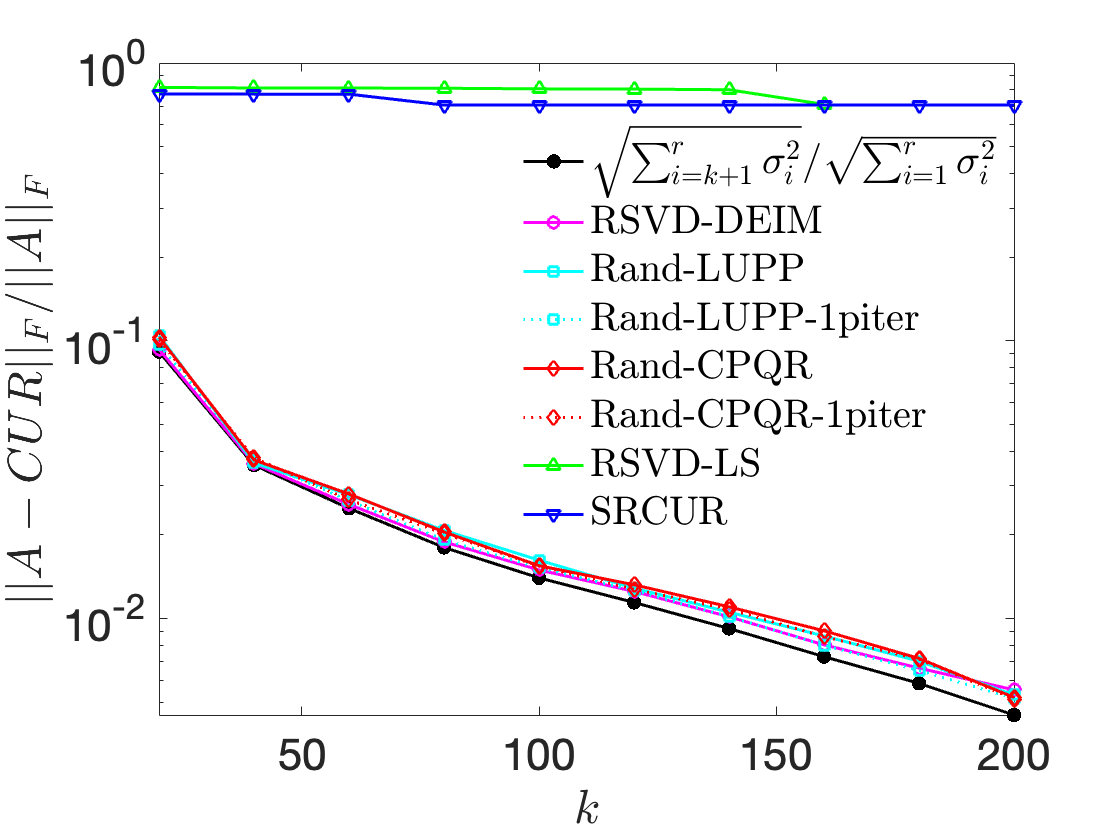}
    \caption{Frobenius norm error.}
    \label{fig:rand-errfro-rank_snn-1e3-1e3_a2b1_k100_r1e3_s1e-3}
    \end{subfigure}
    \begin{subfigure}{0.32\textwidth}
    \centering
    \includegraphics[width=\linewidth]{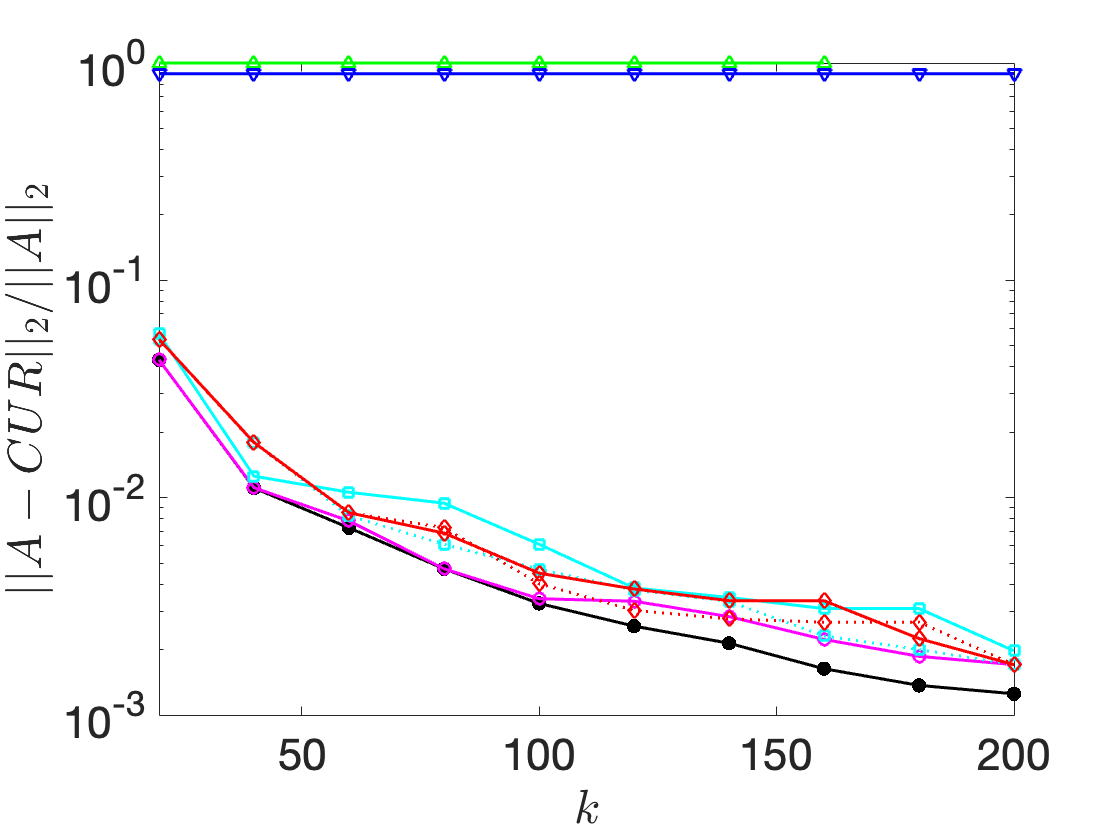}
    \caption{Spectral norm error.}
    \label{fig:rand-err2-rank_snn-1e3-1e3_a2b1_k100_r1e3_s1e-3}
    \end{subfigure}
    \begin{subfigure}{0.32\textwidth}
    \centering
    \includegraphics[width=\linewidth]{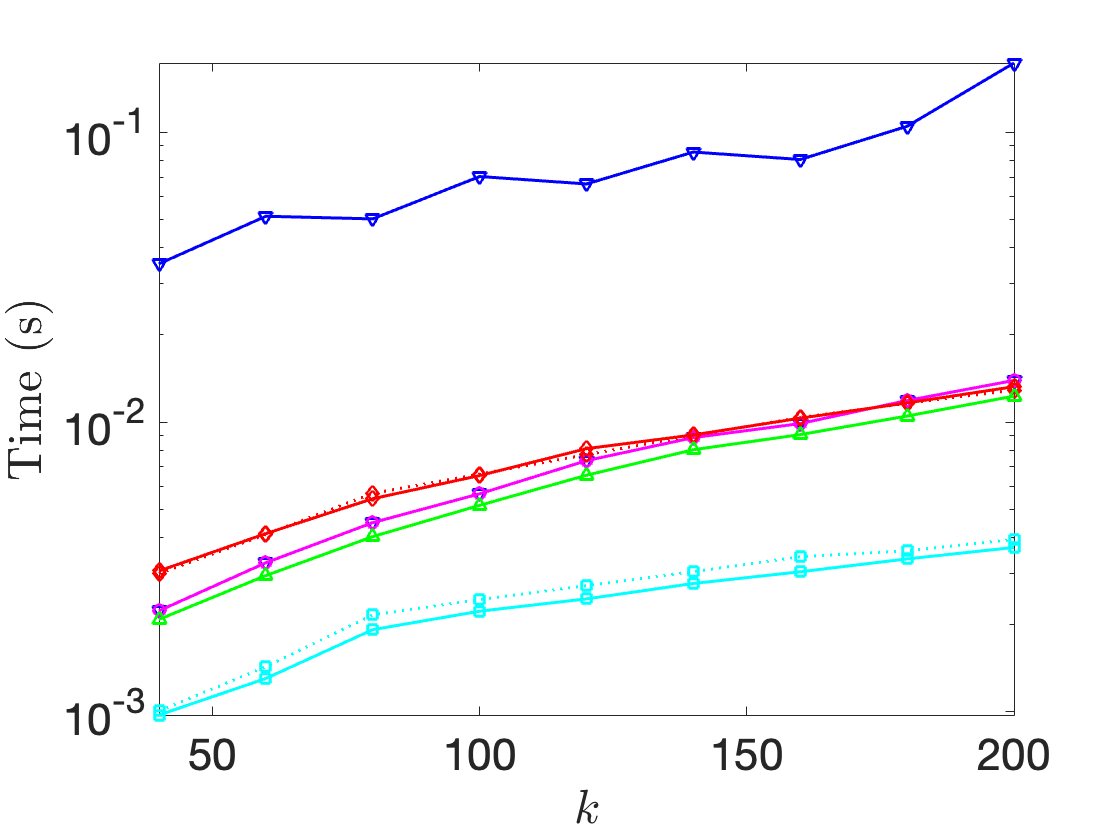}
    \caption{Runtime.}
    \label{fig:rand-time-rank_snn-1e3-1e3_a2b1_k100_r1e3_s1e-3}
    \end{subfigure}
    
    \caption{Relative error and run time of randomized skeleton selection on a $1000 \times 1000$ sparse non-negative random matrix, \texttt{SNN1e3}.}
    \label{fig:rand-err-rank_snn-1e3-1e3_a2b1_k100_r1e3_s1e-3}
\end{figure}

\begin{figure}[!ht]
    \centering
    
    \begin{subfigure}{0.32\textwidth}
    \centering
    \includegraphics[width=\linewidth]{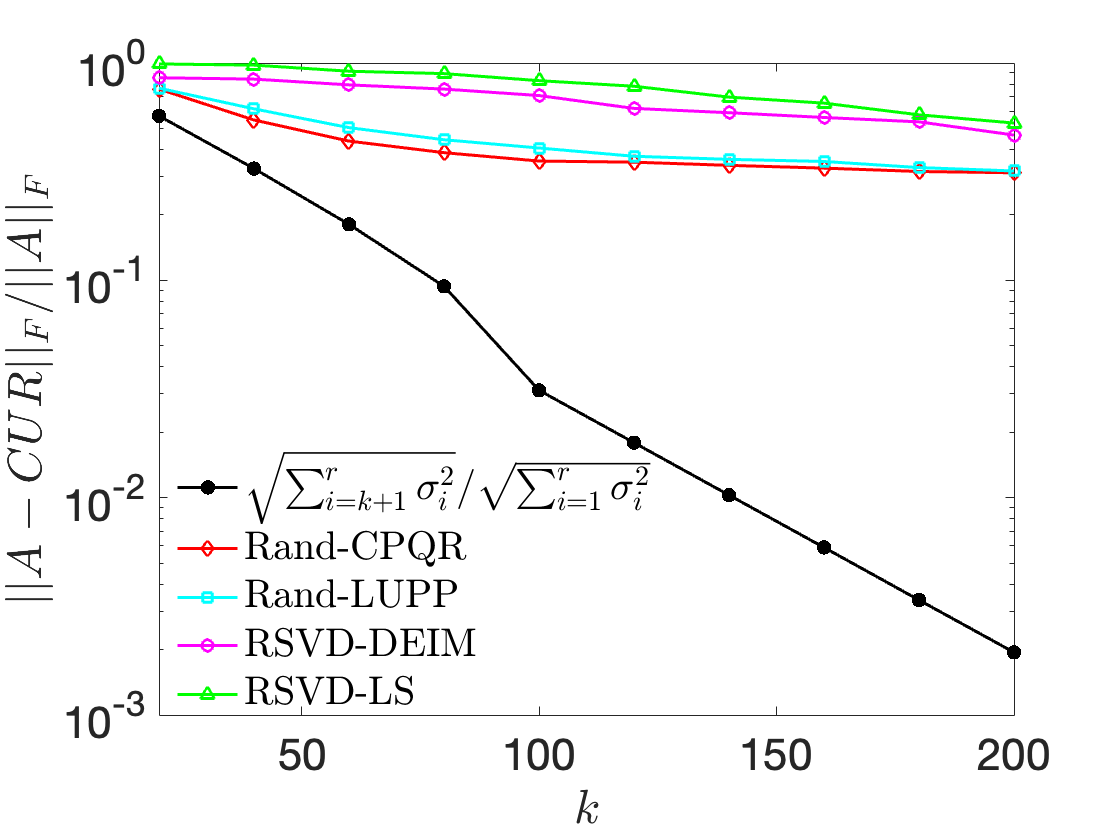}
    \caption{Frobenius norm error.}
    \label{fig:and-errref-rank_snn-1e6-1e6-a2b1-k100-r400-s2or}
    \end{subfigure}
    \begin{subfigure}{0.32\textwidth}
    \centering
    \includegraphics[width=\linewidth]{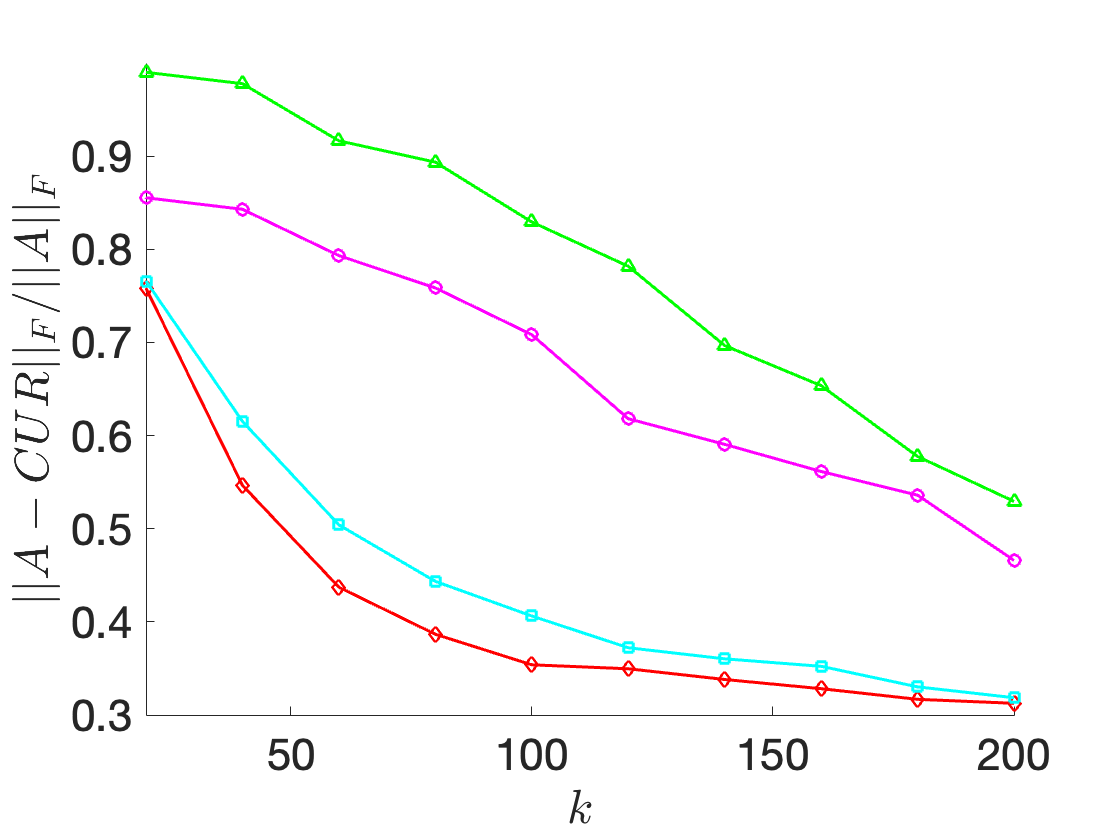}
    \caption{Frobenius norm error zoomed.}
    \label{fig:rand-errfro-rank_snn-1e6-1e6-a2b1-k100-r400-s2or}
    \end{subfigure}
    \begin{subfigure}{0.32\textwidth}
    \centering
    \includegraphics[width=\linewidth]{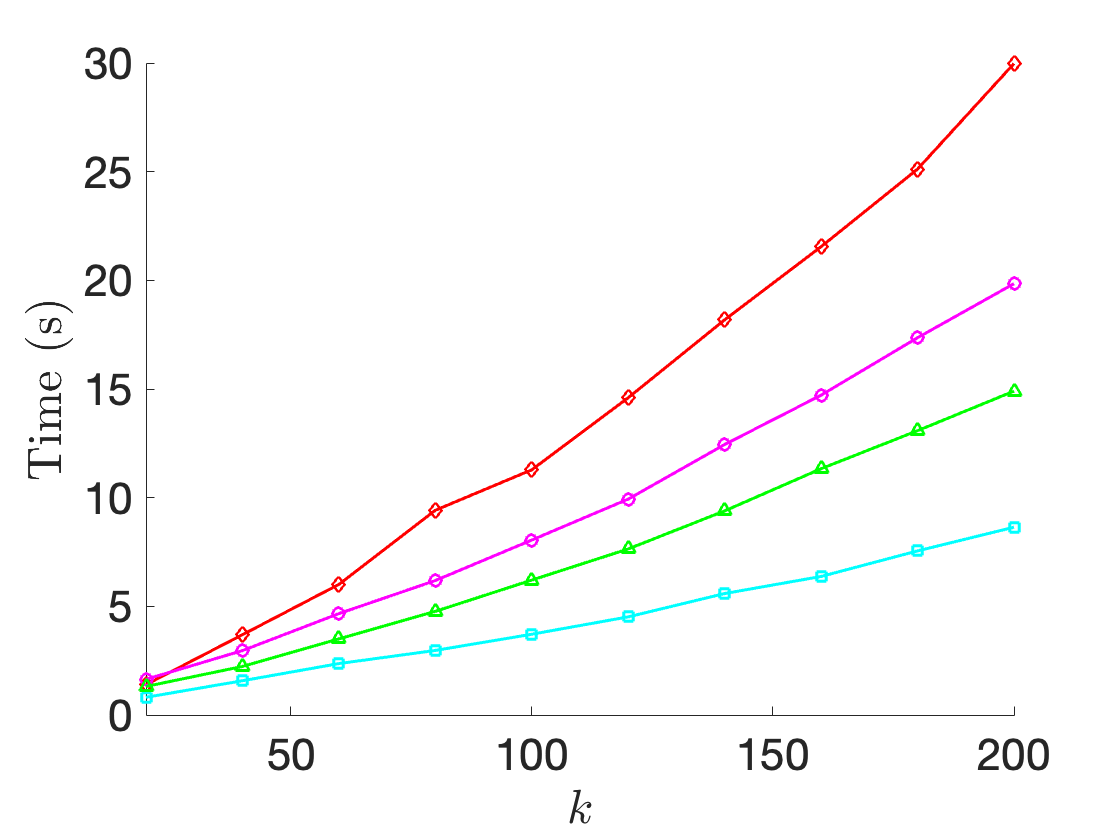}
    \caption{Runtime.}
    \label{fig:rand-time-rank_snn-1e6-1e6-a2b1-k100-r400-s2or}
    \end{subfigure}
    
    \caption{Relative error and run time of randomized skeleton selection on a $10^6 \times 10^6$ sparse non-negative random matrix, \texttt{SNN1e6}.}
    \label{fig:rand-err-rank_snn-1e6-1e6-a2b1-k100-r400-s2or}
\end{figure}

From Figure \ref{fig:rand-err-rank_large}-\ref{fig:rand-err-rank_snn-1e3-1e3_a2b1_k100_r1e3_s1e-3}, we observe that the randomized pivoting-based skeleton selection algorithms that fall in \Cref{algo:sketch_pivot_CUR_general} ($\ie$, Rand-\plu, Rand-\pqr, and RSVD-DEIM) share the similar approximation accuracy, which is considerably higher than that of the RSVD-LS and SRCUR. From the efficiency perspective, Rand-\plu provides the most competitive run time among all the algorithms, especially when $\Ab$ is sparse. Meanwhile, we observe that, for both Rand-\pqr and Rand-\plu, constructing the sketches with one plain power iteration ($\ie$, with \Cref{eq:def_power_iter}) can observably improve the accuracy, without sacrificing the efficiency significantly ($\eg$, in comparison to the randomized DEIM which involves one power iteration with orthogonalization as in \Cref{eq:ortho_power_iter}).

In Figure \ref{fig:rand-err-rank_snn-1e6-1e6-a2b1-k100-r400-s2or}, a similar performance is also observed on a synthetic large-scale problem, \texttt{SNN1e6}, where the matrix is only accessible as a fast matrix-vector multiplication (matvec) oracle such that each matvec takes $o(mn)$ (i.e., $O((m+n)r)$ in our construction) operations. It is worth noticing that the gap between the optimal rank-$k$ approximation error and the CUR approximation error consists of two components:
\begin{enumerate*}[label=(\roman*)]
    \item the suboptimality of skeleton selection from the algorithms, and
    \item the gap between the optimal rank-$k$ approximation error and the CUR approximation error with the optimal skeleton selection from the matrix skeletonization problem.
\end{enumerate*}
Despite the unknown optimal skeleton selection due to its intractability \cite{civril2013exponential}, intuitively, the suboptimality from the matrix skeletonization problem itself accounts for the larger CUR approximation error in Figure \ref{fig:rand-err-rank_snn-1e6-1e6-a2b1-k100-r400-s2or} (\cf Figure \ref{fig:rand-err-rank_large}-\ref{fig:rand-err-rank_snn-1e3-1e3_a2b1_k100_r1e3_s1e-3}) due to the significantly more challenging skeleton selection problem brought by the large matrix size $m=n=10^6 \gg r = \rank(\Ab) = 400$.

\chapter{Randomized Subspace Approximations: Efficient Bounds and Estimates for Canonical Angles}\label{ch:svra}

\subsection*{Abstract}
Randomized subspace approximation with ``matrix sketching'' is an effective approach for constructing approximate partial singular value decompositions (SVDs) of large matrices. 
The performance of such techniques has been extensively analyzed, and very precise estimates on the distribution of the residual errors have been derived.
However, our understanding of the accuracy of the computed singular vectors (measured in terms of the canonical angles between the spaces spanned by the exact and the computed singular vectors, respectively) remains relatively limited.
In this work, we present bounds and estimates for canonical angles of randomized subspace approximation that \emph{can be computed efficiently either a priori or a posteriori}. 
Under moderate oversampling in the randomized SVD, our prior probabilistic bounds are asymptotically tight and can be computed efficiently, while bringing a clear insight into the balance between oversampling and power iterations given a fixed budget on the number of matrix-vector multiplications.
The numerical experiments demonstrate the empirical effectiveness of these canonical angle bounds and estimates on different matrices under various algorithmic choices for the randomized SVD.\footnote{This chapter is based on the following arXiv paper: \\
\bibentry{dong2022efficient}~\citep{dong2022efficient}.}

\section{Introduction}

In light of the ubiquity of high-dimensional data in modern computation, dimension reduction tools like the low-rank matrix decompositions are becoming indispensable tools for managing large data sets. In general, the goal of low-rank matrix decomposition is to identify bases of proper low-dimensional subspaces that well encapsulate the dominant components in the original column and row spaces.
As one of the most well-established forms of matrix decompositions, the truncated singular value decomposition (SVD) is known to achieve the optimal low-rank approximation errors for any given ranks \cite{eckart1936}. Moreover, the corresponding left and right leading singular subspaces can be broadly leveraged for problems like principal component analysis, canonical correlation analysis, spectral clustering \cite{boutsidis2015spectral}, and leverage score sampling for matrix skeleton selection \cite{drineas2008relative,mahoney2009,drineas2012}.

However, for large matrices, the computational cost of classical algorithms for computing the SVD (\cf \cite[Lec.~31]{trefethen1997} or \cite[Sec.~8.6.3]{golub2013}) quickly becomes prohibitive. Fortunately, a randomization framework known as ``matrix sketching'' \cite{halko2011finding,woodruff2015} provides a simple yet effective remedy for this challenge by embedding large matrices to random low-dimensional subspaces where the classical SVD algorithms can be executed efficiently. 

Concretely, for an input matrix $\Ab \in \C^{m \times n}$ and a target rank $k \ll \min(m,n)$, the basic version of the randomized SVD \cite[Alg.~4.1]{halko2011finding} starts by drawing a Gaussian random matrix $\Omegab \in \C^{n \times l}$ for a sample size $l$ that is slightly larger than $k$ so that $k < l \ll \min(m,n)$. Then through a matrix-matrix multiplication $\Xb = \Ab\Omegab$ with $O(mnl)$ complexity, the $n$-dimensional row space of $\Ab$ is embedded to a random $l$-dimensional subspace. With the low-dimensional range approximation $\Xb$, a rank-$l$ randomized SVD $\wh\Ab_l = \wh\Ub_l \wh\Sigmab_l \wh\Vb_l^*$ can be constructed efficiently by computing the QR and SVD of small matrices in $O\rbr{(m+n)l^2}$ time. When the spectral decay in $\Ab$ is slow, a few power iterations $\Xb=\rbr{\Ab\Ab^*}^q\Ab\Omegab$ (usually $q=1,2$) can be incorporated to enhance the accuracy, 
\cf\cite{halko2011finding} Algorithms 4.3 and 4.4.

Let 
$\Ab=\Ub\Sigmab\Vb^*$\footnote{
Here $\Ub \in \C^{m \times r}$, $\Vb \in \C^{n \times r}$, and $\Sigmab \in \C^{r \times r}$. $\Sigmab$ is a diagonal matrix with positive non-increasing diagonal entries, and $r \le \min(m,n)$.
} denote the (unknown) full SVD of $\Ab$. In this work, we explore the alignment between the true leading rank-$k$ singular subspaces $\Ub_k,\Vb_k$ and their respective rank-$l$ approximations $\wh\Ub_l, \wh\Vb_l$ in terms of the canonical angles $\angle\rbr{\Ub_k,\wh\Ub_l}$ and $\angle\rbr{\Vb_k,\wh\Vb_l}$. We introduce prior statistical guarantees and unbiased estimates for these angles with respect to $\Sigmab$, as well as posterior deterministic bounds with the additional dependence on $\wh\Ab_l$, as synopsized below.

\subsection{Our Contributions}
\paragraph{Prior probabilistic bounds and estimates with insight on oversampling-power iteration balance.}
Evaluating the randomized SVD with a fixed budget on the number of matrix-vector multiplications, the computational resource can be leveraged in two ways -- oversampling (characterized by $l-k$) and power iterations (characterized by $q$). A natural question is \emph{how to distribute the computation between oversampling and power iterations for better subspace approximations}? 

Answers to this question are problem-dependent: when aiming to minimize the canonical angles between the true and approximated leading singular subspaces, the prior probabilistic bounds and estimates on the canonical angles provide primary insights. 
To be precise, with isotropic random subspace embeddings and sufficient oversampling, the accuracy of subspace approximations depends jointly on the spectra of the target matrices, oversampling, and the number of power iterations. 
In this work, we present a set of prior probabilistic bounds that precisely quantify the relative benefits of oversampling versus power iterations.
Specifically, the canonical angle bounds in \Cref{thm:space_agnostic_bounds}
\begin{enumerate}[label=(\roman*)]
    \item provide statistical guarantees that are asymptotically tight under sufficient oversampling (\ie, $l=\Omega\rbr{k}$),
    \item unveil a clear balance between oversampling and power iterations for random subspace approximations with given spectra,
    \item can be evaluated in $O(\rank\rbr{\Ab})$ time given access to the (true/estimated) spectra and provide valuable estimations for canonical angles in practice with moderate oversampling (\eg, $l \ge 1.6k$).
\end{enumerate} 
Further, inspired by the derivation of the prior probabilistic bounds, we propose unbiased estimates for the canonical angles with respect to given spectra that admit efficient evaluation and concentrate well empirically.

\paragraph{Posterior residual-based guarantees.}
Alongside the prior probabilistic bounds, we present two sets of posterior canonical angle bounds that hold in the deterministic sense and can be approximated efficiently based on the residuals and the spectrum of $\Ab$.

\paragraph{Numerical comparisons.}
With numerical comparisons among different canonical angle bounds on a variety of data matrices, we aim to explore the question on \emph{how the spectral decay and different algorithmic choices of randomized SVD affect the empirical effectiveness of different canonical angle bounds}.
In particular, our numerical experiments suggest that, for matrices with subexponential spectral decay, the prior probabilistic bounds usually provide tighter (statistical) guarantees than the (deterministic) guarantees from the posterior residual-based bounds, especially with power iterations. By contrast, for matrices with exponential spectral decay, the posterior residual-based bounds can be as tight as the prior probabilistic bounds, especially with large oversampling. 
The code for numerical comparisons is available at \href{https://github.com/dyjdongyijun/Randomized_Subspace_Approximation}{https://github.com/dyjdongyijun/Randomized\_Subspace\_Approximation}.

\subsection{Related Work}

The randomized SVD algorithm (with power iterations) \cite{martinsson2011randomized,halko2011finding} has been extensively analyzed as a low-rank approximation problem where the accuracy is usually measured in terms of residual norms, as well as the discrepancy between the approximated and true spectra \cite{halko2011finding,gu2014subspace,martinsson2019randomized,martinsson2020randomized}. For instance, 
\cite[Thm.~10.7, Thm.~10.8]{halko2011finding}
 show that for a given target rank $k$ (usually $k \ll \min\rbr{m,n}$ for the randomized acceleration to be useful), a small constant oversampling $l \ge 1.1 k$ is sufficient to guarantee that the residual norm of the resulting rank-$l$ approximation is close to the optimal rank-$k$ approximation (\ie, the rank-$k$ truncated SVD) error with high probability. Alternatively, \cite{gu2014subspace} investigates the accuracy of the individual approximated singular values $\wh\sigma_i$ and provides upper and lower bounds for each $\wh\sigma_i$ with respect to the true singular value $\sigma_i$.

In addition to providing accurate low-rank approximations, the randomized SVD algorithm also produces estimates of the leading left and right singular subspaces corresponding to the top singular values. When coupled with power iterations (\cite{halko2011finding} Algorithms 4.3 \& 4.4), such randomized subspace approximations are commonly known as randomized power (subspace) iterations. Their accuracy is explored in terms of canonical angles that measure differences between the unknown true subspaces and their random approximations \cite{boutsidis2015spectral,saibaba2018randomized,nakatsukasa2020sharp}. Generally, upper bounds on the canonical angles can be categorized into two types:
\begin{enumerate*}[label=(\roman*)]
    \item probabilistic bounds that establish prior statistical guarantees by exploring the concentration of the alignment between random subspace embeddings and the unknown true subspace, and
    \item residual-based bounds that can be computed a posteriori from the residual of the resulting approximation.
\end{enumerate*}

The existing prior probabilistic bounds on canonical angles~\cite{boutsidis2015spectral,saibaba2018randomized} mainly focus on the setting where the randomized SVD is evaluated without oversampling or with a small constant oversampling. Concretely, \cite{boutsidis2015spectral} derives guarantees for the canonical angles evaluated without oversampling (\ie, $l=k$) in the context of spectral clustering. Further, by taking a small constant oversampling (\eg, $l \ge k+2$) into account, Saibaba~\cite{saibaba2018randomized} provides a comprehensive analysis for an assortment of canonical angles between the true and approximated leading singular spaces. {Compared with} our results (\Cref{thm:space_agnostic_bounds}), in both no-oversampling and constant-oversampling regimes, the basic forms of the existing prior probabilistic bounds (\eg, \cite{saibaba2018randomized} Theorem 1) generally depend on the unknown singular subspace $\Vb$. Although such dependence is later lifted using the isotropicity and the concentration of the randomized subspace embedding $\Omegab$ (\eg, \cite{saibaba2018randomized} Theorem 6), the separated consideration on the spectra and the singular subspaces introduces unnecessary compromise to the upper bounds (as we will discuss in \Cref{remark:prior_bound_compare_exist}). In contrast, by allowing a more generous choice of multiplicative oversampling $l=\Omega\rbr{k}$, we present a set of space-agnostic bounds 
{(i.e., bounds that hold regardless of the singular vectors of $\Ab$)}
based on an integrated analysis of the spectra and the singular subspaces that appears to be tighter both from derivation and in practice. 

The classical Davis-Kahan $\sin\theta$ and $\tan\theta$ theorems~\cite{daviskahan} for eigenvector perturbation can be used to compute deterministic and computable bounds for the canonical angles. These bounds have the advantage that they give strict bounds (up to the estimation of the so-called gap) rather than estimates or bounds that hold with high probability (although, as we argue below, the failure probability can be taken to be negligibly low). The Davis-Kahan theorems have been extended to perturbation of singular vectors by Wedin~\cite{wedin1973perturbation}, and recent work~\cite{nakatsukasa2020sharp} derives perturbation bounds for singular vectors computed using a subspace projection method. In this work, we establish canonical angle bounds for the singular vectors in the context of (randomized) subspace iterations. Our results indicate clearly that the accuracy of the right and left singular vectors are usually not identical (i.e., $\Vb$ is more accurate with Algorithm~\ref{algo:rsvd_power_iterations}). 

As a roadmap, we formalize the problem setup in \Cref{sec:problem_setup}, including a brief review of the randomized SVD and canonical angles. In \Cref{sec:space_agnostic}, we present the prior probabilistic space-agnostic bounds. Subsequently, in \Cref{sec:space_agnostic_estimation}, we describe a set of unbiased canonical angle estimates that is closely related to the space-agnostic bounds. Then in \Cref{sec:separate_spec_bounds}, we introduce two sets of posterior residual-based bounds. Finally, in \Cref{sec:experiments}, we instantiate the insight cast by the space-agnostic bounds on the balance between oversampling and power iterations and demonstrate the empirical effectiveness of different canonical angle bounds and estimates with numerical comparisons.

\section{Problem Setup}\label{sec:problem_setup}

In this section, we first recapitulate the randomized SVD algorithm (with power iterations) \citep{halko2011finding} for which we analyze the accuracy of the resulting singular subspace approximations. Then, we review the notion of canonical angles \cite{golub2013} that quantify the difference between two subspaces of the same Euclidean space.

\subsection{Notation} 
We start by introducing notations for the SVD of a given matrix $\Ab \in \C^{m \times n}$ of rank $r$:
\begin{align*}
    \Ab 
    = 
    \underset{m \times r}{\Ub}\ \underset{r \times r}{\Sigmab}\ \underset{r \times n}{\Vb^*} 
    = 
    \bmat{\ub_{1} & \dots & \ub_{r}} \bmat{\sigma_{1} && \\ &\ddots& \\ && \sigma_{r}} \bmat{\vb_{1}^* \\ \vdots \\ \vb_{r}^*}.
\end{align*}
For any $1 \le k \le r$, we let $\Ub_{k} \triangleq \sbr{\ub_{1},\dots,\ub_{k}}$ and $\Vb_{k} \triangleq \sbr{\vb_{1},\dots,\vb_{k}}$ denote the orthonormal bases of the dimension-$k$ left and right singular subspaces of $\Ab$ corresponding to the top-$k$ singular values, while $\Ub_{r \setminus k} \triangleq\sbr{\ub_{k+1},\dots,\ub_{r}}$ and $\Vb_{r \setminus k} \triangleq\sbr{\vb_{k+1},\dots,\vb_{r}}$ are orthonormal bases of the respective orthogonal complements. The diagonal submatrices consisting of the spectrum, $\Sigmab_{k} \triangleq \diag\rbr{\sigma_{1}, \dots, \sigma_{k}}$ and $\Sigmab_{r \setminus k} \triangleq \diag\rbr{\sigma_{k+1}, \dots, \sigma_{r}}$, follow analogously.

Meanwhile, for the QR decomposition of an arbitrary matrix $\Mb \in \C^{d \times l}$ ($d \geq l$), we denote $\Mb = \sbr{\Qb_{\Mb},\Qb_{\Mb,\perp}} \bmat{\Rb_{\Mb} \\ \b{0}}$ such that $\Qb_{\Mb} \in \C^{d \times l}$ and $\Qb_{\Mb,\perp} \in \C^{d \times (d-l)}$ consist of orthonormal bases of the subspace spanned by the columns of $\Mb$ and its orthogonal complement.

Furthermore, we denote the spectrum of $\Mb$ by $\sigma\rbr{\Mb}$, a $\rank(\Mb) \times \rank(\Mb)$ diagonal matrix with singular values $\sigma_1\rbr{\Mb} \ge \dots \ge \sigma_{\rank(\Mb)}\rbr{\Mb} > 0$ on the diagonal. 

Generally, we adapt the MATLAB notation for matrix slicing throughout this work. For any $k \in \N$, we denote $[k]=\cbr{1,\dots,k}$.

\subsection{Randomized SVD and Power Iterations}

\begin{algorithm}
\caption{Randomized SVD (with power iterations) \citep{halko2011finding}}\label{algo:rsvd_power_iterations}
\begin{algorithmic}[1]
\REQUIRE $\Ab \in \C^{m \times n}$, power $q \in \cbr{0,1,2,\dots}$, oversampled rank $l \in \N$ ($l< r = \rank(\Ab)$)
\ENSURE $\wh\Ub_l \in \C^{m \times l}$, $\wh\Vb_l \in \C^{n \times l}$, $\wh\Sigmab_l \in \C^{l \times l}$ such that $\wh\Ab_l = \wh\Ub_l \wh\Sigmab_l \wh\Vb_l^*$

\STATE Draw $\Omegab \sim P\rbr{\C^{n \times l}}$ with $\Omega_{ij} \sim \Ncal\rbr{0,l^{-1}}~\iid$ such that $\E\sbr{\Omegab \Omegab^*} = \Ib_n$
\STATE $\Xb^{(q)} = \rbr{\Ab \Ab^*}^q \Ab \Omegab$
\STATE $\Qb_\Xb = \ortho\rbr{\Xb^{(q)}}$
\STATE $\sbr{\wt\Ub_l ,\wh\Sigmab_l, \wh\Vb_l} = \mathrm{svd}\rbr{\Ab^* \Qb_\Xb}$ (where $\wt\Ub_l \in \C^{l \times l}$)
\STATE $\wh\Ub_l = \Qb_\Xb \wt\Ub_l$
\end{algorithmic}   
\end{algorithm} 

As described in \Cref{algo:rsvd_power_iterations}, the randomized SVD provides a rank-$l$ ($l \ll \min(m,n)$) approximation of $\Ab \in \C^{m \times n}$ while grants provable acceleration to the truncated SVD evaluation -- $O\rbr{mnl(2q+1)}$ with the Gaussian random matrix\footnote{Asymptotically, there exist deterministic iterative algorithms for the truncated SVD (\eg, based on Lanczos iterations (\cite{trefethen1997} Algorithm 36.1)) that run in $O\rbr{mnl}$ time. However, compared with these inherently sequential iterative algorithms, the randomized SVD can be executed much more efficiently in practice, even with power iterations (\ie, $q>0$),  since the $O\rbr{mnl(2q+1)}$ computation bottleneck in \Cref{algo:rsvd_power_iterations} involves only matrix-matrix multiplications which are easily parallelizable and highly optimized.}. Such efficiency improvement is achieved by first embedding the high-dimensional row (column) space of $\Ab$ to a low-dimensional subspace via a Johnson-Lindenstrauss transform\footnote{Throughout this work, we focus on Gaussian random matrices (\Cref{algo:rsvd_power_iterations}, Line 1) in the sake of theoretical guarantees, \ie, $\Omegab$ being isotropic and rotationally invariant.} (JLT) $\Omegab$ (known as ``sketching''). Then, SVD of the resulting column (row) sketch $\Xb = \Ab\Omegab$ can be evaluated efficiently in $O\rbr{ml^2}$ time, and the rank-$l$ approximation can be constructed accordingly.

The spectral decay in $\Ab$ has a significant impact on the accuracy of the resulting low-rank approximation from \Cref{algo:rsvd_power_iterations} (as suggested in \cite{halko2011finding} Theorem 10.7 and Theorem 10.8). To remediate the performance of \Cref{algo:rsvd_power_iterations} on matrices with flat spectra, power iterations (\Cref{algo:rsvd_power_iterations}, Line 2) are usually incorporated to enhance the spectral decay. However, without proper orthogonalization, plain power iterations can be numerically unstable, especially for ill-conditioned matrices. For stable power iterations, starting with $\Xb = \Ab\Omegab \in \C^{m \times l}$ of full column rank (which holds almost surely for Gaussian random matrices), we incorporate orthogonalization in each power iteration via the reduced unpivoted QR factorization (each with complexity $O(ml^2)$). Let $\ortho\rbr{\Xb} = \Qb_\Xb \in \C^{m \times l}$ be an orthonormal basis of $\Xb$ produced by the QR factorization. Then, the stable evaluation of $q$ power iterations (\Cref{algo:rsvd_power_iterations}, Line 2) can be expressed as:
\begin{align}\label{eq:stable_power_iter}
    \Xb^{(0)} \gets \ortho\rbr{\Ab\Omegab}, \quad
    \Xb^{(i)} \gets \ortho\rbr{\Ab \ortho\rbr{\Ab^* \Xb^{(i-1)}}} \ \forall\ i \in [q].
\end{align}
 
Notice that in \Cref{algo:rsvd_power_iterations}, with $\Xb = \Xb^{(q)}$, the approximated rank-$l$ SVD of $\Ab$ can be expressed as $\wh\Ab_l = \wh\Ub_l \wh\Sigmab_l \wh\Vb_l^* = \Xb \Yb^*$ where $\Yb = \Xb^\dagger \Ab$.
With $\wh\Ub_l$ and $\wh\Vb_l$ characterizing the approximated $l$-dimensional left and right leading singular subspaces, $\wh\Ub_{m \setminus l} \in \C^{m \times (m-l)}$ and $\wh\Vb_{n \setminus l} \in \C^{n \times (n-l)}$ denote an arbitrary pair of their respective orthogonal complements. For any $1 \le k < l$, we further denote the partitions $\wh\Ub_l = \sbr{\wh\Ub_k, \wh\Ub_{l \setminus k}}$ and $\wh\Vb_l = \sbr{\wh\Vb_k, \wh\Vb_{l \setminus k}}$ where $\wh\Ub_k \in \C^{m \times k}$ and $\wh\Vb_k \in \C^{n \times k}$, respectively.

\subsection{Canonical Angles}

Now, we review the notion of canonical angles \cite{golub2013} that measure distances between two subspaces $\Ucal$, $\Vcal$ of an arbitrary Euclidean space $\C^d$.

\begin{definition}[Canonical angles, \cite{golub2013}]\label{def:canonical-angles}
Given two subspaces $\Ucal,\Vcal \subseteq \C^d$ with dimensions $\dim\rbr{\Ucal}=l$ and $\dim\rbr{\Vcal}=k$ (assuming $l \geq k$ without loss of generality), the canonical angles, denoted by $\angle\rbr{\Ucal,\Vcal}=\diag\rbr{\theta_1,\dots,\theta_k}$, consist of $k$ angles that measure the alignment between $\Ucal$ and $\Vcal$, defined recursively such that
\begin{align*}
    \ub_i, \vb_i ~\triangleq~
    &\argmax~\ub_i^*\vb_i \\
    \t{s.t.}~
    &\ub_i \in \rbr{\Ucal \setminus \spn\cbr{\ub_{\iota}}_{\iota=1}^{i-1}} \cap \SSS^{d-1},\\ 
    &\vb_i \in \rbr{\Vcal \setminus \spn\cbr{\vb_{\iota}}_{\iota=1}^{i-1}} \cap \SSS^{d-1}\\
    \cos \theta_i ~=~
    &\ub_i^* \vb_i \quad \forall~ i=1,\dots,k, \quad
    0 \leq \theta_1 \leq \dots \leq \theta_k \leq \pi/2.
\end{align*}
\end{definition}

For arbitrary full-rank matrices $\Ub \in \C^{d \times l}$ and $\Vb \in \C^{d \times k}$ (assuming $k \leq l \leq d$ without loss of generality), let $\angle\rbr{\Mb,\Nb} \triangleq \angle\rbr{\spn(\Mb),\spn(\Nb)}$ denote the canonical angles between the corresponding spanning subspaces in $\C^d$. For each $i \in [k]$, let $\angle_i\rbr{\Mb,\Nb}$ be the $i$-th (smallest) canonical angle such that $\cos\angle_i\rbr{\Mb,\Nb} = \sigma_i\rbr{\Qb_{\Mb}^*\Qb_{\Nb}}$ and $\sin\angle_i\rbr{\Mb,\Nb} = \sigma_{k-i+1}\rbr{\rbr{\Ib-\Qb_{\Mb}\Qb_{\Mb}^*}\Qb_{\Nb}}$ (\cf \cite{bjorck1973numerical} Section 3).

With the unknown true rank-$k$ truncated SVD $\Ab_k = \Ub_k \Sigmab_k \Vb_k^*$ and an approximated rank-$l$ SVD $\wh\Ab_l = \wh\Ub_l \wh\Sigmab_l \wh\Vb_l^*$ from \Cref{algo:rsvd_power_iterations}, in this work, we mainly focus on the prior and posterior guarantees for the canonical angles $\angle\rbr{\Ub_k,\wh\Ub_l}$ and $\angle\rbr{\Vb_k,\wh\Vb_l}$. Meanwhile, in \Cref{thm:separate_gap_bounds}, we present a set of posterior residual-based upper bounds for the canonical angles $\angle\rbr{\Ub_k,\wh\Ub_k}$ and $\angle\rbr{\Vb_k,\wh\Vb_k}$ as corollaries.

\section{Space-agnostic Bounds under Sufficient Oversampling}\label{sec:space_agnostic}

We start by pointing out the intuition that, under sufficient oversampling, with Gaussian random matrices whose distribution is orthogonally invariant, the alignment between the approximated and true subspaces are independent of the unknown true subspaces, \ie, the canonical angles are space-agnostic, as reflected in the following theorem.

\begin{theorem}\label{thm:space_agnostic_bounds}
For a rank-$l$ randomized SVD (\Cref{algo:rsvd_power_iterations}) with {a} Gaussian embedding {$\Omegab$} and $q \ge 0$ power iterations, when the oversampled rank $l$ satisfies $l = \Omega\rbr{k}$ (where $k$ is the target rank, $k < l < r = \rank(\Ab)$) and $q$ is reasonably small such that $\eta \dfeq \frac{\rbr{\sum_{j=k+1}^r \sigma_j^{4q+2}}^2}{\sum_{j=k+1}^r \sigma_j^{2(4q+2)}}$
\footnote{Notice that $1 < \eta \le r-k$. To the extremes, $\eta = r-k$ when the tail is flat $\sigma_{k+1}=\dots=\sigma_r$; while $\eta \to 1$ when $\sigma_{k+1} \gg \sigma_j$ for all $j=k+2,\dots,r$. In particular, with a relatively flat tail $\Sigmab_{r \setminus k}$ and a reasonably small $q$ (recall that $q=1,2$ is usually sufficient in practice), we have $\eta = \Theta\rbr{r-k}$, and the assumption can be simplified as $r-k=\Omega\rbr{l}$. Although exponential tail decay can lead to small $\eta$ and may render the assumption infeasible in theory, in practice, simply taking $r-k=\Omega\rbr{l}$, $l=\Omega\rbr{k}$, $\eps_1 = \sqrt{\frac{k}{l}}$ and $\eps_2=\sqrt{\frac{l}{r-k}}$ is sufficient to ensure the validity of upper bounds when $q \le 10$ even for matrices with rapid tail decay, as shown in \Cref{subsec:experiment_canonical_angle}.}
satisfies $\eta = \Omega\rbr{l}$, with high probability (at least $1-e^{-\Theta(k)}-e^{-\Theta(l)}$), there exist distortion factors $0< \eps_1, \eps_2 < 1$ such that
\begin{align}
    \label{eq:space_agnostic_left}
    &\sin\angle_i\rbr{\Ub_k, \wh\Ub_l} \le 
    \rbr{1+ \frac{1-\eps_1}{1+\eps_2} \cdot \frac{l}{\sum_{j=k+1}^r \sigma_j^{4q+2}} \cdot \sigma_i^{4q+2}}^{-\frac{1}{2}}
    \\
    \label{eq:space_agnostic_right}
    &\sin\angle_i\rbr{\Vb_k, \wh\Vb_l} \le 
    \rbr{1+ \frac{1-\eps_1}{1+\eps_2} \cdot \frac{l}{\sum_{j=k+1}^r \sigma_j^{4q+4}} \cdot \sigma_i^{4q+4}}^{-\frac{1}{2}}
\end{align}
for all $i \in [k]$, where $\eps_1 = \Theta\rbr{\sqrt{\frac{k}{l}}}$ and $\epsilon_2 = \Theta\rbr{\sqrt{\frac{l}{\eta}}}$.
Furthermore, both bounds are asymptotically tight:
\begin{align}
    \label{eq:space_agnostic_lower_left}
    &\sin\angle_i\rbr{\Ub_k, \wh\Ub_l} \ge \rbr{1+O\rbr{\frac{l \cdot \sigma_i^{4q+2}}{\sum_{j=k+1}^r \sigma_j^{4q+2}} }}^{-\frac{1}{2}}
    \\
    \label{eq:space_agnostic_lower_right}
    &\sin\angle_i\rbr{\Vb_k, \wh\Vb_l} \ge \rbr{1+O\rbr{\frac{l \cdot \sigma_i^{4q+4}}{\sum_{j=k+1}^r \sigma_j^{4q+4}} }}^{-\frac{1}{2}}
\end{align}
where $O\rbr{\cdot}$ suppresses the distortion factors $\frac{1+\eps_1}{1-\eps_2}$
\footnote{
    Despite the asymptotic tightness of \Cref{thm:space_agnostic_bounds} theoretically, in practice, we observe that the empirical validity of lower bounds is more restrictive on oversampling than that of upper bounds. In specific, the numerical observations in \Cref{subsec:experiment_canonical_angle} suggest that $l \ge 1.6 k$ is usually sufficient for the upper bounds to hold; whereas the empirical validity of lower bounds generally requires more aggressive oversampling of at least $l \ge 4 k$, also with slightly larger constants associated with $\eps_1$ and $\eps_2$, as demonstrated in \Cref{subapx:sup_exp_upper_lower_space_agnostic}.
}.
\end{theorem}   

The main insights provided by \Cref{thm:space_agnostic_bounds} include:
\begin{enumerate*}[label=(\roman*)]
    \item improved statistical guarantees for canonical angles under sufficient oversampling (\ie, $l=\Omega\rbr{k}$), as discussed later in \Cref{remark:prior_bound_compare_exist},
    \item a clear view of the balance between oversampling and power iterations for random subspace approximations with given spectra, as instantiated in \Cref{subsec:balance_oversampling_power_iterations_example}, and
    \item affordable upper bounds that can be evaluated in $O(\rank\rbr{\Ab})$ time with access to the (true/estimated) spectra and hold in practice with only moderate oversampling (\eg, $l \ge 1.6k$), as shown in \Cref{subsec:experiment_canonical_angle}.
\end{enumerate*} 

\begin{proof}[Proof of \Cref{thm:space_agnostic_bounds}]
We show the derivation of \Cref{eq:space_agnostic_left} for left canonical angles $\sin\angle_i\rbr{\Ub_k, \wh\Ub_l}$.
The derivation for right canonical angles $\sin\angle_i\rbr{\Vb_k, \wh\Vb_l}$ in \Cref{eq:space_agnostic_right} follows {directly} by replacing the exponent $4q+2$ in \Cref{eq:space_agnostic_left} with $4q+4$ in \Cref{eq:space_agnostic_right}. This slightly larger exponent comes from the additional half power iteration associated with $\wh\Vb_l$ in \Cref{algo:rsvd_power_iterations} (as discussed in \Cref{remark:compare_left_right}), an observation made also in~\cite{saibaba2018randomized}. 

For the rank-$l$ randomized SVD with a Gaussian embedding $\Omegab \in \C^{n \times l}$ and $q$ power iterations, we denote the projected embeddings onto the singular subspaces $\Omegab_1 \triangleq \Vb_{k}^* \Omegab$ and $\Omegab_2 \triangleq \Vb_{r \setminus k}^* \Omegab$, as well as their weighted correspondences $\wt\Omegab_1 \dfeq \Sigmab_k^{2q+1} \Omegab_1$ and $\wt\Omegab_2 \dfeq \Sigmab_{r \setminus k}^{2q+1} \Omegab_2$, such that
\begin{align*}
    \wt\Xb \triangleq \Ub^* \Xb = \Ub^* \Ab \Omegab = \bmat{\Sigmab_k^{2q+1} \Omegab_1 \\ \Sigmab_{r \setminus k}^{2q+1} \Omegab_2} = \bmat{\wt\Omegab_1 \\ \wt\Omegab_2}.
\end{align*}
Then for all $i \in [k]$,
\begin{align*}
    \sin\angle_{k-i+1}\rbr{\Ub_k,\wh\Ub_l} 
    = &\sigma_i\rbr{\rbr{\Ib_m - \wh\Ub_l \wh\Ub_l^*} \Ub_k}
    \\
    \rbr{{\mbox{span}(\Ub_l)=\mbox{span}(\Xb)\  \mbox{and}\ }\Ub \Ub^* \Ub_k = \Ub_k}~
    = &\sigma_i\rbr{\rbr{\Ib_m - \Xb \Xb^{\pinv}} \Ub \Ub^* \Ub_k} 
    \\
    \rbr{\Ub \Ub^* \Xb = \Xb}~
    = &\sigma_i\rbr{\Ub\Ub^* \rbr{\Ib_m - \Xb \Xb^{\pinv}} \Ub \Ub^* \Ub_k} 
    \\
    \rbr{\t{$\Ub$ consists of orthonormal columns}}~
    = &\sigma_i\rbr{\Ub^* \rbr{\Ib_m - \Xb \Xb^{\pinv}} \Ub \Ub^* \Ub_k} 
    \\
    \rbr{\Ub^*\Ub=\Ib_r,~ \Xb^\pinv\Ub=\rbr{\Ub^*\Xb}^\pinv {=\wt\Xb^\pinv}}~
    = &\sigma_i\rbr{\rbr{\Ib_r - \wt\Xb \wt\Xb^{\pinv}} \bmat{\Ib_k \\ \b{0}} }.
\end{align*}
Since $\Xb$ is assumed to {have} full column rank, we have $\wt\Xb \wt\Xb^{\pinv} = \wt\Xb \rbr{\wt\Xb^* \wt\Xb}^{-1} \wt\Xb^*$ {(which is an orthogonal projection)}, and
\begin{align*}
    &\bmat{\Ib_k & \b0} \rbr{\Ib_r - \wt\Xb \wt\Xb^{\pinv}} \bmat{\Ib_k \\ \b0} 
    \\
    = &\bmat{\Ib_k & \b0} \rbr{\Ib_r - \bmat{\wt\Omegab_1 \\ \wt\Omegab_2} \rbr{\wt\Omegab_1^* \wt\Omegab_1 + \wt\Omegab_2^* \wt\Omegab_2}^{-1} \bmat{\wt\Omegab_1^* & \wt\Omegab_2^*}} \bmat{\Ib_k \\ \b0}
    \\
    = &\Ib_k - \wt\Omegab_1 \rbr{\wt\Omegab_1^* \wt\Omegab_1 + \wt\Omegab_2^* \wt\Omegab_2}^{-1} \wt\Omegab_1^*
    \\
    \rbr{\t{Woodbury identity}}~ = &\rbr{\Ib_k + \wt\Omegab_1 \rbr{\wt\Omegab_2^* \wt\Omegab_2} {\wt\Omegab_1^*}}^{-1}.
\end{align*}    
Therefore for all $i \in [k]$,
\begin{align}\label{eq:pf_sin_Uk_whUl}
\begin{split}
    \sin^2\angle_{k-i+1}\rbr{\Ub_k,\wh\Ub_l}  
    = &\sigma_i \rbr{\bmat{\Ib_k & \b0} \rbr{\Ib_r - \wt\Xb \wt\Xb^{\pinv}} \bmat{\Ib_k \\ \b0}} \\
    = &\sigma_i\rbr{\rbr{\Ib_k + \wt\Omegab_1 \rbr{\wt\Omegab_2^* \wt\Omegab_2}^{-1} \wt\Omegab_1^*}^{-1}}.
\end{split}
\end{align}

{By the orthogonal invariance of the distribution of Gaussian embeddings $\Omegab$ together with the orthonormality of $\Vb_k \perp \Vb_{r \setminus k}$, we see that}
$\Omegab_1 \sim P\rbr{\C^{k \times l}}$ and $\Omegab_2 \sim P\rbr{\C^{(r-k) \times l}}$ are independent Gaussian random matrices with the same entry-wise {$\iid$} distribution $\Ncal\rbr{0,l^{-1}}$ as $\Omegab$. 
Therefore by \Cref{lemma:sample_to_population_covariance}, when $l= \Omega(k)$, with high probability (at least $1-e^{-\Theta(k)}$),
\begin{align*}
    \rbr{1-\eps_1} \Sigmab_{k}^{4q+2} 
    \aleq 
    \wt\Omegab_1 \wt\Omegab_1^*
    \aleq
    \rbr{1+\eps_1} \Sigmab_{k}^{4q+2} 
\end{align*}
for some $\epsilon_1 = \Theta\rbr{\sqrt{\frac{k}{l}}}$. 

Analogously, when $r-k \ge \eta = \frac{\tr\rbr{\Sigmab_{r \setminus k}^{4q+2}}^2}{\tr\rbr{\Sigmab_{r \setminus k}^{2(4q+2)}}} = \Omega(l)$, with high probability (at least $1-e^{-\Theta(l)}$),
\begin{align*}
    \frac{1-\eps_2}{l}\tr\rbr{\Sigmab_{r \setminus k}^{4q+2}} \Ib_l 
    \aleq 
    \wt\Omegab_2^* \wt\Omegab_2
    \aleq
    \frac{1+\eps_2}{l}\tr\rbr{\Sigmab_{r \setminus k}^{4q+2}} \Ib_l
\end{align*}
for some $\epsilon_2 = \Theta\rbr{\sqrt{\frac{l}{\eta}}}$.

Therefore by {the} union bound, we have with high probability (at least $1-e^{-\Theta(k)}-e^{-\Theta(l)}$) that,
\begin{align*}
    \rbr{\Ib_k + \wt\Omegab_1 \rbr{\wt\Omegab_2^* \wt\Omegab_2}^{-1} \wt\Omegab_1^*}^{-1}
    \aleq
    \rbr{\Ib_k + \frac{1-\eps_1}{1+\eps_2} \cdot \frac{l}{\tr\rbr{\Sigmab_{r \setminus k}^{4q+2}}} \cdot \Sigmab_k^{4q+2}}^{-1},
\end{align*}    
which leads to \Cref{eq:space_agnostic_left}, while the tightness is implied by 
\begin{align*}
    \rbr{\Ib_k + \wt\Omegab_1 \rbr{\wt\Omegab_2^* \wt\Omegab_2}^{-1} \wt\Omegab_1^*}^{-1}
    \ageq
    \rbr{\Ib_k + \frac{1+\eps_1}{1-\eps_2} \cdot \frac{l}{\tr\rbr{\Sigmab_{r \setminus k}^{4q+2}}} \cdot \Sigmab_k^{4q+2}}^{-1}.
\end{align*}   

The proof of \Cref{eq:space_agnostic_right} follows analogously by replacing the exponents $2q+1$ and $4q+2$ with $2q+2$ and $4q+4$, respectively.
\end{proof} 

\begin{remark}[Comparison with existing probabilistic bounds]\label{remark:prior_bound_compare_exist}
    With access to the unknown right singular subspace $\Vb$, let $\Omegab_1 \triangleq \Vb_{k}^* \Omegab$ and $\Omegab_2 \triangleq \Vb_{r \setminus k}^* \Omegab$. Then, {Saibaba~\cite[Thm.~1]{saibaba2018randomized}}indicates that, for all $i \in [k]$,
    \begin{align}
        \label{eq:saibaba2018_thm1_left}
        &\sin\angle_i\rbr{\Ub_k, \wh\Ub_l} \le
        \rbr{1 + \frac{\sigma_i^{4q+2}}{\sigma_{k+1}^{4q+2} \nbr{\Omegab_2 \Omegab_1^\dagger}_2^2}}^{-\frac{1}{2}},
        \\
        \label{eq:saibaba2018_thm1_right}
        &\sin\angle_i\rbr{\Vb_k, \wh\Vb_l} \le 
        \rbr{1 + \frac{\sigma_i^{4q+4}}{\sigma_{k+1}^{4q+4} \nbr{\Omegab_2 \Omegab_1^\dagger}_2^2}}^{-\frac{1}{2}}.
    \end{align} 
    Further, leveraging existing results on concentration properties of the independent and isotropic Gaussian random matrices $\Omegab_1$ and $\Omegab_2$ (\eg, from the proof of \cite{halko2011finding} Theorem 10.8), \cite{saibaba2018randomized} shows that, when $l \ge k+2$, for any $\delta \in (0,1)$, with probability at least $1-\delta$,
    \begin{align*}
        \nbr{\Omegab_2 \Omegab_1^\dagger}_2 \le 
        \frac{e \sqrt{l}}{l-k+1} \rbr{\frac{2}{\delta}}^{\frac{1}{l-k+1}} \rbr{\sqrt{n-k}+\sqrt{l}+\sqrt{2 \log\frac{2}{\delta}}}.
    \end{align*}    
    
    Without loss of generality, we consider the bounds on $\sin\angle\rbr{\Ub_k,\wh\Ub_l}$. Comparing to the existing bound in \Cref{eq:saibaba2018_thm1_left}, under multiplicative oversampling ($l=\Omega(k)$, $r=\Omega(l)$), \Cref{eq:space_agnostic_left} in \Cref{thm:space_agnostic_bounds} captures the spectral decay on the tail by replacing the denominator term 
    \begin{align*}
        \sigma_{k+1}^{4q+2} \nbr{\Omegab_2 \Omegab_1^\dagger}_2^2
        \quad \text{with} \quad
        \frac{1+\eps_2}{1-\eps_1} \cdot \frac{1}{l}\sum_{j=k+1}^r \sigma_j^{4q+2} = \Theta\rbr{\frac{1}{l}\sum_{j=k+1}^r \sigma_j^{4q+2}}.
    \end{align*}
    We observe that $\frac{1}{l}\sum_{j=k+1}^r \sigma_j^{4q+2} \le \frac{r-k}{l} \sigma_{k+1}^{4q+2}$; while \Cref{lemma:sample_to_population_covariance} implies that, for independent Gaussian random matrices $\Omegab_1 \sim P\rbr{\C^{k \times l}}$ and $\Omegab_2 \sim P\rbr{\C^{(r-k) \times l}}$ with \iid entries from $\Ncal\rbr{0,l^{-1}}$,
    \begin{align*}
        \E\sbr{\nbr{\Omegab_2 \Omegab_1^\dagger}_2^2} = \E_{\Omegab_1}\sbr{\nbr{\rbr{\Omegab_1^\pinv}^* \E_{\Omegab_2}\sbr{\Omegab_2^* \Omegab_2} \Omegab_1^\dagger}_2} = \frac{r-k}{l} \E_{\Omegab_1}\sbr{\nbr{\rbr{\Omegab_1 \Omegab_1^*}^\dagger}_2}.
    \end{align*}
    With non-negligible spectral decay on {the} tail such that $\frac{1}{l}\sum_{j=k+1}^r \sigma_j^{4q+2} \ll \frac{r-k}{l} \sigma_{k+1}^{4q+2}$, when $\frac{1+\eps_2}{1-\eps_1} \cdot \frac{1}{l}\sum_{j=k+1}^r \sigma_j^{4q+2} \ll \sigma_{k+1}^{4q+2} \nbr{\Omegab_2 \Omegab_1^\dagger}_2^2$, \Cref{eq:space_agnostic_left} provides {a} tighter statistical guarantee than \Cref{eq:saibaba2018_thm1_left}, which is also confirmed by numerical observations in \Cref{subsec:experiment_canonical_angle}.

    From the derivation perspective, such improvement is achieved by taking an integrated view on the concentration of $\Sigmab_{r \setminus k}^{2q+1}\Omegab_2$, instead of considering the spectrum and the unknown singular subspace separately.
\end{remark}

\section{Unbiased Space-agnostic Estimates}\label{sec:space_agnostic_estimation}

A natural corollary from the proof of \Cref{thm:separate_gap_bounds} is unbiased estimates for the canonical angles that hold for arbitrary oversampling (\ie, for all $l \ge k$). Further, we will subsequently {see} in \Cref{sec:experiments} that such unbiased estimates also enjoy good empirical concentration.

\begin{proposition}\label{prop:space_agnostic_estimation}
For a rank-$l$ randomized SVD (\Cref{algo:rsvd_power_iterations}) with the Gaussian embedding $\Omegab \sim P\rbr{\C^{n \times l}}$ such that $\Omega_{ij} \sim \Ncal\rbr{0,l^{-1}}~\iid$ and $q \ge 0$ power iterations, for all $i \in [k]$,
\begin{align}\label{eq:space_agnostic_estimation_left}
    \E_{\Omegab} \sbr{\sin\angle_i\rbr{\Ub_k,\wh\Ub_l}}
    = &\E_{\Omegab_1', \Omegab_2'} \sbr{\sigma_i^{-\frac{1}{2}} \rbr{\Ib_k + \Sigmab_k^{2q+1} \Omegab'_1 \rbr{\Omegab_2^{'*} \Sigmab_{r \setminus k}^{4q+2} \Omegab'_2}^{-1} \Omegab_1^{'*} \Sigmab_k^{2q+1}}},
\end{align}
and analogously,
\begin{align}\label{eq:space_agnostic_estimation_right}
    \E_{\Omegab} \sbr{\sin\angle_i\rbr{\Vb_k,\wh\Vb_l}}
    = &\E_{\Omegab_1', \Omegab_2'} \sbr{\sigma_i^{-\frac{1}{2}} \rbr{\Ib_k + \Sigmab_k^{2q+2} \Omegab'_1 \rbr{\Omegab_2^{'*} \Sigmab_{r \setminus k}^{4q+4} \Omegab'_2}^{-1} \Omegab_1^{'*} \Sigmab_k^{2q+2}}},
\end{align}
where $\Omegab'_1 \sim P\rbr{\C^{k \times l}}$ and $\Omegab'_2 \sim P\rbr{\C^{(r-k) \times l}}$ are independent Gaussian random matrices with $\iid$ entries drawn from $\Ncal\rbr{0,l^{-1}}$.
\end{proposition}

To calculate the unbiased estimate, 
{for a modest integer $N$}
we draw a set of independent Gaussian random matrices $\csepp{\Omegab^{(j)}_1 \sim P\rbr{\C^{k \times l}}, \Omegab^{(j)}_2 \sim P\rbr{\C^{(r-k) \times l}}}{j \in [N]}$ and evaluate
\begin{align*}
    &\sin\angle_i\rbr{\Ub_k,\wh\Ub_l} \approx \alpha_i = \frac{1}{N} \sum_{j=1}^N \rbr{1 + \sigma_i^2\rbr{\Sigmab_k^{2q+1} \Omegab^{(j)}_1 \rbr{\Sigmab_{r \setminus k}^{2q+1} \Omegab^{(j)}_2}^\dagger}}^{-\frac{1}{2}},
    \\
    &\sin\angle_i\rbr{\Vb_k,\wh\Vb_l} \approx \beta_i = \frac{1}{N} \sum_{j=1}^N \rbr{1 + \sigma_i^2\rbr{\Sigmab_k^{2q+2} \Omegab^{(j)}_1 \rbr{\Sigmab_{r \setminus k}^{2q+2} \Omegab^{(j)}_2}^\dagger}}^{-\frac{1}{2}},
\end{align*}
for all $i \in [k]$, which can be conducted efficiently in $O\rbr{Nrl^2}$ time. \Cref{algo:unbiased_canonical_angle_estimators} demonstrates the construction of unbiased estimates for $\E \sbr{\sin\angle_i\rbr{\Ub_k,\wh\Ub_l}}=\alpha_i$; while the unbiased estimates for $\E \sbr{\sin\angle_i\rbr{\Vb_k,\wh\Vb_l}}=\beta_i$ can be evaluated analogously by replacing Line 4 with $\wt\Omegab^{(j)}_1 = \Sigmab_k^{2q+2} \Omegab^{(j)}(1:k,:)$, $\wt\Omegab^{(j)}_2 = \Sigmab_{r \setminus k}^{2q+2} \Omegab^{(j)}(k+1:r,:)$. 

\begin{algorithm}
\caption{Unbiased canonical angle estimates}
\label{algo:unbiased_canonical_angle_estimators}
\begin{algorithmic}[1]
\REQUIRE {(Exact or estimated) singular values} $\Sigmab$, rank $k$, sample size $l \ge k$, number of power iterations $q$, number of trials $N$
\ENSURE Unbiased estimates $\E \sbr{\sin\angle_i\rbr{\Ub_k,\wh\Ub_l}}=\alpha_i$ for all $i \in [k]$

\STATE Partition $\Sigmab$ into $\Sigmab_k = \Sigmab(1:k,1:k)$ and $\Sigmab_{r \setminus k} = \Sigmab(k+1:r,k+1:r)$
\FOR{$j=1,\dots,N$}
    \STATE Draw $\Omegab^{(j)} \sim P\rbr{\C^{r \times l}}$ such that $\Omega^{(j)}_{ij} \sim \Ncal\rbr{0,l^{-1}}~\iid$ 
    \STATE $\wt\Omegab^{(j)}_1 = \Sigmab_k^{2q+1} \Omegab^{(j)}(1:k,:)$, $\wt\Omegab^{(j)}_2 = \Sigmab_{r \setminus k}^{2q+1} \Omegab^{(j)}(k+1:r,:)$
    \STATE $\sbr{\Ub_{\wt\Omegab^{(j)}_2}, \Sigmab_{\wt\Omegab^{(j)}_2}, \Vb_{\wt\Omegab^{(j)}_2}} = \t{svd}\rbr{\wt\Omegab^{(j)}_2, \t{``econ''}}$
    \STATE $\nub^{(j)} = \t{svd}\rbr{\wt\Omegab^{(j)}_1 \Vb_{\wt\Omegab^{(j)}_2} \Sigmab_{\wt\Omegab^{(j)}_2}^{-1} \Ub_{\wt\Omegab^{(j)}_2}^*}$
    \STATE $\theta^{(j)}_i = 1/\sqrt{1 + \rbr{\nu^{(j)}_i}^2}$ for all $i \in [k]$
\ENDFOR
\STATE $\alpha_i = \frac{1}{N} \sum_{j=1}^N \theta_i^{(j)}$ for all $i \in [k]$
\end{algorithmic}       
\end{algorithm} 

As demonstrated in \Cref{sec:experiments}, the unbiased estimates concentrate well in practice such that a sample size as small as $N=3$ is {seen to be} sufficient to provide good estimates. Further, with independent Gaussian random matrices, the unbiased estimates in \Cref{prop:space_agnostic_estimation} are also space agnostic, \ie, \Cref{eq:space_agnostic_estimation_left} and \Cref{eq:space_agnostic_estimation_right} only depend on the spectrum $\Sigmab$ but not on the unknown true singular subspaces $\Ub$ and $\Vb$.

\begin{proof}[Proof of \Cref{prop:space_agnostic_estimation}]
To show \Cref{eq:space_agnostic_estimation_left}, we recall from the proof of \Cref{thm:space_agnostic_bounds} that, for the rank-$l$ randomized SVD with a Gaussian embedding $\Omegab \sim P\rbr{\C^{n \times l}}$ and $q$ power iterations, $\Omegab_1 \dfeq \Vb_{k}^* \Omegab$ and $\Omegab_2 \dfeq \Vb_{r \setminus k}^* \Omegab$ are independent Gaussian random matrices with the same entry-wise distribution as $\Omegab$. Therefore, with $\rank\rbr{\Ab}=r$, for all $i \in [k]$,
\begin{align*}
    &\E_{\Omegab} \sbr{\sin\angle_{i}\rbr{\Ub_k,\wh\Ub_l}} 
    \quad \t{(Recall \Cref{eq:pf_sin_Uk_whUl})}
    \\
    = &\E_{\Omegab} \sbr{\sigma_{k-i+1}^{\frac{1}{2}}\rbr{\rbr{\Ib_k + \Sigmab_k^{2q+1} \Omegab_1 \rbr{\Omegab_2^* \Sigmab_{r \setminus k}^{4q+2} \Omegab_2}^{-1} \Omegab_1^* \Sigmab_k^{2q+1}}^{-1}}} 
    \\
    = &\E_{\Omegab} \sbr{\sigma_i^{-\frac{1}{2}}\rbr{\Ib_k + \Sigmab_k^{2q+1} \Omegab_1 \rbr{\Omegab_2^* \Sigmab_{r \setminus k}^{4q+2} \Omegab_2}^{-1} \Omegab_1^* \Sigmab_k^{2q+1}}}
    \\
    = &\underset{\Omegab_1', \Omegab_2'}{\E} \sbr{\sigma_i^{-\frac{1}{2}}\rbr{\Ib_k + \Sigmab_k^{2q+1} \Omegab'_1 \rbr{\Omegab_2^{'*} \Sigmab_{r \setminus k}^{4q+2} \Omegab'_2}^{-1} \Omegab_1^{'*} \Sigmab_k^{2q+1}}}.
\end{align*}
The unbiased estimate in \Cref{eq:space_agnostic_estimation_right} follows analogously.
\end{proof}

As a side note, we point out that, compared with the probabilistic upper bounds \Cref{eq:space_agnostic_left} and \Cref{eq:space_agnostic_right}, the estimates \Cref{eq:space_agnostic_estimation_left} and \Cref{eq:space_agnostic_estimation_right} circumvent overestimation from the operator-convexity of inversion $\sigma \to \sigma^{-1}$,
\begin{align*}
    \E_{\Omegab_2'} \sbr{\rbr{\Omegab_2^{'*} \Sigmab_{r \setminus k}^{4q+2} \Omegab'_2}^{-1}} \ageq 
    \rbr{\E_{\Omegab_2'} \sbr{\Omegab_2^{'*} \Sigmab_{r \setminus k}^{4q+2} \Omegab'_2}}^{-1},
\end{align*}
which implies that
\begin{align*}
    &\E_{\Omegab_1', \Omegab_2'} \sbr{ \rbr{\Ib_k + \Sigmab_k^{2q+1} \Omegab'_1 \rbr{\Omegab_2^{'*} \Sigmab_{r \setminus k}^{4q+2} \Omegab'_2}^{-1} \Omegab_1^{'*} \Sigmab_k^{2q+1}}^{-1} }
    \\
    &\ageq \E_{\Omegab_1'} \sbr{ \rbr{\Ib_k + \Sigmab_k^{2q+1} \Omegab'_1 \rbr{\E_{\Omegab_2'} \sbr{\Omegab_2^{'*} \Sigmab_{r \setminus k}^{4q+2} \Omegab'_2}}^{-1} \Omegab_1^{'*} \Sigmab_k^{2q+1}}^{-1} }.
\end{align*}

\section{Posterior Residual-based Bounds}\label{sec:separate_spec_bounds}

In addition to the prior probabilistic bounds and unbiased estimates, in this section, we introduce two sets of posterior guarantees for the canonical angles that hold deterministically and can be evaluated/approximated efficiently based on the residual of the resulting low-rank approximation $\wh\Ab_l = \wh\Ub_l \wh\Sigmab_l \wh\Vb_l^*$ from \Cref{algo:rsvd_power_iterations}.

\begin{remark}[Generality of residual-based bounds]\label{remark:residula_based_generality}
    It is worth highlighting that both the statements and the proofs of the posterior guarantees \Cref{thm:with_oversmp_computable_det,thm:separate_gap_bounds} to be presented are algorithm-independent. 
    In contrast to \Cref{thm:space_agnostic_bounds} and \Cref{prop:space_agnostic_estimation} whose derivation depends explicitly on the algorithm (\eg, assuming $\Omegab$ being Gaussian in \Cref{algo:rsvd_power_iterations}), the residual-based bounds in \Cref{thm:with_oversmp_computable_det,thm:separate_gap_bounds} hold for general low-rank approximations $\Ab \approx \wh\Ub_l \wh\Sigmab_l\wh\Vb_l^T$.
\end{remark}   

We start with the following proposition that establishes {relations between the canonical angles and} the residuals and the true spectrum $\sigma\rbr{\Ab}$.

\begin{theorem}\label{thm:with_oversmp_computable_det} 
Given any $\wh\Ub_l \in \C^{m \times l}$ and $\wh\Vb_l \in \C^{n \times l}$ with orthonormal columns such that $\range\rbr{\wh\Ub_l} \subseteq \col\rbr{\Ab}$ and $\range\rbr{\wh\Vb_l} \subseteq \row\rbr{\Ab}$, we have for each $i=1,\dots,k$ ($k \le l$),
\begin{align}\label{eq:residual_based_posterior_left}
    &\sin\angle_i\rbr{\Ub_k, \wh\Ub_l} 
    \leq
    \min \cbr{ 
    \frac{\sigma_{k-i+1}\rbr{\rbr{\Ib_m-\wh\Ub_l \wh\Ub_l^*}\Ab}}{\sigma_k},
    \frac{\sigma_1\rbr{\rbr{\Ib_m-\wh\Ub_l \wh\Ub_l^*}\Ab}}{\sigma_i} 
    },
\end{align}
while
\begin{align}\label{eq:residual_based_posterior_right}
    &\sin\angle_i\rbr{\Vb_k, \wh\Vb_l} 
    \leq
    \min \cbr{ 
    \frac{\sigma_{k-i+1}\rbr{\Ab \rbr{\Ib_n-\wh\Vb_l \wh\Vb_l^*}}}{\sigma_k},
    \frac{\sigma_1\rbr{\Ab \rbr{\Ib_n-\wh\Vb_l \wh\Vb_l^*}}}{\sigma_i} 
    }.
\end{align}
\end{theorem}   

\begin{remark}[Left versus right singular subspaces]\label{remark:compare_left_right}
When $\wh\Ub_l$ and $\wh\Vb_l$ consist of approximated left and right singular vectors from \Cref{algo:rsvd_power_iterations}, upper bounds on $\sin\angle_i\rbr{\Vb_k, \wh\Vb_l}$ tend to be smaller than those on $\sin\angle_i\rbr{\Ub_k, \wh\Ub_l}$. This is induced by the algorithmic fact that, in \Cref{algo:rsvd_power_iterations}, $\wh\Vb_l$ is an orthonormal basis of $\Ab^*\Qb_\Xb$, while $\wh\Ub_l$ and $\Qb_\Xb$ are orthonormal bases of $\Xb = \Ab\Omegab$. That is, the evaluation of $\wh\Vb_l$ is enhanced by an additional half power iteration compared with that of $\wh\Ub_l$, which is also reflected by the differences in exponents on $\Sigmab$ (\ie, $2q+1$ versus $2q+2$) in \Cref{thm:space_agnostic_bounds} and \Cref{prop:space_agnostic_estimation}. 
{This difference can be important especially when $q$ is small (\eg, $q=0$). When higher accuracy in the left singular subspace is desirable, one can work with $\Ab^*$.}
\end{remark} 

\begin{proof}[Proof of \Cref{thm:with_oversmp_computable_det}]
Starting with the leading left singular subspace, by definition, for each $i=1,\dots,k$, we have
\begin{align*}
    \sin\angle_i\rbr{\Ub_k, \wh\Ub_l} 
    = &\sigma_{k-i+1}\rbr{\rbr{\Ib_m-\wh\Ub_l \wh\Ub_l^*}\Ub_k}
    \\
    =& \sigma_{k-i+1}\rbr{\rbr{\Ib_m-\wh\Ub_l \wh\Ub_l^*} \Ab \rbr{ \Vb \Sigmab^{-1} \Ub^* \Ub_k } }
    \\
    =& \sigma_{k-i+1}\rbr{ \rbr{\rbr{\Ib_m-\wh\Ub_l \wh\Ub_l^*} \Ab \Vb_k} \Sigmab_k^{-1}}.
\end{align*}
Then, we observe that the following holds simultaneously, 
\begin{align*}
    &\sigma_{k-i+1}\rbr{ \rbr{\rbr{\Ib_m - \wh\Ub_l \wh\Ub_l^*} \Ab \Vb_k} \Sigmab_k^{-1}} \leq \sigma_1\rbr{\rbr{\Ib_m - \wh\Ub_l \wh\Ub_l^*} \Ab \Vb_k} \cdot \sigma_{k-i+1}\rbr{\Sigmab_k^{-1}},
    \\
    &\sigma_{k-i+1}\rbr{ \rbr{\rbr{\Ib_m - \wh\Ub_l \wh\Ub_l^*} \Ab \Vb_k} \Sigmab_k^{-1}} \leq \sigma_{k-i+1}\rbr{\rbr{\Ib_m - \wh\Ub_l \wh\Ub_l^*} \Ab \Vb_k} \cdot \sigma_1\rbr{\Sigmab_k^{-1}},
\end{align*}
where $\sigma_{k-i+1}\rbr{\Sigmab_k^{-1}} = 1/\sigma_i$ and $\sigma_{1}\rbr{\Sigmab_k^{-1}} = 1/\sigma_k$. Finally by \Cref{lemma:Cauchy_interlacing_theorem}, we have
\begin{align*}
    \sigma_{k-i+1}\rbr{\rbr{\Ib_m - \wh\Ub_l \wh\Ub_l^*} \Ab \Vb_k}
    \le 
    \sigma_{k-i+1}\rbr{\rbr{\Ib_m - \wh\Ub_l \wh\Ub_l^*} \Ab}.
\end{align*}

Meanwhile, the upper bound for the leading right singular subspace can be derived analogously by observing that
\begin{align*}
    \sin\angle_i\rbr{\Vb_k, \wh\Vb_l} 
    =&\sigma_{k-i+1}\rbr{ \Vb_k^* \rbr{\Ib_n - \wh\Vb_l\wh\Vb_l^*} } 
    = \sigma_{k-i+1}\rbr{\Sigmab_k^{-1} \Ub_k^* \Ab \rbr{\Ib_n - \wh\Vb_l\wh\Vb_l^*} }.
\end{align*}
\end{proof}

As a potential drawback, although the residuals $\rbr{\Ib_m - \wh\Ub_l \wh\Ub_l^*}\Ab$ and $\Ab\rbr{\Ib_n - \wh\Vb_l \wh\Vb_l^*}$ in \Cref{thm:with_oversmp_computable_det} can be evaluated efficiently in $O(mn)$ and $O(mnl)$ time\footnote{
    For the $O(mn)$ complexity of computing $\rbr{\Ib_m - \wh\Ub_l \wh\Ub_l^*}\Ab$, we assume that $\wh\Ab_l = \wh\Ub_l \wh\Ub_l^* \Ab = \wh\Ub_l \wh\Sigmab_l \wh\Vb_l^*$ is readily available from \Cref{algo:rsvd_power_iterations}. Otherwise (\eg, when \Cref{algo:rsvd_power_iterations} returns $\rbr{\wh\Ub_l, \wh\Sigmab_l, \wh\Vb_l}$ but not $\wh\Ab_l$), the evaluation of $\rbr{\Ib_m - \wh\Ub_l \wh\Ub_l^*}\Ab$ will inevitably take $O(mnl)$ as that of $\Ab\rbr{\Ib_n - \wh\Vb_l \wh\Vb_l^*}$.
}, respectively, the exact evaluation of their full spectra can be unaffordable. A straightforward remedy for this problem is {to use} only the second terms in the right-hand-sides of \Cref{eq:residual_based_posterior_left} and \Cref{eq:residual_based_posterior_right} while estimating $\nbr{\rbr{\Ib_m - \wh\Ub_l \wh\Ub_l^*}\Ab}_2$ and $\nbr{\Ab\rbr{\Ib_n - \wh\Vb_l \wh\Vb_l^*}}_2$ with the randomized power method (\cf \cite{kuczynski1992estimating}, \cite{martinsson2020randomized} Algorithm 4). Alternatively, we leverage the analysis from \cite{nakatsukasa2020sharp} Theorem 6.1 and present the following posterior bounds based only on norms of the residuals which can be estimated efficiently via sampling.

\begin{theorem}\label{thm:separate_gap_bounds}
For any rank-$l$ approximation in the SVD form $\Ab \approx \wh\Ub_l \wh\Sigmab_l \wh\Vb_l^*$ (not necessarily obtained by \Cref{algo:rsvd_power_iterations}), recall the notation that $\wh\Ub_l = \sbr{\wh\Ub_k, \wh\Ub_{l \setminus k}}$, $\wh\Vb_l = \sbr{\wh\Vb_k, \wh\Vb_{l \setminus k}}$, while $\wh\Ub_{m \setminus l}, \wh\Ub_{n \setminus l}$ are the orthogonal complements of $\wh\Ub_l, \wh\Vb_l$, respectively.
Then, with $\Eb_{31} \dfeq \wh\Ub^*_{m \setminus l} \Ab \wh\Vb_k$, 
$\Eb_{32} \dfeq \wh\Ub^*_{m \setminus l} \Ab \wh\Vb_{l \setminus k}$, and 
$\Eb_{33} \dfeq \wh\Ub^*_{m \setminus l} \Ab \wh\Vb_{n \setminus l}$, assuming $\sigma_k > \wh\sigma_{k+1}$ and $\sigma_k > \nbr{\Eb_{33}}_2$, we define the spectral gaps
\begin{align*}
    \gamma_1 \dfeq \frac{\sigma_k^2 - \wh\sigma_{k+1}^2}{\sigma_k},
    \quad
    \gamma_2 \dfeq \frac{\sigma_k^2 - \wh\sigma_{k+1}^2}{\wh\sigma_{k+1}},
    \quad
    \Gamma_1 = \frac{\sigma_k^2-\nbr{\Eb_{33}}_2^2}{\sigma_k},
    \quad
    \Gamma_2 = \frac{\sigma_k^2-\nbr{\Eb_{33}}_2^2}{\nbr{\Eb_{33}}_2}
.\end{align*}   
Then for an arbitrary {unitarily} invariant norm $\nnbr{\cdot}$,
\begin{align}
    \label{eq:separate_gap_Uk_Ul}
    &\nnbr{\sin\angle\rbr{\Ub_k,\wh\Ub_l}} \leq \frac{\nnbr{\sbr{\Eb_{31}, \Eb_{32}}}}{\Gamma_1},
    \\
    \label{eq:separate_gap_Vk_Vl}
    &\nnbr{\sin\angle\rbr{\Vb_k,\wh\Vb_l}} \leq \frac{\nnbr{\sbr{\Eb_{31}, \Eb_{32}}}}{\Gamma_2},
\end{align}
and specifically for the spectral or Frobenius norm $\nbr{\cdot}_\xi$ ($\xi=2,F$),
\begin{align}
    \label{eq:separate_gap_Uk_Uk}
    &\nbr{\sin\angle\rbr{\Ub_k,\wh\Ub_k}}_\xi \leq \frac{\nbr{\sbr{\Eb_{31}, \Eb_{32}}}_\xi}{\Gamma_1} 
    \sqrt{1 + \frac{\nbr{\Eb_{32}}_2^2}{\gamma_2^2}},
    \\
    \label{eq:separate_gap_Vk_Vk}
    &\nbr{\sin\angle\rbr{\Vb_k,\wh\Vb_k}}_\xi \leq \frac{\nbr{\sbr{\Eb_{31}, \Eb_{32}}}_\xi}{\Gamma_1} \sqrt{\frac{\nbr{\Eb_{32}}_2^2}{\gamma_1^2} + \frac{\nbr{\Eb_{33}}_2^2}{\sigma_k^2}}.
\end{align}
Furthermore, for all $i \in [k]$,
\begin{align}
    \label{eq:separate_gap_Uk_Ul_anglewise}
    &\sin\angle_i\rbr{\Ub_k,\wh\Ub_l} \leq \frac{\sigma_k}{\sigma_{k-i+1}} \cdot \frac{\nbr{\sbr{\Eb_{31}, \Eb_{32}}}_2}{\Gamma_1},
    \\
    \label{eq:separate_gap_Vk_Vl_anglewise}
    &\sin\angle_i\rbr{\Vb_k,\wh\Vb_l} \leq \frac{\sigma_k}{\sigma_{k-i+1}} \cdot \frac{\nbr{\sbr{\Eb_{31}, \Eb_{32}}}_2}{\Gamma_2},
    \\
    \label{eq:separate_gap_Uk_Uk_anglewise}
    &\sin\angle_i\rbr{\Ub_k,\wh\Ub_k} \leq \frac{\nbr{\sbr{\Eb_{31}, \Eb_{32}}}_2}{\Gamma_1} \sqrt{1 + \rbr{\frac{\sigma_k}{\sigma_{k-i+1}} \cdot \frac{\nbr{\Eb_{32}}_2}{\gamma_2}}^2 },
    \\
    \label{eq:separate_gap_Vk_Vk_anglewise}
    &\sin\angle_i\rbr{\Vb_k,\wh\Vb_k} \leq \frac{\nbr{\sbr{\Eb_{31}, \Eb_{32}}}_2}{\Gamma_1} \sqrt{\rbr{\frac{\sigma_k}{\sigma_{k-i+1}} \cdot \frac{\nbr{\Eb_{32}}_2}{\gamma_1}}^2 + \rbr{\frac{\nbr{\Eb_{33}}_2}{\sigma_k}}^2 }
.\end{align}
\end{theorem}   

In practice, norms of the residuals can be computed as
\begin{align*}
    &\nnbr{\sbr{\Eb_{31},\Eb_{32}}} = \nnbr{\wh\Ub^*_{m \setminus l} \Ab \wh\Vb_l} = \nnbr{\rbr{\Ib_m - \wh\Ub_l \wh\Ub_l^*} \Ab \wh\Vb_l} = \nnbr{\rbr{\Ab - \wh\Ab_l} \wh\Vb_l},
    \\
    &\nbr{\Eb_{32}}_2 = \nbr{\wh\Ub^*_{m \setminus l} \Ab \wh\Vb_{l \setminus k}}_2 = \nbr{\rbr{\Ab - \wh\Ab_l} \wh\Vb_{l \setminus k}}_2,
    \\
    &\nnbr{\Eb_{33}}_2 = \nbr{\wh\Ub^*_{m \setminus l} \Ab \wh\Vb_{n \setminus l}} = \nbr{\rbr{\Ab - \wh\Ab_l} \rbr{\Ib_n - \wh\Vb_l \wh\Vb_l^*}}_2 = \nbr{\Ab - \Ab \wh\Vb_l \wh\Vb_l^*}_2,
\end{align*}    
where the construction of $\rbr{\Ab - \wh\Ab_l} \wh\Vb_l$, $\rbr{\Ab - \wh\Ab_l} \wh\Vb_{l \setminus k}$, and $\Ab - \Ab \wh\Vb_l \wh\Vb_l^*$ takes $O\rbr{mnl}$ time, while the respective norms can be estimated efficiently via sampling (\cf \cite{martinsson2020randomized} Algorithm 1-4, \cite{meyer2020hutch} Algorithm 1-3, etc.). 

The proof of \Cref{thm:separate_gap_bounds} is 
{similar to that of \cite[Thm.~6.1]{nakatsukasa2020sharp}}.  

\begin{proof}[Proof of \Cref{thm:separate_gap_bounds}]
Let $\wt\Ub_{11} \dfeq \wh\Ub^*_{k} \Ub_k$, $\wt\Ub_{21} \dfeq \wh\Ub^*_{l \setminus k} \Ub_k$, $\wt\Ub_{31} \dfeq \wh\Ub^*_{m \setminus l} \Ub_k$, and $\wt\Vb_{11} \dfeq \wh\Vb^*_{k} \Vb_k$, $\wt\Vb_{21} \dfeq \wh\Vb^*_{l \setminus k} \Vb_k$, and $\wt\Vb_{31} \dfeq \wh\Vb^*_{n \setminus l} \Vb_k$.
We start by expressing the canonical angles in terms of $\wt\Ub_{31}$ and $\wt\Ub_{21}$:
\begin{align*}
    \sin\angle\rbr{\Ub_k, \wh\Ub_l} = 
    \sigma\rbr{\wh\Ub^*_{n \setminus l} \Ub_k} = 
    \sigma\rbr{\wt\Ub_{31}},
\end{align*}
\begin{align*}
    \sin\angle\rbr{\Vb_k, \wh\Vb_l} = 
    \sigma\rbr{\wh\Vb^*_{n \setminus l} \Vb_k} = 
    \sigma\rbr{\wt\Vb_{31}},
\end{align*}
\begin{align*}
    \sin\angle\rbr{\Ub_k, \wh\Ub_k} = 
    \sigma\rbr{\bmat{\wh\Ub^*_{l \setminus k} \\ \wh\Ub^*_{m \setminus l}} \Ub_k} = 
    \sigma\rbr{\bmat{\wt\Ub_{21} \\ \wt\Ub_{31}}},
\end{align*}
\begin{align*}
    \sin\angle\rbr{\Vb_k, \wh\Vb_k} = 
    \sigma\rbr{\bmat{\wh\Vb^*_{l \setminus k} \\ \wh\Vb^*_{n \setminus l}} \Vb_k} = 
    \sigma\rbr{\bmat{\wt\Vb_{21} \\ \wt\Vb_{31}}}.
\end{align*}
By observing that for any rank-$l$ approximation in the SVD form $\Ab \approx \wh\Ub_l \wh\Sigmab_l \wh\Vb_l^*$,
\begin{align*}
    \Ab = 
    \wh\Ub_l \wh\Sigmab_l \wh\Vb_l^* + \wh\Ub_{m \setminus l} \wh\Ub_{m \setminus l}^* \Ab = 
    \bmat{\wh\Ub_k & \wh\Ub_{l \setminus k} & \wh\Ub_{m \setminus l}}
    \bmat{\wh\Sigmab_k & \b0 & \b0 \\
    \b0 & \wh\Sigmab_{l \setminus k} & \b0 \\
    \Eb_{31} & \Eb_{32} & \Eb_{33}}
    \bmat{\wh\Vb_k^* \\ \wh\Vb_{l \setminus k}^* \\ \wh\Vb_{n \setminus l}^*}
,\end{align*}  
we left multiply $\wh\Ub^* = \sbr{\wh\Ub^*_k; \wh\Ub^*_{l \setminus k}; \wh\Ub^*_{m \setminus l}} \in \C^{m \times m}$ and right multiply $\Vb_k$ on both sides and get
\begin{align}\label{eq:proof_mixed_gap_bound_submatrix_relation_31}
    \bmat{\wt\Ub_{11} \\ \wt\Ub_{21} \\ \wt\Ub_{31}} \Sigmab_k = 
    \bmat{\wh\Sigmab_k & \b0 & \b0 \\
    \b0 & \wh\Sigmab_{l \setminus k} & \b0 \\
    \Eb_{31} & \Eb_{32} & \Eb_{33}}
    \bmat{\wt\Vb_{11} \\ \wt\Vb_{21} \\ \wt\Vb_{31}}
,\end{align}    
while left multiplying $\Ub_k^*$ and right multiplying $\wh\Vb = \sbr{\wh\Vb_k, \wh\Vb_{l \setminus k}, \wh\Vb_{n \setminus l}}$ yield
\begin{align}\label{eq:proof_mixed_gap_bound_submatrix_relation_13}
    \Sigmab_k \bmat{\wt\Vb_{11}^* & \wt\Vb_{21}^* & \wt\Vb_{31}^*} =
    \bmat{\wt\Ub_{11}^* & \wt\Ub_{21}^* & \wt\Ub_{31}^*}
    \bmat{\wh\Sigmab_k & \b0 & \b0 \\
    \b0 & \wh\Sigmab_{l \setminus k} & \b0 \\
    \Eb_{31} & \Eb_{32} & \Eb_{33}}
.\end{align}

\paragraph{Bounding $\sigma\rbr{\wt\Ub_{31}}$ and $\sigma\rbr{\wt\Vb_{31}}$.}
To bound $\sigma\rbr{\wt\Ub_{31}}$, we observe the following from the third row of \Cref{eq:proof_mixed_gap_bound_submatrix_relation_31} and the third column of \Cref{eq:proof_mixed_gap_bound_submatrix_relation_13}, 
\begin{align*}
    \wt\Ub_{31} \Sigmab_k = \Eb_{31} \wt\Vb_{11} + \Eb_{32}\wt\Vb_{21} + \Eb_{33}\wt\Vb_{31},
    \quad
    \wt\Ub_{31}^* \Eb_{33} = \Sigmab_k \wt\Vb_{31}^*.
\end{align*}
Noticing that $\sbr{\wt\Vb_{11}; \wt\Vb_{21}} = \wh\Vb_l^* \Vb_k$ and $\nbr{\wh\Vb_l^* \Vb_k}_2 \le 1$, we have
\begin{align*}
    \nnbr{\wt\Ub_{31} \Sigmab_k} 
    \le &\nnbr{\sbr{\Eb_{31},\Eb_{32}}} \nbr{\wh\Vb_l^* \Vb_k}_2 + \nnbr{\Eb_{33} \Eb_{33}^* \wt\Ub_{31} \Sigmab_k^{-1}}
    \\
    \le &\nnbr{\sbr{\Eb_{31},\Eb_{32}}} + \frac{\nbr{\Eb_{33}}_2^2}{\sigma_k^2} \nnbr{\wt\Ub_{31} \Sigmab_k}
\end{align*} 
for all $i \in [\min\rbr{k,m-l}]$, which implies that
\begin{align*}
    \nnbr{\wt\Ub_{31} \Sigmab_k} 
    \le \rbr{1-\frac{\nbr{\Eb_{33}}_2^2}{\sigma_k^2}}^{-1} \nnbr{\sbr{\Eb_{31},\Eb_{32}}}
    = \sigma_k \cdot \frac{\nnbr{\sbr{\Eb_{31},\Eb_{32}}}}{\Gamma_1},
\end{align*}
and leads to 
\begin{align*}
    &\nnbr{\wt\Ub_{31}} 
    \le \frac{1}{\sigma_k} \nnbr{\wt\Ub_{31} \Sigmab_k} 
    \le \frac{\nnbr{\sbr{\Eb_{31},\Eb_{32}}}}{\Gamma_1},
    \\
    &\sigma_i\rbr{\wt\Ub_{31}}
    \le \frac{1}{\sigma_{k-i+1}} \nbr{\wt\Ub_{31} \Sigmab_k}_2 
    \le \frac{\sigma_k}{\sigma_{k-i+1}} \cdot \frac{\nbr{\sbr{\Eb_{31}, \Eb_{32}}}_2}{\Gamma_1}
    \quad \forall~i \in [k],
\end{align*}
where the second line follows from \Cref{lemma:individual_sval_of_product}. These lead to \Cref{eq:separate_gap_Uk_Ul,eq:separate_gap_Uk_Ul_anglewise}

To bound $\sigma\rbr{\wt\Vb_{31}}$, we use the relation $\wt\Ub_{31}^* \Eb_{33} = \Sigmab_k \wt\Vb_{31}^*$,
\begin{align*}
    &\nnbr{\wt\Vb_{31}} \le 
    \frac{\nbr{\Eb_{33}}_2}{\sigma_k} \nnbr{\wt\Ub_{31}} \le 
    \frac{\nnbr{\sbr{\Eb_{31}, \Eb_{32}}}}{\Gamma_2},
    \\
    &\sigma_i\rbr{\wt\Vb_{31}} 
    \le \frac{1}{\sigma_k} \sigma_i\rbr{\Eb_{33}^* \wt\Ub_{31}}
    \le \frac{\nbr{\Eb_{33}}_2}{\sigma_k} \sigma_i\rbr{\wt\Ub_{31}}
    \le \frac{\sigma_k}{\sigma_{k-i+1}} \cdot \frac{\nbr{\sbr{\Eb_{31}, \Eb_{32}}}_2}{\Gamma_2}
    \quad \forall~i \in[k].
\end{align*}    
We therefore have upper bounds \Cref{eq:separate_gap_Vk_Vl,eq:separate_gap_Vk_Vl_anglewise}.

\paragraph{Bounding $\sigma\rbr{\wt\Ub_{21}}$ and $\sigma\rbr{\wt\Vb_{21}}$.}
To bound $\sigma\rbr{\wt\Ub_{21}}$, we leverage the second row of \Cref{eq:proof_mixed_gap_bound_submatrix_relation_31} and the second column of \Cref{eq:proof_mixed_gap_bound_submatrix_relation_13},
\begin{align*}
    \wt\Ub_{21} \Sigmab_k = \wh\Sigmab_{l \setminus k} \wt\Vb_{21}, 
    \quad
    \Sigmab_k \wt\Vb_{21}^* = \wt\Ub_{21}^* \wh\Sigmab_{l \setminus k} + \wt\Ub_{31}^* \Eb_{32}.
\end{align*}    
Up to rearrangement, we observe that
\begin{align*}
    \nnbr{\wt\Ub_{21} \Sigmab_k} =
    & \nnbr{\wh\Sigmab_{l \setminus k} \rbr{\wh\Sigmab_{l \setminus k} \wt\Ub_{21} + \Eb_{32}^* \wt\Ub_{31}} \Sigmab_k^{-1}}
    \\
    \le &\frac{\nbr{\wh\Sigmab_{l \setminus k}}_2^2}{\sigma_k^2} \nnbr{\wt\Ub_{21} \Sigmab_k} + \frac{\nbr{\wh\Sigmab_{l \setminus k}}_2}{\sigma_k} \nnbr{\Eb_{32}^* \wt\Ub_{31}},
\end{align*}
which implies that
\begin{align*}
    &\nnbr{\wt\Ub_{21} \Sigmab_k} 
    \le \rbr{1-\frac{\wh\sigma_{k+1}^2}{\sigma_k^2}}^{-1} \frac{\wh\sigma_{k+1}}{\sigma_k} \nnbr{\Eb_{32}^* \wt\Ub_{31}}
    \le \sigma_k \cdot \frac{\nbr{\Eb_{32}}_2}{\gamma_2} \nnbr{\wt\Ub_{31}},
\end{align*}
and therefore, with \Cref{lemma:individual_sval_of_product}, for all $i \in [k]$,
\begin{align*}
    &\nnbr{\wt\Ub_{21}}
    \le \frac{1}{\sigma_k} \nnbr{\wt\Ub_{21} \Sigmab_k} 
    \le \frac{\nbr{\Eb_{32}}_2}{\gamma_2} \nnbr{\wt\Ub_{31}},
    \\
    &\sigma_i\rbr{\wt\Ub_{21}} 
    \le \frac{1}{\sigma_{k-i+1}} \nbr{\wt\Ub_{21} \Sigmab_k}_2
    \le \frac{\sigma_k}{\sigma_{k-i+1}} \cdot \frac{\nbr{\Eb_{32}}_2}{\gamma_2} \nbr{\wt\Ub_{31}}_2.
\end{align*}
Then, with the stronger inequality for the spectral or Frobenius norm $\nbr{\cdot}_{\xi}$ ($\xi=2,F$), 
\begin{align*}
    \nbr{\bmat{\wt\Ub_{21} \\ \wt\Ub_{31}}}_\xi 
    \le &\sqrt{\nbr{\wt\Ub_{31}}_\xi^2 + \nbr{\wt\Ub_{21}}_\xi^2} 
    \le \nbr{\wt\Ub_{31}}_\xi \sqrt{1 + \frac{\nbr{\Eb_{32}}_2^2}{\gamma_2^2}}
    \\
    \le &\frac{\nbr{\sbr{\Eb_{31}, \Eb_{32}}}_\xi}{\Gamma_1} 
    \sqrt{1 + \frac{\nbr{\Eb_{32}}_2^2}{\gamma_2^2}}
\end{align*} 
leads to \Cref{eq:separate_gap_Uk_Uk}.
Meanwhile for \Cref{eq:separate_gap_Uk_Uk_anglewise}, the individual canonical angles are upper bounded by
\begin{align*}
    \sigma_i\rbr{\bmat{\wt\Ub_{21} \\ \wt\Ub_{31}}} 
    = &\sqrt{\sigma_i\rbr{\wt\Ub_{21}^* \wt\Ub_{21} + \wt\Ub_{31}^* \wt\Ub_{31}}}
    \\
    \rbr{\text{\Cref{lemma:individual_sval_of_sum}}} \quad
    \le &\sqrt{\nbr{\wt\Ub_{31}}_2^2 + \sigma_i\rbr{\wt\Ub_{21}}^2}
    \\
    \le &\nbr{\wt\Ub_{31}}_2 \sqrt{1 + \rbr{\frac{\sigma_k}{\sigma_{k-i+1}} \cdot \frac{\nbr{\Eb_{32}}_2}{\gamma_2}}^2 }
    \\
    \le &\frac{\nbr{\sbr{\Eb_{31}, \Eb_{32}}}_2}{\Gamma_1} \sqrt{1 + \rbr{\frac{\sigma_k}{\sigma_{k-i+1}} \cdot \frac{\nbr{\Eb_{32}}_2}{\gamma_2}}^2 }.
\end{align*}
Analogously, by observing that 
\begin{align*}
    \nnbr{\wt\Vb_{21} \Sigmab_k} =
    & \nnbr{\wh\Sigmab_{l \setminus k}^2 \wt\Vb_{21} \Sigmab_k^{-1} + \Eb_{32}^* \wt\Ub_{31}}
    \le \frac{\nbr{\wh\Sigmab_{l \setminus k}}_2^2}{\sigma_k^2} \nnbr{\wt\Vb_{21} \Sigmab_k} + \nnbr{\Eb_{32}^* \wt\Ub_{31}},
\end{align*}
we have that, by \Cref{lemma:individual_sval_of_product}, for all $i \in [k]$,
\begin{align*}
    &\nnbr{\wt\Vb_{21} \Sigmab_k} 
    \le \rbr{1-\frac{\wh\sigma_{k+1}^2}{\sigma_k^2}}^{-1} \nnbr{\Eb_{32}^* \wt\Ub_{31}}
    \le \sigma_k \cdot \frac{\nbr{\Eb_{32}}_2}{\gamma_1} \nnbr{\wt\Ub_{31}},
    \\
    &\nnbr{\wt\Vb_{21}} 
    \le \frac{1}{\sigma_k} \nnbr{\wt\Vb_{21} \Sigmab_k} 
    \le \frac{\nbr{\Eb_{32}}_2}{\gamma_1} \nnbr{\wt\Ub_{31}},
    \\
    &\sigma_i\rbr{\wt\Vb_{21}}
    \le \frac{1}{\sigma_{k-i+1}} \nbr{\wt\Vb_{21} \Sigmab_k}_2
    \le \frac{\sigma_k}{\sigma_{k-i+1}} \cdot \frac{\nbr{\Eb_{32}}_2}{\gamma_1} \nbr{\wt\Ub_{31}}_2,
\end{align*}
and therefore for the spectral or Frobenius norm $\nbr{\cdot}_{\xi}$ ($\xi=2,F$),
\begin{align*}
    \nbr{\bmat{\wt\Vb_{21} \\ \wt\Vb_{31}}}_\xi
    \le &\sqrt{\nbr{\wt\Vb_{21}}_\xi^2 + \nbr{\wt\Vb_{31}}_\xi^2}  
    \le \nbr{\wt\Ub_{31}}_\xi \sqrt{\frac{\nbr{\Eb_{32}}_2^2}{\gamma_1^2} + \frac{\nbr{\Eb_{33}}_2^2}{\sigma_k^2}}
    \\
    \le &\frac{\nbr{\sbr{\Eb_{31}, \Eb_{32}}}_\xi}{\Gamma_1} \sqrt{\frac{\nbr{\Eb_{32}}_2^2}{\gamma_1^2} + \frac{\nbr{\Eb_{33}}_2^2}{\sigma_k^2}},
\end{align*}  
which leads to \Cref{eq:separate_gap_Vk_Vk}.
Additionally for individual canonical angles $i \in [k]$,
\begin{align*}
    \sigma_i\rbr{\bmat{\wt\Vb_{21} \\ \wt\Vb_{31}}} 
    \le &\sqrt{\sigma_i\rbr{\wt\Vb_{21}}^2 + \nbr{\wt\Vb_{31}}_2^2}
    \quad \rbr{\text{\Cref{lemma:individual_sval_of_sum}}}
    \\
    \le &\nbr{\wt\Ub_{31}}_2 \sqrt{\rbr{\frac{\sigma_k}{\sigma_{k-i+1}} \cdot \frac{\nbr{\Eb_{32}}_2}{\gamma_1}}^2 + \rbr{\frac{\nbr{\Eb_{33}}_2}{\sigma_k}}^2 }
    \\
    \le &\frac{\nbr{\sbr{\Eb_{31}, \Eb_{32}}}_2}{\Gamma_1} \sqrt{\rbr{\frac{\sigma_k}{\sigma_{k-i+1}} \cdot \frac{\nbr{\Eb_{32}}_2}{\gamma_1}}^2 + \rbr{\frac{\nbr{\Eb_{33}}_2}{\sigma_k}}^2 }.
\end{align*}
{This yields \Cref{eq:separate_gap_Vk_Vk_anglewise} and completes the proof.}
\end{proof}

\section{Numerical Experiments}\label{sec:experiments}

First, we present numerical comparisons among different canonical angle upper bounds and the unbiased estimates on the left and right leading singular subspaces of various synthetic and real data matrices. We start by describing the target matrices in \Cref{subsec:target_matrix}.
In \Cref{subsec:experiment_canonical_angle}, we discuss the performance of the unbiased estimates, as well as the relative tightness of the canonical angle bounds, for different algorithmic choices based on the numerical observations.
Second, in \Cref{subsec:balance_oversampling_power_iterations_example}, we present an illustrative example that {provides} insight into the balance between oversampling and power iterations brought by the space-agnostic bounds.

\subsection{Target Matrices}\label{subsec:target_matrix}
We consider several different classes of target matrices, including some synthetic random matrices with 
different spectral patterns, as well as an empirical dataset, as summarized below:
\begin{enumerate}
    \item A random sparse non-negative (SNN) matrix~\cite{sorensen2016deim} $\Ab$ of size $m \times n$ takes the form,
    \begin{equation}
        \label{eq:snn-def}
        \Ab = \text{SNN}\rbr{a,r_1}:= \sum_{i=1}^{r_1} \frac{a}{i} \xb_i \yb_i^T + \sum_{i=r_1+1}^{\min\rbr{m,n}} \frac{1}{i} \xb_i \yb_i^T
    \end{equation}
    where $a>1$ and $r_1 < \min\rbr{m,n}$ control the spectral decay, and $\xb_i \in \C^m$, $\yb_i \in \C^{n}$ are random sparse vectors with non-negative entries.
    In the experiments, we test on two random SNN matrices of size $500 \times 500$ with $r_1 = 20$ and $a=1,100$, respectively.
    
    \item Gaussian dense matrices with controlled spectral decay are randomly generated via {a} similar construction {to} the SNN matrix, with $\xb_j \in \SSS^{m-1}$ and $\yb_j \in \SSS^{n-1}$ in (\ref{eq:snn-def}) replaced by uniformly random dense orthonormal vectors. The generating procedures for $\Ab \in \C^{m \times n}$ with rank $r \leq \min\rbr{m,n}$ can be summarized as following:
    \begin{enumerate}[label=(\roman*)]
        \item Draw Gaussian random matrices, $\Gb_m \in \C^{m \times r}$ and $\Gb_n \in \C^{m \times r}$.
        \item Compute $\Ub = \text{ortho}(\Gb_m) \in \C^{m \times r}$, $\Vb = \text{ortho}(\Gb_n) \in \C^{n \times r}$ as orthonormal bases. 
        \item Given the spectrum $\Sigmab = \diag\rbr{\sigma_1,\dots,\sigma_r}$, we construct $\Ab = \Ub \Sigmab \Vb^*$.
    \end{enumerate}
    In the experiments, we consider two types of spectral decay: 
    \begin{enumerate}[label=(\roman*)]
        \item slower decay with $r_1=20$, $\sigma_i = 1$ for all $i \in [r_1]$, $\sigma_i=1/\sqrt{i-r_1+1}$ for all $i=r_1+1,\dots,r$, and
        \item faster decay with $r_1=20$, $\sigma_i = 1$ for all $i \in [r_1]$, $\sigma_i=\max(0.99^{i-r_1}, 10^{-3})$ for all $i=r_1+1,\dots,r$.
    \end{enumerate}
    
    \item MNIST training set consists of $60,000$ images of hand-written digits from $0$ to $9$. Each image is of size $28 \times 28$. We form the target matrices by uniformly sampling $N = 800$ images from the MNIST training set. The images are flattened and normalized to form a full-rank matrix of size $N \times d$ where $d = 784$ is the size of the flattened images, with entries bounded in $[0,1]$. The nonzero entries take approximately $20\%$ of the matrix for both the training and the testing sets. 
\end{enumerate}

\subsection{Canonical Angle Bounds and Estimates}\label{subsec:experiment_canonical_angle}

Now we present numerical comparisons of the performance of the canonical angle bounds and the unbiased estimates under different algorithmic choices.
Considering the scenario where the true matrix spectra may not be available in practice, we calculate two sets of upper bounds, one from the true spectra $\Sigmab \in \C^{r \times r}$ and the other from the $l$ approximated singular values from \Cref{algo:rsvd_power_iterations}. 
For the later, we pad the approximated spectrum $\wh\Sigmab_l=\diag\rbr{\wh\sigma_1,\dots,\wh\sigma_l}$ with {$r-l$ copies of} $\wh\sigma_l$ and evaluate the canonical angle bounds and estimates with $\wt\Sigmab=\diag\rbr{\wh\sigma_1,\dots,\wh\sigma_l,\dots,\wh\sigma_l} \in \C^{r \times r}$.

\begin{figure}[!ht]
    \centering
    \includegraphics[width=\linewidth]{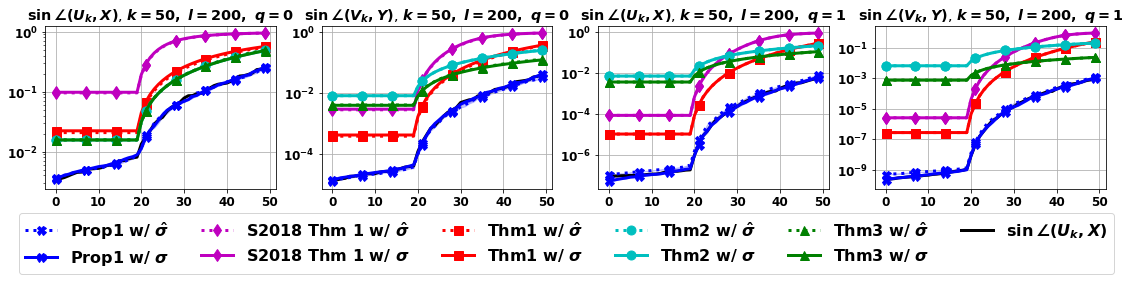}
    \caption{Synthetic Gaussian with the slower spectral decay. $k=50$, $l=200$, $q=0,1$.}
    \label{fig:Gaussian-poly1-kl_k50_l200}
\end{figure}

\begin{figure}[!ht]
    \centering
    \includegraphics[width=\linewidth]{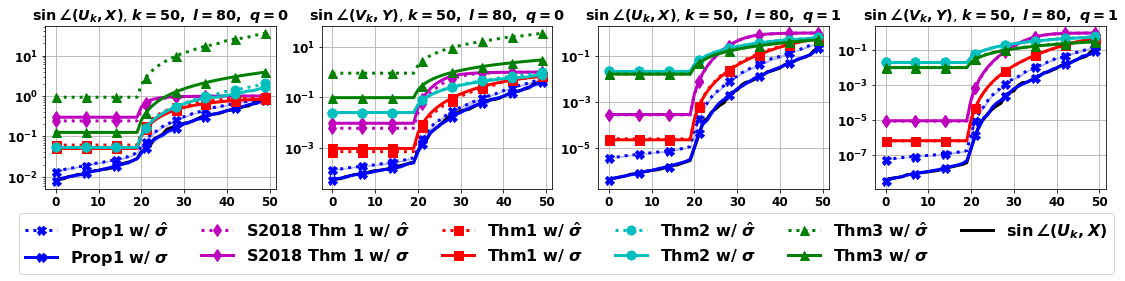}
    \caption{Synthetic Gaussian with the slower spectral decay. $k=50$, $l=80$, $q=0,1$.}
    \label{fig:Gaussian-poly1-kl_k50_l80}
\end{figure}

\begin{figure}[!ht]
    \centering
    \includegraphics[width=\linewidth]{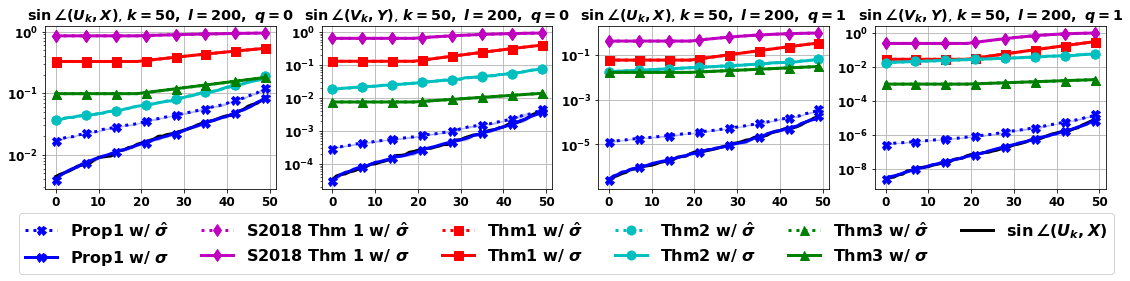}
    \caption{Synthetic Gaussian with the faster spectral decay. $k=50$, $l=200$, $q=0,1$.}
    \label{fig:Gaussian-exp-kl_k50_l200}
\end{figure}

\begin{figure}[!ht]
    \centering
    \includegraphics[width=\linewidth]{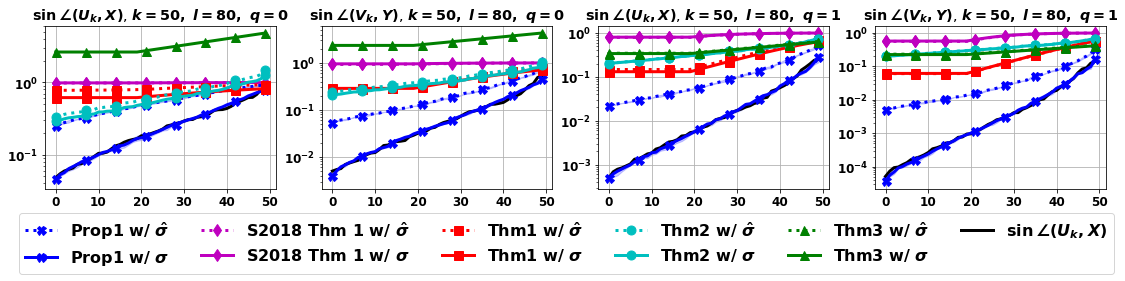}
    \caption{Synthetic Gaussian with the faster spectral decay. $k=50$, $l=80$, $q=0,1$.}
    \label{fig:Gaussian-exp-kl_k50_l80}
\end{figure}

\begin{figure}[!ht]
    \centering
    \includegraphics[width=\linewidth]{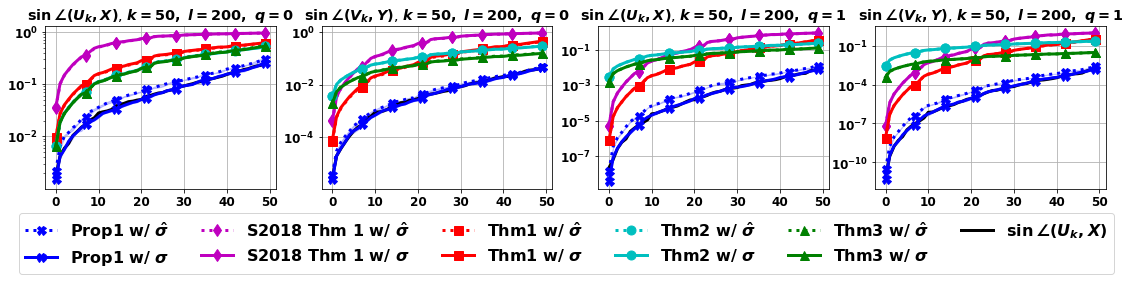}
    \caption{SNN with $r_1=20$, $a=1$. $k=50$, $l=200$, $q=0,1$.}
    \label{fig:SNN-m500n500r20a1-kl_k50_l200}
\end{figure}

\begin{figure}[!ht]
    \centering
    \includegraphics[width=\linewidth]{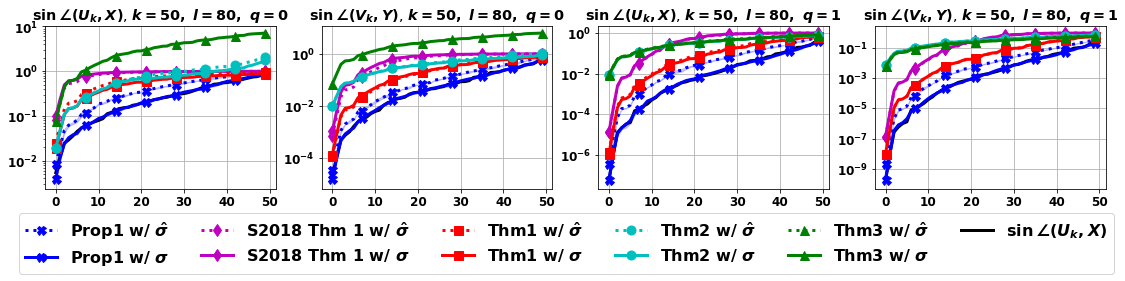}
    \caption{SNN with $r_1=20$, $a=1$. $k=50$, $l=80$, $q=0,1$.}
    \label{fig:SNN-m500n500r20a1-kl_k50_l80}
\end{figure}

\begin{figure}[!ht]
    \centering
    \includegraphics[width=\linewidth]{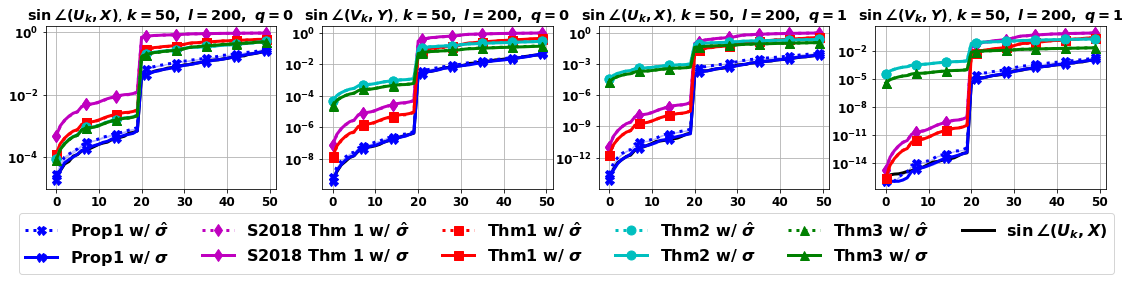}
    \caption{SNN with $r_1=20$, $a=100$. $k=50$, $l=200$, $q=0,1$.}
    \label{fig:SNN-m500n500r20a100-kl_k50_l200}
\end{figure}

\begin{figure}[!ht]
    \centering
    \includegraphics[width=\linewidth]{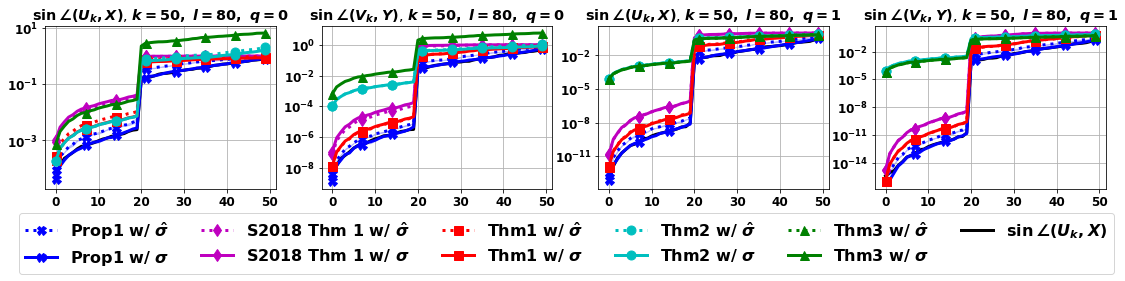}
    \caption{SNN with $r_1=20$, $a=100$. $k=50$, $l=80$, $q=0,1$.}
    \label{fig:SNN-m500n500r20a100-kl_k50_l80}
\end{figure}

\begin{figure}[!ht]
    \centering
    \includegraphics[width=\linewidth]{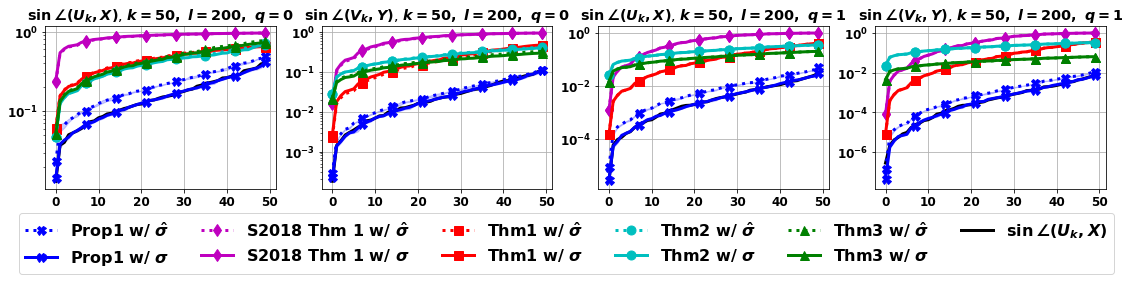}
    \caption{$800$ randomly sampled images from the MNIST training set. $k=50$, $l=200$, $q=0,1$.}
    \label{fig:mnist-train800-kl_k50_l200}
\end{figure}

\begin{figure}[!ht]
    \centering
    \includegraphics[width=\linewidth]{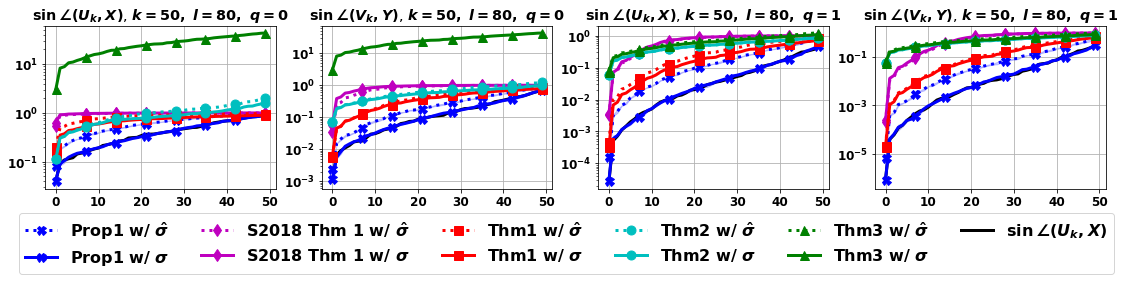}
    \caption{$800$ randomly sampled images from the MNIST training set. $k=50$, $l=80$, $q=0,1$.}
    \label{fig:mnist-train800-kl_k50_l80}
\end{figure}

\begin{figure}[!ht]
    \centering
    \includegraphics[width=\linewidth]{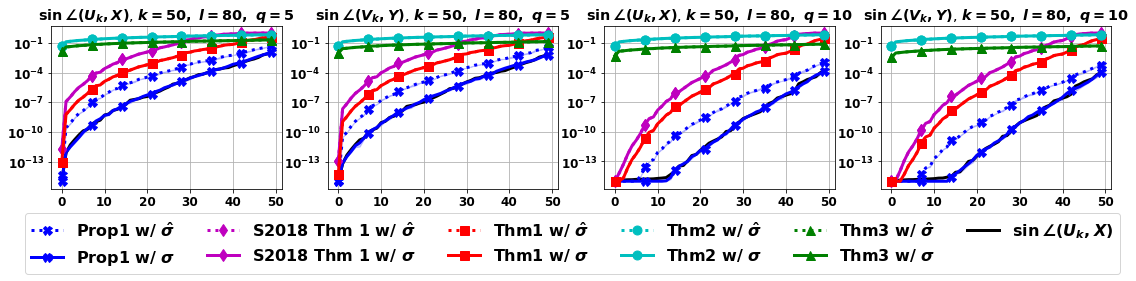}
    \caption{$800$ randomly sampled images from the MNIST training set. $k=50$, $l=80$, $q=5,10$.}
    \label{fig:mnist-train800-kl_k50_l80_p5-10}
\end{figure}

From \Cref{fig:Gaussian-poly1-kl_k50_l200} to \Cref{fig:mnist-train800-kl_k50_l80_p5-10},
\begin{enumerate}
    \item {
        Red lines and dashes (\b{Thm1 w/ $\sigmab$ and $\wh\sigmab$})
    } represent the space-agnostic probabilistic bounds in \Cref{thm:space_agnostic_bounds} evaluated with the true (lines) and approximated (dashes) singular values, $\Sigmab$, and $\wt\Sigmab$, respectively, where we simply ignore tail decay and suppress constants for the distortion factors and set $\eps_1 = \sqrt{\frac{k}{l}}$ and $\eps_2=\sqrt{\frac{l}{r-k}}$ in {\Cref{eq:space_agnostic_left,eq:space_agnostic_right}};
    
    \item {
        Blue lines and dashes (\b{Prop1 w/ $\sigmab$ and $\wh\sigmab$})
    } represent the unbiased space-agnostic estimates in \Cref{prop:space_agnostic_estimation} (averages of $N=3$ independent trials with 
    {blue shades} marking the corresponding minima and maxima in the trials) evaluated with the true (lines) and approximated (dashes) singular values, $\Sigmab$ and $\wt\Sigmab$, respectively;
    
    \item {
        Cyan lines and dashes (\b{Thm2 w/ $\sigmab$ and $\wh\sigmab$})
    } represent the posterior residual-based bounds in \Cref{thm:with_oversmp_computable_det} evaluated with the true (lines) and approximated (dashes) singular values, $\Sigmab$, and $\wt\Sigmab$, respectively;
    
    \item {
        Green lines and dashes (\b{Thm3 w/ $\sigmab$ and $\wh\sigmab$})
    } represent the posterior residual-based bounds \Cref{eq:separate_gap_Uk_Ul} and \Cref{eq:separate_gap_Vk_Vl} in \Cref{thm:separate_gap_bounds} evaluated with the true (lines) and approximated (dashes) singular values, $\Sigmab$, and $\wt\Sigmab$, respectively;
    
    \item {
        Magenta lines and dashes (\b{S2018 Thm1 w/ $\sigmab$ and $\wh\sigmab$})
    } represent the upper bounds in \cite{saibaba2018randomized} Theorem 1 (\ie, \Cref{eq:saibaba2018_thm1_left} and \Cref{eq:saibaba2018_thm1_right}) evaluated with the true (lines) and approximated (dashes) singular values, $\Sigmab$ and $\wt\Sigmab$, respectively, and the unknown true singular subspace such that $\Omegab_1 = \Vb_k^*\Omegab$ and $\Omegab_2 = \Vb_{r \setminus k}^*\Omegab$;
    
    \item {
        Black lines
    } mark the true canonical angles $\sin\angle\rbr{\Ub_k, \wh\Ub_l}$.
\end{enumerate}

We recall from \Cref{remark:compare_left_right} that, by the algorithmic construction of \Cref{algo:rsvd_power_iterations}, for given $q$, canonical angles of the right singular spaces $\sin\angle\rbr{\Vb_k, \wh\Vb_l}$ are evaluated with half more power iterations than those of the left singular spaces $\sin\angle\rbr{\Ub_k, \wh\Ub_l}$. 
That is, $\sin\angle\rbr{\Ub_k, \wh\Ub_l}$, $\sin\angle\rbr{\Vb_k, \wh\Vb_l}$ with $q=0,1$ in \Cref{fig:Gaussian-poly1-kl_k50_l200}-\Cref{fig:mnist-train800-kl_k50_l80} can be viewed as canonical angles of randomized subspace approximation with $q=0,0.5,1,1.5$ power iterations, respectively; while \Cref{fig:mnist-train800-kl_k50_l80_p5-10} corresponds to randomized subspace approximations constructed with $q=5,5.5,10,10.5$ power iterations analogously.

For each set of upper bounds/unbiased estimates, we observe the following.
\begin{enumerate}
    \item The {
        space-agnostic probabilistic bounds (\b{Thm1 w/ $\sigmab$ and $\wh\sigmab$})
    } in \Cref{thm:space_agnostic_bounds} provide tighter statistical guarantees for the canonical angles of all the tested target matrices in comparison to those from {
        \cite{saibaba2018randomized} Theorem 1 (\b{S2018 Thm1 w/ $\sigmab$ and $\wh\sigmab$})
    }, as explained in \Cref{remark:prior_bound_compare_exist}.
    
    \item The 
    {
    unbiased estimators (\b{Prop1 w/ $\sigmab$ and $\wh\sigmab$})} in \Cref{prop:space_agnostic_estimation} yield accurate approximations for the true canonical angles on all the tested target matrices with as few as $N=3$ trials, while enjoying good empirical concentrate. As a potential drawback, the accuracy of the unbiased estimates may be compromised when approaching the machine epsilon (as observed in \Cref{fig:SNN-m500n500r20a100-kl_k50_l200}, $\sin\angle\rbr{\Vb_k,\Yb},~q=1$).

    \item The {
        posterior residual-based bounds (\b{Thm2 w/ $\sigmab$ and $\wh\sigmab$})
    } in \Cref{thm:with_oversmp_computable_det} are relatively tighter among the compared bounds in the setting with larger oversampling ($l=4k$), and no power iterations ($\sin\angle\rbr{\Ub_k,\wh\Ub_l}$ with $q=0$) or exponential spectral decay (\Cref{fig:Gaussian-exp-kl_k50_l200})
    
    \item The {
        posterior residual-based bounds \Cref{eq:separate_gap_Uk_Ul,eq:separate_gap_Vk_Vl} (\b{Thm3 w/ $\sigmab$ and $\wh\sigmab$})
    } in \Cref{thm:separate_gap_bounds} share the similar relative tightness as the posterior residual-based bounds in \Cref{thm:with_oversmp_computable_det}, but are slightly more sensitive to power iterations. As shown in \Cref{fig:Gaussian-exp-kl_k50_l200}, on a target matrix with exponential spectral decay and large oversampling ($l=4k$), \Cref{thm:separate_gap_bounds} gives tighter posterior guarantees when $q>0$. However, with the addition assumptions $\sigma_k > \wh\sigma_{k+1}$ and $\sigma_k > \nbr{\Eb_{33}}_2$, \Cref{thm:separate_gap_bounds} usually requires large oversampling ($l=4k$) in order to provide non-trivial (\ie, {within} the range $[0,1]$) bounds.
\end{enumerate}

For target matrices with various patterns of spectral decay, with different combinations of oversampling ($l=1.6k,4k$) and power iterations ($q=0,1$), we make the following observations on the relative tightness of upper bounds in \Cref{thm:space_agnostic_bounds}, \Cref{thm:with_oversmp_computable_det}, and \Cref{thm:separate_gap_bounds}.
\begin{enumerate}
    \item For target matrices with subexponential spectral decay, the space-agnostic bounds in \Cref{thm:space_agnostic_bounds} are relatively tighter in most tested settings, except for the setting in \Cref{fig:Gaussian-exp-kl_k50_l200} with larger oversampling ($l=200$) and no power iterations ($q=0$).
    
    \item For target matrices with exponential spectral decay (\Cref{fig:Gaussian-exp-kl_k50_l200} and \Cref{fig:Gaussian-exp-kl_k50_l80}), the posterior residual-based bounds in \Cref{thm:with_oversmp_computable_det} and \Cref{thm:separate_gap_bounds} tend to be relatively tighter, especially with large oversampling (\Cref{fig:Gaussian-exp-kl_k50_l200} with $l=4k$). Meanwhile, with power iterations $q>0$, \Cref{thm:separate_gap_bounds} tend to be tighter than \Cref{thm:with_oversmp_computable_det}.
\end{enumerate} 

Furthermore, considering the scenario with an unknown true spectrum $\Sigmab$, we plot estimations for the upper bounds in \Cref{thm:space_agnostic_bounds}, \Cref{thm:with_oversmp_computable_det}, \Cref{thm:separate_gap_bounds}, and the unbiased estimates in \Cref{prop:space_agnostic_estimation}, evaluated with a padded approximation of the spectrum $\wt\Sigmab=\diag\rbr{\wh\sigma_1,\dots,\wh\sigma_l,\dots,\wh\sigma_l}$, which leads to mild overestimations, as marked in dashes from \Cref{fig:Gaussian-poly1-kl_k50_l200} to \Cref{fig:mnist-train800-kl_k50_l80_p5-10}.

\subsection{Balance between Oversampling and Power Iterations}\label{subsec:balance_oversampling_power_iterations_example}

To illustrate the insight cast by \Cref{thm:space_agnostic_bounds} on the balance between oversampling and power iterations, we consider the following synthetic example.
\begin{example}
Given a target rank $k \in \N$, we consider a simple synthetic matrix $\Ab \in \C^{r \times r}$ of size $r = (1+\beta)k$, consisting of random singular subspaces (generated by orthonormalizing Gaussian matrices) and a step spectrum:
\begin{align*}
    \sigma\rbr{\Ab}=\diag(\underbrace{\sigma_1,\dots,\sigma_1}_{\sigma_i=\sigma_1 ~\forall~ i \le k}, \underbrace{\sigma_{k+1},\dots,\sigma_{k+1}}_{\sigma_i=\sigma_{k+1} ~\forall~ i \ge k+1}).
\end{align*}
We fix a budget of $N=\alpha k$ matrix-vector multiplications with $\Ab$ in total. The goal is to distribute the computational budget between the sample size $l$ and the number of power iterations $q$ for the smaller canonical angles $\angle\rbr{\Ub_k, \wh\Ub_l}$.

Leveraging \Cref{thm:space_agnostic_bounds}, we start by fixing $\gamma>1$ associated with the constants $\eps_1=\gamma \sqrt{k/l}$ and $\eps_2=\gamma \sqrt{l/(r-k)}$ in \Cref{eq:space_agnostic_left} such that $l \ge \gamma^2 k$ and $2q+1 < \alpha/\gamma^2$. Characterized by $\gamma$, the right-hand-side of \Cref{eq:space_agnostic_left} under fixed budget $N$ (\ie, $N \ge l(2q+1)$) is defined as:
\begin{align}\label{eq:space_agnostic_example_rhs}
    \phi_\gamma\rbr{q} \dfeq &\rbr{1 + \frac{1-\eps_1}{1+\eps_2} \cdot \frac{l}{r-k} \rbr{\frac{\sigma_1}{\sigma_{k+1}}}^{4q+2}}^{-\frac{1}{2}}
    \\
    = &\rbr{1 + \frac{\alpha-\gamma\sqrt{\alpha (2q+1)}}{\beta(2q+1)+\gamma\sqrt{\alpha\beta(2q+1)}} \rbr{\frac{\sigma_1}{\sigma_{k+1}}}^{4q+2}}^{-\frac{1}{2}}. \nonumber
\end{align}
With the synthetic step spectrum, the dependence of \Cref{eq:space_agnostic_left} on $\sigma\rbr{\Ab}$ is reduced to the spectral gap $\sigma_1/\sigma_{k+1}$ in \Cref{eq:space_agnostic_example_rhs}. 

As a synopsis, \Cref{tab:space_agnostic_example_parameter_summary} summarizes the relevant parameters that characterize the problem setup.
\begin{table}[!h]
    \centering
    \caption{Given $\Ab \in \C^{r \times r}$ with a spectral gap $\sigma_1/\sigma_{k+1}$, a target rank $k$, and a budget of $N$ matrix-vector multiplications, we consider applying \Cref{algo:rsvd_power_iterations} with a sample size $l$ and $q$ power iterations.}
    \label{tab:space_agnostic_example_parameter_summary}
    \begin{tabular}{c|c|c}
    \toprule
        $\alpha$ & budget parameter & $N=\alpha k$ 
        \\
        $\beta$ & size parameter & $r=(1+\beta)k$ 
        \\
        $\gamma$ & oversampling parameter & $l \ge \gamma^2 k$ and $2q+1 \le \frac{\alpha}{\gamma^2}$
        \\
    \bottomrule
    \end{tabular}
\end{table} 

\begin{figure}[!ht]
    \centering
    \includegraphics[width=\textwidth]{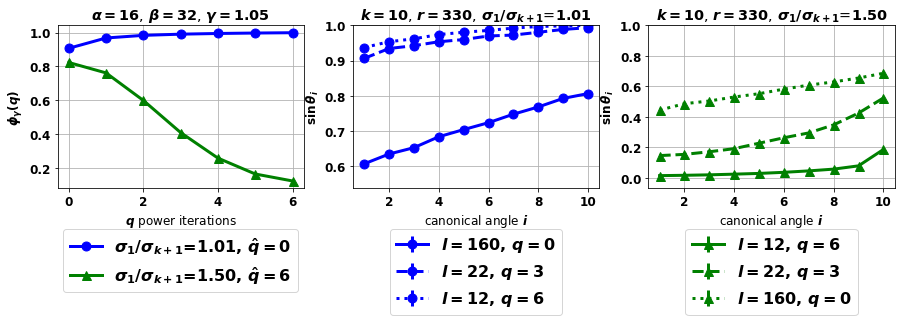}
    \caption{For $k=10$, $\alpha=16$, $\beta=32$, $\gamma=1.05$, the left figure marks $\phi_\gamma(q)$ (\ie, the right-hand side of \Cref{eq:space_agnostic_left} under the fixed budget $N$) with two different spectral gaps ($\wh q = \argmin_{1 \le 2q+1 \le \alpha/\gamma^2} \phi_\gamma\rbr{q}$), while the middle and the right figures demonstrate how the relative magnitudes of canonical angles $\sin\angle_i\rbr{\Ub_k,\wh\Ub_l}$ ($i \in [k]$) under different configurations (\ie, choices of $(l,q)$, showing the averages and ranges of $5$ trials)  align with the trends in $\phi_\gamma(q)$.}
    \label{fig:space-agno-ex_a16-b32-g1}
\end{figure}

\begin{figure}[!ht]
    \centering
    \includegraphics[width=\textwidth]{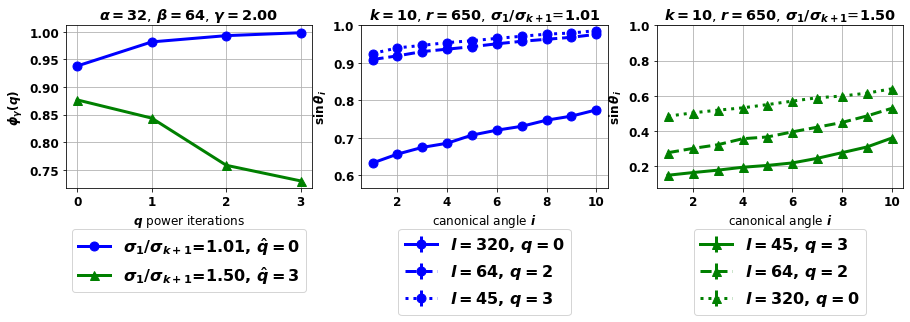}
    \caption{Under the same setup as \Cref{fig:space-agno-ex_a16-b32-g1}, for $k=10$, $\alpha=32$, $\beta=64$, $\gamma=2.00$, the trend in $\phi_\gamma(q)$ also aligns well with that in true canonical angles $\sin\angle_i\rbr{\Ub_k,\wh\Ub_l}$ ($i \in [k]$).}
    \label{fig:space-agno-ex_a32-b64-g2}
\end{figure}

With $k=10$, $\alpha=32$, $\beta=64$, and $\gamma \in \cbr{1.05, 2.00}$, \Cref{fig:space-agno-ex_a16-b32-g1} and \Cref{fig:space-agno-ex_a32-b64-g2} illustrate 
\begin{enumerate*}[label=(\roman*)]
    \item how the balance between oversampling and power iterations is affected by the spectral gap $\sigma_1/\sigma_{k+1}$, and more importantly, 
    \item how \Cref{eq:space_agnostic_example_rhs} unveils the trend in true canonical angles $\sin\angle_i\rbr{\Ub_k,\wh\Ub_l}$ among different configurations $\csepp{(l,q)}{l \ge \gamma^2k, 2q+1 \le \alpha/\gamma^2}$.
\end{enumerate*}

Concretely, both \Cref{fig:space-agno-ex_a16-b32-g1} and \Cref{fig:space-agno-ex_a32-b64-g2} imply that more oversampling (\eg, $q=0$) is preferred when the spectral gap is small (\eg, $\sigma_1/\sigma_{k+1}=1.01$); while more power iterations (\eg, $q=\lfloor \frac{\alpha/\gamma^2-1}{2} \rfloor$) are preferred when the spectral gap is large (\eg, $\sigma_1/\sigma_{k+1}=1.5$). Such trends are both observed in the true canonical angles $\sin\angle_i\rbr{\Ub_k,\wh\Ub_l}$ ($i \in [k]$) and well reflected by $\phi_\gamma\rbr{q}$.
\end{example}

\section{Discussion}
{We presented} prior and posterior bounds and estimates that can be computed efficiently for canonical angles of the randomized subspace approximation. 
Under moderate multiplicative oversampling, our prior probabilistic bounds are space-agnostic (\ie, independent of the unknown true subspaces), asymptotically tight, and can be computed in linear ($O(\rank\rbr{\Ab})$) time, while casting a clear {guidance} on the balance between oversampling and power iterations for a fixed budget of matrix-vector multiplications.
As corollaries of the prior probabilistic bounds, we introduce a set of unbiased canonical angle estimates that are efficiently computable and applicable to arbitrary choices of oversampling with good empirical concentrations.
In addition to the prior bounds and estimates, we further discuss two sets of posterior bounds that provide deterministic guarantees for canonical angles given the computed low-rank approximations.
With numerical experiments, we compare the empirical tightness of different canonical angle bounds and estimates on various data matrices under a diverse set of algorithmic choices for the randomized subspace approximation. 

\chapter{Sample Efficiency of Data Augmentation Consistency Regularization}\label{ch:dac}

\subsection*{Abstract}
Data augmentation is popular in the training of large neural networks; however, currently, theoretical understanding of the discrepancy between different algorithmic choices of leveraging augmented data remains limited.
In this paper, we take a step in this direction -- we first present a simple and novel analysis for linear regression with label invariant augmentations, demonstrating that data augmentation consistency (DAC) is intrinsically more efficient than empirical risk minimization on augmented data (DA-ERM). The analysis is then generalized to misspecified augmentations (i.e., augmentations that change the labels), which again demonstrates the merit of DAC over DA-ERM. 
Further, we extend our analysis to non-linear models (e.g., neural networks) and present generalization bounds. Finally, we perform experiments that make a clean and apples-to-apples comparison (i.e., with no extra modeling or data tweaks) between DAC and DA-ERM using CIFAR-100 and WideResNet; these together demonstrate the superior efficacy of DAC.\footnote{This chapter is based on the following published conference paper: \\
\bibentry{yang2022sample}~\citep{yang2022sample}.}

\section{Introduction}

Modern machine learning models, especially deep learning models, require abundant training samples. Since data collection and human annotation are expensive, data augmentation has become a ubiquitous practice in creating artificial labeled samples and improving generalization performance. This practice is corroborated by the fact that the semantics of images remain the same through simple translations like obscuring, flipping, rotation, color jitter, rescaling~\citep{shorten2019survey}. Conventional algorithms use data augmentation to expand the training data set~\citep{simard1998transformation,krizhevsky2012imagenet,simonyan2014very,he2016deep,cubuk2019autoaugment}. 

Data Augmentation Consistency (DAC) regularization, as an alternative, enforces the model to output similar predictions on the original and augmented samples and has contributed to many recent state-of-the-art supervised or semi-supervised algorithms. This idea was first proposed in~\cite{bachman2014learning} and popularized by~\cite{laine2016temporal,sajjadi2016regularization}, and gained more attention recently with the success of FixMatch~\citep{sohn2020fixmatch} for semi-supervised few-shot learning as well as AdaMatch~\citep{berthelot2021adamatch} for domain adaptation. DAC can utilize unlabeled samples, as one can augment the training samples and enforce consistent predictions without knowing the true labels. This bypasses the limitation of the conventional algorithms that can only augment labeled samples and add them to the training set (referred to as DA-ERM). However, it is not well-understood whether DAC has additional algorithmic benefits compared to DA-ERM. We are, therefore, seeking a theoretical answer. 

Despite the empirical success, the theoretical understanding of data augmentation (DA) remains limited. Existing work \citep{chen2020group,mei2021learning,lyle2019analysis} focused on establishing that augmenting data saves on the number of labeled samples needed for the same level of accuracy. However, none of these explicitly compare the efficacy (in terms of the number of augmented samples) between different algorithmic choices on {\em how to use the augmented samples} in an apples-to-apples way.

In this paper, we focus on the following research question:
\begin{center}

\textit{Is DAC intrinsically more efficient than DA-ERM (even without unlabeled samples)? }
\end{center}

We answer the question affirmatively. We show that DAC is intrinsically more efficient than DA-ERM with a simple and novel analysis for linear regression under label invariant augmentations. We then extend the analysis to misspecified augmentations (i.e., those that change the labels). We further provide generalization bounds under consistency regularization for non-linear models like two-layer neural networks and DNN-based classifiers with expansion-based augmentations. Intuitively, we show DAC is better than DA-ERM in the following sense: 1) DAC enforces stronger invariance in the learned models, yielding smaller estimation error; and 2) DAC better tolerates mis-specified augmentations and incurs smaller approximation error. Our theoretical findings can also explain and guide some technical choices, e.g. why we can use stronger augmentation in consistency regularization but only weaker augmentation when creating pseudo-labels~\citep{sohn2020fixmatch}.  

Specifically, our \textbf{main contributions} are:
\begin{itemize}
    \item \textbf{Theoretical comparisons between DAC and DA-ERM.} We first present a simple and novel result for linear regression, which shows that DAC yields a strictly smaller generalization error than DA-ERM using the same augmented data. Further, we demonstrate that with with the flexibility of hyper-parameter tuning, DAC can better handle data augmentation with small misspecification in the labels. 
    \item \textbf{Extended analysis for non-linear models.} 
    We derive generalization bounds for DAC under two-layer neural networks, and classification with expansion-based augmentations. 
    \item \textbf{Empirical comparisons between DAC and DA-ERM.} We perform experiments that make a clean and apples-to-apples comparison (i.e., with no extra modeling or data tweaks) between DAC and DA-ERM using CIFAR-100 and WideResNet. Our empirical results demonstrate the superior efficacy of DAC.
\end{itemize}

\section{Related Work}

\textbf{Empirical findings. }
Data augmentation (DA) is an essential ingredient for almost every state-of-the-art supervised learning algorithm since the seminal work of \cite{krizhevsky2012imagenet} (see reference therein \citep{simard1998transformation,simonyan2014very,he2016deep,cubuk2019autoaugment,kuchnik2018efficient}). It started from adding augmented data to the training samples via (random) perturbations, distortions, scales, crops, rotations, and horizontal flips. More sophisticated variants were subsequently designed; a non-exhaustive list includes Mixup \citep{zhang2017mixup}, Cutout \citep{devries2017improved}, and Cutmix \citep{yun2019cutmix}. The choice of data augmentation and their combinations require domain knowledge and experts' heuristics, which triggered some automated search algorithms to find the best augmentation strategies~\citep{lim2019fast,cubuk2019autoaugment}. The effects of different DAs are systematically explored in \cite{tensmeyer2016improving}. 

Recent practices not only add augmented data to the training set but also enforce similar predictions by adding consistency regularization~\citep{bachman2014learning,laine2016temporal,sohn2020fixmatch}. One benefit of DAC is the feasibility of exploiting unlabeled data. Therefore input consistency on augmented data also formed a major component to state-of-the-art algorithms for semi-supervised learning~\citep{laine2016temporal,sajjadi2016regularization,sohn2020fixmatch,xie2020self}, self-supervised learning~\citep{chen2020simple}, and unsupervised domain adaptation~\citep{french2017self,berthelot2021adamatch}.

\textbf{Theoretical studies. }
Many interpret the effect of DA as some form of regularization~\citep{he2019data}. Some work focuses on linear transformations and linear models \citep{wu2020generalization} or kernel classifiers \citep{dao2019kernel}. Convolutional neural networks by design enforce translation equivariance symmetry \citep{benton2020learning,li2019enhanced};  further studies have hard-coded CNN's invariance or equivariance to rotation~\citep{cohen2016group,marcos2017rotation,worrall2017harmonic,zhou2017oriented}, scaling~\citep{sosnovik2019scale,worrall2019deep} and other types of transformations.

Another line of works view data augmentation as invariant learning by averaging over group actions~\citep{lyle2019analysis,chen2020group,mei2021learning,wang2020squared,bietti2021sample,shao2022theory}. They consider an ideal setting that is equivalent to ERM with all possible augmented data, bringing a clean mathematical interpretation. 
In contrast, we are interested in a more realistic setting with limited augmented data. In this setting, it is crucial to utilize the limited data with proper training methods, the difference of which cannot be revealed under previously studied settings.  

Some more recent work investigates the feature representation learning procedure with DA for self-supervised learning tasks~\citep{garg2020functional,wen2021toward,haochen2021provable,von2021self}. \cite{cai2021theory,wei2021theoretical} studied the effect of data augmentation with label propagation. Data augmentation is also deployed to improve robustness \citep{rajput2019does}, to facilitate domain adaptation and domain generalization~\citep{cai2021theory,sagawa2019distributionally}.

\section{Problem Setup and Data Augmentation Consistency}\label{sec:general}

Consider the standard supervised learning problem setup: $\xb \in \Xcal$ is input feature, and $y \in \Ycal$ is its label (or response). Let $\Pgt$ be the true distribution of $\rbr{\xb, y}$ (i.e., the label distribution follows $y \sim \Pgt(y | \xb)$). We have the following definition for label invariant augmentation.

\begin{definition}[Label Invariant Augmentation] \label{def:causal_invar_data_aug}
For any sample $\xb \in \Xcal$, we say that a random transformation $A:\Xcal \to \Xcal$ is a label invariant augmentation if and only if $A\rbr{\xb}$ satisfies $\Pgt(y | \xb) = \Pgt(y | A(\xb))$.
\end{definition}

Our work largely relies on label invariant augmentation but also extends to augmentations that incur small misspecification in their labels. Therefore our results apply to the augmentations achieved via certain transformations (e.g., random cropping, rotation), and we do not intend to cover augmentations that can largely alter the semantic meanings (e.g., MixUp \citep{zhang2017mixup}). See examples of data augmentation in \Cref{sec:linear_regression_label_invariant,sec:beyond_linear}.

Now we introduce the learning problem on an augmented dataset. Let $(\Xb, \yb) \in \Xcal^N \times \Ycal^N$ be a training set consisting of $N$ $\iid$ samples. Besides the original $\rbr{\Xb,\yb}$, each training sample is provided with $\alpha$ augmented samples. The features of the augmented dataset $\widetilde\Acal(\xb) \in \Xcal^{(1+\alpha) N}$ is:
\begin{align*}
    \widetilde\Acal(\Xb) = \sbr{\xb_{1}; \cdots; \xb_{N}; \xb_{1,1}; \cdots; \xb_{N,1}; \cdots; \xb_{1, \alpha}; \cdots; \xb_{N, \alpha}} \in \Xcal^{(1+\alpha) N},
\end{align*}
where $\xb_i$ is in the original training set and $\xb_{i, j}, \forall j \in [\alpha]$ are the augmentations of $\xb_i$. The labels of the augmented samples are kept the same, which can be denoted as $\widetilde \Mb \yb \in \Ycal^{(1+\alpha)N}$, where $\wt\Mb \in \R^{(1+\alpha) N \times N}$ is a vertical stack of $(1+\alpha)$ identity mappings. 

\textbf{Data Augmentation Consistency Regularization.} Let $\Hcal = \cbr{h: \Xcal \rightarrow \Ycal}$ be a well-specified function class (e.g., for linear regression problems, $\exists h^* \in \Hcal$, s.t. $h^*(\xb) = \E[y |\xb]$) that we hope to learn from. Without loss of generality, we assume that each function $h \in \Hcal$ can be expressed as $h = f_h \circ \phi_h$, where $\phi_h \in \Phi = \cbr{\phi: \Xcal \rightarrow \Wcal}$ is a proper representation mapping and $f_h \in \Fcal = \cbr{f:\Wcal \rightarrow \Ycal}$ is a predictor on top of the learned representation. We tend to decompose $h$ such that $\phi_h$ is a powerful feature extraction function whereas $f_h$ can be as simple as a linear combiner. For instance, in a deep neural network, all the layers before the final layer can be viewed as feature extraction $\phi_h$, and the predictor $f_h$ is the final linear combination layer.

For a loss function $l: \Ycal \times \Ycal \rightarrow \R$ and a metric $\varrho$ properly defined on the representation space $\Wcal$, learning with data augmentation consistency (DAC) regularization is:
\begin{align}\label{eq:dac_soft}
    \argmin_{h\in\Hcal}\sum_{i=1}^{N}l(h(\xb_i), y_i) + \underbrace{\lambda\sum_{i=1}^N\sum_{j=1}^{\alpha} \varrho\rbr{\phi_h(\xb_i), \phi_h(\xb_{i, j})}}_{\textit{DAC regularization}}.
\end{align}

Note that the DAC regularization in \Cref{eq:dac_soft} can be easily implemented empirically as a regularizer. Intuitively, DAC regularization penalizes the representation difference between the original sample $\phi_h(\xb_i)$ and the augmented sample $\phi_h(\xb_{i,j})$, with the belief that similar samples (i.e., original and augmented samples) should have similar representations. When the data augmentations do not alter the labels, it is reasonable to enforce a strong regularization (i.e., $\lambda \rightarrow \infty$) -- since the conditional distribution of $y$ does not change. The learned function $\widehat h^{dac}$ can then be written as the solution of a constrained optimization problem:
\begin{align}\label{eq:dac_hard}
\begin{split}
    & \widehat h^{dac} \triangleq \argmin_{h \in \Hcal} \sum_{i=1}^N l(h(\xb_i), y_i)\quad \text{s.t.}\quad \phi_h(\xb_i) = \phi_h(\xb_{i,j}),~ \forall i \in [N], j \in [\alpha].
\end{split}
\end{align}

In the rest of the paper, we mainly focus on the data augmentations satisfying \Cref{def:causal_invar_data_aug} and our analysis relies on the formulation of \Cref{eq:dac_hard}. When the data augmentations alter the label distributions (i.e., not satisfying \Cref{def:causal_invar_data_aug}), it becomes necessary to adopt a finite $\lambda$ for \Cref{eq:dac_soft}, and such extension is discussed in \Cref{subsec:finite_lambda}.

\section{Linear Model and Label Invariant Augmentations}\label{sec:linear_regression_label_invariant}

In this section, we show the efficacy of DAC regularization with linear regression under label invariant augmentations (\Cref{def:causal_invar_data_aug}). 

To see the efficacy of DAC regularization (i.e., \Cref{eq:dac_hard}), we revisit a more commonly adopted training method here -- empirical risk minimization on augmented data (DA-ERM):
\begin{align}\label{eq:plain_erm}
    \wh h^{\herm} \triangleq \argmin_{h\in \Hcal} \sum_{i=1}^N l (h(\xb_i), y_i) + \sum_{i=1}^N\sum_{j=1}^{\alpha}l (h(\xb_{i,j}), y_i).
\end{align}
Now we show that the DAC regularization (\Cref{eq:dac_hard}) learns more efficiently than DA-ERM. Consider the following setting: given $N$ observations $\Xb \in \R^{N \times d}$, the responses $\yb \in \R^N$ are generated from a linear model $\yb = \Xb\thetab^* + \epsb$, where $\epsb\in \R^N$ is zero-mean noise with $\E\sbr{\epsb\epsb^\top} = \sigma^2 \Ib_N$. Recall that $\widetilde \Acal(\Xb)$ is the entire augmented dataset, and $\widetilde \Mb \yb$ corresponds to the labels. We focus on the fixed design excess risk of $\thetab$ on $\widetilde \Acal(\Xb)$, which is defined as $L(\thetab) \triangleq \frac{1}{(1+\alpha) N}\norm{\widetilde \Acal(\Xb)\thetab - \widetilde \Acal(\Xb)\thetab^*}_2^2$.

Let $\Deltab \triangleq \wt\Acal\rbr{\Xb} - \wt\Mb\Xb$ and $\dau \triangleq \rank\rbr{\Deltab}$ measure the number of dimensions in the row space of $\Xb$ perturbed by augmentations (which can be intuitively view as the ``strength'' of data augmentations where the larger $\dau$ implies the stronger perturbation brought by $\wt\Acal(\Xb)$ to $\Xb$). Assuming that $\wt\Acal(\Xb)$ has full column rank (such that the linear regression problem has a unique solution), we have the following result for learning by DAC versus DA-ERM.

\begin{theorem}[Informal result on linear regression (formally in \Cref{thm:formal_linear_regression})]\label{thm:informal_linear_regression}
    Learning with DAC regularization, 
    \begin{align*}
        \E_{\epsb}\sbr{L(\widehat \thetab^{dac}) - L(\thetab^*)} = \frac{(d - \dau)\sigma^2}{N},
    \end{align*}
    while learning with ERM directly on the augmented dataset, there exists $d' \in [0, \dau]$ such that
    \begin{align*}
        \E_{\epsb}\sbr{L(\widehat \thetab^{\herm}) - L(\thetab^*)} = \frac{(d - \dau + d')\sigma^2}{N}.
    \end{align*}
\end{theorem}

Formally, we have $d' \triangleq \frac{\tr\rbr{ \rbr{\projAX - \Pb_{\Scal}} {\wt\Mb \widetilde \Mb^\top} }}{1+\alpha}$, where $\projAX \triangleq \widetilde \Acal(\Xb) \widetilde \Acal(\Xb)^\pinv$, and $\Pb_\Scal$ is the projector onto $\Scal \triangleq \cbr{\widetilde \Mb \Xb \thetab~|~\forall \thetab \in \R^d, s.t. \rbr{\widetilde\Acal(\Xb) - \widetilde \Mb \Xb}\thetab = 0}$. Under standard conditions (e.g., $\xb$ is sub-Gaussian and $N$ is not too small), it is not hard to extend \Cref{thm:informal_linear_regression} to random design (i.e., the more commonly acknowledged generalization bound) with the same order.

\begin{remark}[Why DAC is more effective]
    Intuitively, DAC reduces the dimensions from $d$ to $d - \dau$ by enforcing consistency regularization. DA-ERM, on the other hand, still learns in the original $d$-dimensional space. $d'$ characterizes such difference.
\end{remark}

Now we take a closer look at $d' \triangleq \frac{\tr\rbr{ \rbr{\projAX - \Pb_{\Scal}} {\wt\Mb \widetilde \Mb^\top} }}{1+\alpha}$ characterizing the discrepancy between DAC and DA-ERM. We first observe that $\sigma^2 \cdot {\wt \Mb \widetilde \Mb^\top}$ is the noise covariance matrix of the augmented dataset. $\tr\rbr{\Pb_\Scal {\wt\Mb \widetilde \Mb^\top}}$ represents the variance of $\widehat\thetab^{dac}$, while $\tr\rbr{\projAX {\wt\Mb \widetilde \Mb^\top}}$ denotes the variance of $\widehat\thetab^{\herm}$. Therefore, $d' \propto \tr\rbr{\rbr{\projAX - \Pb_\Scal} {\wt\Mb \widetilde \Mb^\top}}$ measures the excess variance of $\widehat\thetab^{\herm}$ in comparison to $\widehat\thetab^{dac}$. When $\projAX \neq \Pb_\Scal$ (a common scenario as instantiated in \Cref{example:linear_regression}), DAC is strictly better than DA-ERM.

\begin{figure}[!h]
    \centering
    \includegraphics[width=.6\linewidth]{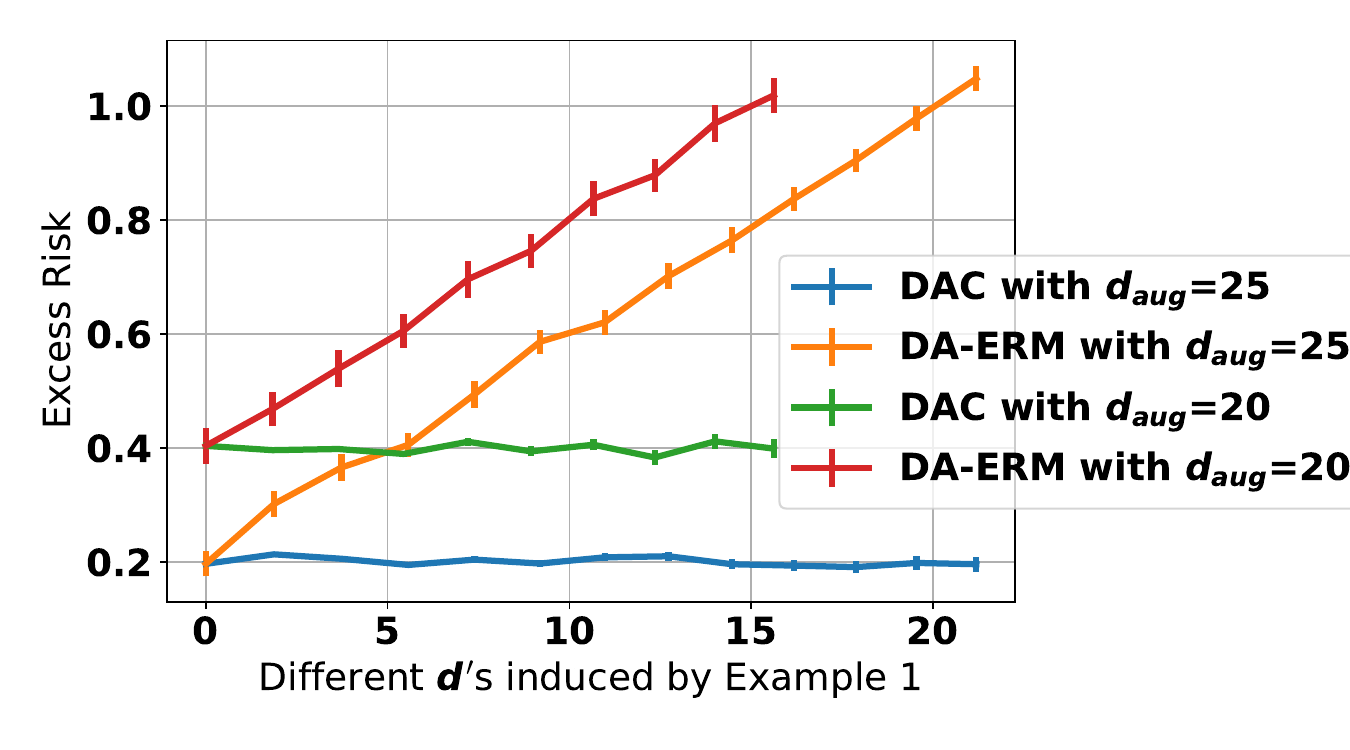}
    \caption{Comparison of DAC regularization and DA-ERM (\Cref{example:linear_regression}). The results precisely match \Cref{thm:informal_linear_regression}. DA-ERM depends on the $d'$ induced by different augmentations, while the DAC regularization works equally well for all $d'$ and better than the DA-ERM. Further, both DAC and DA-ERM are affected by $\dau$, the number of dimensions perturbed by $\wt\Acal(\Xb)$.}
    \label{fig:linear_regression}
\end{figure}

\begin{example}\label{example:linear_regression} Consider a 30-dimensional linear regression. The original training set contains 50 samples. The inputs $\xb_i$s are generated independently from $\Ncal(0, \Ib_{30})$ and we set $\thetab^* = [\thetab_c^*; \mathbf{0}]$ with $\thetab_c^* \sim \Ncal(0, \Ib_5)$  and $\mathbf{0} \in \R^{25}$. The noise variance $\sigma$ is set to $1$. We partition $\xb$ into 3 parts $[x_{c1}, x_{e_1}, x_{e2}]$ and take the following augmentations: $A([x_{c1}; x_{e1}; x_{e2}]) = [x_{c1}; 2x_{e1}; -x_{e2}], x_{c1}\in\R^{d_{c1}}, x_{e1}\in \R^{d_{e1}}, x_{e2}\in\R^{d_{e2}}$, where $d_{c1} + d_{e1} + d_{e2} = 30$. 

Notice that the augmentation perturbs $x_{e1}$ and $x_{e2}$ and leaving $x_{c1}$ unchanged, we therefore have $\dau = 30 - d_{c1}$. By changing $d_{c1}$ and $d_{e1}$, we can have different augmentations with different $\dau, d'$. The results for $\dau \in \cbr{20, 25}$ and various $d'$s are presented in \Cref{fig:linear_regression}. The excess risks precisely match \Cref{thm:informal_linear_regression}. It confirms that the DAC regularization is strictly better than DA-ERM for a wide variety of augmentations.
\end{example}

\section{Beyond Label Invariant Augmentation}\label{subsec:finite_lambda}

In this section, we extend our analysis to misspecified augmentations by relaxing the label invariance assumption (such that $\Pgt(y | \xb) \neq \Pgt(y | A(\xb))$). With an illustrative linear regression problem, we show that DAC also brings advantages over DA-ERM for misspecified augmentations. 

We first recall the linear regression setup: given a set of $N$ $\iid$ samples $\rbr{\Xb, \yb}$ that follows $\yb = \Xb \thetab^* + \epsb$ where $\epsb$ are zero-mean independent noise with $\E\sbr{\epsb \epsb^\top} = \sigma^2 \Ib_N$, we aim to learn the unknown ground truth $\thetab^*$. For randomly generated misspecified augmentations $\wt \Acal(\Xb)$ that alter the labels (\ie, $\wt\Acal\rbr{\Xb}\thetab^* \neq \wt\Mb\Xb\thetab^*$), a proper consistency constraint is $\norm{\phi_h(\xb_i) - \phi_h(\xb_{i, j})}_2 \le \constmis$ (where $\xb_{i, j}$ is an augmentation of $\xb_i$, noticing that $\constmis = 0$ corresponds to label invariant augmentations in \Cref{def:causal_invar_data_aug}). 
For $\constmis>0$, the constrained optimization is equivalent to:
\begin{align}\label{eq:dac_soft_reg}
\begin{split}
    \wh\thetab^{dac} = &\argmin_{\thetab \in \R^d} \frac{1}{N}\nbr{\Xb \thetab - \yb}_2^2 
    + \frac{\lambda }{\rbr{1+\alpha} N} \norm{\rbr{\wt\Acal\rbr{\Xb} - \wt\Mb\Xb} \thetab}_2^2
\end{split}
\end{align}
for some finite $0 < \lambda < \infty$. We compare $\wh\thetab^{dac}$ to the solution learned with ERM on augmented data (as in \Cref{eq:plain_erm}):
\begin{align*}
    \wh\thetab^{\herm} = \argmin_{\thetab \in \R^d} \frac{1}{\rbr{1+\alpha}N} \nbr{\wt\Acal\rbr{\Xb} \thetab - \wt\Mb \yb}_2^2.
\end{align*} 
Let $\covtr \triangleq \frac{1}{N} \Xb^\top \Xb$ and $\covall \triangleq \frac{1}{(1+\alpha)N}\wt\Acal\rbr{\Xb}^\top \wt\Acal\rbr{\Xb}$. 
With $\Sb = \frac{1}{1+\alpha}\wt\Mb^\top \wt\Acal\rbr{\Xb}$, $\Deltab \triangleq \wt\Acal\rbr{\Xb} - \wt\Mb\Xb$, and its reweighted analog $\wt\Deltab \triangleq \rbr{\wt\Mb \Xb} \wt\Acal\rbr{\Xb}^\pinv \Deltab$, we further introduce positive semidefinite matrices: $\covs \triangleq \frac{1}{N}\Sb^\top \Sb$, $\covaug \triangleq \frac{1}{(1+\alpha)N}\Deltab^\top \Deltab$, and $\covaugwt \triangleq \frac{1}{(1+\alpha)N} \wt\Deltab^\top \wt\Deltab$.
For demonstration purpose, we consider fixed $\Xb$ and $\wt\Acal\rbr{\Xb}$, with respect to which we introduce distortion factors $c_X, c_S>0$ as the minimum constants that satisfy $\covall \aleq c_X \covtr$ and $\covall \aleq c_S \covs$ (notice that such $c_X, c_S$ exist almost surely when $\Xb$ and $\wt\Acal\rbr{\Xb}$ are drawn from absolutely continuous marginal distributions). 

Recall $\dau \triangleq \rank\rbr{\Deltab}$ from \Cref{sec:linear_regression_label_invariant}. Let $\projrg \triangleq \Deltab^\pinv \Deltab$ denote the rank-$\dau$ orthogonal projector onto $\range\rbr{\Deltab^\top}$.
Then, for $L(\thetab) = \frac{1}{N}\nbr{\Xb\thetab - \yb}_2^2$, we have the following result:

\begin{figure}
    \begin{subfigure}{0.45\columnwidth}
    \centering
	\includegraphics[width=\linewidth]{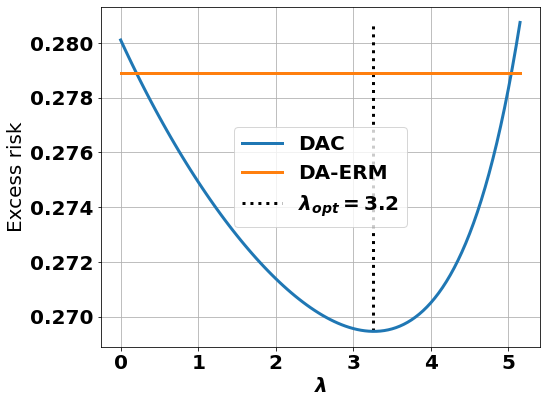}
	\vspace{-2em}
	\caption{Comparison of DAC with different $\lambda$ (optimal choice at $\lambda_{\textit{opt}}=3.2$) and DA-ERM in \Cref{example:misspec}, where $\dau=24$ and $\alpha=1$. The results demonstrate that, with a proper $\lambda$, DAC can outperform DA-ERM under misspecified augmentations.}
    \label{fig:misspec_lambda}
    \end{subfigure}
    \hfill
    \begin{subfigure}{0.45\columnwidth}
    \centering
	\includegraphics[width=\linewidth]{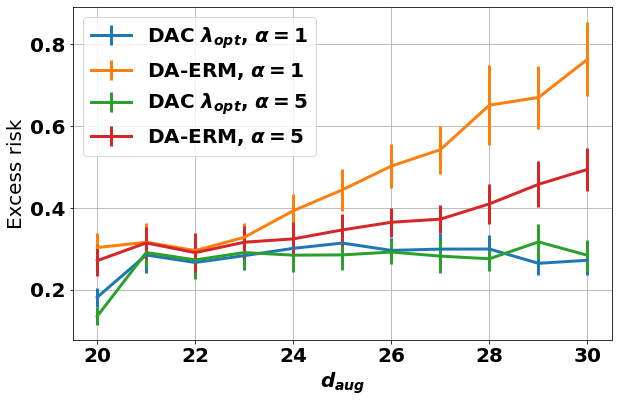}
	\vspace{-2em}
	\caption{Comparison of DAC with the optimal $\lambda$ and DA-ERM in \Cref{example:misspec} for different augmentation strength $\dau$. $\dau=20$ corresponds to the label-invariance augmentations, whereas increasing $\dau$ leads to more misspecification. 
	}
    \label{fig:misspec_daug}
    \end{subfigure}
    \caption{Comparisons of DAC and DA-ERM under misspecification.}
\end{figure}

\begin{theorem}\label{thm:formal_linear_regression_soft}
Learning with DAC regularization (\Cref{eq:dac_soft_reg}), we have that, at the optimal $\lambda$\footnote{A positive (semi)definite matrix $\Sigmab$ induces a (semi)norm: $\nbr{\ub}_{\Sigmab} = \rbr{\ub^{\top} \Sigmab \ub}^{1/2}$ for all conformable $\ub$.},
\begin{align*}
\begin{split}
    \E_{\epsb}\sbr{L(\wh\thetab^{dac}) - L\rbr{\thetab^*}} \leq \frac{\sigma^2 \rbr{d-\dau}}{N}
    \norm{\projrg \thetab^*}_{\covaug} \sqrt{\frac{\sigma^2}{N} \tr\rbr{\covtr \covaug^\pinv}}&,
\end{split}
\end{align*}
whereas learning with DA-ERM (\Cref{eq:plain_erm}),
\begin{align*}
    \E_{\epsb} \sbr{L(\wh\thetab^{\herm}) - L\rbr{\thetab^*}} \geq \frac{\sigma^2 d}{N c_X c_S} + \nbr{\projrg \thetab^*}_{\covaugwt}^2.
\end{align*}
Here, $\projrg \thetab^*$ measures the misspecification in $\thetab^*$ by the augmentations $\wt\Acal\rbr{\Xb}$.
\end{theorem}

One advantage of DAC regularization derives from its flexibility in choosing regularization parameter $\lambda$. With a proper $\lambda$ (\eg, see \Cref{fig:misspec_lambda}) that matches the misspecification $\constmis^2 = \frac{1}{\rbr{1+\alpha} N} \norm{\rbr{\wt\Acal\rbr{\Xb} - \wt\Mb\Xb} \thetab^*}^2_2 = \nbr{\projrg \thetab^*}^2_{\covaug}$, DAC effectively reduces the function class from $\R^d$ to $\csepp{\thetab}{\nbr{\projrg \thetab}_{\covaug} \le \constmis}$ and therefore improves the sample efficiency.

Another advantage of DAC is that, in contrast to DA-ERM, the consistency regularization term in \Cref{eq:dac_soft_reg} refrains from learning the original labels with misspecified augmentations $\E_{\epsb}\sbr{\wt\Mb\yb} \neq \wt\Acal\rbr{\Xb}\thetab^*$ when a suitable $\constmis$ is identified implicitly via $\lambda$. This allows DAC to learn from fewer but stronger (potentially more severely misspecified) augmentations (\eg, \Cref{fig:misspec_daug}). Specifically, as $N \to \infty$, the excess risk of DAC with the optimal $\lambda$ converges to zero by learning from unbiased labels $\E_{\epsb}\sbr{\yb} = \Xb\thetab^*$, whereas DA-ERM suffers from a bias term $\nbr{\projrg \thetab^*}_{\covaugwt}^2 > 0$ due to the bias from misspecified augmentations. 

\begin{example}\label{example:misspec} 
As in \Cref{example:linear_regression}, we consider a linear regression problem of dimension $d=30$ with $\alpha \ge 1$ misspecified augmentations on $N=50$ $\iid$ training samples drawn from $\Ncal(\b0, \Ib_{d})$.  
We aim to learn $\thetab^* = [\thetab_c^*; \mathbf{0}] \in \R^{d}$ (where $\thetab_c^* \in \cbr{-1,+1}^{d_c}$, $d_c=10$) under label noise $\sigma=0.1$. 
The misspecified augmentations mimic the effect of color jitter by adding $\iid$ Gaussian noise entry-wisely to the last $\dau$ feature coordinates: $\wt\Acal\rbr{\Xb} = \sbr{\Xb;\Xb'}$ where $\Xb'_{ij} = \Xb_{ij} + \Ncal\rbr{0, 0.1}$ for all $i \in [N]$, $d-\dau+1 \le j \le d$ -- such that $\dau=\rank\rbr{\Deltab}$ with probability $1$. The $(d-\dau+1),\dots,d_c$-th coordinates of $\thetab^*$ are misspecified by the augmentations.

As previously discussed on \Cref{thm:formal_linear_regression_soft}, DAC is more robust than DA-ERM to misspecified augmentations, and therefore can learn with fewer (smaller $\alpha$) and stronger (larger $\dau$) augmentations. In addition, DAC generally achieves better generalization than DA-ERM with limited samples.
\end{example}

\section{Beyond Linear Model}\label{sec:beyond_linear}

In this section, we extend our analysis of DAC regularization to non-linear models, including the two-layer neural networks, and DNN-based classifiers with expansion-based augmentations.  

Further, in addition to the popular in-distribution setting where we consider a unique distribution $\Pgt$ for both training and testing, DAC regularization is also known to improve out-of-distribution generalization for settings like domain adaptation. We defer detailed discussion on such advantage of DAC regularization for linear regression in the domain adaptation setting to \Cref{apx:case_ood}.

\subsection{Two-layer Neural Network} \label{sec:example_2relu}
We first generalize our analysis to an illustrative nonlinear model -- two-layer ReLU network.
With $\Xcal = \R^d$ and $\Ycal = \R$, we consider a ground truth distribution $\Pgt\rsep{y}{\xb}$ induced by $y = \rbr{\xb^{\top} \Bb^*}_+ \wb^* + \eps$. 
For the unknown ground truth function $h^*\rbr{\xb} \triangleq \rbr{\xb^{\top} \Bb^*}_+ \wb^*$, $(\cdot)_+ \triangleq \max(0,\cdot)$ denotes the element-wisely ReLU function; ${\Bb^*} = \bmat{\bb_1^* \dots \bb_k^* \dots \bb_q^*} \in \R^{d \times q}$ consists of $\bb_k^* \in \SSS^{d-1}$ for all $k \in [q]$; and $\eps \sim \Ncal\rbr{0, \sigma^2}$ is $\iid$ Gaussian noise.
In terms of the function class $\Hcal$, for some constant $C_w \geq \norm{\wb^*}_1$, let
\begin{align*}
    \Hcal = \csepp{h(\xb) = (\xb^{\top} \Bb)_+ \wb}{\Bb=[\bb_1 \dots \bb_q] \in \R^{d \times q}, \norm{\bb_k}_2=1\ \forall\ j \in [q], \norm{\wb}_1 \leq C_w},
\end{align*}
such that $h^* \in \Hcal$. For regression, we again consider square loss $l(h(\xb), y) = \frac{1}{2}(h(\xb)-y)^2$ and learn with DAC on the first layer: $\rbr{\xb_{i}^{\top} \Bb}_{+} = \rbr{\xb_{i,j}^{\top} \Bb}_{+}$.  

Let $\Deltab \triangleq \wt\Acal(\Xb)-\wt\Mb\Xb$, and $\projnull$ be the projector onto the null space of $\Deltab$. Under mild regularity conditions (\ie, $\alpha N$ being sufficiently large, $\xb$ being subgaussian, and distribution of $\Deltab$ being absolutely continuous, as specified in \Cref{apx:pf_case_2layer_relu}), regression over two-layer ReLU networks with the DAC regularization generalizes as following:
\begin{theorem}[Informal result on two-layer neural network with DAC (formally in \Cref{thm:case_2layer_relu_risk})]\label{thm:case_2layer_relu_risk_informal}
Conditioned on $\Xb$ and $\Deltab$, with $L(h) = \frac{1}{N}\nbr{h(\Xb) - h^*(\Xb)}_2^2$ and $\sqrt{\frac{1}{N}\sum_{i=1}^N \nbr{\projnull \xb_i}^2_2} \le \Cnull$, for any $\delta \in (0,1)$, with probability at least $1-\delta$ over $\epsb$, 
\begin{align*}
    L\rbr{\wh{h}^{dac}} - L\rbr{h^*}
    \lesssim 
    \sigma C_w \Cnull \rbr{\frac{1}{\sqrt{N}} + \sqrt{\frac{\log(1/\delta)}{N}}}.
\end{align*}
\end{theorem} 

Recall $\dau = \rank(\Deltab)$. With a sufficiently large $N$ (as specified in \Cref{apx:pf_case_2layer_relu}), we have $\Cnull \lesssim \sqrt{d-\dau}$ with high probability\footnote{Here we only account for the randomness in $\Xb$ but not that in $\Deltab|\Xb$ which characterizes $\dau$ for conciseness. We refer the readers to \Cref{apx:pf_case_2layer_relu} for a formal tail bound on $\Cnull$.}. Meanwhile, applying DA-ERM directly on the augmented samples achieves no better than $L(\wh h^{\herm}) - L(h^{*}) \lesssim \sigma C_w {\max\rbr{\sqrt{\frac{d}{(\alpha+1) N}}, \sqrt{\frac{d - \dau}{N}}}}$, where the first term corresponds to the generalization bound for a $d$-dimensional regression with $(\alpha+1) N$ \emph{i.i.d.} samples (in contrast to augmented samples that are potentially dependent); and the second term follows as the augmentations $\wt\Acal\rbr{\Xb}$ keep a $(d - \dau)$-dimensional subspace (\ie, the null space of $\Deltab=\wt\Acal(\Xb)-\wt\Mb\Xb$) intact, in which DA-ERM can only rely on the $N$ original samples for learning. In specific, the first term will dominate the $\max$ with limited augmented data (i.e., $\alpha$ being small). 

Comparing the two, we see that DAC tends to be more efficient than DA-ERM, and such advantage is enhanced with strong but limited data augmentations (i.e., large $\dau$ and small $\alpha$). For instance, with $\alpha = 1$ and $\dau = d - 1$, the generalization error of DA-ERM scales as $\sqrt{{d}/{N}}$, while DAC yields a dimension-free $\sqrt{{1}/{N}}$ error.

As a synopsis for the regression cases in \Cref{sec:linear_regression_label_invariant}, \Cref{subsec:finite_lambda}, and \Cref{sec:example_2relu} generally, the effect of DAC regularization can be casted as a dimension reduction by $\dau$ -- dimension of the subspace perturbed by data augmentations where features contain scarce label information.

\subsection{Classification with Expansion-based Augmentations}\label{subsec:expansion_based}

A natural generalization of the dimension reduction viewpoint on DAC regularization in the regression setting is the complexity reduction for general function classes. Here we demonstrate the power of DAC on function class reduction in a DNN-based classification setting.

Concretely, we consider a multi-class classification problem: given a probability space $\Xcal$ with marginal distribution $\Pgt(\xb)$ and $K$ classes $\Ycal=[K]$, let $h^*: \Xcal \to [K]$ be the ground truth classifier, partitioning $\Xcal$ into $K$ disjoint sets $\cbr{\Xcal_k}_{k \in [K]}$ such that $\Pgt\rbr{y|\xb} = \b{1}\cbr{y = h^*\rbr{\xb}} = \b{1}\cbr{\xb \in \Xcal_y}$.
In the classification setting, we replace \Cref{def:causal_invar_data_aug} with the notion of \textit{expansion-based data augmentations} introduced in \cite{wei2021theoretical, cai2021theory}.

\begin{definition}[Expansion-based augmentations (formally in \Cref{def:generalized-causal-invariant-data-augmentation})]
\label{def:generalized-causal-invariant-data-augmentation_informal}
With respect to an augmentation function $\Acal:\Xcal \to 2^{\Xcal}$, let $\nbh(S) \triangleq \cup_{\xb \in S} \cbr{\xb' \in \Xcal ~\big|~ \Acal(\xb) \cap \Acal(\xb') \neq \emptyset}$ be the neighborhood of $S \subseteq \Xcal$.
For any $c>1$, we say that $\Acal$ induces $c$-expansion-based data augmentations if (a) $\cbr{\xb} \subsetneq \Acal(\xb) \subseteq \cbr{\xb' \in \Xcal ~|~ h^*(\xb) = h^*(\xb')}$ for all $\xb \in \Xcal$; and (b) for all $k \in [K]$, given any $S \subseteq \Xcal$ with $\Pgt\rbr{S \cap \Xcal_k} \leq \frac{1}{2}$, $\Pgt\rbr{\nbh\rbr{S} \cap \Xcal_k} \geq \min\cbr{c \cdot \Pgt\rbr{S \cap \Xcal_k},1}$.
\end{definition}

Particularly, \Cref{def:generalized-causal-invariant-data-augmentation_informal}(a) enforces that the ground truth classifier $h^*$ is invariant throughout each neighborhood. Meanwhile, the expansion factor $c$ in \Cref{def:generalized-causal-invariant-data-augmentation_informal}(b) serves as a quantification of augmentation strength -- a larger $c$ implies a stronger augmentation $\Acal$.

We aim to learn $h(\xb) \triangleq \argmax_{k \in [K]}\ f(\xb)_k$ with loss $l_{01}\rbr{h(\xb),y} = \b1\cbr{h(\xb) \neq y}$ from $\Hcal$ induced by the class of $p$-layer fully connected neural networks with maximum width $q$, $\Fcal = \csepp{f: \Xcal \to \R^K}{f = f_{2p-1} \circ \dots \circ f_1,}$ where $f_{2\iota-1}(\xb) = \Wb_{\iota} \xb,\ f_{2\iota}(\epsb)=\varphi(\epsb)$, $\Wb_{\iota} \in \R^{d_{\iota} \times d_{\iota-1}}$ $\forall \iota \in [p]$, $q \triangleq \max_{\iota\in[p]} d_{\iota}$, and $\varphi$ is the activation function. 

Over a general probability space $\Xcal$, DAC with expansion-based augmentations requires stronger conditions than merely consistent classification over $\Acal(\xb_i)$ for all labeled training samples $i \in [N]$. Instead, we enforce a large robust margin $m_{\Acal}(f,\xb^u)$ (adapted from \cite{wei2021theoretical}, see \Cref{apx:generalized_DAC}) over an finite set of unlabeled samples $\Xb^u$ that is independent of $\Xb$ and drawn $\iid$ from $P(\xb)$. Intuitively, $m_{\Acal}(f, \xb^u)$ measures the maximum allowed perturbation in all parameters of $f$ such that predictions remain consistent throughout $\Acal\rbr{\xb^u}$ ($\eg$, $m_{\Acal}(f,\xb^u) > 0$ is equivalent to enforcing consistent classification outputs).
For any $0< \tau \leq \max_{f \in \Fcal}\ \inf_{\xb^u \in \Xcal} m_{\Acal}(f, \xb^u)$, the DAC regularization reduces the function class $\Hcal$ to
\begin{align*}
    \Hred \triangleq \csepp{h \in \Hcal}{m_{\Acal}(f,\xb^u)> \tau \quad \forall\ \xb^u \in \Xb^u}.
\end{align*}
Then for $\hgdacfin = \argmin_{h \in \Hred} \frac{1}{N} \sum_{i=1}^N l_{01}\rbr{h(\xb_i),y_i}$, we have the following.

\begin{theorem}[Informal result on classification with DAC (formally in \Cref{thm:generalized-dac-finite-unlabeled})]
\label{thm:generalized-dac-finite-unlabeled_informal}
Given an augmentation function $\Acal$ that induces $c$-expansion-based data augmentations (\Cref{def:generalized-causal-invariant-data-augmentation_informal}) such that
\begin{align*}
    \mu \triangleq \sup_{h \in \Hred} \PP_{\Pgt} \sbr{\exists\ \xb' \in \Acal(\xb): h(\xb) \neq h(\xb')} \leq \frac{c-1}{4},
\end{align*}
for any $\delta \in (0,1)$, with probability at least $1-\delta$, we have $\mu \leq \wt O \rbr{\frac{\sum_{\iota=1}^p \sqrt{q} \norm{\Wb_{\iota}}_F}{\tau \sqrt{\abbr{\Xb^u}}} + \sqrt{\frac{p \log \abbr{\Xb^u}}{\abbr{\Xb^u}}}}$ such that
\begin{align*}
    L_{01}\rbr{\hgdacfin} - L_{01}\rbr{h^*} 
    \lesssim & \sqrt{\frac{K \log (N)}{N} + \frac{\mu}{\min\cbr{c-1,1}}} 
    + \sqrt{\frac{\log(1/\delta)}{N}}.
\end{align*}
\end{theorem}  

In particular, DAC regularization leverages the unlabeled samples $\Xb^u$ and effectively decouples the labeled sample complexity $N = \wt O\rbr{K}$ from the complexity of the function class $\Hcal$ (characterized by $\cbr{\Wb_\iota}_{\iota \in [p]}$ and $q$ and encapsulated in $\mu$) via the reduced function class $\Hred$. Notably, \Cref{thm:generalized-dac-finite-unlabeled_informal} is reminiscent of \cite{wei2021theoretical} Theorem 3.6, 3.7, and \cite{cai2021theory} Theorem 2.1, 2.2, 2.3. We unified the existing theories under our function class reduction viewpoint to demonstrate its generality.

\section{Experiments}\label{sec:dac_experiments}

In this section, we empirically verify that training with DAC learns more efficiently than DA-ERM. The dataset is derived from CIFAR-100, where we randomly select 10,000 labeled data as the training set (i.e., 100 labeled samples per class). During the training time, given a training batch, we generate augmentations by RandAugment \citep{cubuk2020randaugment}. We set the number of augmentations per sample to 7 unless otherwise mentioned.

The experiments focus on comparisons of 1) training with consistency regularization (DAC), and 2) empirical risk minimization on the augmented dataset (DA-ERM). We use the same network architecture (a WideResNet-28-2 \citep{zagoruyko2016wide}) and the same training settings (e.g., optimizer, learning rate schedule, etc) for both methods. We defer the detailed experiment settings to \Cref{apdx:exp_detail}. Our test set is the standard CIFAR-100 test set, and we report the average and standard deviation of the testing accuracy of 5 independent runs. The consistency regularizer is implemented as the $l_2$ distance of the model's predictions on the original and augmented samples.

\begin{table*}[!t]
\centering
\caption{Testing accuracy of DA-ERM and DAC with different $\lambda$'s (regularization coeff.).}
\label{table:different_lambda}
\begin{tabular}{c|ccccc}
\hline
\multirow{2}{*}{DA-ERM} & \multicolumn{5}{c}{DAC Regularization}                                              \\
                                       & $\lambda=0$ & $\lambda=1$ & $\lambda=5$ & $\lambda=10$ & $\lambda=20$            \\ \hline
$69.40 \pm 0.05$                                  & $62.82 \pm 0.21$       & $68.63 \pm 0.11$       & $\mathbf{70.56 \pm 0.07}$       & $\mathbf{70.52 \pm 0.14}$        & $68.65 \pm 0.27$     \\ \hline
\end{tabular}
\end{table*}

\textbf{Efficacy of DAC regularization.} We first show that the DAC regularization learns more efficiently than DA-ERM. The results are listed in \Cref{table:different_lambda}. In practice, the augmentations almost always alter the label distribution, we therefore follow the discussion in \cref{subsec:finite_lambda} and adopt a finite $\lambda$ (i.e., the multiplicative coefficient before the DAC regularization, see \Cref{eq:dac_soft}). With proper choice of $\lambda$, training with DAC significantly improves over DA-ERM.

\begin{table*}[!t]
\centering
\caption{Testing accuracy of DA-ERM and DAC with different numbers of augmentations.}
\label{table:different_number_of_aug}
\begin{tabular}{c|cccc}
\hline
Number of Augmentations & 1 & 3 & 7 & 15 \\ \hline
DA-ERM                     & $67.92 \pm 0.08$  & $69.04 \pm 0.05$  & $69.25 \pm 0.16$  &  $69.30 \pm 0.11$  \\
DAC ($\lambda=10$)       & $\mathbf{70.06 \pm 0.08}$  & $\mathbf{70.77 \pm 0.20}$  & $\mathbf{70.74 \pm 0.11}$  &  $\mathbf{70.31 \pm 0.12}$  \\ \hline
\end{tabular}
\end{table*}

\begin{table*}[!t]
\centering
\caption{Testing accuracy of ERM and DAC regularization with different numbers of labeled data.}
\label{table:different_labeled_samples}
\begin{tabular}{c|ccc}
\hline
Number of Labeled Data & 1000 & 10000 & 20000  \\ \hline
DA-ERM                    & $31.11 \pm 0.30$ & $68.89 \pm 0.07$ &  $\mathbf{76.79 \pm 0.13}$   \\
DAC ($\lambda=10$)      & $\mathbf{33.59 \pm 0.41}$ & $\mathbf{70.71 \pm 0.10}$  & $\mathbf{76.86 \pm 0.16}$   \\ \hline
\end{tabular}
\end{table*}

\begin{table*}[!t]
\centering
\caption{DAC performs well under misspecified augmentations after tuning $\lambda$.}
\label{table:misspecified_lambda}
\begin{tabular}{c|c|c|c|c}
\hline
No Augmentation & DA-ERM & DAC ($\lambda=0.1$) & DAC ($\lambda=1$) & DAC ($\lambda=10$) \\ \hline
$62.82 \pm 0.21$ & $61.35 \pm 0.27$   & $63.73 \pm 0.33$ & $\mathbf{64.30 \pm 0.20} $    & $64.00 \pm 0.26$    \\ \hline
\end{tabular}
\end{table*}

\begin{table}[t]
\centering
\caption{DAC helps FixMatch when the unlabeled data is scarce.}
\vspace{-0.5em}
\label{table:combining_with_SSL}
\begin{tabular}{c|ccc}
\hline
Number of Unlabeled Data     & 5000 & 10000 & 20000 \\ \hline
FixMatch                     &  67.74    &  69.23     & 70.76          \\
FixMatch + DAC ($\lambda=1$) &  \textbf{71.24}    &   \textbf{72.7}    &  \textbf{74.04}     \\ \hline
\end{tabular}
\end{table}

\textbf{DAC regularization helps more with limited augmentations.} Our theoretical results suggest that the DAC regularization learns efficiently with a limited number of augmentations. While keeping the number of labeled samples to be 10,000, we evaluate the performance of the DAC regularization and DA-ERM with different numbers of augmentations. The number of augmentations for each training sample ranges from 1 to 15, and the results are listed in \Cref{table:different_number_of_aug}. The DAC regularization offers a more significant improvement when the number of augmentations is small. This clearly demonstrates that the DAC regularization learns more efficiently than DA-ERM.

\begin{figure}[!t]
    \centering
	\includegraphics[clip, trim={50 60 50 60}, width=.6\linewidth]{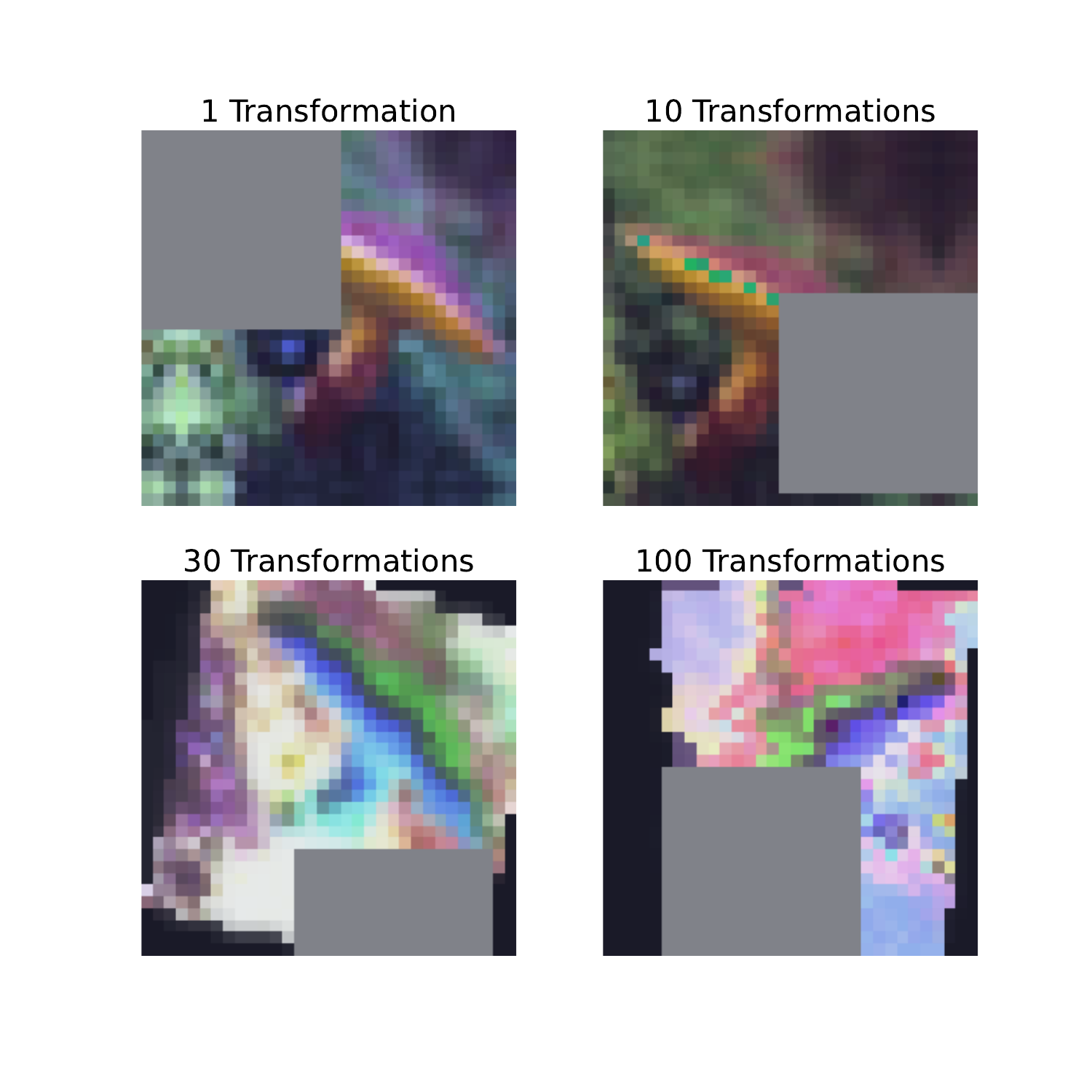}
	\vspace{-2em}
	\caption{Different numbers of transformations.}
    \label{fig:augmentation_strength}
\end{figure}

\textbf{DAC regularization helps more when data is scarce.} We conduct experiments with different numbers of labeled samples, ranging from 1,000 (i.e., 10 images per class) to 20,000 samples (i.e., 200 images per class). We generate 3 augmentations for each of the samples during the training time, and the results are presented in \Cref{table:different_labeled_samples}. Notice that the DAC regularization gives a bigger improvement over DA-ERM when the labeled samples are scarce. This matches the intuition that when there are sufficient training samples, data augmentation is less necessary. Therefore, the difference between different ways of utilizing the augmented samples becomes diminishing.

\textbf{DAC performs well under misspecified augmentations.}
As suggested by \Cref{thm:formal_linear_regression_soft}, DAC is more robust to misspecified augmentations with proper $\lambda$. We further empirically verify this result with misspecified augmentations - where the augmentations are generated by applying 100 random transformations. When too many transformations are applied (see illustration in \Cref{fig:augmentation_strength}), the augmentation will alter the label distribution and is thus misspecified. The results are presented in \Cref{table:misspecified_lambda}. Notice that with $\lambda = 1$, the DAC delivers the best accuracy, which supports our theoretical results.

Further, comparing the results of \Cref{table:different_lambda} and \Cref{table:misspecified_lambda}, we see that the optimal $\lambda$ is different when the augmentations are misspecified. Because of the flexibility in choosing $\lambda$, DAC is able to outperform DA-ERM, which matches the result in \Cref{thm:formal_linear_regression_soft}.

\textbf{Combining with a semi-supervised learning algorithm.} Here we show that the DAC regularization can be easily extended to the semi-supervised learning setting. We take the previously established semi-supervised learning method FixMatch \citep{sohn2020fixmatch} as the baseline and adapt the FixMatch by combining it with the DAC regularization. Specifically, besides using FixMatch to learn from the unlabeled data, we additionally generate augmentations for the labeled samples and apply DAC. In particular, we focus on the data-scarce regime by only keeping 10,000 labeled samples and at most 20,000 unlabeled samples. Results are listed in \Cref{table:combining_with_SSL}. We see that the DAC regularization also improves the performance of FixMatch when the unlabeled samples are scarce. This again demonstrates the efficiency of learning with DAC.

\section{Conclusion}

In this paper, we take a step toward understanding the statistical efficiency of DAC with limited data augmentations. At the core, DAC is statistically more efficient because it reduces problem dimensions by enforcing consistency regularization. 

We demonstrate the benefits of DAC compared to DA-ERM (expanding training set with augmented samples) both theoretically and empirically. Theoretically, we show a strictly smaller generalization error under linear regression, and explicitly characterize the generalization upper bound for two-layer neural networks and expansion-based data augmentations. We further show that DAC better handles the label misspecification caused by strong augmentations. Empirically, we provide apples-to-apples comparisons between DAC and DA-ERM. These together demonstrate the superior efficacy of DAC over DA-ERM.


\chapter{Adaptively Weighted Data Augmentation Consistency Regularization: Application in Medical Image Segmentation}\label{ch:adawac}

\subsection*{Abstract}
Concept shift is a prevailing problem in natural tasks like medical image segmentation where samples usually come from different subpopulations with variant correlations between features and labels. One common type of concept shift in medical image segmentation is the ``information imbalance'' between \emph{label-sparse} samples with few (if any) segmentation labels and \emph{label-dense} samples with plentiful labeled pixels.
Existing distributionally robust algorithms have focused on adaptively truncating/down-weighting the ``less informative'' (\ie, label-sparse in our context) samples.  
To exploit data features of label-sparse samples more efficiently, we propose an adaptively weighted online optimization algorithm --- \ours --- to incorporate data augmentation consistency regularization in sample reweighting. Our method introduces a set of trainable weights to balance the supervised loss and unsupervised consistency regularization of each sample separately. At the saddle point of the underlying objective, the weights assign label-dense samples to the supervised loss and label-sparse samples to the unsupervised consistency regularization.
We provide a convergence guarantee by recasting the optimization as online mirror descent on a saddle point problem. Our empirical results demonstrate that \ours not only enhances the segmentation performance and sample efficiency but also improves the robustness to concept shift on various medical image segmentation tasks with different UNet-style backbones.\footnote{This chapter is based on the following published conference paper: \\
\bibentry{dong2022adawac}~\citep{dong2022adawac}.}

\section{Introduction}
Modern machine learning is revolutionizing the field of medical imaging, especially in computer-aided diagnosis with computed tomography (CT) and magnetic resonance imaging (MRI) scans. 
However, classical learning objectives like empirical risk minimization (ERM) generally assume that training samples are independently and identically (\iid) distributed, whereas real-world medical image data rarely satisfy this assumption. 
\Cref{fig:synapse_case40_sparse_dense_ce_weights} instantiates a common observation in medical image segmentation where the segmentation labels corresponding to different cross-sections of the human body tend to have distinct proportions of labeled (\ie, non-background) pixels, which is accurately reflected by the evaluation of supervised cross-entropy loss during training. 
We refer to this as the ``information imbalance'' among samples, as opposed to the well-studied ``class imbalance''~\citep{wong2018segmentation,taghanaki2019combo,yeung2022unified} among the numbers of segmentation labels in different classes.
Such information imbalance induces distinct difficulty/paces of learning with the cross-entropy loss for different samples~\citep{wang2021survey, tullis2011effectiveness,tang2018attention,hacohen2019power}.
Specifically, we say a sample is \emph{label-sparse} when it contains very few (if any) segmentation labels; in contrast, a sample is \emph{label-dense} when its segmentation labels are prolific. Motivated by the information imbalance among samples, we explore the following questions:
\begin{center}
    \textit{
        What is the effect of separation between sparse and dense labels on segmentation?
        \\ 
        Can we leverage such information imbalance to improve the segmentation accuracy?
    }   
\end{center}

We formulate the mixture of label-sparse and label-dense samples as a concept shift --- a type of distribution shift in the conditional distribution of labels given features $P\rsep{\yb}{\xb}$. 
Coping with concept shifts, prior works have focused on adaptively truncating (hard-thresholding) the empirical loss associated with label-sparse samples.  These include the Trimmed Loss Estimator~\citep{shen2019learning}, MKL-SGD~\citep{pmlr-v108-shah20a}, Ordered SGD~\citep{kawaguchi2020ordered}, and the quantile-based Kacmarz algorithm~\citep{haddock2022quantile}. 
Alternatively, another line of works~\citep{wang2018iterative,sagawa2019distributionally} proposes to relax the hard-thresholding operation to soft-thresholding by down-weighting instead of truncating the less informative samples.  
However, diminishing sample weights reduces the importance of both the features and the labels simultaneously, which is still not ideal as the potentially valuable information in the features of the label-sparse samples may not be fully used.

\begin{figure}[ht]
    \centering
    \includegraphics[width=\linewidth]{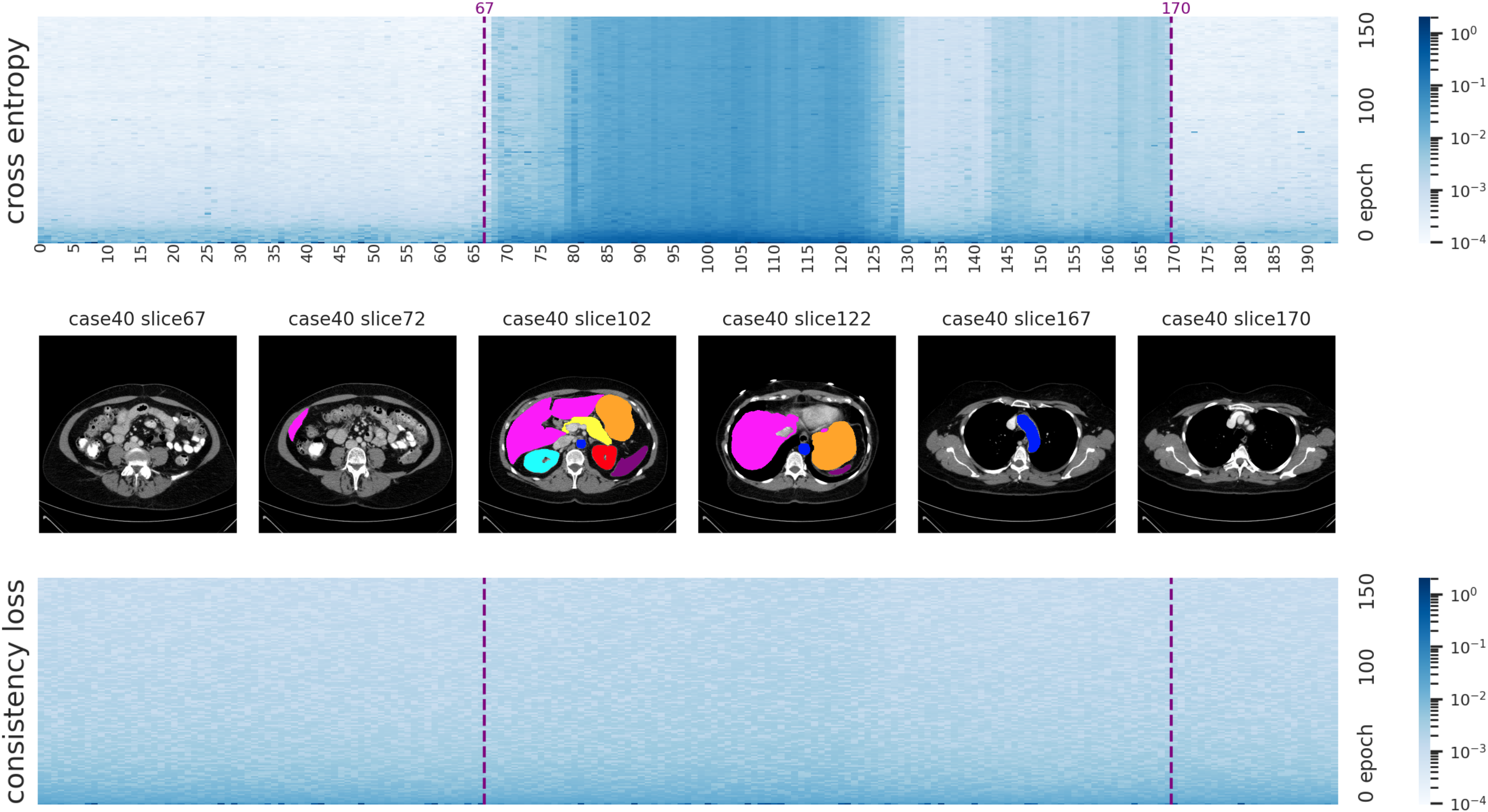}
    \caption{Evolution of cross-entropy losses versus consistency regularization terms for slices at different cross-sections of the human body in the Synapse dataset (described in \Cref{sec:experiments_adawac}) during training.}
    \label{fig:synapse_case40_sparse_dense_ce_weights}
\end{figure}

For further exploitation of the feature of training samples, we propose the incorporation of \emph{data augmentation consistency regularization} on label-sparse samples. 
As a prevalent strategy for utilizing unlabeled data, consistency regularization~\citep{bachman2014learning,laine2016temporal,sohn2020fixmatch} encourages data augmentations of the same samples to lie in the vicinity of each other on a proper manifold.  
For medical imaging segmentation, consistency regularization has been extensively studied in the semi-supervised learning setting~\citep{bortsova2019semi,zhao2019data,li2020transformation,wang2021deep,zhang2021multi,zhou2021ssmd,basak2022embarrassingly} as a strategy for overcoming label scarcity.
Nevertheless, unlike general vision tasks, for medical image segmentation, the scantiness of unlabeled image data can also be a problem due to regulations and privacy considerations~\cite{karimi2020deep}, which makes it worthwhile to reminisce the more classical supervised learning setting. 
In contrast to the aforementioned semi-supervised strategies, we  explore the potency of consistency regularization in the \emph{supervised learning} setting by leveraging the information in the features of label-sparse samples via data augmentation consistency regularization.

To naturally distinguish the label-sparse and label-dense samples, we make a key observation that the unsupervised consistency regularization on encoder layer outputs (of a UNet-style architecture) is much more uniform across different subpopulations than the supervised cross-entropy loss (as exemplified in \Cref{fig:synapse_case40_sparse_dense_ce_weights}).
Since the consistency regularization is characterized by the marginal distribution of features $P(\xb)$ but not labels, and therefore is less affected by the concept shift in $P\rsep{\yb}{\xb}$, it serves as a natural reference for separating the label-sparse and label-dense samples. 
In light of this observation, we present the \emph{weighted data augmentation consistency (WAC) regularization} --- a minimax formulation that reweights the cross-entropy loss versus the consistency regularization associated with each sample via a set of trainable weights. 
At the saddle point of this minimax formulation, the WAC regularization automatically separates samples from different subpopulations by assigning all weights to the consistency regularization for label-sparse samples, and all weights to the cross-entropy terms for label-dense samples.

We further introduce an adaptively weighted online optimization algorithm --- \ours --- for solving the minimax problem posed by the WAC regularization, which is inspired by a mirror-descent-based algorithm for distributionally robust optimization~\citep{sagawa2019distributionally}. By adaptively learning the weights between the cross-entropy loss and consistency regularization of different samples, \ours comes with both a convergence guarantee and empirical success.

The main contributions are summarized as follows:
\begin{itemize}
    \item We introduce the \emph{WAC regularization} that leverages the consistency regularization on the encoder layer outputs (of a UNet-style architecture) as a natural reference to distinguish the label-sparse and label-dense samples (\Cref{sec:wac}). 
    \item We propose an adaptively weighted online optimization algorithm --- \ours --- for solving the WAC regularization problem with a convergence guarantee (\Cref{sec:ada_wac}). 
    \item Through extensive experiments on different medical image segmentation tasks with different UNet-style backbone architectures, we demonstrate the effectiveness of \ours not only for enhancing the segmentation performance and sample efficiency but also for improving the robustness to concept shift (\Cref{sec:experiments_adawac}).
\end{itemize}

\subsection{Related Work}

\paragraph{Sample reweighting.} 
Sample reweighting is a popular strategy for dealing with distribution/subpopulation shifts in training data where different weights are assigned to samples from different subpopulations. 
In particular, the distributionally-robust optimization (DRO) framework~\citep{bental2013robust,duchi2016statistics,duchi2018learning,sagawa2019distributionally} considers a collection of training sample groups from different distributions.  With the explicit grouping of samples, the goal is to minimize the worst-case loss over the groups.
Without prior knowledge of sample grouping, importance sampling~\citep{needell2014stochastic,zhao2015stochastic,alain2015variance,loshchilov2015online,gopal2016adaptive,katharopoulos2018not}, iterative trimming~\citep{kawaguchi2020ordered,shen2019learning}, and empirical-loss-based reweighting~\citep{Wu_Xie_Du_Ward_2022} are commonly incorporated in the stochastic optimization process for adaptive reweighting and separation of samples from different subpopulations.

\paragraph{Data augmentation consistency regularization.}
As a popular way of exploiting data augmentations, consistency regularization encourages models to learn the vicinity among augmentations of the same sample based on the assumption that data augmentations generally preserve the semantic information in data and therefore lie closely on proper manifolds. Beyond being a powerful building block in semi-supervised~\citep{bachman2014learning,sajjadi2016regularization,laine2016temporal,sohn2020fixmatch,berthelot2019mixmatch} and self-supervised~\citep{wu2018unsupervised,he2020momentum,chen2020simple,grill2020bootstrap} learning, the incorporation of data augmentation and consistency regularization also provably improves generalization and feature learning even in the supervised learning setting~\citep{yang2022sample,shen2022data}.

For medical imaging, data augmentation consistency regularization is generally leveraged as a semi-supervised learning tool~\citep{bortsova2019semi,zhao2019data,li2020transformation,wang2021deep,zhang2021multi,zhou2021ssmd,basak2022embarrassingly}. In efforts to incorporate consistency regularization in segmentation tasks with augmentation-sensitive labels, \cite{li2020transformation} encourages transformation consistency between predictions with augmentations applied to the image inputs and the segmentation outputs.
\cite{basak2022embarrassingly} penalizes inconsistent segmentation outputs between teacher-student models, with MixUp \citep{zhang2017mixup} applied to image inputs of the teacher model and segmentation outputs of the student model.
Instead of enforcing consistency in the segmentation output space as above, we leverage the insensitivity of sparse labels to augmentations and encourage consistent encodings (in the latent space of encoder outputs) on label-sparse samples.

\section{Problem Setup}\label{sec:problem_setup_adawac}

\paragraph{Notation.} 
For any $K \in \N$, we denote $\sbr{K} = \cbr{1,\dots,K}$. 
We represent the elements and subtensors of an arbitrary tensor by adapting the syntax for Python slicing on the subscript (except counting from $1$). For example, $\iloc{\xb}{i,j}$ denotes the $(i,j)$-entry of the two-dimensional tensor $\xb$, and $\iloc{\xb}{i,:}$ denotes the $i$-th row.
Let $\II$ be a function onto $\cbr{0,1}$ such that, for any event $e$, $\iffun{e}=1$ if $e$ is true and $0$ otherwise.
For any distribution $P$ and $n \in \N$, we let $P^n$ denote the joint distribution of $n$ samples drawn \iid from $P$.
Finally, we say that an event happens with high probability (\whp) if the event takes place with probability $1-\Omega\rbr{\poly\rbr{n}}^{-1}$.

\subsection{Pixel-wise Classification with Sparse and Dense Labels}\label{subsec:pixel-wise-classification}

We consider medical image segmentation as a pixel-wise multi-class classification problem where we aim to learn a pixel-wise classifier $h:\Xcal \to [K]^d$ that serves as a good approximation to the ground truth $h^*: \Xcal \to [K]^d$. 

Recall the separation of cross-entropy losses between samples with different proportions of background pixels from \Cref{fig:synapse_case40_sparse_dense_ce_weights}.
We refer to a sample $\rbr{\xb,\yb} \in \Xcal \times [K]^d$ as \emph{label-sparse} if most pixels in $\yb$ are labeled as background; for these samples, the cross-entropy loss on $\rbr{\xb,\yb}$ converges rapidly in the early stage of training.
Otherwise, we say that $\rbr{\xb,\yb}$ is \emph{label-dense}.
Formally, we describe such variation as a concept shift in the data distribution.
\begin{definition}[Mixture of label-sparse and label-dense subpopulations]
We assume that \emph{label-sparse and label-dense samples} are drawn from $P_0$ and $P_1$ with distinct conditional distributions $P_0\rsep{\yb}{\xb}$ and $P_1\rsep{\yb}{\xb}$ but common marginal distribution $P\rbr{\xb}$ such that $P_i\rbr{\xb,\yb} = P_i\rsep{\yb}{\xb} P\rbr{\xb}$ ($i=0,1$).
For $\xi \in [0,1]$, we define $P_\xi$ as a data distribution where $\rbr{\xb,\yb} \sim P_\xi$ is drawn either from $P_1$ with probability $\xi$ or from $P_0$ with probability $1-\xi$.
\end{definition}

We aim to learn a pixel-wise classifier from a function class $\Hcal \ni h_\theta = \argmax_{k \in [K]} \iloc{f_\theta\rbr{\xb}}{j,:}$ for all $j \in [d]$ where the underlying function $f_\theta \in \Fcal$, parameterized by some $\theta \in \Fcal_\theta$, admits an encoder-decoder structure:
\begin{align}\label{eq:encoder_decoder}
    \csepp{f_\theta = \phi_\theta \circ \psi_\theta}{\phi_\theta: \Xcal \to \Zcal, \psi_\theta:\Zcal \to [0,1]^{d \times K}}.
\end{align}
Here $\phi_\theta, \psi_\theta$ correspond to the encoder and decoder functions, respectively. The parameter space $\Fcal_\theta$ is equipped with the norm $\nbr{\cdot}_\Fcal$ and its dual norm $\nbr{\cdot}_{\Fcal,*}$\footnote{
For \ours (\Cref{prop:convergence_srw_dac} in \Cref{sec:ada_wac}), $\Fcal_\theta$ is simply a subspace in the Euclidean space with dimension equal to the total number of parameters for each $\theta \in \Fcal_\theta$, with $\nbr{\cdot}_\Fcal$ and $\nbr{\cdot}_{\Fcal,*}$ both being the $\ell_2$-norm.
}. $\rbr{\Zcal, \varrho}$ is a latent metric space. 

To learn from segmentation labels, we consider the \emph{averaged cross-entropy loss}:
\begin{align}\label{eq:cross_entropy_loss}
\begin{split}
    \lossce\rbr{\theta;(\xb,\yb)} 
    = &-\frac{1}{d} \sum_{j=1}^d \sum_{k=1}^K \iffun{\iloc{\yb}{j}=k} \cdot \log\rbr{ \iloc{f_\theta\rbr{\xb}}{j,k} } \\
    = &-\frac{1}{d} \sum_{j=1}^d \log\rbr{\iloc{f_\theta\rbr{\xb}}{j, \iloc{\yb}{j}}}.
\end{split}
\end{align} 
We assume the proper learning setting: there exists $\theta^* \in \bigcap_{\xi \in [0,1]} \argmin_{\theta \in \Fcal_\theta} \E_{\rbr{\xb,\yb}\sim P_\xi}\sbr{\lossce\rbr{\theta;(\xb,\yb)} }$, which is invariant with respect to $\xi$.\footnote{
We assume proper learning only to \begin{enumerate*}[label=(\roman*)]
    \item highlight the invariance of the desired ground truth to $\xi$ that can be challenging to learn with finite samples in practice and
    \item provide a natural pivot for the convex and compact neighborhood $\Fnb{\gamma}$ of ground truth $\theta^*$ in \Cref{ass:separability_sparse_dense} granted by the pretrained initialization, where $\theta^*$ can also be replaced with the pretrained initialization weights $\theta_0 \in \Fnb{\gamma}$.
    In particular, neither our theory nor the \ours algorithm requires the function class $\Fcal$ to be expressive enough to truly contain such $\theta^*$.
\end{enumerate*}
}

\subsection{Augmentation Consistency Regularization}
Despite the invariance of $f_{\theta^*}$ to $P_\xi$ on the population loss, with a finite number of training samples in practice, the predominance of label-sparse samples would be problematic. As an extreme scenario for the pixel-wise classifier with encoder-decoder structure (\Cref{eq:encoder_decoder}), when the label-sparse samples are predominant ($\xi \ll 1$), a decoder function $\psi_\theta$ that predicts every pixel as background can achieve near-optimal cross-entropy loss, regardless of the encoder function $\phi_\theta$, considerably compromising the test performance (\cf \Cref{tab:synapse_sample_eff}). 
To encourage legit encoding even in the absence of sufficient dense labels, we leverage the unsupervised consistency regularization on the \emph{encoder function} $\phi_\theta$ based on data augmentations. 

Let $\Acal$ be a distribution over transformations on $\Xcal$ where for any $\xb \in \Xcal$, each $A \sim \Acal$ ($A:\Xcal \to \Xcal$) induces an augmentation $A\rbr{\xb}$ of $\xb$ that perturbs low-level information in $\xb$. 
We aim to learn an encoder function $\phi_\theta:\Xcal \to \Zcal$ that is capable of filtering out low-level information from $\xb$ and therefore provides similar encodings for augmentations of the same sample. 
Recalling the metric $\varrho$ (\eg, the Euclidean distance) on $\Zcal$, for a given scaling hyperparameter $\lambdac>0$, we measure the similarity between augmentations with a consistency regularization term on $\phi_\theta\rbr{\cdot}$: for any $A_1,A_2 \sim \Acal^2$,
\begin{align}\label{eq:dac}
    \regdac\rbr{\theta;\xb,A_1,A_2} \dfeq \lambdac \cdot \varrho\Big( \phi_\theta\rbr{A_1(\xb)}, \phi_\theta\rbr{A_2(\xb)} \Big).
\end{align}

For the $n$ training samples $\cbr{\rbr{\xb_i, \yb_i}}_{i \in [n]}\sim P_\xi^n$, we consider $n$ pairs of data augmentation transformations $\cbr{\rbr{A_{i,1}, A_{i,2}}}_{i \in [n]} \sim \Acal^{2n}$.
In the basic version, we encourage the similar encoding $\phi_\theta\rbr{\cdot}$ of the augmentation pairs $\rbr{{A_{i,1}\rbr{\xb_i}}, {A_{i,2}\rbr{\xb_i}}}$ for all $i \in [n]$ via consistency regularization:
\begin{align}\label{eq:objective_dac_encoder_plain}
    \min_{\theta \in \Fnb{\gamma}} \frac{1}{n} \sum_{i=1}^n 
    \lossce\rbr{\theta;\rbr{\xb_i,\yb_i}} + 
    \regdac\rbr{\theta;\xb_i,A_{i,1},A_{i,2}}.
\end{align}

We enforce consistency on $\phi_\theta\rbr{\cdot}$ in light of the encoder-decoder architecture: the encoder is generally designed to abstract essential information and filters out low-level non-semantic perturbations (\eg, those introduced by augmentations), while the decoder recovers the low-level information for the pixel-wise classification.
Therefore, with different $A_1, A_2 \sim \Acal$, the encoder output $\phi_\theta\rbr{\cdot}$ tends to be more consistent than the other intermediate layers, especially for label-dense samples.

\section{Weighted Augmentation Consistency (WAC) Regularization}\label{sec:wac}

As the motivation, we begin with a key observation about the averaged cross-entropy:  
\begin{remark}[Separation of averaged cross-entropy loss on $P_0$ and $P_1$]\label{rmk:separation_average_cross_entropy}
    As demonstrated in \Cref{fig:synapse_case40_sparse_dense_ce_weights}, the sparse labels from $P_0$ tend to be much easier to learn than the dense ones from $P_1$, leading to considerable separation of averaged cross-entropy losses on the sparse and dense labels after a sufficient number of training epochs.  In other words, $\lossce\rbr{\theta;\rbr{\xb,\yb}} \ll \lossce\rbr{\theta;\rbr{\xb',\yb'}}$ for label-sparse samples $\rbr{\xb,\yb} \sim P_0$ and label-dense samples $\rbr{\xb',\yb'} \sim P_1$ with high probability.
\end{remark}

Although \Cref{eq:objective_dac_encoder_plain} with consistency regularization alone can boost the segmentation accuracy during testing (\cf \Cref{tab:ablation}), it does not take the separation between label-sparse and label-dense samples into account. In \Cref{sec:experiments_adawac}, we will empirically demonstrate that proper exploitation of such separation, like the formulation introduced below, can lead to improved classification performance.

We formalize the notion of separation between $P_0$ and $P_1$ with the consistency regularization (\Cref{eq:dac}) as a reference in the following assumption
\footnote{
Although \Cref{ass:separability_sparse_dense} may seem to be rather strong, it is only required for the separation guarantee of label-sparse and label-dense samples with high probability in \Cref{prop:spontaneous_separation_sparse_dense}, but not for the adaptive weighting algorithm introduced in \Cref{sec:ada_wac} or in practice for the experiments.
}.

\begin{assumption}[$n$-separation between $P_0$ and $P_1$]\label{ass:separability_sparse_dense}
    Given a sufficiently small $\gamma>0$, let $\Fnb{\gamma} = \csepp{\theta \in \Fcal_\theta}{\nbr{\theta - \theta^*}_{\Fcal} \le \gamma}$ be a compact and convex neighborhood of well-trained pixel-wise classifiers\footnote{With pretrained initialization, we assume that the optimization algorithm is always probing in $\Fnb{\gamma}$.}. We say that \emph{$P_0$ and $P_1$ are $n$-separated over $\Fnb{\gamma}$} if there exists $\omega>0$ such that with probability $1-\Omega\rbr{n^{1+\omega}}^{-1}$ over $\rbr{\rbr{\xb,\yb},\rbr{A_1,A_2}} \sim P_\xi \times \Acal^2$, the following hold:
    \begin{enumerate}[label=(\roman*)]
        \item $\lossce\rbr{\theta;\rbr{\xb,\yb}} < \regdac\rbr{\theta;\xb,A_1,A_2}$ for all $\theta \in \Fnb{\gamma}$ given $\rbr{\xb,\yb} \sim P_0$;
        \item $\lossce\rbr{\theta;\rbr{\xb,\yb}} > \regdac\rbr{\theta;\xb,A_1,A_2}$ for all $\theta \in \Fnb{\gamma}$ given $\rbr{\xb,\yb} \sim P_1$.
    \end{enumerate}
\end{assumption}
This assumption is motivated by the empirical observation that the perturbation in $\phi_\theta\rbr{\cdot}$ induced by $\Acal$ is more uniform across $P_0$ and $P_1$ than the averaged cross-entropy losses, as instantiated in \Cref{fig:synapse-weights}. 

Under \Cref{ass:separability_sparse_dense}, up to a proper scaling hyperparameter $\lambdac$, the consistency regularization (\Cref{eq:dac}) can separate the averaged cross-entropy loss (\Cref{eq:cross_entropy_loss}) on $n$ label-sparse and label-dense samples with probability $1-\Omega\rbr{n^{\omega}}^{-1}$ (as explained formally in \Cref{subapx:spontaneous_separation_sparse_dense}). 
In particular, the larger $n$ corresponds to the stronger separation between $P_0$ and $P_1$. 

With \Cref{ass:separability_sparse_dense}, we introduce a minimax formulation that incentivizes the separation of label-sparse and label-dense samples automatically 
by introducing a flexible weight $\iloc{\betab}{i} \in [0,1]$ that balances $\lossce\rbr{\theta;\rbr{\xb_i,\yb_i}}$ and $\regdac\rbr{\theta;\xb_i,A_{i,1},A_{i,2}}$ for each of the $n$ samples.
\begin{align}\label{eq:objective_dac_encoder_weighted_adawac}
\begin{split}
    &\wh\theta^\wdac, \wh\betab \in 
    \argmin_{\theta \in \Fnb{\gamma}}~
    \argmax_{\betab \in [0,1]^n}~ 
    \wh L^\wdac\rbr{\theta,\betab} 
    \\
    \wh L^\wdac\rbr{\theta,\betab} \dfeq &\frac{1}{n} \sum_{i=1}^n \iloc{\betab}{i} \cdot \lossce\rbr{\theta;\rbr{\xb_i,\yb_i}} + (1-\iloc{\betab}{i}) \cdot \regdac\rbr{\theta;\xb_i,A_{i,1},A_{i,2}}. 
\end{split}
\end{align}

With convex and continuous loss and regularization terms (formally in \Cref{prop:spontaneous_separation_sparse_dense}), \Cref{eq:objective_dac_encoder_weighted_adawac} admits a saddle point corresponding to $\wh\betab$ which separates the label-sparse and label-dense samples under \Cref{ass:separability_sparse_dense}.
\begin{proposition}[Formal proof in \Cref{subapx:spontaneous_separation_sparse_dense}]\label{prop:spontaneous_separation_sparse_dense}
    Assume that 
    $\lossce\rbr{\theta;\rbr{\xb,\yb}}$ and $\regdac\rbr{\theta;\xb,A_{1},A_{2}}$ are convex and continuous in $\theta$ for all $(\xb,\yb) \in \Xcal \times [K]^d$ and $A_1,A_2 \sim \Acal^2$; $\Fnb{\gamma} \subset \Fcal_\theta$ is compact and convex.
    If $P_0$ and $P_1$ are $n$-separated (\Cref{ass:separability_sparse_dense}), then there exists $\wh\betab \in \cbr{0,1}^n$ and $\wh\theta^\wdac \in \argmin_{\theta \in \Fnb{\gamma}} \wh L^\wdac\rbr{\theta, \wh\betab}$ such that 
    \begin{align}\label{eq:saddle_point_def}
    \begin{split}
        \min_{\theta \in \Fnb{\gamma}} \wh L^\wdac\rbr{\theta, \wh\betab} 
        = &\wh L^\wdac\rbr{\wh\theta^\wdac, \wh\betab} 
        = \max_{\betab \in [0,1]^n} \wh L^\wdac\rbr{\wh\theta^\wdac, \betab}.
    \end{split}
    \end{align}
    Further, $\wh\betab$ separates the label-sparse and label-dense samples---$\iloc{\wh\betab}{i}=\iffun{\rbr{\xb_i,\yb_i} \sim P_1}$---\whp. 
\end{proposition}
That is, for $n$ samples drawn from a mixture of $n$-separated $P_0$ and $P_1$, the saddle point of $L^\wdac_i \rbr{\theta,\betab}$ in \Cref{eq:objective_dac_encoder_weighted_adawac} corresponds to $\iloc{\betab}{i} = 0$ on label-sparse samples (\ie, learning from the unsupervised consistency regularization), and $\iloc{\betab}{i} = 1$ on label-dense samples (\ie, learning from the supervised averaged cross-entropy loss).

\begin{remark}[Connection to hard-thresholding algorithms]\label{rmk:relation_hard_thresholding}
    The saddle point of \Cref{eq:objective_dac_encoder_weighted_adawac} is closely related to hard-thresholding algorithms like Ordered SGD~\citep{kawaguchi2020ordered} and iterative trimmed loss~\citep{shen2019learning}. 
    In each iteration, these algorithms update the model only on a proper subset of training samples based on the (ranking of) current empirical risks. 
    Compared to hard-thresholding algorithms,  
    \begin{enumerate*}[label=(\roman*)]
        \item \Cref{eq:objective_dac_encoder_weighted_adawac} additionally leverages the unused samples (\eg, label-sparse samples) for unsupervised consistency regularization on data augmentations;
        \item meanwhile, it does not require prior knowledge of the sample subpopulations (\eg, $\xi$ for $P_\xi$) which is essential for hard-thresholding algorithms. 
    \end{enumerate*}

    \Cref{eq:objective_dac_encoder_weighted_adawac} further facilitates the more flexible optimization process. As we will empirically show in \Cref{tab:trim}, despite the close relation between \Cref{eq:objective_dac_encoder_weighted_adawac} and the hard-thresholding algorithms (\Cref{rmk:relation_hard_thresholding}), such updating strategies may be suboptimal for solving \Cref{eq:objective_dac_encoder_weighted_adawac}.
\end{remark}

\section{Adaptively Weighted Augmentation Consistency (\ours)}\label{sec:ada_wac}

Inspired by the breakthrough made by \cite{sagawa2019distributionally} in the distributionally-robust optimization (DRO) setting where gradient updating on weights is shown to enjoy better convergence guarantees than hard thresholding, we introduce an adaptively weighted online optimization algorithm (\Cref{alg:srw_dac}) for solving \Cref{eq:objective_dac_encoder_weighted_adawac} based on online mirror descent.

In contrast to the commonly used stochastic gradient descent (SGD), the flexibility of online mirror descent in choosing the associated norm space not only allows gradient updates on sample weights but also grants distinct learning dynamics to sample weights $\betab_t$ and model parameters $\theta_t$, which leads to the following convergence guarantee.
\begin{proposition}[Formally in \Cref{prop:convergence_srw_dac_formal}, proof in \Cref{subapx:convergence_wdac}, assumptions instantiated in \Cref{ex:convex_continuous}]\label{prop:convergence_srw_dac}
    Assume that 
    $\lossce\rbr{\theta;\rbr{\xb,\yb}}$ and $\regdac\rbr{\theta;\xb,A_{1},A_{2}}$ are convex and continuous in $\theta$ for all $(\xb,\yb) \in \Xcal \times [K]^d$ and $A_1,A_2 \sim \Acal^2$.  Assume moreover that  $\Fnb{\gamma} \subset \Fcal_\theta$ is convex and compact. If there exist
    \footnote{
        Following the convention, we use $*$ in subscription to denote the dual spaces. For instance, recalling the parameter space $\Fcal_\theta$ characterized by the norm $\nbr{\cdot}_{\Fcal}$ from \Cref{subsec:pixel-wise-classification}, we use $\nbr{\cdot}_{\Fcal,*}$ to denote its dual norm; while $C_{\theta,*}, C_{\betab,*}$ upper bound the dual norms of the gradients with respect to $\theta$ and $\betab$.
    } 
    $C_{\theta,*} > 0$ and $C_{\betab,*} > 0$ such that 
    \begin{align*}
        &\frac{1}{n} \sum_{i=1}^n \nbr{\nabla_\theta \wh L_i^\wdac\rbr{\theta,\betab}}_{\Fcal,*}^2 \le C_{\theta,*}^2 \\
        &\frac{1}{n} \sum_{i=1}^n \max\cbr{\lossce\rbr{\theta;\rbr{\xb_i,\yb_i}}, \regdac\rbr{\theta;\xb_i,A_{i,1},A_{i,2}}}^2 
        \le C_{\betab,*}^2
    \end{align*}
    for all $\theta \in \Fnb{\gamma}$, $\betab \in [0,1]^n$,
    then with $\eta_\theta = \eta_{\betab} = \frac{2}{\sqrt{5T \rbr{\gamma^2 C_{\theta,*}^2 + 2 n C_{\betab,*}^2}}}$, \Cref{alg:srw_dac} provides
    \begin{align*}
        \E\sbr{\max_{\betab \in [0,1]^n} \wh L^\wdac\rbr{\overline{\theta}_T, \betab} - \min_{\theta \in \Fnb{\gamma}} \wh L^\wdac\rbr{\theta, \overline{\betab}_T}}& \\
        \le 2 \sqrt{5 \rbr{\gamma^2 C_{\theta,*}^2 + 2 n C_{\betab,*}^2} \Big/ T}&
    \end{align*}
    where $\overline{\theta}_T = \frac{1}{T} \sum_{t=1}^T \theta_t$ and $\overline{\betab}_T = \frac{1}{T} \sum_{t=1}^T \betab_t$. 
\end{proposition} 

\begin{algorithm}[h]
\caption{Adaptively Weighted Augmentation Consistency (\ours)}\label{alg:srw_dac}
\begin{algorithmic}
    \STATE {\bfseries Input:} 
    Training samples $\cbr{\rbr{\xb_i, \yb_i}}_{i \in [n]}\sim P_\xi^n$, 
    augmentations $\cbr{\rbr{A_{i,1}, A_{i,2}}}_{i \in [n]} \sim \Acal^{2n}$,
    maximum number of iterations $T \in \N$, 
    learning rates $\eta_\theta, \eta_{\betab} > 0$, 
    pretrained initialization for the pixel-wise classifier $\theta_0 \in \Fnb{\gamma}$.
    
    \STATE Initialize the sample weights $\betab_0 = \b{1}/2 \in [0,1]^n$.
    
    \FOR{$t = 1,\dots,T$}
        \STATE Sample $i_t \sim [n]$ uniformly
        \STATE $\bb \gets \sbr{\iloc{\rbr{\betab_{t-1}}}{i_t}, 1-\iloc{\rbr{\betab_{t-1}}}{i_t}}$
        \STATE $\iloc{\bb}{1} \gets \iloc{\bb}{1} \cdot \exp\rbr{\eta_{\betab} \cdot \lossce\rbr{\theta_{t-1};\rbr{\xb_{i_t},\yb_{i_t}}}}$
        \STATE $\iloc{\bb}{2} \gets \iloc{\bb}{2} \cdot \exp\rbr{ \eta_{\betab} \cdot \regdac\rbr{\theta_{t-1};\xb_{i_t},A_{i_t,1},A_{i_t,2}} }$
        \STATE $\betab_t \gets \betab_{t-1}$, $\iloc{\rbr{\betab_t}}{i_t} \gets \iloc{\bb}{1}/\nbr{\bb}_1$
        \STATE $\theta_t \gets \theta_{t-1} - \eta_\theta \cdot \Big(
        \iloc{\rbr{\betab_t}}{i_t} \cdot \nabla_\theta\lossce\rbr{\theta_{t-1};\rbr{\xb_{i_t},\yb_{i_t}}} + \rbr{1-\iloc{\rbr{\betab_t}}{i_t}} \cdot \nabla_\theta\regdac\rbr{\theta_{t-1};\xb_{i_t},A_{i_t,1},A_{i_t,2}} \Big)$
    \ENDFOR
\end{algorithmic}
\end{algorithm}

In addition to the convergence guarantee, \Cref{alg:srw_dac} also demonstrates superior performance over hard-thresholding algorithms for segmentation problems in practice (\Cref{tab:trim}).
An intuitive explanation is that instead of filtering out all the label-sparse samples via hard thresholding, the adaptive weighting allows the model to learn from some sparse labels at the early epochs, while smoothly down-weighting $\lossce$ of these samples since learning sparse labels tends to be easier (\Cref{rmk:separation_average_cross_entropy}). 
With the learned model tested on a mixture of label-sparse and label-dense samples, learning sparse labels at the early stage is crucial for accurate segmentation.

\section{Experiments}\label{sec:experiments_adawac}
In this section, we investigate the proposed \ours algorithm (\Cref{alg:srw_dac}) on different medical image segmentation tasks with different UNet-style architectures. We first demonstrate the performance improvements brought by \ours in terms of sample efficiency and robustness to concept shift (\Cref{tab:synapse_sample_eff}). Then, we verify the empirical advantage of \ours compared to the closely related hard-thresholding algorithms as discussed in \Cref{rmk:relation_hard_thresholding} (\Cref{tab:trim}). Our ablation study (\Cref{tab:ablation}) further illustrates the indispensability of both sample reweighting and consistency regularization, the deliberate combination of which leads to the superior performance of \ours\footnote{We release our code anonymously at \href{https://anonymous.4open.science/r/adawac-F5F8}{https://anonymous.4open.science/r/adawac-F5F8.}}.

\paragraph{Experiment setup.} 
We conduct experiments on two medical image segmentation tasks: abdominal CT segmentation for Synapse multi-organ dataset (Synapse)\footnote{\href{https://www.synapse.org/\#!Synapse:syn3193805/wiki/217789}{{https://www.synapse.org/\#!Synapse:syn3193805/wiki/217789}}} and cine-MRI segmentation for Automated cardiac diagnosis challenge dataset (ACDC)\footnote{\href{https://www.creatis.insa-lyon.fr/Challenge/acdc/}{https://www.creatis.insa-lyon.fr/Challenge/acdc/}}, with two UNet-like architectures: TransUNet~\citep{chen2021transunet} and UNet~\cite{ronneberger2015unet} (deferred to \Cref{subsec:exp_unet}). For the main experiments with TransUNet in \Cref{sec:experiments_adawac}, we follow the official implementation in \citep{chen2021transunet} and use ERM+SGD as the baseline. We evaluate segmentations with two standard metrics---the average Dice-similarity coefficient (DSC) and the average $95$-percentile of Hausdorff distance (HD95). Dataset and implementation details are deferred to \Cref{app:imp}. Given the sensitivity of medical image semantics to perturbations, our experiments only involve simple augmentations (\ie, rotation and mirroring) adapted from \citep{chen2021transunet}.

It is worth highlighting that, in addition to the information imbalance among samples caused by the concept shift discussed in this work, the pixel-wise class imbalance (\eg, the predominance of background pixels) is another well-investigated challenge for medical image segmentation, where coupling the dice loss~\citep{wong2018segmentation,taghanaki2019combo,yeung2022unified} in the objective is a common remedy used in many state-of-the-art methods~\citep{chen2021transunet,cao2021swin}. The implementation of \ours also leverages the dice loss to alleviate pixel-wise class imbalance. We defer the detailed discussion to \Cref{apx:dice}.

\subsection{Segmentation Performance of \ours with TransUNet}\label{subsec:exp_transunet}
\paragraph{Segmentation on Synapse.}
\Cref{fig:synapse} visualizes the segmentation predictions on $6$ Synapse test slices given by models trained via \ours(ours) and via the baseline (ERM+SGD) with TransUNet~\citep{chen2021transunet}. We observe that \ours provides more accurate predictions on the segmentation boundaries and captures small organs better than the baseline.

\begin{figure}[ht]
    \centering
    \includegraphics[width=\linewidth]{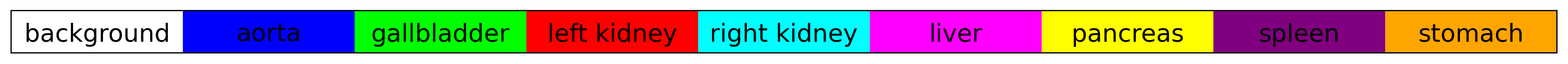}
    \includegraphics[width=\linewidth]{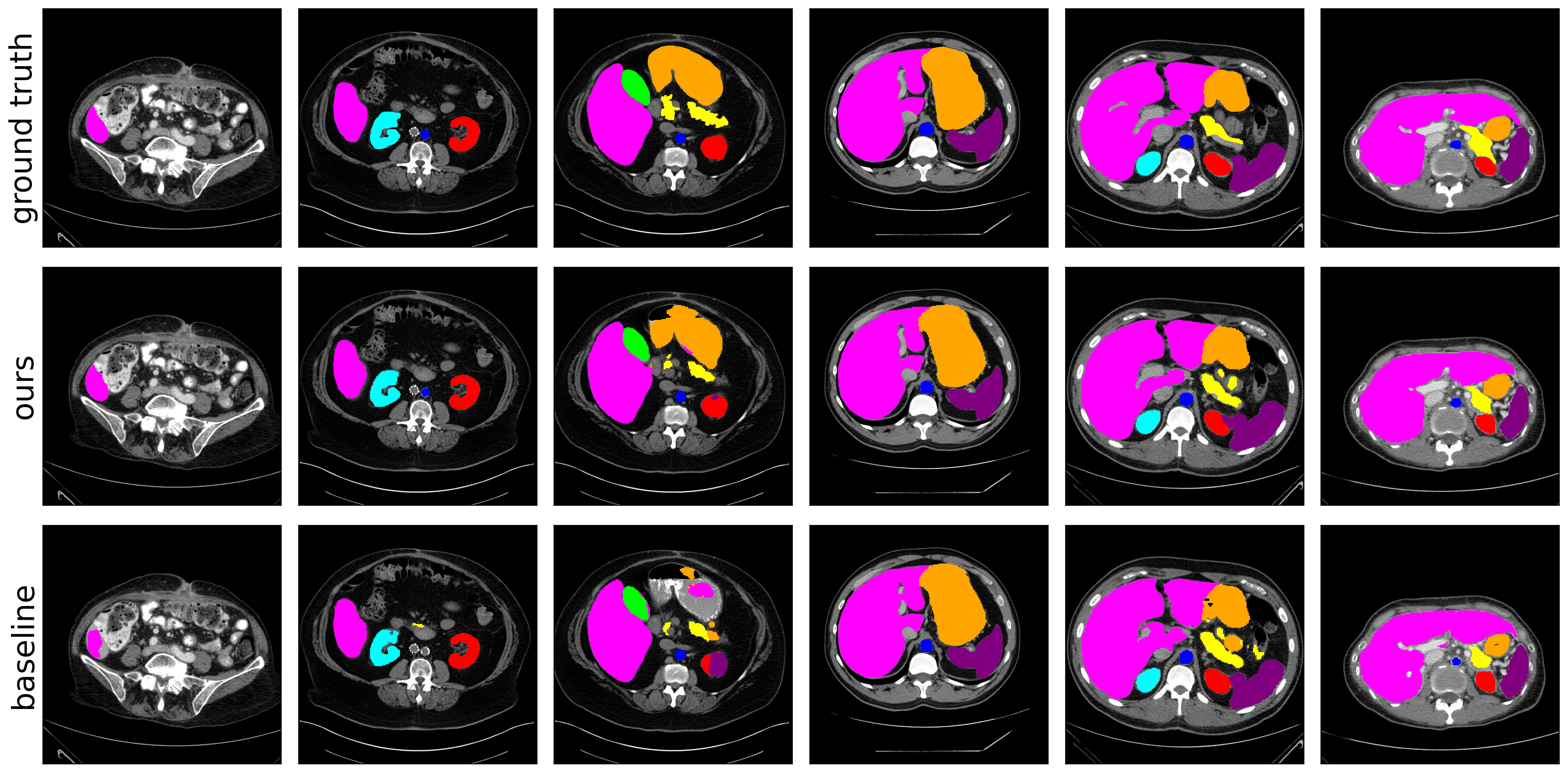}
    \caption{Visualization of segmentation predictions with TransUNet~\citep{chen2021transunet} on Synapse. Top to bottom: ground truth, ours (\ours), baseline.}
    \label{fig:synapse}
\end{figure}

\paragraph{Visualization of \ours.} 
As shown in \Cref{fig:synapse-weights}, with $\lossce\rbr{\theta_t;(\xb_i,\yb_i)}$ (\Cref{eq:cross_entropy_loss}) of label-sparse versus label-dense slices weakly separated in the early epochs, the model further learns to distinguish $\lossce\rbr{\theta_t;(\xb_i,\yb_i)}$ of label-sparse/label-dense slices during training. By contrast, $\regdac\rbr{\theta_t;\xb_i,A_{i,1},A_{i,2}}$ (\Cref{eq:dac}) remains mixed for all slices throughout the entire training process. As a result, the CE weights of label-sparse slices are much smaller than those of label-dense ones, pushing \ours to learn more image representations but fewer pixel classifications for slices with sparse labels and learn more pixel classifications for slices with dense labels.

\begin{figure}[ht]
    \centering
    \includegraphics[width=\linewidth]{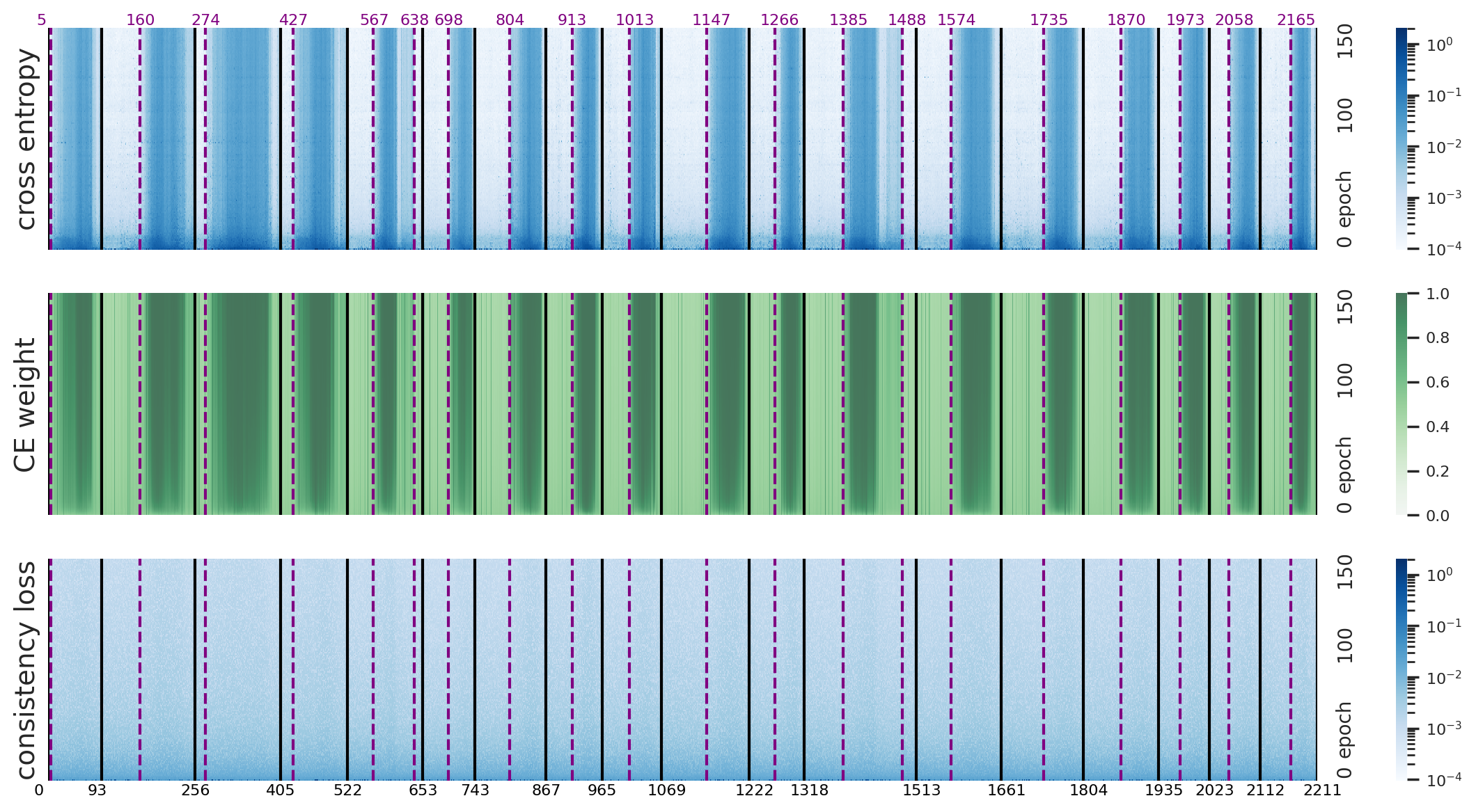}
    \caption{$\lossce\rbr{\theta_t;(\xb_i,\yb_i)}$ (top), CE weights $\betab_t$ (middle), and $\regdac\rbr{\theta_t;\xb_i,A_{i,1},A_{i,2}}$ (bottom) of the entire Synapse training process. The x-axis indexes slices 0--2211. The y-axis enumerates epochs 0--150. Individual cases (patients) are partitioned by black lines, while purple lines separate slices with/without non-background pixels.}
    \label{fig:synapse-weights}
\end{figure}

\paragraph{Sample efficiency and robustness.} 
We first demonstrate the \emph{sample efficiency} of \ours in comparison to the baseline (ERM+SGD) when training only on different subsets of the full Synapse training set (``\b{full}'' in \Cref{tab:synapse_sample_eff}). Specifically, 
\begin{enumerate*}[label=(\roman*)]
    \item  \b{half-slice} contains slices with even indices only in each case (patient)\footnote{Such sampling is equivalent to doubling the time interval between two consecutive scans or halving the scanning frequency in practice, resulting in the halving of sample size.};
    \item \b{half-vol} consists of 9 cases uniformly sampled from the total 18 cases in \b{full} where different cases tend to have distinct $\xi$s (\ie, ratios of label-dense samples);
    \item \b{half-sparse} takes the first half slices in each case, most of which tend to be label-sparse (\ie, $\xi$s are made to be small).
\end{enumerate*}
As shown in \Cref{tab:synapse_sample_eff}, the model trained with \ours on \b{half-slice} generalizes as well as a baseline model trained on \b{full}, if not better.
Moreover, the \b{half-vol} and \b{half-sparse} experiments illustrate the \emph{robustness} of \ours to concept shift. 
Furthermore, such sample efficiency and distributional robustness of \ours extend to the more widely used UNet architecture. We defer the detailed results and discussions on UNet to \Cref{subsec:exp_unet}.

\begin{table*}[!ht]
    \caption{\ours with TransUNet trained on the full Synapse and its subsets. 
    }\vspace{-.5em}
    \label{tab:synapse_sample_eff}
    \centering
    \begin{adjustbox}{width=1\textwidth}  
    \begin{tabular}{ll|cc|cccccccc}
    \toprule
    Training & Method &  DSC $\uparrow$ & HD95 $\downarrow$ & Aorta & Gallbladder & Kidney (L) & Kidney (R) & Liver & Pancreas & Spleen & Stomach 
    \\
    \midrule
    \multirow{2}{*}{full} 
    & baseline & 76.66 $\pm$ 0.88 & 29.23 $\pm$ 1.90 & 87.06 & 55.90 & 	81.95 & 75.58 & 94.29 & 56.30 & 86.05 & 76.17
    \\
    & \ours & \best{79.04 $\pm$ 0.21} & \best{27.39 $\pm$ 1.91} & 87.53 & 56.57 & 83.23 & 81.12 & 94.04 & 62.05 & 89.51 & 78.32
    \\
    \midrule
    \multirow{2}{*}{half-slice} 
    & baseline & 74.62 $\pm$ 0.78 & 31.62 $\pm$ 8.37 & 86.14 & 44.23 & 79.09 & 78.46 & 93.50 & 55.78 & 84.54 & 75.24
    \\
    & \ours & \best{77.37 $\pm$ 0.40} & \best{29.56 $\pm$ 1.09} & 86.89 & 55.96 & 82.15 & 78.63 & 94.34 & 57.36 & 86.60 & 77.05
    \\
    \midrule
    \multirow{2}{*}{half-vol} 
    & baseline & 71.08 $\pm$ 0.90 & 46.83 $\pm$ 2.91 & 84.38 & 46.71 & 78.19 & 74.55 & 92.02 & 48.03 & 76.28 & 68.47
    \\
    & \ours & \best{73.81 $\pm$ 0.94} & \best{35.33 $\pm$ 0.92} & 84.37 & 48.14 & 80.32 & 77.39 & 93.23 & 52.78 & 83.50 & 70.79
    \\
    \midrule
    \multirow{2}{*}{half-sparse} 
    & baseline & 31.74 $\pm$ 2.78 & 69.72 $\pm$ 1.37 & 65.71 & 8.33 & 59.46 & 51.59 & 51.18 & 10.72 & 6.92 & 0.00
    \\
    & \ours & \best{41.03 $\pm$ 2.12} & \best{59.04 $\pm$ 12.32} & 71.27 & 8.33 & 69.14 & 63.09 & 64.29 & 17.74 & 30.77 & 3.57
    \\
    \bottomrule
    \end{tabular}
    \end{adjustbox}
\end{table*}

\paragraph{Comparison with hard-thresholding algorithms.}
\Cref{tab:trim} illustrates the empirical advantage of \ours over the hard-thresholding algorithms, as suggested in \Cref{rmk:relation_hard_thresholding}. In particular, we consider the following hard-thresholding algorithms:
\begin{enumerate*}[label=(\roman*)]
    \item \b{trim-train} learns only from slices with at least one non-background pixel and trims the rest in each iteration on the fly;
    \item \b{trim-ratio} ranks the cross-entropy loss $\lossce\rbr{\theta_t;(\xb_i,\yb_i)}$ in each iteration (mini-batch) and trims samples with the lowest cross-entropy losses at a fixed ratio -- the ratio of all-background slices in the full training set ($1-\frac{1280}{2211} \approx 0.42$); 
    \item \b{ACR} further incorporates the data augmentation consistency regularization directly via the addition of $\regdac\rbr{\theta_t;\xb_i,A_{i,1},A_{i,2}}$ without reweighting;
    \item \b{pseudo-\ours} simulates the sample weights $\betab$ at the saddle point and learns via $\lossce\rbr{\theta_t;(\xb_i,\yb_i)}$ on slices with at least one non-background pixel while via $\regdac\rbr{\theta_t;\xb_i,A_{i,1},A_{i,2}}$ otherwise.
\end{enumerate*}
We see that naive incorporation of \b{ACR} brings less observable boosts to the hard-thresholding methods. Therefore, the deliberate combination via reweighting in \ours is essential for performance improvement.

\begin{table*}[!ht]
    \caption{\ours versus hard-thresholding algorithms with TransUNet on Synapse.}\vspace{-.5em}
    \label{tab:trim}
    \centering
    \begin{adjustbox}{width=\textwidth}
    \begin{tabular}{l|ccccccc}
    \toprule
    \multirow{2}{*}{Method} & \multirow{2}{*}{baseline} & \multicolumn{2}{c}{trim-train} & \multicolumn{2}{c}{trim-ratio} & \multirow{2}{*}{pseudo-\ours}  & \multirow{2}{*}{\ours} 
    \\ 
    \cline{3-6} 
    & & & +ACR &  & +ACR & \\
    \midrule
    DSC $\uparrow$ & 76.66 $\pm$ 0.88 & 76.80 $\pm$ 1.13 & 78.42 $\pm$ 0.17 & 76.49 $\pm$ 0.16 & 77.71 $\pm$ 0.56 & 77.72 $\pm$ 0.65 & \best{79.04 $\pm$ 0.21} \\
    HD95 $\downarrow$ & 29.23 $\pm$ 1.90 & 32.05 $\pm$ 2.34 & 27.84 $\pm$ 1.16 & 31.96 $\pm$ 2.60 & 28.51 $\pm$ 2.66 & 28.45 $\pm$ 1.18 & \best{27.39 $\pm$ 1.91} \\
    \bottomrule
    \end{tabular}
    \end{adjustbox}
\end{table*}

\paragraph{Segmentation on ACDC.}
Performance improvements granted by \ours are also observed on the ACDC dataset (\Cref{tab:acdc}). We defer detailed visualization of ACDC segmentation to \Cref{apx:additional_experiment}.
\begin{table*}[!ht]
    \caption{\ours with TransUNet trained on ACDC.}\vspace{-.5em}
    \label{tab:acdc}
    \centering
    \begin{adjustbox}{width=.6\textwidth} 
    \begin{tabular}{l|cc|ccc}
    \toprule
    Method & DSC $\uparrow$ & HD95 $\downarrow$  & RV & Myo & LV \\
    \midrule
    TransUNet &	89.40 $\pm$ 0.22 & 2.55 $\pm$ 0.37 & 89.17 & 83.24 & 95.78 \\
    \ours (ours) &	\best{90.67 $\pm$ 0.27} & \best{1.45 $\pm$ 0.55} & 90.00 & 85.94 & 96.06 \\
    \bottomrule
    \end{tabular}
    \end{adjustbox}
\end{table*}

\subsection{Ablation Study}
\paragraph{On the influence of consistency regularization.} 
To illustrate the role of consistency regularization in \ours, we consider the \b{reweight-only} scenario with $\lambdac=0$ such that $\regdac\rbr{\theta_t;\xb_i,A_{i,1},A_{i,2}} \equiv 0$ and therefore $\iloc{\bb}{2}$ (\Cref{alg:srw_dac} line 7) remains intact. With zero consistency regularization in \ours, reweighting alone brings little improvement (\Cref{tab:ablation}).

\paragraph{On the influence of sample reweighting.} 
We then investigate the effect of sample reweighting under different reweighting learning rates $\eta_{\betab}$ (recall \Cref{alg:srw_dac}):
\begin{enumerate*}[label=(\roman*)]
    \item \b{ACR-only} for $\eta_{\betab} = 0$ (equivalent to the naive addition of $\regdac\rbr{\theta_t;\xb_i,A_{i,1},A_{i,2}}$),
    \item \b{\ours-0.01} for $\eta_{\betab} = 0.01$, and 
    \item \b{\ours-1.0} for $\eta_{\betab} = 1.0$.
\end{enumerate*} 
As \Cref{tab:ablation} implies, when removing reweighting from \ours, augmentation consistency regularization alone improves DSC slightly from $76.28$ (baseline) to $77.89$ (ACR-only), whereas \ours boosts DSC to $79.12$ (\ours-1.0) with a proper choice of $\eta_{\betab}$.

\begin{table*}[!ht]
    \caption{Ablation study of \ours with TransUNet trained on Synapse.}\vspace{-.5em}
    \label{tab:ablation}
    \centering
    \begin{adjustbox}{width=1\textwidth}  
    \begin{tabular}{l|cc|cccccccc}
    \toprule
    Method &  DSC $\uparrow$ & HD95 $\downarrow$ & Aorta & Gallbladder & Kidney (L) & Kidney (R) & Liver & Pancreas & Spleen & Stomach \\
    \midrule
    baseline &	76.66 $\pm$ 0.88 & 29.23 $\pm$ 1.90 & 87.06 & 55.90 & 	81.95 & 75.58 & 94.29 & 56.30 & 86.05 & 76.17 \\
    reweight-only & 76.91 $\pm$ 0.88 & 30.92 $\pm$ 2.37 & 87.18 & 52.89	& 82.15 & 77.11	& 94.15 & 58.35	& 86.36 & 77.08 \\
    ACR-only & 78.01 $\pm$ 0.62 & 27.78 $\pm$ 2.80 & 87.51 & 58.79 & 83.39 & 79.26 & 94.70 & 58.99 & 86.02 & 75.43 \\
    \ours-0.01 & 77.75 $\pm$ 0.23 &	28.02 $\pm$ 3.50 & 87.33 & 56.68 & 83.35 & 78.53 & 94.45 & 57.02 & 87.72 & 76.94 \\
    \ours-1.0 & \best{79.04 $\pm$ 0.21} & \best{27.39 $\pm$ 1.91} & 87.53 & 56.57 & 83.23 & 81.12 & 94.04 & 62.05 & 89.51 & 78.32 \\
    \bottomrule
    \end{tabular}
    \end{adjustbox}
\end{table*}


\section{Discussion}

In this paper, we explore the information imbalance commonly observed in medical image segmentation and exploit the information in features of label-sparse samples via \ours, an adaptively weighted online optimization algorithm. 
\ours can be viewed as a careful combination of adaptive sample reweighting and data augmentation consistency regularization.
By casting the information imbalance among samples as a concept shift in the data distribution, we leverage the unsupervised data augmentation consistency regularization on the encoder layer outputs (of UNet-style architectures) as a natural reference for distinguishing the label-sparse and label-dense samples via the comparisons against the supervised average cross-entropy loss.
We formulate such comparisons as a weighted augmentation consistency (WAC) regularization problem and propose \ours for iterative and smooth separation of samples from different subpopulations with a convergence guarantee.
Our experiments on various medical image segmentation tasks with different UNet-style architectures empirically demonstrate the effectiveness of \ours not only in improving the segmentation performance and sample efficiency but also in enhancing the distributional robustness to concept shifts.

\paragraph{Limitations and future directions.}
From an algorithmic perspective, a limitation of this work is the utilization of the encoder layer outputs $\phi_\theta\rbr{\cdot}$ for data augmentation consistency regularization, which resulted in \ours being specifically tailored to UNet-style backbones. However, our method can be generalized to other architectures in principle by selecting a representation extractor in the network that
\begin{enumerate*}[label=(\roman*)]
    \item well characterizes the marginal distribution of features $P\rbr{\xb}$
    \item while being robust to the concept shift in $P\rsep{\yb}{\xb}$.
\end{enumerate*}
Further investigation into such generalizations is a promising avenue for future research.

Meanwhile, noticing the prevalence of concept shifts in natural data, especially for dense prediction tasks like segmentation and detection, we hope to extend the application/idea of \ours beyond medical image segmentation as a potential future direction.


\appendices     
\chapter{Appendix for \Cref{ch:svra}}

\section{Technical Lemmas}\label{subapx:technical_proofs}
\begin{lemma}\label{lemma:sample_to_population_covariance}
Let $\xb \in \R^d$ be a random vector with $\E[\xb]=\b{0}$, $\E[\xb\xb^{\top}] = \Sigmab$, and $\overline{\xb} = \Sigmab^{-1/2} \xb$ 
\footnote{In the case where $\Sigmab$ is rank-deficient, we slightly abuse the notation such that $\Sigmab^{-1/2}$ and $\Sigmab^{-1}$ refer to the respective pseudo-inverses.}
being $\rho^2$-subgaussian\footnote{A random vector $\vb \in \R^d$ is $\rho^2$-subgaussian if for any unit vector $\ub \in \mathbb{S}^{d-1}$, $\ub^{\top} \vb$ is $\rho^2$-subgaussian, $\E \sbr{\exp(s \cdot \ub^{\top} \vb)} \leq \exp\rbr{s^2 \rho^2/2}$ for all $s \in \R$.}.
Given a set of $\iid$ samples of $\xb$, $\Xb=[\xb_1;\dots;\xb_n]$, and a set of weights corresponding to the samples, $\csepp{w_i>0}{i \in [n]}$, let $\Wb=\diag\rbr{w_1,\dots,w_n}$.
If $n \ge \frac{\tr\rbr{\Wb}^2}{\tr\rbr{\Wb^2}} \ge \frac{20736 \rho^4 d}{\eps^2} + \frac{10368\rho^4 \log(1/\delta)}{\eps^2}$, then with probability at least $1-\delta$,
\begin{align*}
    \rbr{1-\eps} \tr\rbr{\Wb} \Sigmab \aleq \Xb^{\top} \Wb \Xb \aleq \rbr{1+\eps} \tr\rbr{\Wb} \Sigmab
\end{align*}
Concretely, with $\Wb = \Ib_n$, $n = \frac{\tr\rbr{\Wb}^2}{\tr\rbr{\Wb^2}} = \Omega\rbr{\rho^4 d}$, and $\eps=\Theta\rbr{\rho^2 \sqrt{\frac{d}{n}}}$, $(1-\eps) \Sigmab \aleq \frac{1}{n}\Xb^{\top} \Xb \aleq (1+\eps) \Sigmab$ with high probability (at least $1-e^{-\Theta(d)}$).
\end{lemma}

\begin{proof}
We first denote $\Pb_{\Xcal} \triangleq \Sigmab \Sigmab^{\pinv}$ as the orthogonal projector onto the subspace $\Xcal \subseteq \R^d$ supported by the distribution of $\xb$.
With the assumptions $\E[\xb]=\b{0}$ and $\E[\xb\xb^{\top}] = \Sigmab$, we observe that $\E \sbr{\overline{\xb}} = \b{0}$ and $\E\sbr{\overline{\xb} \overline{\xb}^{\top}} = \E \sbr{\rbr{\Sigmab^{-1/2}\xb} \rbr{\Sigmab^{-1/2}\xb}^{\top}} = \Pb_{\Xcal}$.
Given the sample set $\Xb$ of size $n \gg \rho^4\rbr{d+\log(1/\delta)}$ for any $\delta \in (0,1)$, we let
\begin{align*}
    \Ub = \frac{1}{\tr\rbr{\Wb}} \sum_{i=1}^n w_i \rbr{\Sigmab^{-1/2}\xb} \rbr{\Sigmab^{-1/2}\xb}^{\top} - \Pb_{\Xcal}.
\end{align*}
Then the problem can be reduced to showing that, for any $\eps>0$, with probability at least $1-\delta$, $\norm{\Ub}_2 \leq \eps$.
For this, we leverage the $\eps$-net argument as follows.

For an arbitrary $\vb \in \Xcal \cap\ \mathbb{S}^{d-1}$, we have
\begin{align*}
    \vb^{\top} \Ub \vb = 
    \frac{1}{\tr\rbr{\Wb}} \sum_{i=1}^n 
    w_i \rbr{\vb^{\top} \rbr{\Sigmab^{-1/2}\xb} \rbr{\Sigmab^{-1/2}\xb}^{\top} \vb - 1} = 
    \frac{1}{\tr\rbr{\Wb}} \sum_{i=1}^n 
    w_i \rbr{\rbr{\vb^{\top} \overline{\xb}_i}^2 - 1},
\end{align*}
where, given $\overline{\xb}_i$ being $\rho^2$-subgaussian, 
$\vb^{\top} \overline{\xb}_i$ is $\rho^2$-subgaussian. 
Since 
\begin{align*}
    \E \sbr{\rbr{\vb^{\top} \overline{\xb}_i}^2} = \vb^{\top} \E \sbr{\overline{\xb}_i \overline{\xb}_i^{\top}} \vb = 1,
\end{align*}
we know that $\rbr{\vb^{\top} \overline{\xb}_i}^2-1$ is $16\rho^2$-subexponential\footnote{We abbreviate $\rbr{\nu,\nu}$-subexponential (\ie, recall that a random variable $X$ is $\rbr{\nu,\alpha}$-subexponential if $\E\sbr{\exp\rbr{s X}} \le \exp\rbr{s^2 \nu^2/2}$ for all $\abbr{s} \le 1/\alpha$) simply as $\nu$-subexponential.}.
With $\beta_i \dfeq \frac{w_i}{\tr\rbr{\Wb}}$ for all $i \in [n]$ such that $\betab = \sbr{\beta_1;\dots;\beta_n}$, we recall Bernstein's inequality \cite[Theorem 2.8.2]{vershynin2018high}\cite[Section 2.1.3]{wainwright2019high},
\begin{align*}
    \PP\sbr{\abbr{\vb^{\top} \Ub \vb} = \abbr{\sum_{i=1}^n \beta_i \cdot \rbr{\rbr{\vb^{\top} \overline{\xb}_i}^2-1}} > t} \leq 
    2 \exp \rbr{-\frac{1}{2} 
    \min\rbr{\frac{t^2}{\rbr{16 \rho^2}^2 \nbr{\betab}_2^2}, 
    \frac{t}{16 \rho^2 \nbr{\betab}_\infty}}},
\end{align*}
where $\nbr{\betab}_2^2 = \frac{\tr\rbr{\Wb^2}}{\tr\rbr{\Wb}^2}$ and $\nbr{\betab}_\infty = \frac{\max_{i \in [n]} w_i}{\tr\rbr{\Wb}}$.

Let $N \subset \Xcal \cap\ \mathbb{S}^{d-1}$ be an $\eps_1$-net such that $\abbr{N} = \rbr{1+\frac{2}{\eps_1}}^d$. 
Then for some $0 < \eps_2 \leq 16 \rho^2 \frac{\nbr{\betab}_2^2}{\nbr{\betab}_\infty}$, by the union bound,
\begin{align*}
    \PP \sbr{\underset{\vb \in N}{\max}: \abbr{\vb^{\top} \Ub \vb} > \eps_2} 
    \leq \ 
    & 2 \abbr{N} \exp \rbr{-\frac{1}{2}
    \min\rbr{\frac{\eps_2^2}{\rbr{16 \rho^2}^2 \nbr{\betab}_2^2}, 
    \frac{\eps_2}{16 \rho^2 \nbr{\betab}_\infty}} } 
    \\ \leq \ 
    & \exp \rbr{d\log\rbr{1+\frac{2}{\eps_1}} -\frac{1}{2} \cdot \frac{\eps_2^2}{\rbr{16 \rho^2}^2 \nbr{\betab}_2^2} } = \delta
\end{align*}
whenever $\frac{1}{\nbr{\betab}_2^2} = \frac{\tr\rbr{\Wb}^2}{\tr\rbr{\Wb^2}} = 2 \rbr{\frac{16 \rho^2}{\eps_2}}^2 \rbr{d\log\rbr{1+\frac{2}{\eps_1}} + \log\frac{1}{\delta}}$ where $1 < \frac{\tr\rbr{\Wb}^2}{\tr\rbr{\Wb^2}} \le n$. 

Now for any $\vb \in \Xcal \cap\ \mathbb{S}^{d-1}$, there exists some $\vb' \in N$ such that $\norm{\vb - \vb'}_2 \leq \eps_1$. 
Therefore,
\begin{align*}
    \abbr{\vb^{\top} \Ub \vb} 
    \ = \
    & \abbr{\vb'^{\top} \Ub \vb' + 
    2 \vb'^{\top} \Ub \rbr{\vb - \vb'} + 
    \rbr{\vb - \vb'}^{\top} \Ub \rbr{\vb - \vb'} }
    \\ \leq \ 
    & \rbr{ \underset{\vb \in N}{\max}: \abbr{\vb^{\top} \Ub \vb} } + 
    2 \norm{\Ub}_2 \norm{\vb'}_2 \norm{\vb-\vb'}_2 + 
    \norm{\Ub}_2 \norm{\vb-\vb'}_2^2
    \\ \leq \ 
    & \rbr{ \underset{\vb \in N}{\max}: \abbr{\vb^{\top} \Ub \vb} } + 
    \norm{\Ub}_2
    \rbr{2 \eps_1 + \eps_1^2}.
\end{align*}
Taking the supremum over $\vb \in \mathbb{S}^{d-1}$, with probability at least $1-\delta$,
\begin{align*}
    \underset{\vb \in \Xcal \cap\ \mathbb{S}^{d-1}}{\max}: \abbr{\vb^{\top} \Ub \vb}
    = 
    \norm{\Ub}_2
    \leq 
    \eps_2 + \norm{\Ub}_2 \rbr{2 \eps_1 + \eps_1^2},
    \qquad
    \norm{\Ub}_2
    \leq 
    \frac{\eps_2}{2 - \rbr{1+\eps_1}^2}.
\end{align*}
With $\eps_1 = \frac{1}{3}$, we have $\eps = \frac{9}{2}\eps_2$.

Overall, if $n \ge \frac{\tr\rbr{\Wb}^2}{\tr\rbr{\Wb^2}} \ge \frac{1024 \rho^4 d}{\eps_2^2} + \frac{512 \rho^4}{\eps_2^2}\log\frac{1}{\delta}$, then with probability at least $1-\delta$, we have $\norm{\Ub}_2 \leq \eps$.

As a concrete instance, when $\Wb=\Ib_n$ and $n = \frac{\tr\rbr{\Wb}^2}{\tr\rbr{\Wb^2}} \ge 9^2 \cdot 1025 \cdot \rho^4 d$, by taking $\eps_2 = \sqrt{\frac{1025 \rho^4 d}{n}}$, we have $\nbr{\Ub}_2 \le \frac{1}{2} \sqrt{\frac{9^2 \cdot 1025 \cdot \rho^4 d}{n}}$ with high probability (at least $1-\delta$ where $\delta = \exp\rbr{-\frac{d}{512}}$).
\end{proof}

\begin{lemma}[Cauchy interlacing theorem]\label{lemma:Cauchy_interlacing_theorem}
Given an arbitrary matrix $\Ab \in \C^{m \times n}$ and an orthogonal projection $\Qb \in \C^{n \times k}$ with orthonormal columns, for all $i=1,\dots,k$,
\begin{align*}
    \sigma_{i}\rbr{\Ab \Qb} \leq \sigma_{i}\rbr{\Ab}.
\end{align*}
\end{lemma}

\begin{proof}[Proof of \Cref{lemma:Cauchy_interlacing_theorem}]
Let $\csepp{\vb_j \in \C^k}{j=1,\dots,k}$ be the right singular vectors of $\Ab\Qb$. By the min-max theorem (\cf \cite{golub2013} Theorem 8.6.1), 
\begin{align*}
    \sigma_{i}\rbr{\Ab \Qb}^2
    = \min_{\xb \in \spn\cbr{\vb_1,\dots,\vb_i}\setminus\rbr{\b0}} \frac{\xb^\top \Qb^\top \Ab^\top \Ab \Qb \xb}{\xb^\top \xb}
    \leq \max_{\dim\rbr{\Vcal}=i} \min_{\xb \in \Vcal\setminus\rbr{\b0}} \frac{\xb^\top \Ab^\top \Ab \xb}{\xb^\top \xb} = \sigma_i\rbr{\Ab}^2.
\end{align*}  
\end{proof}

\begin{lemma}[\citep{horn_johnson_2012} (7.3.14)]\label{lemma:individual_sval_of_product}
For arbitrary matrices $\Ab,\Bb \in \C^{m \times n}$,
\begin{align}\label{eq:lemma_individual_sval_of_product}
    \sigma_i\rbr{\Ab \Bb^*} \le \sigma_i\rbr{\Ab} \sigma_j\rbr{\Bb}
\end{align} 
for all $i \in [\rank\rbr{\Ab}]$, $j \in [\rank\rbr{\Bb}]$ such that $i+j-1 \in [\rank\rbr{\Ab\Bb^*}]$.
\end{lemma}

\begin{lemma}[[\citep{horn_johnson_2012} (7.3.13)] \label{lemma:individual_sval_of_sum}
For arbitrary matrices $\Ab,\Bb \in \C^{m \times n}$,
\begin{align}\label{eq:lemma_individual_sval_of_sum}
    \sigma_{i+j-1}\rbr{\Ab + \Bb} \le \sigma_i\rbr{\Ab} + \sigma_j\rbr{\Bb}
\end{align}    
for all $i \in [\rank\rbr{\Ab}]$, $j \in [\rank\rbr{\Bb}]$ such that $i+j-1 \in [\rank\rbr{\Ab+\Bb}]$.
\end{lemma}

\section{Supplementary Experiments}

\subsection{Upper and Lower Space-agnostic Bounds}\label{subapx:sup_exp_upper_lower_space_agnostic}

In this section, we visualize and compare the upper and lower bounds in \Cref{thm:space_agnostic_bounds} under the sufficient multiplicative oversampling regime (\ie, $l = 4k$. Recall that $k < l < r = \rank\rbr{\Ab}$ where $k$ is the target rank, $l$ is the oversampled rank, and $r$ is the full rank of the matrix $\Ab$).

\begin{figure}[!ht]
    \centering
    \includegraphics[width=\linewidth]{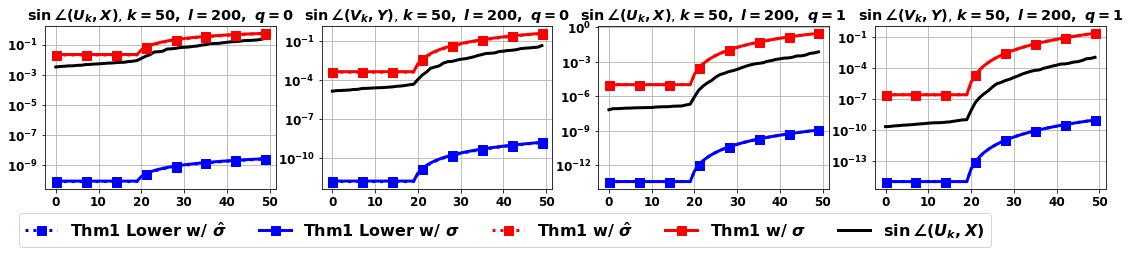}
    \caption{Synthetic Gaussian with the slower spectral decay. $k=50$, $l=200$, $q=0,1$.}
    \label{fig:lower_Gaussian-poly1-kl_k50_l200}
\end{figure}

\begin{figure}[!ht]
    \centering
    \includegraphics[width=\linewidth]{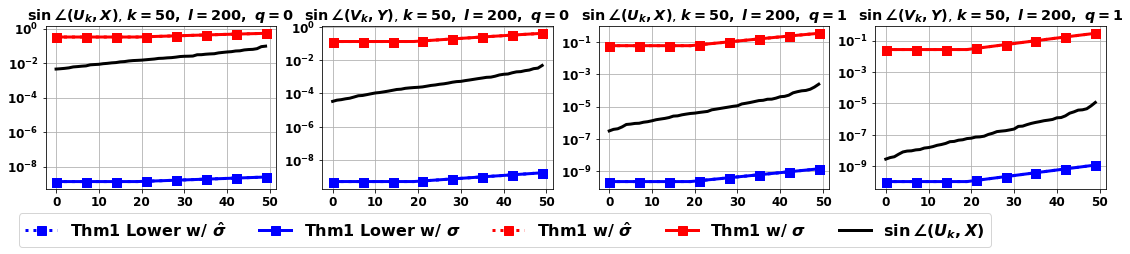}
    \caption{Synthetic Gaussian with the faster spectral decay. $k=50$, $l=200$, $q=0,1$.}
    \label{fig:lower_Gaussian-exp-kl_k50_l200}
\end{figure}

\begin{figure}[!ht]
    \centering
    \includegraphics[width=\linewidth]{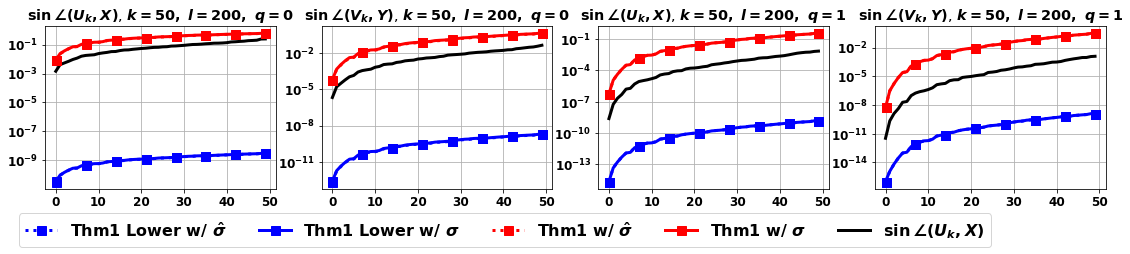}
    \caption{SNN with $r_1=20$, $a=1$. $k=50$, $l=200$, $q=0,1$.}
    \label{fig:lower_SNN-m500n500r20a1-kl_k50_l200}
\end{figure}

\begin{figure}[!ht]
    \centering
    \includegraphics[width=\linewidth]{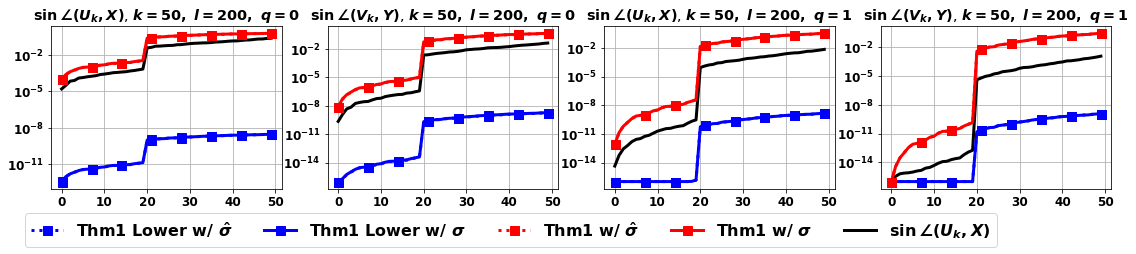}
    \caption{SNN with $r_1=20$, $a=100$. $k=50$, $l=200$, $q=0,1$.}
    \label{fig:lower_SNN-m500n500r20a100-kl_k50_l200}
\end{figure}

\begin{figure}[!ht]
    \centering
    \includegraphics[width=\linewidth]{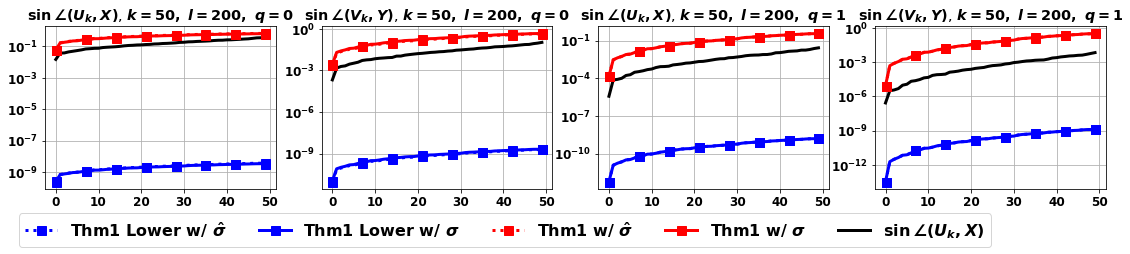}
    \caption{$800$ randomly sampled images from the MNIST training set. $k=50$, $l=200$, $q=0,1$.}
    \label{fig:lower_mnist-train800-kl_k50_l200}
\end{figure}

With the same set of target matrices described in \Cref{subsec:target_matrix}, from \Cref{fig:lower_Gaussian-poly1-kl_k50_l200} to \Cref{fig:lower_mnist-train800-kl_k50_l200},
\begin{enumerate}
    \item {\color{red} Red lines and dashes} represent the upper bounds in \Cref{eq:space_agnostic_left} and \Cref{eq:space_agnostic_left} evaluated with the true (lines) and approximated (dashes) singular values, $\Sigmab$, and $\wt\Sigmab$, respectively, where we simply ignore tail decay and suppress constants for the distortion factors and set 
    \begin{align*}
        \eps_1 = \sqrt{\frac{k}{l}}
        \quad\text{and}\quad
        \eps_2 = \sqrt{\frac{l}{r-k}}.
    \end{align*}
    
    \item {\color{blue} Blue lines and dashes} present the lower bounds in \Cref{eq:space_agnostic_lower_left} and \Cref{eq:space_agnostic_lower_right} evaluated with $\Sigmab$ and $\wt\Sigmab$, respectively, and slightly larger constants associated with the distortion factors 
    \begin{align*}
        \eps'_1 = 2\sqrt{\frac{k}{l}}
        \quad\text{and}\quad
        \eps'_2 = 2\sqrt{\frac{l}{r-k}}.
    \end{align*}
\end{enumerate}

The numerical observations imply that the empirical validity of lower bounds requires more aggressive oversampling than that of upper bounds. In particular, we recall from \Cref{subsec:experiment_canonical_angle} that $l \ge 1.6k$ is usually sufficient for the upper bounds to hold numerically. In contrast, the lower bounds generally require at least $l \ge 4k$, with slightly larger constants associated with the distortion factors $\eps_1 = \Theta\rbr{\sqrt{k/l}}$ and $\eps_2 = \Theta\rbr{\sqrt{l/\rbr{r-k}}}$.
\chapter{Appendix for \Cref{ch:dac}}

\section{Linear Regression Models}\label{apdx:general}

In this section, we present formal proofs for the results on linear regression in the fixed design where the training samples $\rbr{\Xb,\yb}$ and their augmentations $\wt\Acal\rbr{\Xb}$ are considered to be fixed. We discuss two types of augmentations: the label invariant augmentations in \Cref{sec:linear_regression_label_invariant} and the misspecified augmentations in \Cref{subsec:finite_lambda}.

\subsection{Linear Regression with Label Invariant Augmentations}\label{apdx:linear_regression}
 
For fixed $\wt\Acal(\Xb)$, let $\Deltab \triangleq \widetilde \Acal(\Xb) - \widetilde \Mb \Xb$ in this section. We recall that $\dau = \rank\rbr{\Deltab}$ since there is no randomness in $\widetilde \Acal, \Xb$ in fix design setting. Assuming that $\widetilde \Acal(\Xb)$ admits full column rank, we have the following theorem on the excess risk of DAC and ERM:
\begin{theorem}[Formal restatement of \Cref{thm:informal_linear_regression} on linear regression.]\label{thm:formal_linear_regression}
Learning with DAC regularization, we have $\E\sbr{L(\widehat \thetab^{dac}) - L(\thetab^*)} = \frac{(d - \dau)\sigma^2}{N}$, while learning with ERM directly on the augmented dataset, we have $\E\sbr{L(\widehat \thetab^{\herm}) - L(\thetab^*)} = \frac{(d - \dau + d')\sigma^2}{N}$. $d'$ is defined as
\begin{align*}
    d' \triangleq \frac{\tr\rbr{\widetilde\Mb^\top\rbr{\projAX - \Pb_{\Scal}}\widetilde\Mb}}{1+\alpha},
\end{align*}
where $d' \in [0, \dau]$ with $\projAX = \widetilde \Acal(\Xb)\rbr{\widetilde \Acal(\Xb)^\top\widetilde \Acal(\Xb)}^{-1}\widetilde \Acal(\Xb)^\top$ and $\Pb_\Scal \in \R^{(\alpha+1)N \times (\alpha+1)N}$ is the orthogonal projector onto $\Scal \triangleq \cbr{\widetilde \Mb \Xb \thetab~|~\forall \thetab \in \R^d, s.t. \rbr{\widetilde\Acal(\Xb) - \widetilde \Mb \Xb}\thetab = \b0}$.
\end{theorem}

\begin{proof}
With $L(\thetab) \triangleq \frac{1}{(1+\alpha) N}\norm{\widetilde \Acal(\Xb)\thetab - \widetilde \Acal(\Xb)\thetab^*}_2^2$, the excess risk of ERM on the augmented training set satisfies that:
\begin{align*}
    \E\sbr{L(\widehat \thetab^{\herm})} &= \frac{1}{(1 + \alpha) N}\E\sbr{\norm{\widetilde \Acal(\Xb)\widehat\thetab^{\herm} - \widetilde \Acal(\Xb)\thetab^{*}}_{2}^{2}} \\
    &= \frac{1}{(1 + \alpha) N}\E\sbr{\norm{\widetilde \Acal(\Xb)(\widetilde \Acal(\Xb)^{\top}\widetilde \Acal(\Xb))^{-1}\widetilde \Acal(\Xb)^{\top}(\widetilde \Acal(\Xb)\thetab^{*} + \widetilde \Mb\epsb) - \widetilde \Acal(\Xb)\thetab^{*}}_{2}^{2}} \\
    &=\frac{1}{(1 + \alpha) N}\E\sbr{\norm{\projAX\widetilde\Acal(\Xb)\thetab^{*} + \projAX \widetilde \Mb\epsb - \widetilde \Acal(\Xb)\thetab^{*}}_{2}^{2}} \\
    &=\frac{1}{(1 + \alpha) N}\E\sbr{\norm{\projAX \widetilde \Mb\epsb}_{2}^{2}} \\
    &=\frac{1}{(1 + \alpha) N}\E\sbr{\tr(\epsb^{\top} \widetilde \Mb^{\top}\projAX \widetilde \Mb\epsb)} \\
    &=\frac{\sigma^{2}}{(1 + \alpha) N} \tr\rbr{\wt\Mb^{\top} \projAX \wt\Mb}
.\end{align*}
Let $\Ccal_{\widetilde \Acal(\Xb)}$ and $\Ccal_{\widetilde \Mb}$ denote the column space of $\widetilde \Acal(\Xb)$ and $\widetilde \Mb$, respectively. Notice that $\Scal$ is a subspace of both $\Ccal_{\widetilde \Acal(\Xb)}$ and $\Ccal_{\widetilde \Mb}$. Observing that $\dau = \rank\rbr{\Deltab} = \rank\rbr{\Pb_{\Scal}}$, we have 
\begin{align*}
\E\sbr{L(\widehat \thetab^{\herm})}
=&\frac{\sigma^{2}}{(1 + \alpha) N} \tr(\widetilde\Mb^{\top}\projAX\widetilde\Mb)
\\
= &\frac{\sigma^{2}}{(1 + \alpha) N} \tr(\widetilde\Mb^{\top}\Pb_\Scal\widetilde\Mb) + \frac{\sigma^{2}}{(1+\alpha)N} \tr(\widetilde\Mb^{\top}(\projAX - \Pb_\Scal)\widetilde\Mb)
\\
= &\frac{\sigma^{2}}{(1 + \alpha) N} \tr(\widetilde\Mb^{\top}\Pb_\Scal\widetilde\Mb) + \frac{\sigma^{2}}{N}\cdot \frac{\tr(\widetilde\Mb^{\top}(\projAX -\Pb_\Scal)\widetilde\Mb)}{1+\alpha}
\end{align*}

By the data augmentation consistency constraint, we are essentially solving the linear regression on the $(d-\dau)$-dimensional space $\cbr{\thetab~|~\Deltab \thetab = 0}$. The rest of proof is identical to standard regression analysis, with features first projected to $\Scal$: 
\begin{align*}
\E\sbr{L(\widehat \thetab^{dac})} &= \frac{1}{(1 + \alpha)N}\E\sbr{\norm{\widetilde \Acal(\Xb)\widehat\thetab^{dac} - \widetilde \Acal(\Xb)\thetab^{*}}_{2}^{2}} 
\\
&= \frac{1}{(1 + \alpha) N}\E\sbr{\norm{\widetilde \Acal(\Xb)(\widetilde \Acal(\Xb)^{\top}\widetilde \Acal(\Xb))^{-1}\widetilde \Acal(\Xb)^{\top}\Pb_\Scal(\widetilde \Acal(\Xb)\thetab^{*} + \widetilde \Mb\epsb) - \widetilde \Acal(\Xb)\thetab^{*}}_{2}^{2}} 
\\
&=\frac{1}{(1 + \alpha) N}\E\sbr{\norm{ \projAX \Pb_\Scal \widetilde\Acal(\Xb)\thetab^{*} + \projAX \Pb_\Scal \widetilde \Mb\epsb - \widetilde \Acal(\Xb)\thetab^{*}}_{2}^{2}} 
\\
& \quad \rbr{\t{since}\ \wt\Acal(\Xb) \thetab^* \in \Scal,\ \t{and}\ \projAX \Pb_\Scal = \Pb_\Scal\ \t{since}\ \Scal \subseteq \Ccal_{\wt\Acal(\Xb)}}
\\
&=\frac{1}{(1 + \alpha)N}\E\sbr{\norm{\Pb_\Scal \widetilde \Mb\epsb}_{2}^{2}} \\
&=\frac{\sigma^{2}}{(1 + \alpha) N} \tr(\widetilde\Mb^{\top}\Pb_\Scal\widetilde\Mb) \\
&= \frac{(d - \dau)\sigma^{2}}{N}.
\end{align*}

\end{proof}

\subsection{Linear Regression Beyond Label Invariant Augmentations}\label{apx:finite_lambda}

\begin{proof}[Proof of \Cref{thm:formal_linear_regression_soft}]
With $L(\thetab) \triangleq \frac{1}{N}\norm{\Xb\thetab - \Xb\thetab^*}_2^2 = \nbr{\thetab - \thetab^*}_{\covtr}^2$, we start by partitioning the excess risk into two parts -- the variance from label noise and the bias from feature-label mismatch due to augmentations ($\ie$, $\wt\Acal\rbr{\Xb}\thetab^* \neq \wt\Mb\Xb\thetab^*$): 
\begin{align*}
    \E_{\epsb}\sbr{L\rbr{\thetab} - L\rbr{\thetab^*}} 
    = \E_{\epsb}\sbr{\norm{\thetab - \thetab^*}_{\covtr}^2} 
    = \underbrace{\E_{\epsb}\sbr{\nbr{\thetab - \E_{\epsb}\sbr{\thetab}}_{\covtr}^2}}_{\t{Variance}} + 
    \underbrace{\nbr{\E_{\epsb}\sbr{\thetab} - \thetab^*}_{\covtr}^2}_{\t{Bias}}
.\end{align*}

First, we consider learning with DAC regularization with some finite $0<\lambda<\infty$,
\begin{align*}
    \wh\thetab^{dac} = \argmin_{\thetab \in \R^d} \frac{1}{N} \nbr{\Xb \thetab - \yb}_2^2
    + \frac{\lambda}{\rbr{1+\alpha} N} \norm{\rbr{\wt\Acal\rbr{\Xb} - \wt\Mb\Xb} \thetab}_2^2.
\end{align*}
By setting the gradient of \Cref{eq:dac_soft_reg} with respect to $\thetab$ to $\b0$, with $\yb = \Xb \thetab^* + \epsb$, we have
\begin{align*}
    \wh\thetab^{dac} = \frac{1}{N} \rbr{\covtr + \lambda \covaug}^{\pinv} \Xb^\top \rbr{\Xb \thetab^* + \epsb},
\end{align*}    
Then with $\E_{\epsb}\sbr{\wh\thetab^{dac}} = \rbr{\covtr + \lambda \covaug}^{\pinv} \covtr \thetab^*$, 
\begin{align*}
    \t{Var} = \E_{\epsb}\sbr{\nbr{\frac{1}{N} \rbr{\covtr + \lambda \covaug}^{\pinv} \Xb^\top \epsb}_{\covtr}^2},
    \quad
    \t{Bias} = \nbr{\rbr{\covtr + \lambda \covaug}^{\pinv} \covtr \thetab^* - \thetab^*}_{\covtr}^2.
\end{align*}

For the variance term, we have
\begin{align*}
    \t{Var} 
    = &\frac{\sigma^2}{N} \tr\rbr{\rbr{\covtr + \lambda \covaug}^{\pinv} \covtr \rbr{\covtr + \lambda \covaug}^{\pinv} \covtr}
    \\
    = &\frac{\sigma^2}{N} \tr\rbr{\sbr{\covtr^{1/2} \rbr{\covtr + \lambda \covaug}^{\pinv} \covtr^{1/2}}^2}
    \\
    = &\frac{\sigma^2}{N} \tr\rbr{ \rbr{\Ib_d + \lambda \covtr^{-1/2} \covaug \covtr^{-1/2}}^{-2} }
\end{align*}
For the semi-positive definite matrix $\covtr^{-1/2} \covaug \covtr^{-1/2}$, we introduce the spectral decomposition:
\begin{align*}
    \covtr^{-1/2} \covaug \covtr^{-1/2} = \underset{d \times \dau}{\Qb}\ \underset{\dau \times \dau}{\Gammab}\ \Qb^\top, 
    \quad
    \Gammab = \diag\rbr{\gamma_1,\dots,\gamma_{\dau}},
\end{align*}
where $\Qb$ consists of orthonormal columns and $\gamma_1 \geq \dots \geq \gamma_{\dau} > 0$. Then
\begin{align*}
    \t{Var}
    = \frac{\sigma^2}{N} \tr\rbr{\rbr{\Ib_d - \Qb\Qb^\top} + \Qb \rbr{\Ib_{\dau} + \lambda\Gammab}^{-2} \Qb^\top}
    = \frac{\sigma^2 \rbr{d-\dau}}{N} + \frac{\sigma^2}{N} \sum_{i=1}^{\dau} \frac{1}{\rbr{1+\lambda \gamma_i}^2}.
\end{align*}    

For the bias term, we observe that
\begin{align*}
    \t{Bias} 
    = &\nbr{\rbr{\covtr + \lambda \covaug}^{\pinv} \covtr \thetab^* - \thetab^*}_{\covtr}^2
    \\
    = &\nbr{\rbr{\covtr + \lambda \covaug}^{\pinv} \rbr{-\lambda\covaug} \thetab^*}_{\covtr}^2
    \\
    = &\nbr{ \rbr{\Ib_d + \lambda \covtr^{-\frac{1}{2}} \covaug \covtr^{-\frac{1}{2}}}^{-1} \rbr{\lambda \covtr^{-\frac{1}{2}} \covaug \covtr^{-\frac{1}{2}}} \rbr{\covtr^{1/2} \projrg \thetab^*} }_2^2.
\end{align*}
Then with $\vthetab \triangleq \covtr^{1/2} \projrg \thetab^*$, we have
\begin{align*}
    \t{Bias} 
    = \sum_{i=1}^{\dau} \vartheta_i^2 \rbr{\frac{\lambda \gamma_i}{1 + \lambda \gamma_i}}^2
\end{align*}

To simply the optimization of regularization parameter $\lambda$, we leverage upper bounds of the variance and bias terms:
\begin{align*}
    & \t{Var} - \frac{\sigma^2 \rbr{d-\dau}}{N} 
    \leq \frac{\sigma^2}{N} \sum_{i=1}^{\dau} \frac{1}{\rbr{1+\lambda \gamma_i}^2} 
    \leq \frac{\sigma^2}{2 N \lambda} \sum_{i=1}^{\dau} \frac{1}{\gamma_i} 
    \leq \frac{\sigma^2}{2 N \lambda} \tr\rbr{\covtr \covaug^\pinv},
    \\
    & \t{Bias} 
    = \sum_{i=1}^{\dau} \vartheta_i^2 \rbr{\frac{\lambda \gamma_i}{1 + \lambda \gamma_i}}^2
    \leq \frac{\lambda}{2} \sum_{i=1}^{\dau} \vartheta_i^2 \gamma_i
    = \frac{\lambda}{2} \nbr{\projrg \thetab^*}^2_{\covaug}
.\end{align*}
Then with $\lambda = \sqrt{ \frac{\sigma^2 \tr\rbr{\covtr \covaug^\pinv}}{N \norm{\projrg \thetab^*}_{\covaug}^{2}} }$, we have the generalization bound for $\wh\thetab^{dac}$ in \Cref{thm:formal_linear_regression_soft}.

Second, we consider learning with DA-ERM:
\begin{align*}
    \wh\thetab^{\herm} = \argmin_{\thetab \in \R^d} \frac{1}{\rbr{1+\alpha}N} \nbr{\wt\Acal\rbr{\Xb} \thetab - \wt\Mb \yb}_2^2.
\end{align*} 
With 
\begin{align*}
    \wh\thetab^{\herm} = \frac{1}{(1+\alpha)N} \covall^{-1} \wt\Acal\rbr{\Xb}^\top \wt\Mb \rbr{\Xb \thetab^* + \epsb}
,\end{align*}  
we again partition the excess risk into the variance and bias terms.
For the variance term, with the assumptions $\covall \aleq c_X \covtr$ and $\covall \aleq c_S \covs$, we have
\begin{align*}
    \t{Var} 
    = & \E_{\epsb}\sbr{\nbr{\frac{1}{(1+\alpha)N} \covall^{-1} \wt\Acal\rbr{\Xb}^\top \wt\Mb \epsb}_{\covtr}^2}
    \\
    = & \E_{\epsb}\sbr{\nbr{\frac{1}{N} \covall^{-1} \Sb^\top \epsb}_{\covtr}^2}
    \\
    = & \frac{\sigma^2}{N} \tr\rbr{\covtr \covall^{-1} \covs \covall^{-1}}
    \\
    \geq & \frac{\sigma^2}{N} \tr\rbr{\frac{1}{c_X c_S} \Ib_d}
    = \frac{\sigma^2 d}{N c_X c_S}
.\end{align*}

Additionally, for the bias term, we have
\begin{align*}
    \t{Bias} 
    = & \nbr{\frac{1}{(1+\alpha)N} \covall^{-1} \wt\Acal\rbr{\Xb}^\top \wt\Mb\Xb \thetab^* - \thetab^*}_{\covtr}^2
    \\
    = & \nbr{\rbr{\wt\Acal\rbr{\Xb}^\top \wt\Acal\rbr{\Xb}}^{-1} \wt\Acal\rbr{\Xb}^\top \Deltab \rbr{\projrg \thetab^*}}^2_{\covtr}
    \\
    = & \nbr{\wt\Acal\rbr{\Xb}^\pinv \Deltab \rbr{\projrg \thetab^*}}^2_{\covtr}
    = \nbr{\projrg \thetab^*}_{\covaugwt}^2
.\end{align*}
Combining the variance and bias leads to the generalization bound for $\wh\thetab^{\herm}$ in \Cref{thm:formal_linear_regression_soft}.
\end{proof}

\section{Two-layer Neural Network Regression}\label{apx:pf_case_2layer_relu}

In the two-layer neural network regression setting with $\Xcal = \R^d$ described in \Cref{sec:example_2relu}, let $\Xb \sim P^N(\xb)$ be a set of $N$ $\iid$ samples drawn from the marginal distribution $P(\xb)$ that satifies the following.
\begin{assumption}[Regularity of marginal distribution]
\label{ass:observable_marginal_distribution}
Let $\xb \sim \Pgt(\xb)$ be zero-mean $\E[\xb]=\b{0}$, with covairance matrix $\E[\xb \xb^{\top}]=\Sigmab_{\xb} \succ 0$ whose eigenvalues are bounded by constant factors $\Omega(1)=\sigma_{\min}(\Sigmab_\xb) \le \sigma_{\max}(\Sigmab_\xb) = O(1)$, such that $(\Sigmab_{\xb}^{-1/2} \xb)$ is $\rho^2$-subgaussian
\footnote{A random vector $\vb \in \R^d$ is $\rho^2$-subgaussian if for any unit vector $\ub \in \mathbb{S}^{d-1}$, $\ub^{\top} \vb$ is $\rho^2$-subgaussian, $\E \sbr{\exp(s \cdot \ub^{\top} \vb)} \leq \exp\rbr{s^2 \rho^2/2}$.}.
\end{assumption}    

For the sake of analysis, we isolate the augmented part in  $\wt\Acal(\Xb)$ and denote the set of these augmentations as
\begin{align*}
    \Acal(\Xb) = \sbr{\xb_{1,1}; \cdots; \xb_{N,1}; \cdots; \xb_{1, \alpha}; \cdots; \xb_{N, \alpha}} \in \Xcal^{\alpha N},
\end{align*}
where for each sample $i \in [N]$, $\cbr{\xb_{i,j}}_{j \in [\alpha]}$ is a set of $\alpha$ augmentations generated from $\xb_i$, and $\Mb \in \R^{\alpha N \times N}$ is the vertical stack of $\alpha$ $N \times N$ identity matrices.
Analogous to the notions with respect to $\wt\Acal(\Xb)$ in the linear regression cases in \Cref{apdx:general}, in this section, we denote $\Deltab \triangleq \Acal(\Xb) - \Mb\Xb$ and quantify the augmentation strength as
\begin{align*}
    \dau \triangleq \rank\rbr{\Deltab} = \rank\rbr{\wt\Acal\rbr{\Xb} - \wt\Mb\Xb}
\end{align*}
such that $0 \leq \dau \leq \min\rbr{d, \alpha N}$ can be intuitively interpreted as the number of dimensions in the span of the unlabeled samples, $\row(\Xb)$, perturbed by the augmentations.

Then, to learn the ground truth distribution $\yb = h^*(\Xb) + \epsb = \rbr{\Xb\Bb^*}_+ \wb^* + \epsb$ where $\epsb \sim \Ncal(\b0, \sigma^2 \Ib_N)$, training with the DAC regularization can be formulated explicitly as
\begin{align*}
    \wh{\Bb}^{dac}, \wh{\wb}^{dac} ~=~
    &\underset{\Bb \in \R^{d \times q}, \wb \in \R^q}{\argmin}\ \frac{1}{N} \norm{\yb - \rbr{\Xb\Bb}_+ \wb}_2^2 \\
    &\t{s.t.} \quad
    \Bb = \bmat{\bb_1 \dots \bb_k \dots \bb_q},
    \ \bb_k \in \mathbb{S}^{d-1}\ \forall\ k \in [q],
    \quad
    \norm{\wb}_1 \leq C_w \\
    &\rbr{\Aemp\rbr{\Xb} \Bb}_+ = \rbr{\Mb\Xb\Bb}_+.
\end{align*}
For the resulted minimizer $\wh{h}^{dac}(\xb) \triangleq (\xb^{\top} \wh\Bb^{dac})_+ \wh\wb^{dac}$, we have the following.
\begin{theorem}[Formal restatement of \Cref{thm:case_2layer_relu_risk_informal} on two-layer neural network with DAC]
\label{thm:case_2layer_relu_risk}
Under \Cref{ass:observable_marginal_distribution}, we suppose $\Xb$ and $\Deltab$ satisfy that
(a) $\alpha N \geq 4 \dau$; and
(b) $\Deltab$ admits an absolutely continuous distribution. Then conditioned on $\Xb$ and $\Deltab$, with $L(h) = \frac{1}{N}\nbr{h(\Xb) - h^*(\Xb)}_2^2$ and $\frac{1}{N}\sum_{i=1}^N \nbr{\projnull \xb_i}^2_2 \leq \Cnull^2$ for some $\Cnull>0$, for all $\delta \in (0,1)$, with probability at least $1-\delta$ (over $\epsb$), 
\begin{align*}
    L\rbr{\wh{h}^{dac}} - L\rbr{h^*}
    \lesssim 
    \sigma C_w \Cnull \rbr{\frac{1}{\sqrt{N}} + \sqrt{\frac{\log(1/\delta)}{N}}}.
\end{align*}
\end{theorem}

Moreover, to account for randomness in $\Xb$ and $\Deltab$, we introduce the following notion of augmentation strength.
\begin{definition}[Augmentation strength]\label{def:daug} 
For any $\delta \in [0,1)$, let
\begin{align*}
    \dau(\delta) \triangleq \argmax_{d'}\ \PP_{\Deltab} \sbr{\rank \rbr{\Deltab}<d'} \le \delta.
\end{align*}
\end{definition}
Intuitively, the \textit{augmentation strength} $\dau$ ensures that the feature subspace perturbed by the augmentations in $\Aemp(\Xb)$ has a minimum dimension $\dau(\delta)$ with probability at least $1 - \delta$. A larger $\dau(\delta)$ corresponds to stronger augmentations. For instance, when $\Aemp(\Xb)=\Mb\Xb$ almost surely ($\eg$, when the augmentations are identical copies of the original samples, corresponding to the weakest augmentation -- no augmentations at all), we have $\dau(\delta) = \dau = 0$ for all $\delta<1$. Whereas for randomly generated augmentations, $\dau$ is likely to be larger (i.e., with more dimensions being perturbed). For example in \Cref{example:misspec}, for a given $\dau$, with random augmentations $\Aemp\rbr{\Xb} = \Xb'$ where $\Xb'_{ij} = \Xb_{ij} + \Ncal\rbr{0, 0.1}$ for all $i \in [N]$, $d-\dau+1 \le j \le d$, we have $\rank\rbr{\Deltab}=\dau$ with probability $1$. That is $\dau(\delta)=\dau$ for all $\delta \ge 0$.

Leveraging the notion of augmentation strength in \Cref{def:daug}, we show that the stronger augmentations lead to the better generalization by reducing $\Cnull$ in \Cref{thm:case_2layer_relu_risk}.
\begin{corollary}\label{coro:two_layer_daug_bound}
When $N \gg \rho^4 d$ and $\alpha N \ge d$, for any $\delta \in (0,1)$, with probability at least $1-\delta$ (over $\Xb$ and $\Deltab$),
we have $\Cnull \lesssim \sqrt{d - \dau(\delta)}$.
\end{corollary}

To prove \Cref{thm:case_2layer_relu_risk}, we start by showing that, with sufficient samples ($\alpha N \ge 4 \dau$), consistency of the first layer outputs over the samples implies consistency of those over the population.
\begin{lemma}
\label{lemma:total_invertible_meas_zero}
Under the assumptions in \Cref{thm:case_2layer_relu_risk}, every size-$\dau$ subset of rows in $\Deltab = \Aemp(\Xb) - \Mb \Xb$ is linearly independent almost surely.
\end{lemma}

\begin{proof}
[Proof of \Cref{lemma:total_invertible_meas_zero}]
Since $\alpha N > \dau$, it is sufficient to show that a random matrix with an absolutely continuous distribution is totally invertible \footnote{A matrix is totally invertible if all its square submatrices are invertible.} almost surely.

It is known that for any dimension $m \in \N$, an $m \times m$ square matrix $\Sb$ is singular if $\det(\Sb) = 0$ where entries of $\Sb$ lie within the roots of the polynomial equation specified by the determinant.
Therefore, the set of all singular matrices in $\R^{m \times m}$ has Lebesgue measure zero, 
\begin{align*}
    \lambda\rbr{\csepp{\Sb \in \R^{m \times m}}{\det(\Sb) = 0}} = 0
.\end{align*}
Then, for an absolutely continuous probability measure $\mu$ with respect to $\lambda$, we also have
\[
    \PP_{\mu}\sbr{\Sb \in \R^{m \times m}\ \t{is singular}} = 
    \mu \rbr{\csepp{\Sb \in \R^{m \times m}}{\det(\Sb) = 0}} = 0.
\]
Since a general matrix $\Rb$ contains only finite number of submatrices, when $\Rb$ is drawn from an absolutely continuous distribution, by the union bound, $\PP\sbr{\Rb\ \t{cotains a singular submatrix}} = 0$. 
That is, $\Rb$ is totally invertible almost surely.
\end{proof}

\begin{lemma}
\label{lemma:2layer-relu-input-consistency-exclude-spurious}
Under the assumptions in \Cref{thm:case_2layer_relu_risk}, the hidden layer in the two-layer ReLU network learns $\kernel\rbr{\Deltab}$, the invariant subspace under data augmentations : with high probability,
\begin{align*}
    \rbr{\xb^{\top} \wh{\Bb}^{dac}}_+ = \rbr{\xb^{\top} \projnull \wh{\Bb}^{dac}}_+
    \quad \forall ~ \xb \in \Xcal.
\end{align*}    
\end{lemma} 

\begin{proof}[Proof of \Cref{lemma:2layer-relu-input-consistency-exclude-spurious}]
We will show that for all $\bb_k = \projnull \bb_k + \projrg \bb_k$, $k \in [q]$, $\projrg \bb_k = \b{0}$ with high probability, which then implies that given any $\xb \in \Xcal$, $(\xb^{\top} \bb_k)_+ = (\xb^{\top} \projnull \bb_k)_+$ for all $k \in [q]$.

For any $k \in [q]$ associated with an arbitrary fixed $\bb_k \in \mathbb{S}^{d-1}$, let $\Xb_k \triangleq \Xb_k \projnull + \Xb_k \projrg \in \Xcal^{N_k}$ be the inclusion-wisely maximum row subset of $\Xb$ such that $\Xb_k \bb_k > \b{0}$ element-wisely. 
Meanwhile, we denote $\Aemp(\Xb_k) = \Mb_k \Xb_k \projnull + \Aemp(\Xb_k) \projrg \in \Xcal^{\alpha N_k}$ as the augmentation of $\Xb_k$ where $\Mb_k \in \R^{\alpha N_k \times N_k}$ is the vertical stack of $\alpha$ identity matrices with size $N_k \times N_k$.
Then the DAC constraint implies that $(\Aemp(\Xb_k) - \Mb_k \Xb_k) \projrg \bb_k = \b{0}$.

With \Cref{ass:observable_marginal_distribution}, for a fixed $\bb_k \in \mathbb{S}^{d-1}$, $\PP[\xb^{\top} \bb_k > 0] = \frac{1}{2}$. Then, with the Chernoff bound,
\begin{align*}
    \PP\sbr{N_k < \frac{N}{2} - t} \leq e^{-\frac{2 t^2}{N}},
\end{align*}
which implies that, $N_k \geq \frac{N}{4}$ with high probability.

Leveraging the assumptions in \Cref{thm:case_2layer_relu_risk}, $\alpha N \geq 4 \dau$ implies that $\alpha N_k \geq \dau$. 
Therefore by \Cref{lemma:total_invertible_meas_zero}, $\row\rbr{\Aemp(\Xb_k) - \Mb_k \Xb_k} = \row\rbr{\Deltab}$ with probability $1$.
Thus, $(\Aemp(\Xb_k) - \Mb_k \Xb_k) \projrg \bb_k = \b{0}$ enforces that $\projrg \bb_k = \b{0}$.
\end{proof}

\begin{proof}[Proof of \Cref{thm:case_2layer_relu_risk}]
Conditioned on $\Xb$ and $\Deltab$, we are interested in the excess risk $L(\wh h^{dac}) - L(h^*) = \frac{1}{N} \norm{(\Xb \wh{\Bb}^{dac})_+ \wh{\wb}^{dac} - (\Xb\Bb^*)_+ \wb^*}_2^2$ with randomness on $\epsb$.

We first recall that \Cref{lemma:2layer-relu-input-consistency-exclude-spurious} implies $\wh h^{dac} \in \Hred = \csepp{h(\xb) = \rbr{\xb^\top \Bb}_+ \wb}{\Bb \in \Bcal,~\norm{\wb}_1 \leq C_w}$ where 
\begin{align*}
    \Bcal \triangleq \cbr{\Bb=[\bb_1 \dots \bb_q] ~|~ \norm{\bb_k}=1\ \forall\ k \in [q], (\Xb \Bb)_+ = (\Xb \projnull \Bb)_+ }.
\end{align*}    
Leveraging Equation (21) and (22) in \cite{du2020fewshot},
since $(\Bb^*, \wb^*)$ is feasible under the constraint, by the basic inequality,
\begin{align}
    \label{eq:2layer-relu-basic-ineq}
    \norm{\yb - (\Xb \wh{\Bb}^{dac})_+ \wh{\wb}^{dac}}_2^2 
    \leq
    \norm{\yb - (\Xb \Bb^*)_+ \wb^*}_2^2.
\end{align}
Knowing that $\yb = (\Xb \Bb^*)_+ \wb^* + \epsb$ with $\epsb \sim \Ncal\rbr{\b0, \sigma^2\Ib_N}$, we can rewrite \Cref{eq:2layer-relu-basic-ineq} as
\begin{align*}
    \frac{1}{N} \norm{(\Xb \wh{\Bb}^{dac})_+ \wh{\wb}^{dac} - (\Xb\Bb^*)_+ \wb^*}_2^2
    \le &\frac{2}{N} \epsb^\top \rbr{(\Xb \wh{\Bb}^{dac})_+ \wh{\wb}^{dac} - (\Xb\Bb^*)_+ \wb^*}
    \\
    \le & 4 \sup_{h \in \Hred} \frac{1}{N} \epsb^\top h(\Xb)
\end{align*}    
First, we observe that $\sigma^{-1}\E_{\epsb}\sbr{\sup_{h \in \Hred} \frac{1}{N} \epsb^\top h(\Xb)} = \wh{\Gfrak}_{\Xb}\rbr{\Hred}$ measures the empirical Gaussian width of $\Hred$ over $\Xb$. Moreover, by observing that for any $h \in \Hred$ and $\xb_i \in \Xb$,
\begin{align*}
    &\abbr{h(\xb_i)} 
    \le \nbr{\rbr{\Bb^\top \xb_i}_+}_{\infty} \nbr{\wb}_1 
    \le \max_{k \in [q]} \abbr{\bb_k^\top \projnull \xb_i} \nbr{\wb}_1 
    \le \nbr{\projnull \xb_i}_2 \nbr{\wb}_1,
    \\
    &\frac{1}{N}\nbr{h(\Xb)}^2_2
    = \frac{1}{N} \sum_{i=1}^N \abbr{h(\xb_i)}^2
    \le \nbr{\wb}_1^2 \cdot \frac{1}{N}\sum_{i=1}^N \nbr{\projnull \xb_i}_2^2
    \le C_w^2 \Cnull^2
\end{align*}
and 
\begin{align*}
    &\abbr{\sup_{h \in \Hred} \frac{1}{N} \epsb_1^\top h(\Xb) - \sup_{h \in \Hred} \frac{1}{N} \epsb_2^\top h(\Xb)}
    \\
    \le &\abbr{\sup_{h \in \Hred} \frac{1}{N} h(\Xb)^\top \rbr{\epsb_1 - \epsb_2}}
    \\
    \le &\frac{1}{\sqrt{N}} \nbr{\frac{1}{\sqrt{N}} h(\Xb)}_2 \nbr{\epsb_1 - \epsb_2}_2
    \\
    \le &\frac{C_w \Cnull}{\sqrt{N}} \nbr{\epsb_1 - \epsb_2}_2,
\end{align*}    
we know that the function $\epsb \to \sup_{h \in \Hred} \frac{1}{N} \epsb^\top h(\Xb)$ is $\frac{\Cnull C_w}{\sqrt{N}}$-Lipschitz in $\ell_2$ norm. Therefore, by \cite{wainwright2019} Theorem 2.26, we have that with probability at least $1-\delta$,
\begin{align*}
    \sup_{h \in \Hred} \frac{1}{N} \epsb^\top h(\Xb) \le \sigma \cdot \rbr{\wh{\Gfrak}_{\Xb}\rbr{\Hred} + C_w \Cnull \sqrt{\frac{2 \log(1/\delta)}{N}}}
,\end{align*}
where the empirical Gaussian complexity is upper bounded by
\begin{align*}
    \wh{\Gfrak}_{\Xb}\rbr{\Hred} 
    =
    & \underset{\gb \sim \Ncal(\b{0}, \Ib_N)}{\E} \sbr{
    \underset{\Bb \in \Bcal, \norm{\wb}_1 \leq R}{\sup}\ 
    \frac{1}{N} \gb^{\top} (\Xb\Bb)_+ \wb } \\ 
    \leq
    & \frac{C_w}{N}\ \underset{\gb}{\E} \sbr{
    \underset{\Bb \in \Bcal}{\sup}\ 
    \norm{(\Xb\Bb)_+^{\top} \gb}_{\infty}} \\ 
    =
    & \frac{C_w}{N}\ \underset{\gb}{\E} \sbr{
    \underset{\bb \in \mathbb{S}^{d-1}}{\sup}\ 
    \gb^{\top} \rbr{\Xb \projnull \bb}_+} 
    \quad
    \rbr{\t{\Cref{lemma:tech_gaussian_width_lipschitz}, $(\cdot)_+$ is $1$-Lipschitz} } \\
    \leq 
    & \frac{C_w}{N}\ \underset{\gb}{\E} \sbr{
    \underset{\bb \in \mathbb{S}^{d-1}}{\sup}\ 
    \gb^{\top} \Xb \projnull \bb} \\
    =
    & \frac{C_w}{N}\ \underset{\gb}{\E} 
    \sbr{\norm{\projnull \Xb^{\top} \gb}_2} \\
    \leq
    & \frac{C_w}{N}\ \rbr{\underset{\gb}{\E} \sbr{\norm{\projnull \Xb^{\top} \gb}_2^2}}^{1/2} \\ = 
    & \frac{C_w}{N}\ \sqrt{\tr(\projnull \Xb^{\top} \Xb \projnull)} \\ = 
    & \frac{C_w \Cnull}{\sqrt{N}}.
\end{align*}
\end{proof}

\begin{proof}[Proof of \Cref{coro:two_layer_daug_bound}]
By \Cref{def:daug}, we have with probability at least $1-\delta$ that $\dau = \rank(\projrg) \ge \dau(\delta)$ and $\rank(\projnull) \le d-\dau(\delta)$.
Meanwhile, leveraging \Cref{lemma:sample-population-covariance}, we have that under \Cref{ass:observable_marginal_distribution} and with $N \gg \rho^4 d$, with high probability,
\begin{align*}
    \norm{\frac{1}{N} \projnull \Xb^{\top} \Xb \projnull}_2 
    \leq \norm{\frac{1}{N} \Xb^{\top} \Xb }_2
    \le 1.1 C \lesssim 1.
\end{align*}
Therefore, there exists $\Cnull > 0$ with $\frac{1}{N} \sum_{i=1}^n \norm{\projnull \xb_i}_2^2 \leq \Cnull^2$ such that, with probability at least $1-\delta$,
\begin{align*}
    \Cnull^2 \leq \rbr{d-\dau} \cdot \norm{\frac{1}{N} \projnull \Xb^{\top} \Xb \projnull}_2 \lesssim d-\dau(\delta).
\end{align*}
\end{proof}

\section{Classification with Expansion-based Augmentations}\label{apx:generalized_DAC}

We first recall the multi-class classification problem setup in \Cref{subsec:expansion_based}, while introducing some helpful notions.
For an arbitrary set $\Xcal$, let $\Ycal=[K]$, and $h^*: \Xcal \to [K]$ be the ground truth classifier that partitions $\Xcal$: for each $k \in [K]$, let $\Xcal_k \triangleq \cbr{\xb \in \Xcal ~|~ h^*(\xb)=k}$, with $\Xcal_i \cap \Xcal_j = \emptyset, \forall i \neq j$. 
In addition, for an arbitrary classifier $h: \Xcal \to [K]$, we denote the majority label with respect to $h$ for each class,
\begin{align*}
    \wh{y}_k \triangleq \underset{y \in [K]}{\argmax}\ \PP_{\Pgt} \sbr{h(\xb)=y ~\big|~ \xb \in \Xcal_k} 
    \quad \forall\ k \in [K],
\end{align*}
along with the respective class-wise local and global minority sets,
\begin{align*}
    M_k \triangleq \cbr{\xb \in \Xcal_k ~\big|~ h(\xb) \neq \wh{y}_k} \subsetneq \Xcal_k
    \quad \forall\ k \in [K], \quad
    M \triangleq \bigcup_{k=1}^K M_k.
\end{align*}

Given the marginal distribution $\Pgt\rbr{\xb}$, we introduce the \textit{expansion-based data augmentations} that concretizes \Cref{def:causal_invar_data_aug} in the classification setting:

\begin{definition}[Expansion-based data augmentations, \cite{cai2021theory}]
\label{def:generalized-causal-invariant-data-augmentation}
We call $\Acal:\Xcal \to 2^{\Xcal}$ an augmentation function that induces expansion-based data augmentations if $\Acal$ is class invariant: $\cbr{\xb} \subsetneq \Acal(\xb) \subseteq \cbr{\xb' \in \Xcal ~|~ h^*(\xb) = h^*(\xb')}$ for all $\xb \in \Xcal$.
Let 
\begin{align*}
    \nbh(\xb) \triangleq \cbr{\xb' \in \Xcal ~\big|~ \Acal(\xb) \cap \Acal(\xb') \neq \emptyset},
    \quad
    \nbh(S) \triangleq \cup_{\xb \in S} \nbh(\xb)
\end{align*}
be the neighborhoods of $\xb \in \Xcal$ and $S \subseteq \Xcal$ with respect to $\Acal$.
Then, $\Acal$ satisfies
\begin{enumerate}[nosep,leftmargin=*,label=(\alph*)]
    \item \underline{$(q,\xi)$-constant expansion} if given any $S \subseteq \Xcal$ with $\Pgt\rbr{S} \geq q$ and $\Pgt\rbr{S \cap \Xcal_k} \leq \frac{1}{2}$ for all $k \in [K]$,
    $\Pgt\rbr{\nbh\rbr{S}} \geq \min\cbr{\Pgt\rbr{S},\xi} + \Pgt\rbr{S}$;
    \item \underline{$(a,c)$-multiplicative expansion} if for all $k \in [K]$, given any $S \subseteq \Xcal$ with $\Pgt\rbr{S \cap \Xcal_k} \leq a$,
    $\Pgt\rbr{\nbh\rbr{S} \cap \Xcal_k} \geq \min\cbr{c \cdot \Pgt\rbr{S \cap \Xcal_k},1}$.
\end{enumerate}
\end{definition}

On \Cref{def:generalized-causal-invariant-data-augmentation}, we first point out that the ground truth classifier is invariant throughout the neighborhood: given any $\xb \in \Xcal$, $h^*\rbr{\xb} = h^*\rbr{\xb'}$ for all $\xb' \in \nbh(\xb)$. 
Second, in contrast to the linear regression and two-layer neural network cases where we assume $\Xcal \subseteq R^d$, with the expansion-based data augmentation over a general $\Xcal$, the notion of $\dau$ in \Cref{def:daug} is not well-established. Alternatively, we leverage the concept of constant / multiplicative expansion from \cite{cai2021theory}, and quantify the augmentation strength with parameters $(q,\xi)$ or $(a,c)$. Intuitively, the strength of expansion-based data augmentations is characterized by expansion capability of $\Acal$: for a neighborhood $S \subseteq \Xcal$ of proper size (characterized by $q$ or $a$ under measure $\Pgt$), the stronger augmentation $\Acal$ leads to more expansion in $\nbh(S)$, and therefore larger $\xi$ or $c$. For example in \Cref{def:generalized-causal-invariant-data-augmentation_informal}, we use an expansion-based augmentation function $\Acal$ that satisfies $\rbr{\frac{1}{2}, c}$-multiplicative expansion.

Adapting the existing setting in \cite{wei2021theoretical, cai2021theory}, we concretize the classifier class $\Hcal$ with a function class $\Fcal \subseteq \cbr{f:\Xcal \to \R^K}$ of fully connected neural networks such that $\Hcal = \csepp{h(\xb) \triangleq \argmax_{k \in [K]}\ f(\xb)_k}{f \in \Fcal}$.
To constrain the feasible hypothesis class through the DAC regularization with finite unlabeled samples, we recall the notion of all-layer-margin, $m: \Fcal \times \Xcal \times \Ycal \to \R_{\geq 0}$ (from \cite{wei2021theoretical}) that measures the maximum possible perturbation in all layers of $f$ while maintaining the prediction $y$. 
Precisely, given any $f \in \Fcal$ such that $f\rbr{\xb} = \Wb_p \varphi\rbr{\dots\varphi\rbr{\Wb_1 \xb}\dots}$ for some activation function $\varphi: \R \to \R$ and parameters $\cbr{\Wb_{\iota} \in \R^{d_{\iota} \times d_{\iota-1}}}_{\iota=1}^p$, we can write $f = f_{2p-1} \circ \dots \circ f_1$ where $f_{2\iota-1}(\xb) = \Wb_{\iota} \xb$ for all $\iota \in [p]$ and $f_{2\iota}(\zb)=\varphi(\zb)$ for $\iota \in [p-1]$.
For an arbitrary set of perturbation vectors $\deltab = \rbr{\deltab_1,\dots,\deltab_{2p-1}}$ such that $\deltab_{2\iota-1}, \deltab_{2\iota} \in \R^{d_{\iota}}$ for all $\iota$, let $f(\xb,\deltab)$ be the perturbed neural network defined recursively such that
\begin{align*}
    & \wt{\zb}_1 = f_1\rbr{\xb} + \norm{\xb}_2 \deltab_1, \\
    & \wt{\zb}_{\iota} = f_{\iota}\rbr{\wt{\zb}_{\iota-1}} + \norm{\wt{\zb}_{\iota-1}}_2 \deltab_{\iota}
    \quad \forall\ \iota=2,\dots,2p-1, \\
    & f(\xb,\deltab) = \wt{\zb}_{2p-1}.
\end{align*}   
The all-layer-margin $m(f,\xb,y)$ measures the minimum norm of the perturbation $\deltab$ such that $f(\xb,\deltab)$ fails to provide the classification $y$,
\begin{align}
    \label{eq:def-all-layer-margin}
    m(f,\xb,y) \triangleq
    \underset{\deltab = \rbr{\deltab_1,\dots,\deltab_{2p-1}}}{\min}
    \sqrt{\sum_{\iota=1}^{2p-1} \norm{\deltab_{\iota}}_2^2}
    \quad \t{s.t.} \quad
    \underset{k \in [K]}{\argmax}\ f(\xb,\deltab)_k \neq y.
\end{align}
With the notion of all-layer-margin established, for any $\Acal:\Xcal \to 2^{\Xcal}$ that satisfies conditions in \Cref{def:generalized-causal-invariant-data-augmentation}, the robust margin is defined as
\begin{align*}
    m_{\Acal}(f,\xb) \triangleq \underset{\xb' \in \Acal(\xb)}{\sup}\  
    m\rbr{f, \xb', \argmax_{k \in [K]}\ f(\xb)_k}.
\end{align*}
Intuitively, the robust margin measures the maximum possible perturbation in all-layer weights of $f$ such that predictions on all data augmentations of $\xb$ remain consistent. For instance, $m_{\Acal}(f,\xb) > 0$ is equivalent to enforcing $h(\xb) = h(\xb')$ for all $\xb' \in \Acal\rbr{\xb}$.

To achieve finite sample guarantees, DAC regularization requires stronger consistency conditions than merely consistent classification outputs ($\ie$, $m_{\Acal}(f,\xb) > 0$). Instead, we enforce $m_{\Acal}(f,\xb) > \tau$ for any $0 < \tau < \max_{f \in \Fcal}\ \inf_{\xb \in \Xcal} m_{\Acal}(f, \xb)$\footnote{The upper bound on $\tau$ ensures the proper learning setting, $\ie$, there exists $f \in \Fcal$ such that $m_\Acal\rbr{f,\xb} > \tau$ for all $\xb \in \Xcal$.} over an finite set of unlabeled samples $\Xb^u$ that is independent of $\Xb$ and drawn $\iid$ from $P(\xb)$.
Then, learning the classifier with zero-one loss $l_{01}\rbr{h(\xb),y} = \b1\cbr{h(\xb) \neq y}$ from a class of $p$-layer fully connected neural networks with maximum width $q$,
\begin{align*}
    \Fcal = 
    \csepp{f: \Xcal \to \R^K}
    {f = f_{2p-1} \circ \dots \circ f_1,\ f_{2\iota-1}(\xb) = \Wb_{\iota} \xb,\ f_{2\iota}(\zb)=\varphi(\zb)},
\end{align*}
where $\Wb_{\iota} \in \R^{d_{\iota} \times d_{\iota-1}}$ for all $\iota \in [p]$, and $q \triangleq \max_{\iota=0,\dots,p} d_{\iota}$, we solve
\begin{align}\label{eq:gdac_finite}
    \hgdacfin\ \triangleq\ 
    & \underset{h \in \Hcal}{\argmin}\ 
    \wh{L}^{dac}_{01}(h) = 
    \frac{1}{N} \sum_{i=1}^N \b{1}\cbr{h\rbr{\xb_i} \neq h^*\rbr{\xb_i}} 
    \\
    & \t{s.t.} \quad
    m_{\Acal}(f,\xb^u_i)> \tau \quad \forall\ i \in [\abbr{\Xb^u}] \nonumber 
\end{align}
for any $0 < \tau < \max_{f \in \Fcal}\ \inf_{\xb \in \Xcal} m_{\Acal}(f, \xb)$.
The corresponding reduced function class is given by
\begin{align*}
    \Hred \triangleq 
    \csepp{h \in \Hcal}
    {m_{\Acal}(f,\xb^u_i)> \tau \quad \forall\ i \in [\abbr{\Xb^u}]}.
\end{align*}
Specifically, with $\mu \triangleq \sup_{h \in \Hred} \PP_{\Pgt} \sbr{\exists\ \xb' \in \Acal(\xb): h(\xb) \neq h(\xb')}$, \cite{wei2021theoretical, cai2021theory} demonstrate the following for $\Hred$:
\begin{proposition}
[\cite{wei2021theoretical} Theorem 3.7, \cite{cai2021theory} Proposition 2.2]
\label{prop:non-robust-upper-bound}
For any $\delta \in (0,1)$, with probability at least $1-\delta/2$ (over $\Xb^u$), 
\begin{align*}
    \mu 
    \leq \wt{O} \rbr{\frac{\sum_{\iota=1}^p \sqrt{q} \norm{\Wb_{\iota}}_F}{\tau \sqrt{\abbr{\Xb^u}}} + \sqrt{\frac{\log\rbr{1/\delta} + p \log \abbr{\Xb^u}}{\abbr{\Xb^u}}}}
,\end{align*}
where $\wt{O}\rbr{\cdot}$ hides polylogarithmic factors in $\abbr{\Xb^u}$ and $d$.
\end{proposition}   

Leveraging the existing theory above on finite sample guarantee of the maximum possible inconsistency, we have the following.
\begin{theorem}[Formal restatement of \Cref{thm:generalized-dac-finite-unlabeled_informal} on classification with DAC]
\label{thm:generalized-dac-finite-unlabeled}
Learning the classifier with DAC regularization in \Cref{eq:gdac_finite} provides that, for any $\delta \in (0,1)$, with probability at least $1-\delta$,
\begin{align}
    \label{eq:generalization-bound-generalized-dac-finite}
    L_{01}\rbr{\hgdacfin} - L_{01}\rbr{h^*} \leq 
    4 \Rndac + \sqrt{\frac{2 \log(4/\delta)}{N}},
\end{align}
where with $0 < \mu < 1$ defined in \Cref{prop:non-robust-upper-bound}, for any $0 \leq q < \frac{1}{2}$ and $c > 1+4\mu$, 
\begin{enumerate}[label=(\alph*),nosep]
    \item when $\Acal$ satisfies $(q,2\mu)$-constant expansion, $\Rndac \leq \sqrt{\frac{2K \log(2N)}{N} + 2 \max\cbr{q,2\mu}}$;
    \item when $\Acal$ satisfies $(\frac{1}{2},c)$-multiplicative expansion, $\Rndac \leq \sqrt{\frac{2K \log(2N)}{N} + \frac{4 \mu}{\min\cbr{c-1,1}}}$.
\end{enumerate}
\end{theorem} 

First, to quantify the function class complexity and relate it to the generalization error, we leverage the notion of Rademacher complexity and the associated standard generalization bound.
\begin{lemma}\label{lemma:rademacher_generalization_bound}
Given a fixed function class $\Hred$ (\ie, conditioned on $\Xb^u$) and a $B$-bounded and $C_l$-Lipschitz loss function $l$, let $\wh L(h) = \frac{1}{N} \sum_{i=1}^N l(h(\xb_i),y_i)$, $L(h) = \E\sbr{l(h(\xb_i),y_i)}$, and $\wh h^{dac} = \argmin_{h \in \Hred} \wh L(h)$. Then for any $\delta \in (0,1)$, with probability at least $1-\delta$ over $\Xb$, 
\begin{align*}
    L(\widehat h^{dac}) - L(h^*) \le & 4 C_l \cdot \Rfrak_N\rbr{\Hred} + \sqrt{\frac{2B^2\log(4/\delta)}{N}}.
\end{align*}
\end{lemma}

\begin{proof}[Proof of \Cref{lemma:rademacher_generalization_bound}]
We first decompose the expected excess risk as
\[
L(\wh{h}^{dac}) - L(h^*) = 
\rbr{L(\wh{h}^{dac}) - \wh{L}(\wh{h}^{dac})} + \rbr{\wh{L}(\wh{h}^{dac}) - \wh{L}(h^*)} + \rbr{\wh{L}(h^*) - L(h^*)},
\]
where $\wh{L}(\wh{h}^{dac}) - \wh{L}(h^*) \leq 0$ by the basic inequality. Since both $\wh h^{dac}, h^* \in \Hred$, we then have
\begin{align*}
    L(\wh{h}^{dac}) - L(h^*) 
    \leq 2 \sup_{h \in \Hred}\ \abbr{L(h) - \wh{L}(h)}.
\end{align*}    
Let $g^+(\Xb,\yb) = \sup_{h \in \Hred}: L(h) - \wh{L}(h)$ and 
$g^-(\Xb,\yb) = \sup_{h \in \Hred}: - L(h) + \wh{L}(h)$. Then,
\[
\PP\sbr{L(\wh{h}^{dac}) - L(h^*) \geq \epsilon} \leq
\PP\sbr{g^+(\Xb,\yb) \geq \frac{\eps}{2}} + 
\PP\sbr{g^-(\Xb,\yb) \geq \frac{\eps}{2} }.
\]
We will derive a tail bound for $g^+(\Xb,\yb)$ with the standard inequalities and symmetrization argument \cite{wainwright2019,bartlett2003}, while the analogous statement holds for $g^-(\Xb,\yb)$.

Let $(\Xb^{(1)}, \yb^{(1)})$ be a sample set generated by replacing an arbitrary sample in $(\Xb, \yb)$ with an independent sample $(\xb, y) \sim P(\xb,y)$. Since $l$ is $B$-bounded, we have $\abbr{g^+(\Xb, \yb) - g^+(\Xb^{(1)}, \yb^{(1)})} \leq B/N$. Then, via McDiarmid's inequality \cite{bartlett2003},
\[
    \PP\sbr{g^+(\Xb, \yb) \geq \E[g^+(\Xb, \yb)] + t} 
    \leq \exp\rbr{-\frac{2 N t^2}{B^2}}.
\]
For an arbitrary sample set $\rbr{\Xb,\yb}$, let $\wh{L}_{\rbr{\Xb,\yb}}\rbr{h} = \frac{1}{N} \sum_{i=1}^N l\rbr{h(\xb_i),y_i}$ be the empirical risk of $h$ with respect to $\rbr{\Xb,\yb}$.
Then, by a classical symmetrization argument (e.g., proof of \cite{wainwright2019} Theorem 4.10), we can bound the expectation: for an independent sample set $\rbr{\Xb',\yb'} \in \Xcal^N \times \Ycal^N$ drawn $\iid$ from $\Pgt$, 
\begin{align*}
    \E\sbr{g^+(\Xb, \yb)} 
    = & \E_{(\Xb,\yb)} \sbr{ \sup_{h \in \Hred}\ L(h) - \wh{L}_{(\Xb,\yb)}(h)}
    \\
    = & \E_{(\Xb,\yb)} \sbr{ \sup_{h \in \Hred}\ \E_{(\Xb',\yb')} \sbr{\wh L_{(\Xb',\yb')}(h)} - \wh{L}_{(\Xb,\yb)}(h)}
    \\
    = & \E_{(\Xb,\yb)} \sbr{ \sup_{h \in \Hred}\ \E_{(\Xb',\yb')} \sbr{\wh L_{(\Xb',\yb')}(h) - \wh{L}_{(\Xb,\yb)}(h) ~\middle|~ \rbr{\Xb,\yb}} }
    \\
    \leq & \E_{(\Xb,\yb)} \sbr{ \E_{(\Xb',\yb')} \sbr{ \sup_{h \in \Hred}\ \wh L_{(\Xb',\yb')}(h) - \wh{L}_{(\Xb,\yb)}(h) ~\middle|~ \rbr{\Xb,\yb}} } 
    \\
    & \rbr{\t{Law of iterated conditional expectation}}
    \\
    = & \E_{\rbr{\Xb,\yb,\Xb',\yb'}} \sbr{\sup_{h \in \Hred}\ \wh L_{(\Xb',\yb')}(h) - \wh{L}_{(\Xb,\yb)}(h) }
\end{align*}
Since $\rbr{\Xb,\yb}, \rbr{\Xb',\yb'}$ are drawn $\iid$ from $\Pgt$, we can introduce $\iid$ Rademacher random variables $\rb = \cbr{r_i \in \cbr{-1,+1} ~|~ i \in [N]}$ (independent of both $\rbr{\Xb,\yb}$ and $\rbr{\Xb',\yb'}$) such that
\begin{align*}
    \E\sbr{g^+(\Xb, \yb) } 
    \leq & \E_{\rbr{\Xb,\yb,\Xb',\yb',\rb}} \sbr{\sup_{h \in \Hred}\ \frac{1}{N} \sum_{i=1}^N r_i \cdot \rbr{l\rbr{h\rbr{\xb'_i},y'_i} - l\rbr{h\rbr{\xb_i},y_i}} }
    \\
    \leq & 2\ E_{\rbr{\Xb,\yb,\rb}} \sbr{ \sup_{h \in \Hred}\ \frac{1}{N} \sum_{i=1}^N r_i \cdot l\rbr{h\rbr{\xb_i},y_i} }
    \\
    \leq & 2\ \Rfrak_N\rbr{l \circ \Hred}
\end{align*}
where $l \circ \Hred = \cbr{l(h(\cdot), \cdot): \Xcal \times \Ycal \to \R: h \in \Hred}$ is the loss function class, and 
\begin{align*}
    \Rfrak_N\rbr{\Fcal} 
    \triangleq E_{\rbr{\Xb,\yb,\rb}} \sbr{ \sup_{f \in \Fcal}\ \frac{1}{N} \sum_{i=1}^N r_i \cdot f\rbr{\xb_i,y_i}}
\end{align*}
denotes the Rademacher complexity. 
Analogously, $\E[g^-(\Xb, \yb)] \leq 2\Rfrak_N\rbr{l \circ \Hred}$. 
Therefore, assuming that $\dacop(\Hcal) \subseteq \Hred(\Hcal)$ holds, with probability at least $1-\delta/2$,
\[
L(\wh{h}^{dac}) - L(h^*) \leq 4 \Rfrak_N\rbr{l \circ \Hred} + \sqrt{\frac{2 B^2 \log(4/\delta)}{N}}
\]

Finally, since $l(\cdot, y)$ is $C_l$-Lipschitz for all $y \in \Ycal$, by \cite{ledoux2013} Theorem 4.12, we have $\Rfrak_N\rbr{l \circ \Hred} \leq C_l \cdot \Rfrak_N\rbr{\Hred} $.
\end{proof}

\begin{lemma}
[\cite{cai2021theory}, Lemma A.1]
\label{lemma:minority-set-upper-bound-expansion-assumption}
For any $h \in \Hred$, when $\Pgt$ satisfies
\begin{enumerate}[label=(\alph*), nosep]
    \item $\rbr{q, 2\mu}$-constant expansion with $q < \frac{1}{2}$, $\Pgt\rbr{M} \leq \max\cbr{q, 2\mu}$;
    \item $\rbr{\frac{1}{2},c}$-multiplicative expansion with $c>1+4\mu$, $\Pgt\rbr{M} \leq \max\cbr{\frac{2\mu}{c-1},2\mu}$.
\end{enumerate}
\end{lemma} 

\begin{proof}
[Proof of \Cref{lemma:minority-set-upper-bound-expansion-assumption}]
We start with the proof for \Cref{lemma:minority-set-upper-bound-expansion-assumption} (a).
By definition of $M_k$ and $\wh{y}_k$, we know that $M_k = M \cap \Xcal_k \leq \frac{1}{2}$. Therefore, for any $0 < q < \frac{1}{2}$, one of the following two cases holds:
\begin{enumerate}[label=(\roman*), nosep]
    \item $\Pgt\rbr{M} < q$;
    \item $\Pgt\rbr{M} \geq q$.
    Since $\Pgt\rbr{M \cap \Xcal_k} < \frac{1}{2}$ for all $k \in [K]$ holds by construction, with the $\rbr{q, 2\mu}$-constant expansion, $\Pgt\rbr{\nbh\rbr{M}} \geq \min\cbr{\Pgt\rbr{M}, 2\mu} + \Pgt\rbr{M}$.
    
    Meanwhile, since the ground truth classifier $h^*$ is invariant throughout the neighborhoods, $\nbh\rbr{M_k} \cap \nbh\rbr{M_{k'}} = \emptyset$ for $k \neq k'$, and therefore $\nbh\rbr{M} \backslash M = \bigcup_{k=1}^K \nbh\rbr{M_k} \backslash M_k$ with each $\nbh\rbr{M_k} \backslash M_k$ disjoint.
    Then, we observe that for each $\xb \in \nbh\rbr{M} \backslash M$, here exists some $k = h^*\rbr{\xb}$ such that $\xb \in \nbh\rbr{M_k} \backslash M_k$.
    $\xb \in \Xcal_k \backslash M_k$ implies that $h\rbr{\xb} = \wh{y}_k$, while
    $\xb \in \nbh\rbr{M_k}$ suggests that there exists some $\xb' \in \Acal\rbr{\xb} \cap \Acal\rbr{\xb''}$ where $\xb'' \in M_k$ such that either 
    $h\rbr{\xb'} = \wh{y}_k$ and $h\rbr{\xb'} \neq h\rbr{\xb''}$ for $\xb' \in \Acal\rbr{\xb''}$, 
    or 
    $h\rbr{\xb'} \neq \wh{y}_k$ and $h\rbr{\xb'} \neq h\rbr{\xb}$ for $\xb' \in \Acal\rbr{\xb}$. Therefore, we have 
    \[
    \Pgt\rbr{\nbh\rbr{M} \backslash M} \leq 
    2 \PP_{\Pgt}\sbr{\exists\ \xb' \in \Acal(\xb)\ \t{s.t.}\ h(\xb) \neq h(\xb')} \leq 2\mu.
    \]
    Moreover, since $\Pgt\rbr{\nbh\rbr{M}} - \Pgt\rbr{M} \leq \Pgt\rbr{\nbh\rbr{M} \backslash M} \leq 2 \mu$, we know that 
    \[
    \min\cbr{\Pgt\rbr{M}, 2\mu} + \Pgt\rbr{M} \leq 
    \Pgt\rbr{\nbh\rbr{M}} \leq 
    \Pgt\rbr{M} + 2 \mu. 
    \]
    That is, $\Pgt\rbr{M} \leq 2 \mu$.
\end{enumerate}
Overall, we have $\Pgt\rbr{M} \leq \max\cbr{q, 2\mu}$.

To show \Cref{lemma:minority-set-upper-bound-expansion-assumption} (b), we recall from \cite{wei2021theoretical} Lemma B.6 that for any $c>1+4\mu$, $\rbr{\frac{1}{2},c}$-multiplicative expansion implies $\rbr{\frac{2\mu}{c-1}, 2\mu}$-constant expansion. 
Then leveraging the proof for \Cref{lemma:minority-set-upper-bound-expansion-assumption} (a), with $q = \frac{2\mu}{c-1}$, we have $\Pgt\rbr{M} \leq \max\cbr{\frac{2\mu}{c-1}, 2\mu}$.
\end{proof}

\begin{proof}
[Proof of \Cref{thm:generalized-dac-finite-unlabeled}]
To show \Cref{eq:generalization-bound-generalized-dac-finite}, we leverage \Cref{lemma:rademacher_generalization_bound} and observe that $B = 1$ with the zero-one loss. Therefore, conditioned on $\Hred$ (which depends only on $\Xb^u$ but not on $\Xb$), for any $\delta \in (0,1)$, with probability at least $1-\delta/2$,
\begin{align*}
    L_{01}\rbr{\hgdacfin} - L_{01}\rbr{h^*} \leq 
    4 \Rfrak_N\rbr{l_{01} \circ \Hred} + \sqrt{\frac{2 \log(4/\delta)}{N}}
.\end{align*}
For the upper bounds of the Rademacher complexity, let $\wt{\mu} \triangleq \sup_{h \in \Hred} \Pgt\rbr{M}$ where $M$ denotes the global minority set with respect to $h \in \Hred$.
\Cref{lemma:minority-set-upper-bound-expansion-assumption} suggests that
\begin{enumerate}[label=(\alph*), nosep]
    \item when $\Pgt$ satisfies $(q,2\mu)$-constant expansion for some $q < \frac{1}{2}$, $\wt{\mu} \leq \max\cbr{q,2\mu}$; while
    \item when $\Pgt$ satisfies $(\frac{1}{2},c)$-multiplicative expansion for some $c > 1+4\mu$,
    $\wt{\mu} \leq \frac{2\mu}{\min\cbr{c-1,1}}$.
\end{enumerate}
Then, it is sufficient to show that, conditioned on $\Hred$, 
\begin{align}\label{eq:pf_Rndac}
    \Rfrak_N\rbr{l_{01} \circ \Hred} \leq \sqrt{\frac{2K \log(2N)}{N} + 2\wt{\mu}}. 
\end{align}    

To show this, we first consider a fixed set of $n$ observations in $\Xcal$, $\Xb = \sbr{\xb_1,\dots,\xb_N}^{\top} \in \Xcal^N$. 
Let the number of distinct behaviors over $\Xb$ in $\Hred$ be
\begin{align*}
    \mathfrak{s}\rbr{l_{01} \circ \Hred, \Xb} \triangleq
    \abbr{\cbr{\sbr{l_{01} \circ h\rbr{\xb_1},\dots,l_{01} \circ h\rbr{\xb_N}} ~\big|~ h \in \Hred}}.
\end{align*}
Then, by the Massart's finite lemma, the empirical rademacher complexity with respect to $\Xb$ is upper bounded by
\begin{align*}
    \wh{\Rfrak}_{\Xb}\rbr{l_{01} \circ \Hred} \leq \sqrt{\frac{2 \log \mathfrak{s}\rbr{l_{01} \circ \Hred, \Xb}}{N}}.
\end{align*}
By the concavity of $\sqrt{\log\rbr{\cdot}}$, we know that, 
\begin{align}
\label{eq:rademacher-upper-from-shatter}
    \Rfrak_N\rbr{l_{01} \circ \Hred} = 
    & \E_{\Xb} \sbr{\wh{\Rfrak}_{\Xb}\rbr{l_{01} \circ \Hred}} \leq 
    \E_{\Xb} \sbr{\sqrt{\frac{2 \log \mathfrak{s}\rbr{l_{01} \circ \Hred, \Xb}}{N}}} 
    \nonumber \\ \leq & 
    \sqrt{\frac{2 \log \E_{\Xb} \sbr{\mathfrak{s}\rbr{l_{01} \circ \Hred, \Xb}}}{N}}.
\end{align}    
Since $\Pgt\rbr{M} \leq \wt{\mu} \leq \frac{1}{2}$ for all $h \in \Hred$, we have that, conditioned on $\Hred$,
\begin{align*}
    \E_{\Xb} \sbr{\mathfrak{s}\rbr{l_{01} \circ \Hred, \Xb}} \leq & 
    \sum_{r=0}^N \binom{N}{r} \wt{\mu}^r \rbr{1-\wt{\mu}}^{N-r} \cdot \binom{N-r-1}{\min(K,N-r)-1} 2^{K+r}
    \\ \leq &
    (2N)^K \sum_{r=0}^N \binom{N}{r} \rbr{2\wt{\mu}}^r \rbr{1-\wt{\mu}}^{N-r}
    \\ = &
    (2N)^K \rbr{1-\wt{\mu} + 2\wt{\mu}}^N
    \\ \leq &
    (2N)^K \cdot e^{N \wt{\mu}}.
\end{align*}    
Plugging this into \Cref{eq:rademacher-upper-from-shatter} yields \Cref{eq:pf_Rndac}. Finally, the randomness in $\Hred$ is quantified by $\wt\mu, \mu$, and upper bounded by \Cref{prop:non-robust-upper-bound}. 
\end{proof}

\section{Supplementary Application: Domain Adaptation}\label{apx:case_ood}

As a supplementary example, we demonstrate the possible failure of DA-ERM, and alternatively how DAC regularization can serve as a remedy. Concretely, we consider an illustrative linear regression problem in the domain adaptation setting: with training samples drawn from a source distribution $P^s$ and generalization (in terms of excess risk) evaluated over a related but different target distribution $P^t$. 
With distinct $\E_{P^s}\sbr{y|\xb}$ and $\E_{P^t}\sbr{y|\xb}$, we assume the existence of an unknown but unique inclusionwisely maximal invariant feature subspace $\Xcal_r \subset \Xcal$ such that $P^s\sbr{y|\xb \in \Xcal_r} = P^t\sbr{y|\xb \in \Xcal_r}$, we aim to demonstrate the advantage of the DAC regularization over the ERM on augmented training set, with a provable separation in the respective excess risks.

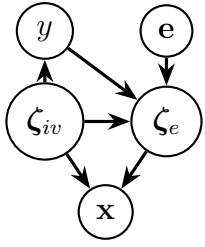
\begin{figure}[!ht]
    \begin{center}
    \begin{tikzpicture}
    \begin{scope}[every node/.style={circle,thick,draw}]
        \node (x) at (0,0) {$\xb$};
        \node (xc) at (-0.8,1.2) {$\zetab_{iv}$};
        \node (xe) at (0.8,1.2) {$\zetab_e$};
        \node (e) at (0.8,2.4) {$\eb$};
        \node (y) at (-0.8,2.4) {$y$};
    \end{scope}
    
    \begin{scope}[>={Stealth[black]},
                  every node/.style={fill=white,circle},
                  every edge/.style={draw=black,very thick}]
        \path [->] (xc) edge (y);
        \path [->] (y) edge (xe);
        \path [->] (xc) edge (xe);
        \path [->] (e) edge (xe);
        \path [->] (xc) edge (x);
        \path [->] (xe) edge (x);
    \end{scope}
    \end{tikzpicture}
    \end{center}
    \caption{Causal graph shared by $P^s$ and $P^t$.}
    \label{fig:data-distribution-causal-graph}
\end{figure}

\paragraph{Source and target distributions.}
Formally, the source and target distributions are concretized with the causal graph in \Cref{fig:data-distribution-causal-graph}. For both $P^s$ and $P^t$, the observable feature $\xb$ is described via a linear generative model in terms of two latent features, the `invariant' feature $\zetab_{iv} \in \R^{d_{iv}}$ and the `environmental' feature $\zetab_e \in \R^{d_e}$:
\begin{align*}
    \xb = g(\zetab_{iv}, \zetab_e) \triangleq \Sb \bmat{\zetab_{iv}; \zetab_e} = \Sb_{iv} \zetab_{iv}  + \Sb_e \zetab_e,
\end{align*}
where $\Sb = \bmat{\Sb_{iv},\Sb_e} \in \R^{d \times (d_{iv} + d_e)}$ ($d_{iv}+d_e \leq d$) consists of orthonormal columns.
Let the label $y$ depends only on the invariant feature $\zetab_{iv}$ for both domains,
\begin{align*}
    y = \rbr{\thetab^{*}}^{\top} \xb + z 
    = \rbr{\thetab^{*}}^{\top} \Sb_{iv} \zetab_{iv} + z,
    \quad 
    z \sim \Ncal\rbr{0, \sigma^2}, 
    \quad
    z \perp \zetab_{iv},
\end{align*}
for some $\thetab^* \in \range\rbr{\Sb_{iv}}$ such that $P^s\sbr{y|\zetab_{iv}} = P^t\sbr{y|\zetab_{iv}}$, while the environmental feature $\zetab_e$ is conditioned on $y$, $\zetab_{iv}$, (along with the Gaussian noise $z$), and varies across different domains $\eb$ with $\E_{P^s}\sbr{y|\xb} \neq \E_{P^t}\sbr{y|\xb}$. In other words, with the square loss $l(h(\xb), y) = \frac{1}{2}(h(\xb)-y)^2$, the optimal hypotheses that minimize the expected excess risk over the source and target distributions are distinct. Therefore, learning via the ERM with training samples from $P^s$ can overfit the source distribution, in which scenario identifying the invariant feature subspace $\range\rbr{\Sb_{iv}}$ becomes indispensable for achieving good generalization in the target domain.

For $P^s$ and $P^t$, we assume the following regularity conditions: 
\begin{assumption}[Regularity conditions for $P^s$ and $P^t$]
\label{ass:case_ood_distribution}
Let $P^s$ satisfy \Cref{ass:observable_marginal_distribution}. While $P^t$ satisfies that $\E_{P^t}[\xb\xb^{\top}] \succ 0$, and
\begin{enumerate}[label=(\alph*),nosep,leftmargin=*]
    \item for the invariant feature, $c_{t,iv} \Ib_{d_{iv}} \aleq \E_{P^t}[\zetab_{iv} \zetab_{iv}^{\top}] \aleq C_{t,iv} \Ib_{d_{iv}}$ for some $C_{t,iv} \geq c_{t,iv} = \Theta(1)$; 
    \item for the environmental feature, $\E_{P^t}[\zetab_{e} \zetab_{e}^{\top}] \ageq c_{t,e} \Ib_{d_e}$ for some $c_{t,e} > 0$, and $\E_{P^t}\sbr{z \cdot \zetab_e} = \b{0}$.
\end{enumerate}
\end{assumption}

\paragraph{Training samples and data augmentations.}
Let $\Xb=\sbr{\xb_1;\dots;\xb_N}$ be a set of $N$ samples drawn $\iid$ from $P^s(\xb)$ such that $\yb=\Xb\thetab^* + \zb$ where $\zb \sim \Ncal(\b0,\sigma^2\Ib_N)$.
Recall that we denote the augmented training sets, including/excluding the original samples, respectively, with
\begin{align*}
    &\wt\Acal(\Xb) = \sbr{\xb_{1}; \cdots; \xb_{N}; \xb_{1,1}; \cdots; \xb_{N,1}; \cdots; \xb_{1, \alpha}; \cdots; \xb_{N, \alpha}} \in \Xcal^{(1+\alpha) N}, 
    \\
    &\Aemp(\Xb) = \sbr{\xb_{1,1}; \cdots; \xb_{N,1}; \cdots; \xb_{1, \alpha}; \cdots; \xb_{N, \alpha}} \in \Xcal^{\alpha N}.
\end{align*}
In particular, we consider a set of augmentations that only perturb the environmental feature $\zetab_e$, while keep the invariant feature $\zetab_{iv}$ intact:
\begin{align}\label{eq:def_envi_data_aug}
    \Sb_{iv}^{\top} \xb_i = \Sb_{iv}^{\top} \xb_{i,j},
    \quad
    \Sb_{e}^{\top} \xb_i \neq \Sb_{e}^{\top} \xb_{i,j}
    \quad
    \forall\ i \in [n],\ j \in [\alpha].
\end{align}
We recall the notion $\Deltab\triangleq\Aemp\rbr{\Xb}-\Mb\Xb$ such that $\dau \triangleq \rank\rbr{\Deltab} = \rank\rbr{\wt\Acal\rbr{\Xb}-\wt\Mb\Xb}$ ($0 \leq \dau \leq d_e$), and assume that $\Xb$ and $\Aemp(\Xb)$ are representative enough:
\begin{assumption}[Diversity of $\Xb$ and $\Aemp(\Xb)$]
\label{ass:case_ood_sample}
$(\Xb,\yb) \in \Xcal^n \times \Ycal^n$ is sufficiently large with $n \gg \rho^4 d$, $\thetab^* \in \row(\Xb)$, and $\dau = d_e$.
\end{assumption}

\paragraph{Excess risks in target domain.}
Learning from the linear hypothesis class $\Hcal = \csepp{h(\xb) = \xb^{\top}\thetab}{\thetab \in \R^{d}}$, with the DAC regularization on $h\rbr{\xb_i}=h\rbr{\xb_{i,j}}$, we have
\begin{align*}
    &\wh{\thetab}^{dac}\ =\ 
    \underset{\thetab \in \Hred}{\argmin}\
    \frac{1}{2N} \norm{\yb - \Xb\thetab}_2^2,
    \quad 
    \Hred = 
    \csepp{h\rbr{\xb}= \thetab^{\top} \xb} 
    {\Deltab\thetab = \b0},
\end{align*}
while with the ERM on augmented training set,
\begin{align*}
    \wh{\thetab}^{\herm}\ =\ 
    & \underset{\thetab \in \R^{d}}{\argmin}\ 
    \frac{1}{2 (1+\alpha) N} \norm{\wt\Mb\yb - \wt\Acal(\Xb)\thetab}_2^2,
\end{align*}    
where $\Mb$ and $\wt\Mb$ denote the vertical stacks of $\alpha$ and $1+\alpha$ identity matrices of size $n \times n$, respectively as denoted earlier.

We are interested in the excess risk on $P^t$: $L_t\rbr{\thetab} - L_t\rbr{\thetab^*}$ where $L_t\rbr{\thetab} \triangleq \E_{P^t\rbr{\xb,y}} \sbr{\frac{1}{2} (y - \xb^{\top} \thetab)^2}$.
\begin{theorem}[Domain adaptation with DAC]
\label{thm:domain-adaption-data-aug-consistency}
Under \Cref{ass:case_ood_distribution}(a) and \Cref{ass:case_ood_sample}, $\wh{\thetab}^{dac}$ satisfies that, with constant probability,
\begin{align}
    \label{eq:domain-adaption-data-aug-consistency}
    \E_{P^s}\sbr{L_t(\wh{\thetab}^{dac}) - L_t(\thetab^*)} 
    \ \lesssim\ 
    \frac{\sigma^2 d_{iv}}{N}.
\end{align}
\end{theorem}

\begin{theorem}[Domain adaptation with ERM on augmented samples]
\label{thm:domain-adaption-plain-data-aug}
Under \Cref{ass:case_ood_distribution} and \Cref{ass:case_ood_sample}, $\wh{\thetab}^{dac}$ and $\wh{\thetab}^{\herm}$ satisfies that, 
\begin{align}
    \label{eq:domain-adaption-plain-data-aug-excess-risk}
    \E_{P^s}\sbr{L_t(\wh{\thetab}^{\herm}) - L_t(\thetab^*)} 
    \ \geq \ 
    \E_{P^s}\sbr{L_t(\wh{\thetab}^{dac}) - L_t(\thetab^*)}
    + c_{t,e} \cdot \EER_e,
\end{align}
for some $\EER_e > 0$.
\end{theorem}   
In contrast to $\wh{\thetab}^{dac}$ where the DAC constraints enforce $\Sb_e^{\top} \wh{\thetab}^{dac} = \b{0}$ with a sufficiently diverse $\Aemp\rbr{\Xb}$ (\Cref{ass:case_ood_sample}), the ERM on augmented training set fails to filter out the environmental feature in $\wh{\thetab}^{\herm}$: $\Sb_e^{\top} \wh{\thetab}^{\herm} \neq \b{0}$. As a consequence, the expected excess risk of $\wh{\thetab}^{\herm}$ in the target domain can be catastrophic when $c_{t,e} \to \infty$, as instantiated by \Cref{example:ood}.

\paragraph{Proofs and instantiation.}
Recall that for $\Deltab \triangleq \Aemp(\Xb) - \Mb\Xb$, $\projnull \triangleq \Ib_d - \Deltab^{\dagger} \Deltab$ denotes the orthogonal projector onto the dimension-$(d-\dau)$ null space of $\Deltab$. 
Furthermore, let $\Pb_{iv} \triangleq \Sb_{iv} \Sb_{iv}^{\top}$ and $\Pb_{e} \triangleq \Sb_{e} \Sb_{e}^{\top}$ be the orthogonal projectors onto the invariant and environmental feature subspaces, respectively, such that $\xb = \Sb_{iv} \zetab_{iv}  + \Sb_e \zetab_e = \rbr{\Pb_{iv}+\Pb_e} \xb$ for all $\xb$.

\begin{proof}[Proof of \Cref{thm:domain-adaption-data-aug-consistency}]

By construction \Cref{eq:def_envi_data_aug}, $\Deltab\Pb_{iv} = \b{0}$, and it follows that $\Pb_{iv} \aleq \projnull$. 
Meanwhile from \Cref{ass:case_ood_sample}, $\dau=d_e$ implies that $\dim\rbr{\projnull} = d_{iv}$. 
Therefore, $\Pb_{iv} = \projnull$, and the data augmentation consistency constraints can be restated as
\begin{align*}
\Hred = \csepp{h\rbr{\xb} = \thetab^{\top} \xb}{\projnull\thetab = \thetab}
= 
\csepp{h\rbr{\xb} = \thetab^{\top} \xb}{\Pb_{iv} \thetab = \thetab}   
\end{align*}
Then with $\thetab^* \in \row(\Xb)$ from \Cref{ass:case_ood_sample}, 
\begin{align*}
    \wh{\thetab}^{dac} - \thetab^*
    = 
    \frac{1}{N} \wh{\Sigmab}_{\Xb_{iv}}^{\dagger} \Pb_{iv} \Xb^{\top} (\Xb \Pb_{iv} \thetab^* + \zb) - \thetab^*
    =
    \frac{1}{N} \wh{\Sigmab}_{\Xb_{iv}}^{\dagger} \Pb_{iv} \Xb^{\top} \zb,
\end{align*}
where $\wh{\Sigmab}_{\Xb_{iv}} \triangleq \frac{1}{N} \Pb_{iv} \Xb^{\top} \Xb \Pb_{iv}$.
Since $\wh{\thetab}^{dac} - \thetab^* \in \col\rbr{\Sb_{iv}}$, we have
$\E_{P^t}\sbr{z \cdot \xb^{\top} \Pb_e (\wh{\thetab}^{dac} - \thetab^*)} = 0$.
Therefore, let $\bs{\Sigma}_{\xb,t} \triangleq \E_{P^t}[\xb\xb^{\top}]$, with high probability,
\begin{align*}
    E_{P^s} \sbr{L_t(\wh{\thetab}^{dac}) - L_t(\thetab^*)} 
    = \ 
    & E_{P^s} \sbr{\frac{1}{2} \norm{\wh{\thetab}^{dac} - \thetab^*}_{\Sigmab_{\xb,t}}^2} 
    \\ = \ 
    & \tr \rbr{
    \frac{1}{2N} \E_{P^s}\sbr{\zb \zb^{\top}}\ 
    \E_{P^s}\sbr{\rbr{\frac{1}{N} \Pb_{iv} \Xb^{\top} \Xb \Pb_{iv}}^{\dagger}}
    \ \Sigmab_{\xb,t}}
    \\ = \ 
    & \tr \rbr{
    \frac{\sigma^2}{2N}\ 
    \E_{P^s} \sbr{\wh{\Sigmab}_{\Xb_{iv}}^{\dagger}}\ 
    \Sigmab_{\xb,t} }
    \\ \leq \ 
    & C_{t,iv}\ 
    \frac{\sigma^2}{2N}\ 
    tr \rbr{\E_{P^s} \sbr{\wh{\Sigmab}_{\Xb_{iv}}^{\dagger}}} \quad \rbr{\t{\Cref{lemma:sample-population-covariance}},\ \emph{w.h.p.}}
    \\ \lesssim \ 
    & \frac{\sigma^2}{2N} \tr\rbr{\rbr{\E_{P^s} \sbr{\Pb_{iv} \xb \xb^{\top} \Pb_{iv}}}^{\dagger}}
    \\ \leq \ 
    & \frac{\sigma^2 d_{iv}}{2N c} 
    \ \lesssim \  
    \frac{\sigma^2 d_{iv}}{2N}.
\end{align*}
\end{proof}

\begin{proof}[Proof of \Cref{thm:domain-adaption-plain-data-aug}]
Let $\wh{\Sigmab}_{\wt\Acal\rbr{\Xb}} \triangleq \frac{1}{(1+\alpha)N} \wt\Acal\rbr{\Xb}^{\top} \wt\Acal\rbr{\Xb}$.
Then with $\thetab^* \in \row(\Xb)$ from \Cref{ass:case_ood_sample}, we have $\thetab^* = \wh{\Sigmab}_{\wt\Acal\rbr{\Xb}}^{\dagger} \wh{\Sigmab}_{\wt\Acal\rbr{\Xb}} \thetab^*$. 
Since $\thetab^* \in \col\rbr{\Sb_{iv}}$, 
$\wt\Mb \Xb \thetab^* = \wt\Mb \Xb \Pb_{iv} \thetab^* = \wt\Acal(\Xb) \thetab^*$.
Then, the ERM on the augmented training set yields 
\begin{align*}
    \wh{\thetab}^{\herm} - \thetab^*
    \ = \ 
    & \frac{1}{(1+\alpha) N} \wh{\Sigmab}_{\wt\Acal\rbr{\Xb}}^{\dagger} \wt\Acal(\Xb)^{\top} \wt\Mb (\Xb \thetab^* + \zb) 
    -
    \wh{\Sigmab}_{\wt\Acal\rbr{\Xb}}^{\dagger} \wh{\Sigmab}_{\wt\Acal\rbr{\Xb}} \thetab^*
    \\ =  
    & \frac{1}{(1+\alpha) N} \wh{\Sigmab}_{\wt\Acal\rbr{\Xb}}^{\dagger} \wt\Acal(\Xb)^{\top} \wt\Mb \zb.
\end{align*}

Meanwhile with $\E_{P^t}\sbr{z \cdot \zetab_e} = \b{0}$ from \Cref{ass:case_ood_distribution}, we have $\E_{P^t}\sbr{z \cdot \Pb_e\xb} = \b{0}$. Therefore, by recalling that $\bs{\Sigma}_{\xb,t} \triangleq \E_{P^t}[\xb\xb^{\top}]$,
\begin{align*}
    L_t(\thetab) - L_t(\thetab^*)
    \ = \  
    \E_{P^t\rbr{\xb}}\sbr{\frac{1}{2} \rbr{\xb^{\top} (\thetab - \thetab^*)}^2 + 
    z \cdot \xb^{\top} \Pb_e (\thetab - \thetab^*)}
    \ = \ 
    \frac{1}{2} \norm{\thetab^* - \thetab}_{\Sigmab_{\xb,t}}^2,
\end{align*}
such that the expected excess risk can be expressed as
\begin{align*}
    \E_{P^s} \sbr{L_t(\wh{\thetab}^{\herm}) - L_t(\thetab^*)} 
    =  
    \frac{1}{2 (1+\alpha)^2 N^2} \tr \rbr{
    \E_{P^s}\sbr{
    \wh{\Sigmab}_{\wt\Acal\rbr{\Xb}}^{\dagger} 
    \rbr{\wt\Acal(\Xb)^{\top} \wt\Mb \zb \zb^{\top} \wt\Mb^{\top} \wt\Acal(\Xb) }
    \wh{\Sigmab}_{\wt\Acal\rbr{\Xb}}^{\dagger} }
    \Sigmab_{\xb,t} },
\end{align*}
where let $\wh{\Sigmab}_{\wt\Acal\rbr{\Xb_e}} \triangleq \Pb_e \wh{\Sigmab}_{\wt\Acal\rbr{\Xb}} \Pb_e$,
\begin{align*}
    & \E_{P^s}\sbr{ 
    \wh{\Sigmab}_{\wt\Acal\rbr{\Xb}}^{\dagger} 
    \rbr{\wt\Acal(\Xb)^{\top} \wt\Mb \zb \zb^{\top} \wt\Mb^{\top} \wt\Acal(\Xb) }
    \wh{\Sigmab}_{\wt\Acal\rbr{\Xb}}^{\dagger} }
    \\ \ageq \ 
    & \E_{P^s}\sbr{ \rbr{
    \Pb_{iv} \wh{\Sigmab}_{\wt\Acal\rbr{\Xb}}^{\dagger} \Pb_{iv}
    + \Pb_e \wh{\Sigmab}_{\wt\Acal\rbr{\Xb}}^{\dagger} \Pb_e }
    \wt\Acal(\Xb)^{\top} \wt\Mb\zb \zb^{\top} \wt\Mb^{\top} \wt\Acal(\Xb) 
    \rbr{ \Pb_{iv} \wh{\Sigmab}_{\wt\Acal\rbr{\Xb}}^{\dagger} \Pb_{iv} +
    \Pb_e \wh{\Sigmab}_{\wt\Acal\rbr{\Xb}}^{\dagger} \Pb_e
    } }
    \\ \ageq \ 
    & \sigma^2 (1+\alpha)^2 N \cdot 
    \E_{P^s}\sbr{\wh{\Sigmab}_{\Xb_{iv}}^{\dagger}}
    + 
    \E_{P^s}\sbr{ 
    \wh{\Sigmab}_{\wt\Acal\rbr{\Xb_e}}^{\dagger} 
    \wt\Acal(\Xb_e)^{\top} \wt\Mb \zb\zb^{\top} \wt\Mb^{\top} \wt\Acal(\Xb_e) 
    \wh{\Sigmab}_{\wt\Acal\rbr{\Xb_e}}^{\dagger} }.
\end{align*}
We denote 
\begin{align*}
    \EER_e \triangleq \ 
    \tr\rbr{
    \E_{P^s} \sbr{ \frac{1}{2 (1+\alpha)^2 N^2}
    \wh{\Sigmab}_{\wt\Acal\rbr{\Xb_e}}^{\dagger} 
    \wt\Acal(\Xb_e)^{\top} \wt\Mb \zb\zb^{\top} \wt\Mb^{\top} \wt\Acal(\Xb_e) 
    \wh{\Sigmab}_{\wt\Acal\rbr{\Xb_e}}^{\dagger} } },
\end{align*}
and observe that
\begin{align*}
    \EER_e = \E_{P^s} \sbr{ 
    \frac{1}{2} \norm{\frac{1}{(1+\alpha) N} 
    \wh{\Sigmab}_{\wt\Acal\rbr{\Xb_e}}^{\dagger} 
    \wt\Acal(\Xb_e)^{\top} \wt\Mb \zb }_2^2 } 
    > 0.
\end{align*}
Finally, we complete the proof by partitioning the lower bound for the target expected excess risk of $\wh{\thetab}^{\herm}$ into the invariantand environmental parts such that
\begin{align*}
    & \E_{P^s}\sbr{L_t(\wh{\thetab}^{\herm}) - L_t(\thetab^*)} 
    \\ \geq \ 
    & \underbrace{ \tr \rbr{\frac{\sigma^2}{2N}\ 
    \E_{P^s} \sbr{\wh{\Sigmab}_{\Xb_{iv}}^{\dagger}} 
    \Sigmab_{\xb,t} } 
    }_{=\E\sbr{L_t(\wh{\thetab}^{dac}) - L_t(\thetab^*)} }
    \ \\
    & + \
    \underbrace{\tr \rbr{ 
    \E_{P^s} \sbr{ \frac{1}{2 (1+\alpha)^2 N^2}
    \wh{\Sigmab}_{\wt\Acal\rbr{\Xb_e}}^{\dagger} 
    \wt\Acal(\Xb_e)^{\top} \wt\Mb \zb\zb^{\top} \wt\Mb^{\top} \wt\Acal(\Xb_e) 
    \wh{\Sigmab}_{\wt\Acal\rbr{\Xb_e}}^{\dagger} } 
    \Sigmab_{\xb,t} } 
    }_{\t{expected excess risk from environmental feature subspace} \geq c_{t,e} \cdot \EER_e}
    \\ \geq \ 
    & \E_{P^s} \sbr{L_t(\wh{\thetab}^{dac}) - L_t(\thetab^*)}
    + c_{t,e} \cdot \EER_e.
\end{align*}
\end{proof}

Now we construct a specific domain adaptation example with a large separation ($\ie$, proportional to $d_e$) in the target excess risk between learning with the DAC regularization ($\ie$, $\wh{\thetab}^{dac}$) and with the ERM on augmented training set ($\ie$, $\wh{\thetab}^{\herm}$). 
\begin{example}
\label{example:ood}
We consider $P^s$ and $P^t$ that follow the same set of relations in \Cref{fig:data-distribution-causal-graph}, except for the distributions over $\eb$ where $P^s\rbr{\eb} \neq P^t\rbr{\eb}$. 
Precisely, let the environmental feature $\zetab_e$ depend on $(\zetab_{iv}, y, \eb)$:
\begin{align*}
    \zetab_e = \sgn\rbr{y - \rbr{\thetab^*}^{\top} \Sb_{iv} \zetab_{iv}} \eb = \sgn(z) \eb,
    \quad
    z \sim \Ncal(0, \sigma^2),
    \quad 
    z \perp \eb,
\end{align*}    
where $\eb \sim \Ncal\rbr{\b0, \Ib_{d_e}}$ for $P^s(\eb)$ and $\eb \sim \Ncal\rbr{\b0, \sigma_t^2 \Ib_{d_e}}$ for $P^t(\eb)$, $\sigma_t \geq c_{t,e}$ (recall $c_{t,e}$ from \Cref{ass:case_ood_distribution}).
Assume that the training set $\Xb$ is sufficiently large, $n \gg d_e + \log\rbr{1/\delta}$ for some given $\delta \in (0,1)$.
Augmenting $\Xb$ with a simple by common type of data augmentations -- the linear transforms, we let
\begin{align*}
    \wt\Acal(\Xb) = \sbr{\Xb; \rbr{\Xb\Ab_1}; \dots; \rbr{\Xb\Ab_\alpha}},
    \quad
    \Ab_j = \Pb_{iv} + \ub_j \vb_j^{\top},
    \quad
    \ub_j,\vb_j \in \col\rbr{\Sb_e}
    \quad \forall\ j \in [\alpha],
\end{align*}
and define
\begin{align*}
    \nu_1 \triangleq \max \cbr{1} \cup \csepp{\sigma_{\max}(\Ab_j)}{j\in[\alpha]}
    \quad \t{and} \quad 
    \nu_2 \triangleq \sigma_{\min}\rbr{\frac{1}{1+\alpha} \rbr{\Ib_{d} + \sum_{j=1}^{\alpha} \Ab_k}},
\end{align*}
where $\sigma_{\min}(\cdot)$ and $\sigma_{\max}(\cdot)$ refer to the minimum and maximum singular values, respectively.
Then under \Cref{ass:case_ood_distribution} and \Cref{ass:case_ood_sample}, with constant probability,
\[
\E_{P^s}\sbr{L_t(\wh{\thetab}^{\herm}) - L_t(\thetab^*)} 
\gtrsim
\E_{P^s}\sbr{L_t(\wh{\thetab}^{dac}) - L_t(\thetab^*)} + 
c_{t,e} \cdot \frac{\sigma^2 d_e}{2N}. 
\]
\end{example}

\begin{proof}[Proof of \Cref{example:ood}]
With the specified distribution, for $\Eb = \sbr{\eb_1; \dots; \eb_N} \in \R^{N \times d_e}$,
\begin{align*}
    &\wh{\Sigmab}_{\wt\Acal\rbr{\Xb_e}} = 
    \frac{1}{(1+\alpha) N} 
    \Sb_e \rbr{\Eb^{\top} \Eb + \sum_{j=1}^{\alpha} \Ab_j^{\top} \Eb^{\top} \Eb
    \Ab_j} \Sb_e^{\top}
    \aleq
    \frac{\nu_1^2}{N} \Sb_e \Eb^{\top} \Eb \Sb_e^{\top},
    \\
    &\frac{1}{(1+\alpha) N} \wt\Acal(\Xb_e)^{\top} \wt\Mb\zb = 
    \rbr{\frac{1}{1+\alpha} \rbr{\Ib_{d} + \sum_{j=1}^{\alpha} \Ab_j}}^{\top}
    \frac{1}{N} \Sb_e \Eb^{\top} \abbr{\zb}.
\end{align*}
By \Cref{lemma:sample-population-covariance}, under \Cref{ass:case_ood_distribution} and \Cref{ass:case_ood_sample}, we have that with high probability, $0.9 \Ib_{d_e} \aleq \frac{1}{N} \Eb^{\top} \Eb \aleq 1.1 \Ib_{d_e}$.
Therefore with $\Eb$ and $\zb$ being independent, 
\begin{align*}
    & \EER_e 
    \ = \ 
    \E_{P^s} \sbr{ 
    \frac{1}{2} \norm{\frac{1}{(1+\alpha) N} 
    \wh{\Sigmab}_{\wt\Acal\rbr{\Xb_e}}^{\dagger} 
    \wt\Acal(\Xb_e)^{\top} \wt\Mb \zb }_2^2 } 
    \\ \geq \ 
    & \frac{\sigma^2}{2N}\ \frac{\nu_2^2}{\nu_1^4}\ 
    \tr\rbr{\E_{P^s} \sbr{ 
    \rbr{\frac{1}{N} \Sb_e \Eb^{\top} \Eb \Sb_e^{\top}}^{\dagger} } }
    \\ \gtrsim 
    &\frac{\sigma^2}{2N}\ \frac{\nu_2^2}{\nu_1^4} d_e
    \\ \gtrsim \ 
    & \frac{\sigma^2 d_e}{2N},
\end{align*}
and the rest follows from \Cref{thm:domain-adaption-plain-data-aug}.
\end{proof}

\section{Technical Lemmas}

\begin{lemma}
\label{lemma:sample-population-covariance}
We consider a random vector $\xb \in \R^d$ with $\E[\xb]=\b{0}$, $\E[\xb\xb^{\top}] = \Sigmab$, and $\overline{\xb} = \Sigmab^{-1/2} \xb$ 
\footnote{In the case where $\Sigmab$ is rank-deficient, we slightly abuse the notation such that $\Sigmab^{-1/2}$ and $\Sigmab^{-1}$ refer to the respective pseudo-inverses.}
being $\rho^2$-subgaussian.
Given an $\iid$ sample of $\xb$, $\Xb=[\xb_1,\dots,\xb_n]^{\top}$, for any $\delta \in (0,1)$, if $n \gg \rho^4 d$, then $0.9 \Sigmab \aleq \frac{1}{n}\Xb^{\top} \Xb \aleq 1.1 \Sigmab$ with high probability.
\end{lemma}

\begin{proof}
We first denote $\Pb_{\Xcal} \triangleq \Sigmab \Sigmab^{\pinv}$ as the orthogonal projector onto the subspace $\Xcal \subseteq \R^d$ supported by the distribution of $\xb$.
With the assumptions $\E[\xb]=\b{0}$ and $\E[\xb\xb^{\top}] = \Sigmab$, we observe that $\E \sbr{\overline{\xb}} = \b{0}$ and $\E\sbr{\overline{\xb} \overline{\xb}^{\top}} = \E \sbr{\xb \Sigmab^{-1} \xb^{\top}} = \Pb_{\Xcal}$.
Given the sample set $\Xb$ of size $n \gg \rho^4\rbr{d+\log(1/\delta)}$ for some $\delta \in (0,1)$, we let $\Ub = \frac{1}{n} \sum_{i=1}^n \xb_i \Sigmab^{-1} \xb_i^{\top} - \Pb_{\Xcal}$. 
Then the problem can be reduced to showing that, with probability at least $1-\delta$, $\norm{\Ub}_2 \leq 0.1$.
For this, we leverage the $\eps$-net argument as following.

For an arbitrary $\vb \in \Xcal \cap\ \mathbb{S}^{d-1}$, we have
\begin{align*}
    \vb^{\top} \Ub \vb = 
    \frac{1}{n} \sum_{i=1}^n 
    \rbr{\vb^{\top} \xb_i \Sigmab^{-1} \xb_i^{\top} \vb - 1} = 
    \frac{1}{n} \sum_{i=1}^n 
    \rbr{\rbr{\vb^{\top} \overline{\xb}_i}^2 - 1},
\end{align*}
where, given $\overline{\xb}_i$ being $\rho^2$-subgaussian, 
$\vb^{\top} \overline{\xb}_i$ is $\rho^2$-subgaussian. 
Since 
\begin{align*}
    \E \sbr{\rbr{\vb^{\top} \overline{\xb}_i}^2} = \vb^{\top} \E \sbr{\overline{\xb}_i \overline{\xb}_i^{\top}} \vb = 1,
\end{align*}
we know that $\rbr{\vb^{\top} \overline{\xb}_i}^2-1$ is $16\rho^2$-subexponential.
Then, we recall the Bernstein's inequality,
\begin{align*}
    \PP\sbr{\abbr{\vb^{\top} \Ub \vb} > \eps} \leq 
    2 \exp \rbr{-\frac{n}{2} 
    \min\rbr{\frac{\eps^2}{\rbr{16 \rho^2}^2}, 
    \frac{\eps}{16 \rho^2}}}.
\end{align*}

Let $N \subset \Xcal \cap\ \mathbb{S}^{d-1}$ be an $\eps_1$-net such that $\abbr{N} = e^{O\rbr{d}}$. 
Then for some $0 < \eps_2 \leq 16 \rho^2$, by the union bound,
\begin{align*}
    \PP \sbr{\underset{\vb \in N}{\max}: \abbr{\vb^{\top} \Ub \vb} > \eps_2} 
    \leq \ 
    & 2 \abbr{N} \exp \rbr{-\frac{n}{2}
    \min\rbr{\frac{\eps_2^2}{\rbr{16 \rho^2}^2}, 
    \frac{\eps_2}{16 \rho^2}} } 
    \\ \leq \ 
    & \exp \rbr{O\rbr{d} -\frac{n}{2} \cdot \frac{\eps_2^2}{\rbr{16 \rho^2}^2} } \leq \delta
\end{align*}
whenever $n > \frac{2 \rbr{16 \rho^2}^2}{\eps_2^2} \rbr{\Theta\rbr{d} + \log\frac{1}{\delta}}$. By taking $\delta = \exp\rbr{-\frac{1}{4}\rbr{\frac{\eps_2}{16 \rho^2}}^2 n}$, we have that $\underset{\vb \in N}{\max}\abbr{\vb^{\top} \Ub \vb} \leq \eps_2$ with high probability when $n > 4 \rbr{\frac{16 \rho^2}{\eps_2}}^2 \Theta\rbr{d}$, and taking $n \gg \rho^4 d$ is sufficient.

Now for any $\vb \in \Xcal \cap\ \mathbb{S}^{d-1}$, there exists some $\vb' \in N$ such that $\norm{\vb - \vb'}_2 \leq \eps_1$. 
Therefore,
\begin{align*}
    \abbr{\vb^{\top} \Ub \vb} 
    \ = \
    & \abbr{\vb'^{\top} \Ub \vb' + 
    2 \vb'^{\top} \Ub \rbr{\vb - \vb'} + 
    \rbr{\vb - \vb'}^{\top} \Ub \rbr{\vb - \vb'} }
    \\ \leq \ 
    & \rbr{ \underset{\vb \in N}{\max}: \abbr{\vb^{\top} \Ub \vb} } + 
    2 \norm{\Ub}_2 \norm{\vb'}_2 \norm{\vb-\vb'}_2 + 
    \norm{\Ub}_2 \norm{\vb-\vb'}_2^2
    \\ \leq \ 
    & \rbr{ \underset{\vb \in N}{\max}: \abbr{\vb^{\top} \Ub \vb} } + 
    \norm{\Ub}_2
    \rbr{2 \eps_1 + \eps_1^2}.
\end{align*}
Taking the supremum over $\vb \in \mathbb{S}^{d-1}$, with probability at least $1-\delta$,
\begin{align*}
    \underset{\vb \in \Xcal \cap\ \mathbb{S}^{d-1}}{\max}: \abbr{\vb^{\top} \Ub \vb}
    = 
    \norm{\Ub}_2
    \leq 
    \eps_2 + \norm{\Ub}_2 \rbr{2 \eps_1 + \eps_1^2},
    \qquad
    \norm{\Ub}_2
    \leq 
    \frac{\eps_2}{2 - \rbr{1+\eps_1}^2}.
\end{align*}
With $\eps_1 = \frac{1}{3}$ and $\eps_2 = \frac{1}{45}$, we have $\frac{\eps_2}{2 - \rbr{1+\eps_1}^2} = \frac{1}{10}$.

Overall, if $n \gg \rho^4 d$, then with high probability, we have $\norm{\Ub}_2 \leq 0.1$.
\end{proof}

\begin{lemma}\label{lemma:tech_gaussian_width_lipschitz}
Let $U \subseteq \R^d$ be an arbitrary subspace in $\R^d$, and $\gb \sim \Ncal\rbr{\b0,\Ib_d}$ be a Gaussian random vector. Then for any continuous and $C_l$-Lipschitz function $\varphi: \R \to \R$ ($\ie$, $\abbr{\varphi(u)-\varphi(u')} \leq C_l \cdot \abbr{u-u'}$ for all $u,u' \in \R$), 
\begin{align*}
    \E_{\gb} \sbr{\underset{\ub \in U}{\sup}\ \gb^{\top} \varphi(\ub)} \leq C_l \cdot \E_{\gb} \sbr{\underset{\ub \in U}{\sup}\ \gb^{\top} \ub},
\end{align*}
where $\varphi$ acts on $\ub$ entry-wisely, $\rbr{\varphi(\ub)}_j = \varphi(u_j)$.
In other words, the Gaussian width of the image set $\varphi(U) \triangleq \cbr{\varphi(\ub) \in \R^d ~|~ \ub \in U}$ is upper bounded by that of $U$ scaled by the Lipschitz constant.
\end{lemma} 

\begin{proof}
\begin{align*}
    \E_{\gb} \sbr{\underset{\ub \in U}{\sup}\ \gb^{\top} \varphi(\ub)}
    =
    & \frac{1}{2} \E_{\gb} \sbr{\underset{\ub \in U}{\sup}\ \gb^{\top} \varphi(\ub) + \underset{\ub' \in U}{\sup}\ \gb^{\top} \varphi(\ub)} \\
    =
    & \frac{1}{2} \E_{\gb} \sbr{\underset{\ub,\ub' \in U}{\sup}\ \gb^{\top} \rbr{\varphi(\ub)-\varphi(\ub')}} \\
    \leq 
    & \frac{1}{2} \E_{\gb} \sbr{\underset{\ub,\ub' \in U}{\sup}\ \sum_{j=1}^d \abbr{g_j} \abbr{\varphi(u_j)-\varphi(u'_j)}} 
    \quad 
    \rbr{\t{since}\ \varphi\ \t{is $C_l$-Lipschitz}} \\
    \leq 
    & \frac{C_l}{2} \E_{\gb} \sbr{\underset{\ub,\ub' \in U}{\sup}\ \sum_{j=1}^d \abbr{g_j} \abbr{u_j-u'_j}} \\
    = 
    & \frac{C_l}{2} \E_{\gb} \sbr{\underset{\ub,\ub' \in U}{\sup}\ \gb^{\top} \rbr{\ub-\ub'}} \\
    = 
    & \frac{C_l}{2} \E_{\gb} \sbr{\underset{\ub \in U}{\sup}\ \gb^{\top} \ub + \underset{\ub' \in U}{\sup}\ \gb^{\top} \rbr{-\ub'}} \\
    = 
    & C_l \cdot \E_{\gb} \sbr{\underset{\ub \in U}{\sup}\ \gb^{\top} \ub}
\end{align*}
\end{proof}

\section{Experiment Details}\label{apdx:exp_detail}

In this section, we provide the details of our experiments. Our code is adapted from the publicly released repo: \href{https://github.com/kekmodel/FixMatch-pytorch}{https://github.com/kekmodel/FixMatch-pytorch}.

\textbf{Dataset:} Our training dataset is derived from CIFAR-100, where the original dataset contains 50,000 training samples of 100 different classes. Out of the original 50,000 samples, we randomly select 10,000 labeled data as training set (i.e., 100 labeled samples per class). To see the impact of different training samples, we also trained our model with dataset that contains 1,000 and 20,000 samples. Evaluations are done on standard test set of CIFAR-100, which contains 10,000 testing samples. 

\textbf{Data Augmentation:} During the training time, given a training batch, we generate corresponding augmented samples by RandAugment \citep{cubuk2020randaugment}. We set the number of augmentations per sample to 7, unless otherwise mentioned.

To generate an augmented image, the RandAugment draws $n$ transformations uniformaly at random from 14 different augmentations, namely \{identity, autoContrast, equalize, rotate, solarize, color, posterize, contrast, brightness, sharpness, shear-x, shear-y, translate-x, translate-y\}. The RandAugment provides each transformation with a single scalar (1 to 10) to control the strength of each of them, which we always set to 10 for all transformations. By default, we set $n=2$ (i.e., using 2 random transformations to generate an augmented sample). To see the impact of different augmentation strength, we choose $n\in\cbr{1,2, 5, 10}$. Examples of augmented samples are shown in \Cref{fig:augmentation_strength}.

\textbf{Parameter Setting:} The batch size is set to 64 and the entire training process takes $2^{15}$ steps. During the training, we adopt the SGD optimizer with momentum set to 0.9, with learning rate for step $i$ being $0.03 \times \cos{\rbr{\frac{i \times 7\pi}{2^{15}\times 16}}}$.  

\textbf{Additional Settings for the semi-supervised learning results:} For the results on semi-supervised learning, besides the 10,000 labeled samples, we also draw additionally samples (ranging from 5,000 to 20,000) from the training set of the original CIFAR-100. We remove the labels of those additionally sampled images, as they serve as ``unlabeled" samples in the semi-supervised learning setting. The FixMatch implementation follows the publicly available on in \href{https://github.com/kekmodel/FixMatch-pytorch}{https://github.com/kekmodel/FixMatch-pytorch}.
\chapter{Appendix for \Cref{ch:adawac}}

\section{Separation of Label-sparse and Label-dense Samples}\label{subapx:spontaneous_separation_sparse_dense}

\begin{proof}[Proof of \Cref{prop:spontaneous_separation_sparse_dense}]
    We first observe that, since 
    $\lossce\rbr{\theta;\rbr{\xb,\yb}}$ and $\regdac\rbr{\theta;\xb,A_{1},A_{2}}$ are convex and continuous in $\theta$ for all $(\xb,\yb) \in \Xcal \times \Ycal$ and $A_1,A_2 \sim \Acal^2$, for all $i \in [n]$, $\wh L^\wdac_i\rbr{\theta,\betab}$ is continuous, convex in $\theta$, and affine (thus concave) in $\betab$; and therefore so is $\wh L^\wdac \rbr{\theta,\betab}$. 
    Then with the compact and convex domains $\theta \in \Fnb{\gamma}$ and $\betab \in [0,1]^n$, Sion's minimax theorem~\citep{sion1958minimax} suggests the minimax equality,
    \begin{align}\label{eq:pf_minimax_optimal}
        \min_{\theta \in \Fnb{\gamma}}~ \max_{\betab \in [0,1]^n}~ \wh L^\wdac\rbr{\theta,\betab} = \max_{\betab \in [0,1]^n}~ \min_{\theta \in \Fnb{\gamma}}~ \wh L^\wdac\rbr{\theta,\betab},
    \end{align}
    where $\inf, \sup$ can be replaced by $\min, \max$ respectively due to compactness of the domains.

    Further, by the continuity and convexity-concavity of $\wh L^\wdac \rbr{\theta,\betab}$, the pointwise maximum $\max_{\betab \in [0,1]^n}\wh L^\wdac\rbr{\theta,\betab}$ is lower semi-continuous and convex in $\theta$ while the pointwise minimum $\min_{\theta \in \Fnb{\gamma}}\wh L^\wdac\rbr{\theta,\betab}$ is upper semi-continuous and concave in $\betab$. Then via Weierstrass' theorem (\cite{bertsekas2009convex}, Proposition 3.2.1), there exist $\wh\theta^\wdac \in \Fnb{\gamma}$ and $\wh\betab \in [0,1]^n$ that achieve the minimax optimal by minimizing $\max_{\betab \in [0,1]^n}\wh L^\wdac\rbr{\theta,\betab}$ and maximizing $\min_{\theta \in \Fnb{\gamma}}\wh L^\wdac\rbr{\theta,\betab}$. Along with \Cref{eq:pf_minimax_optimal}, such $\rbr{\wh\theta^\wdac, \wh\betab}$ provides a saddle point for \Cref{eq:objective_dac_encoder_weighted_adawac} (\cite{bertsekas2009convex}, Proposition 3.4.1).

    Next, we show via contradiction that there exists a saddle point with $\wh\betab$ attained on a vertex $\wh\betab \in \cbr{0,1}^n$. Suppose the opposite, then for any saddle point $\rbr{\wh\theta^\wdac, \wh\betab}$, there must be an $i \in [n]$ with $\iloc{\wh\betab}{i} \in (0,1)$, where we have the following contradictions:
    \begin{enumerate}[label=(\roman*)]
        \item If $\lossce\rbr{\wh\theta^\wdac;\rbr{\xb_i,\yb_i}} < \regdac\rbr{\wh\theta^\wdac;\xb_i,A_{i,1},A_{i,2}}$, decreasing $\iloc{\wh\betab}{i}>0$ to $\iloc{\wh\betab'}{i}=0$ leads to $\wh L^\wdac\rbr{\wh\theta^\wdac, \wh\betab'} > \wh L^\wdac\rbr{\wh\theta^\wdac, \wh\betab}$, contradicting \Cref{eq:saddle_point_def}.
        \item If $\lossce\rbr{\wh\theta^\wdac;\rbr{\xb_i,\yb_i}} > \regdac\rbr{\wh\theta^\wdac;\xb_i,A_{i,1},A_{i,2}}$, increasing $\iloc{\wh\betab}{i}<1$ to $\iloc{\wh\betab'}{i}=1$ again leads to $\wh L^\wdac\rbr{\wh\theta^\wdac, \wh\betab'} > \wh L^\wdac\rbr{\wh\theta^\wdac, \wh\betab}$, contradicting \Cref{eq:saddle_point_def}.
        \item If $\lossce\rbr{\wh\theta^\wdac;\rbr{\xb_i,\yb_i}} = \regdac\rbr{\wh\theta^\wdac;\xb_i,A_{i,1},A_{i,2}}$, $\iloc{\wh\betab}{i}$ can be replaced with any value in $[0,1]$, including $0,1$.
    \end{enumerate}
    Therefore, there must be a saddle point $\rbr{\wh\theta^\wdac, \wh\betab}$ with $\wh\betab \in \cbr{0,1}^n$ such that
    \begin{align*}
        \iloc{\betab}{i} = \iffun{\lossce\rbr{\wh\theta^\wdac;\rbr{\xb_i,\yb_i}} > \regdac\rbr{\wh\theta^\wdac;\xb_i,A_{i,1},A_{i,2}}}.
    \end{align*}

    Finally, it remains to show that \whp over $\cbr{\rbr{\xb_i, \yb_i}}_{i \in [n]}\sim P_\xi^n$ and $\cbr{\rbr{A_{i,1}, A_{i,2}}}_{i \in [n]} \sim \Acal^{2n}$,
    \begin{enumerate}[label=(\roman*)]
        \item $\lossce\rbr{\wh\theta^\wdac;\rbr{\xb_i,\yb_i}} \le \regdac\rbr{\wh\theta^\wdac;\xb_i,A_{i,1},A_{i,2}}$ for all $\rbr{\xb_i,\yb_i} \sim P_0$; and
        \item $\lossce\rbr{\wh\theta^\wdac;\rbr{\xb_i,\yb_i}} > \regdac\rbr{\wh\theta^\wdac;\xb_i,A_{i,1},A_{i,2}}$ for all $\rbr{\xb_i,\yb_i} \sim P_1$; 
    \end{enumerate}
    which leads to $\iloc{\betab}{i}=\iffun{\rbr{\xb_i,\yb_i} \sim P_1}$ \whp as desired.
    To illustrate this, we begin by observing that when $P_0$ and $P_1$ are $n$-separated (\Cref{ass:separability_sparse_dense}), since $\wh\theta^\wdac \in \Fnb{\gamma}$, there exists some $\omega>0$ such that for each $i \in [n]$,
    \begin{align*}
        \PP\ssepp{\lossce\rbr{\wh\theta^\wdac;\rbr{\xb_i,\yb_i}} < \regdac\rbr{\wh\theta^\wdac;\xb_i,A_{i,1},A_{i,2}}}{\rbr{\xb_i, \yb_i} \sim P_0} \ge 1-\frac{1}{\Omega\rbr{n^{1+\omega}}},
    \end{align*}
    and 
    \begin{align*}
        \PP\ssepp{\lossce\rbr{\wh\theta^\wdac;\rbr{\xb_i,\yb_i}} > \regdac\rbr{\wh\theta^\wdac;\xb_i,A_{i,1},A_{i,2}}}{\rbr{\xb_i, \yb_i} \sim P_1} \ge 1-\frac{1}{\Omega\rbr{n^{1+\omega}}}.
    \end{align*}
    Therefore by the union bound over the set of $n$ samples $\cbr{\rbr{\xb_i, \yb_i}}_{i \in [n]}\sim P_\xi^n$,
    \begin{align}\label{eq:pf_union_bound_P0}
        \PP\sbr{\lossce\rbr{\wh\theta^\wdac;\rbr{\xb_i,\yb_i}} < \regdac\rbr{\wh\theta^\wdac;\xb_i,A_{i,1},A_{i,2}} ~\forall~ \rbr{\xb_i, \yb_i} \sim P_0} \ge 1-\frac{1}{\Omega\rbr{n^\omega}},
    \end{align}
    and 
    \begin{align}\label{eq:pf_union_bound_P1}
        \PP\sbr{\lossce\rbr{\wh\theta^\wdac;\rbr{\xb_i,\yb_i}} > \regdac\rbr{\wh\theta^\wdac;\xb_i,A_{i,1},A_{i,2}} ~\forall~ \rbr{\xb_i, \yb_i} \sim P_1} \ge 1-\frac{1}{\Omega\rbr{n^\omega}}.
    \end{align}
    Applying the union bound again on \Cref{eq:pf_union_bound_P0} and \Cref{eq:pf_union_bound_P1}, we have the desired condition holds with probability $1-\Omega\rbr{n^\omega}^{-1}$, \ie, \whp.
\end{proof}

\section{Convergence of \ours}\label{subapx:convergence_wdac}

Recall the underlying function class $\Fcal \ni f_\theta$ parameterized by some $\theta \in \Fcal_\theta$ that we aim to learn for the pixel-wise classifier $h_\theta = \argmax_{k \in [K]} \iloc{f_\theta\rbr{\xb}}{j,:}$, $j \in [d]$:
\begin{align}
    \Fcal = \csepp{f_\theta = \phi_\theta \circ \psi_\theta}{\phi_\theta: \Xcal \to \Zcal,~ \psi_\theta:\Zcal \to [0,1]^{d \times K}},
\end{align}
where $\phi_\theta, \psi_\theta$ correspond to the encoder and decoder functions. Formally, we consider an inner product space of parameters $\rbr{\Fcal_\theta, \abr{\cdot,\cdot}_\Fcal}$ with the induced norm $\nbr{\cdot}_\Fcal$ and dual norm $\nbr{\cdot}_{\Fcal,*}$. 

For any $d \in \N$, let $\Delta_d^n \dfeq \csepp{\sbr{\betab_1;\dots;\betab_n} \in [0,1]^{n \times d}}{\nbr{\betab_i}_1 = 1 ~\forall~ i \in [n]}$. Then \Cref{eq:objective_dac_encoder_weighted_adawac} can be reformulated as:
\begin{align}\label{eq:objective_dac_encoder_weighted_l1_ball_formulation}
    &\wh\theta^\wdac, \wh\betab = 
    \argmin_{\theta \in \Fnb{\gamma}}~
    \argmax_{\Bb \in \Delta_2^n}~ \cbr{\wh L^\wdac\rbr{\theta,\Bb} \dfeq \frac{1}{n} \sum_{i=1}^n \wh L^\wdac_i\rbr{\theta,\Bb}},
    \\
    &\wh L^\wdac_i \rbr{\theta,\Bb} \dfeq 
    \iloc{\Bb}{i,1} \cdot \lossce\rbr{\theta;\rbr{\xb_i,\yb_i}} + 
    \iloc{\Bb}{i,2} \cdot \regdac\rbr{\theta;\xb_i,A_{i,1},A_{i,2}}. \nonumber
\end{align}

\begin{proposition}[Convergence (formal restatement of \Cref{prop:convergence_srw_dac})]\label{prop:convergence_srw_dac_formal}
    Assume that 
    $\lossce\rbr{\theta;\rbr{\xb,\yb}}$ and $\regdac\rbr{\theta;\xb,A_{1},A_{2}}$ are convex and continuous in $\theta$ for all $(\xb,\yb) \in \Xcal \times \Ycal$ and $A_1,A_2 \sim \Acal^2$, and that $\Fnb{\gamma} \subset \Fcal_\theta$ is convex and compact. 
    If there exist
    \begin{enumerate}[label=(\roman*)]
        \item $C_{\theta,*} > 0$ such that $\frac{1}{n} \sum_{i=1}^n \nbr{\nabla_\theta \wh L_i^\wdac\rbr{\theta,\Bb}}_{\Fcal,*}^2 \le C_{\theta,*}^2$ for all $\theta \in \Fnb{\gamma}$, $\Bb \in \Delta_2^n$ and 
        \item $C_{\Bb,*} > 0$ such that $\frac{1}{n} \sum_{i=1}^n \max\cbr{\lossce\rbr{\theta;\rbr{\xb_i,\yb_i}}, \regdac\rbr{\theta;\xb_i,A_{i,1},A_{i,2}} }^2 \le C_{\Bb,*}^2$ for all $\theta \in \Fnb{\gamma}$,
    \end{enumerate} 
    then with $\eta_\theta = \eta_{\Bb} = 2 \Big/ \sqrt{5T \rbr{\gamma^2 C_{\theta,*}^2 + 2 n C_{\Bb,*}^2}}$, \Cref{alg:srw_dac} provides the convergence guarantee for the duality gap $\Ecal\rbr{\overline{\theta}_T, \overline{\Bb}_T} \dfeq \max_{\Bb \in \Delta_2^n} \wh L^\wdac\rbr{\overline{\theta}_T, \Bb} - \min_{\theta \in \Fnb{\gamma}} \wh L^\wdac\rbr{\theta, \overline{\Bb}_T}$:
    \begin{align*}
        \E\sbr{\Ecal\rbr{\overline{\theta}_T, \overline{\Bb}_T} }
        \le 2 \sqrt{\frac{5 \rbr{\gamma^2 C_{\theta,*}^2 + 2 n C_{\Bb,*}^2}}{T}},
    \end{align*}
    where $\overline{\theta}_T = \frac{1}{T} \sum_{t=1}^T \theta_t$ and $\overline{\Bb}_T = \frac{1}{T} \sum_{t=1}^T \Bb_t$.
\end{proposition} 

\begin{proof}[Proof of \Cref{prop:convergence_srw_dac_formal}]
The proof is an application of the standard convergence guarantee for the online mirror descent on saddle point problems, as recapitulated in \Cref{lemma:mirror_descent_convergence_saddle_point}.  

Specifically, for $\Bb \in \Delta_2^n$, we use the norm $\nbr{\Bb}_{1,2} \dfeq \sqrt{\sum_{i=1}^n \rbr{\sum_{j=1}^2 \abbr{\iloc{\Bb}{i,j}}}^2}$ with its dual norm $\nbr{\Bb}_{1,2,*} \dfeq \sqrt{\sum_{i=1}^n \rbr{\max_{j \in [2]} \abbr{\iloc{\Bb}{i,j}}}^2}$. 
We consider a mirror map $\varphi_{\Bb}: [0,1]^{n \times 2} \to \R$ such that $\varphi_{\Bb}\rbr{\Bb} = \sum_{i=1}^n \sum_{j=1}^2 \iloc{\Bb}{i,j} \log \iloc{\Bb}{i,j}$. We observe that, since $\iloc{\Bb}{i,:}, \iloc{\Bb'}{i,:} \in \Delta_2$ for all $i \in [n]$,
\begin{align*}
    D_{\varphi_{\Bb}}\rbr{\Bb, \Bb'} 
    = \sum_{i=1}^n \sum_{j=1}^2 \iloc{\Bb}{i,j} \log \frac{\iloc{\Bb}{i,j}}{\iloc{\Bb'}{i,j}} 
    \ge \frac{1}{2} \sum_{i=1}^n \rbr{\sum_{j=1}^2 \abbr{\iloc{\Bb}{i,j} - \iloc{\Bb'}{i,j}}}^2
    = \frac{1}{2} \nbr{\Bb - \Bb'}_{1,2}^2,
\end{align*}
and therefore $\varphi_{\Bb}$ is $1$-strongly convex with respect to $\nbr{\cdot}_{1,2}$. With such $\varphi_\Bb$, we have the associated Fenchel dual $\varphi^*_\Bb\rbr{\Gb} = \sum_{i=1}^n \log \rbr{\sum_{j=1}^2 \exp\rbr{\iloc{\Gb}{i,j}}}$, along with the gradients
\begin{align*}
    \iloc{\nabla \varphi_\Bb\rbr{\Bb}}{i,j} = 1 + \log \iloc{\Bb}{i,j},
    \quad
    \iloc{\nabla \varphi^*_\Bb\rbr{\Gb}}{i,j} = \frac{\exp\rbr{\iloc{\Gb}{i,j}}}{\sum_{j=1}^2 \exp\rbr{\iloc{\Gb}{i,j}}},
\end{align*}
such that the mirror descent update on $\Bb$ is given by
\begin{align*}
    \iloc{\rbr{\Bb_{t+1}}}{i,j} 
    = &\nabla\varphi_\Bb^*\rbr{\nabla\varphi_\Bb\rbr{\Bb_t} - \eta_\Bb \cdot \nabla_\Bb\wh L^\wdac_{i_t}\rbr{\theta_t, \Bb_{t}}}
    \\
    = &\frac{\iloc{\rbr{\Bb_{t}}}{i,j} \exp\rbr{\eta_\Bb \cdot \iloc{\rbr{\nabla_\Bb \wh L^\wdac_{i_t}\rbr{\theta_t, \Bb_t}}}{i,j}}}{\sum_{j=1}^2 \iloc{\rbr{\Bb_{t}}}{i,j} \exp\rbr{\eta_\Bb \cdot \iloc{\rbr{\nabla_\Bb \wh L^\wdac_{i_t}\rbr{\theta_t, \Bb_t}}}{i,j}}}.
\end{align*}    
For $i_t \sim [n]$ uniformly, the stochastic gradient with respect to $\Bb$ satisfies
\begin{align*}
    &\E_{i_t \sim [n]}\sbr{\nbr{\nabla_\Bb \wh L^\wdac_{i_t}\rbr{\theta_t, \Bb_t}}_{1,2,*}^2}
    \\
    = &\frac{1}{n} \sum_{i_t=1}^n \max\cbr{ \lossce\rbr{\theta_t;\rbr{\xb_{i_t},\yb_{i_t}}}, \regdac\rbr{\theta_t;\xb_{i_t},A_{i_t,1},A_{i_t,2}} }^2
    \le C_{\Bb,*}^2.
\end{align*}
Further, in the distance induced by $\varphi_\Bb$, we have
\begin{align*}
    R_{\Delta_2^n}^2 
    \dfeq \max_{\Bb \in \Delta_2^n} \varphi_{\Bb}\rbr{\Bb} - \min_{\Bb \in \Delta_2^n} \varphi_{\Bb}\rbr{\Bb}
    = 0 - \sum_{i=1}^n \sum_{j=1}^2 \frac{1}{2} \log \frac{1}{2}
    = n.
\end{align*}

Meanwhile, for $\theta \in \Fnb{\gamma}$, we consider the norm $\nbr{\theta}_\Fcal \dfeq \sqrt{\abr{\theta,\theta}_\Fcal}$ induced by the inner product that characterizes $\Fcal_\theta$, with the associated dual norm $\nbr{\cdot}_{\Fcal,*}$. We use a mirror map $\varphi_\theta:\Fcal_\theta \to \R$ such that $\varphi_\theta\rbr{\theta} = \frac{1}{2}\nbr{\theta-\theta^*}^2_{\Fcal}$. By observing that
\begin{align*}
    D_{\varphi_\theta}\rbr{\theta,\theta'} = \frac{1}{2} \nbr{\theta-\theta'}_{\Fcal}^2 \quad \forall~ \theta,\theta' \in \Fcal.
\end{align*}
we have $\varphi_\theta$ being $1$-strongly convex with respect to $\nbr{\cdot}_{\Fcal}$.
With the gradient of $\varphi_\theta$, $\nabla\varphi_\theta(\theta) = \theta-\theta^*$, and that of its Fenchel dual $\nabla\varphi_\theta^*(g)=g+\theta^*$, at the $(t+1)$-th iteration, we have
\begin{align*}
    \theta_{t+1} 
    = \nabla\varphi_\theta^*\rbr{\nabla\varphi_\theta\rbr{\theta_t} - \eta_\theta \cdot \nabla_\theta\wh L^\wdac_{i_t}\rbr{\theta_t, \Bb_{t+1}}}
    = \theta_t - \eta_\theta \cdot \nabla_\theta\wh L^\wdac_{i_t}\rbr{\theta_t, \Bb_{t+1}}.
\end{align*}
For $i_t \sim [n]$ uniformly, the stochastic gradient with respect to $f$ satisfies that
\begin{align*}
    \E_{i_t \sim [n]}\sbr{\nbr{\nabla_\theta \wh L^\wdac_{i_t}\rbr{\theta_t, \Bb_{t+1}}}_{\Fcal,*}^2}
    = \frac{1}{n} \sum_{i_t=1}^n \nbr{\nabla_\theta \wh L_{i_t}^\wdac\rbr{\theta_t,\Bb_{t+1}}}_{\Fcal,*}^2 
    \le C_{\theta,*}^2.
\end{align*}
Further, in light of the definition of $\Fnb{\gamma}$, since $\theta^* \in \Fnb{\gamma}$, with $\theta^* = \argmin_{\theta \in \Fnb{\gamma}} \varphi_\theta(\theta)$ and $\theta' = \argmax_{\theta \in \Fnb{\gamma}} \varphi_\theta(\theta)$, we have 
\begin{align*}
    R_{\Fnb{\gamma}}^2 
    \dfeq \max_{\theta \in \Fnb{\gamma}} \varphi_\theta\rbr{\theta} - \min_{\theta \in \Fnb{\gamma}} \varphi_\theta\rbr{\theta}
    = \frac{1}{2}\nbr{\theta' - \theta^*}_{\Fcal}^2 
    \le \frac{\gamma^2}{2}.
\end{align*}

Finally, leveraging \Cref{lemma:mirror_descent_convergence_saddle_point} completes the proof.
\end{proof}

We recall the standard convergence guarantee for online mirror descent on saddle point problems. In general, we consider a stochastic function $F:\Ucal \times \Vcal \times \Ical \to \R$ with the randomness of $F\rbr{u,v;i}$ on $i \in \Ical$. 
Overloading notation $\Ical$ both as the distribution of $i$ and as the support, we are interested in solving the saddle point problem on the expectation function
\begin{align}\label{eq:apx_omd_saddle_point_problem}
    \min_{u \in \Ucal} \max_{v \in \Vcal} f\rbr{u,v}
    \quad \t{where} \quad
    f\rbr{u,v} \dfeq \E_{i \sim \Ical}\sbr{F\rbr{u,v;i}}.
\end{align}

\begin{assumption}\label{ass:apx_omd_objective}
Assume that the stochastic objective satisfies the following:
\begin{enumerate}[label=(\roman*)]
    \item For every $i \in \Ical$, $F\rbr{\cdot,v,i}$ is convex for all $v \in \Vcal$ and $F\rbr{u,\cdot,i}$ is concave for all $u \in \Ucal$.
    \item The stochastic subgradients $G_u\rbr{u,v;i} \in \partial_u F\rbr{u,v;i}$ and $G_v\rbr{u,v;i} \in \partial_v F\rbr{u,v;i}$ with respect to $u$ and $v$ evaluated at any $\rbr{u,v} \in \Ucal \times \Vcal$ provide unbiased estimators for some respective subgradients of the expectation function: for any $\rbr{u,v} \in \Ucal \times \Vcal$, there exist some $g_u\rbr{u,v} \dfeq \E_{i \sim \Ical}\sbr{G_u\rbr{u,v;i}} \in \partial_u f\rbr{u,v}$ and $g_v\rbr{u,v} \dfeq \E_{i \sim \Ical}\sbr{G_v\rbr{u,v;i}} \in \partial_v f\rbr{u,v}$.
    \item Let $\nbr{\cdot}_\Ucal$ and $\nbr{\cdot}_\Vcal$ be arbitrary norms that are well-defined on $\Ucal$ and $\Vcal$, while $\nbr{\cdot}_{\Ucal,*}$ and $\nbr{\cdot}_{\Vcal,*}$ be their respective dual norms. There exist constants $C_{u,*}, C_{v,*} > 0$ such that
    \begin{align*}
        \E_{i \sim \Ical}\sbr{\nbr{G_u\rbr{u,v;i}}_{\Ucal,*}^2} \le C_{u,*}^2
        \quad\t{and}\quad
        \E_{i \sim \Ical}\sbr{\nbr{G_v\rbr{u,v;i}}_{\Vcal,*}^2} \le C_{v,*}^2
        \quad\forall~ \rbr{u,v} \in \Ucal \times \Vcal.
    \end{align*}    
\end{enumerate} 
\end{assumption}   

For online mirror descent, we further introduce two mirror maps that induce distances on $\Ucal$ and $\Vcal$, respectively.
\begin{assumption}\label{ass:apx_omd_mirror_map}
Let $\varphi_u: \Dcal_u \to \R$ and $\varphi_v: \Dcal_v \to \R$ satisfy the following:
\begin{enumerate}[label=(\roman*)]
    \item $\Ucal \subseteq \Dcal_u \cup \partial \Dcal_u$, $\Ucal \cap \Dcal_u \neq \emptyset$ and $\Vcal \subseteq \Dcal_v \cup \partial \Dcal_v$, $\Vcal \cap \Dcal_v \neq \emptyset$.
    \item $\varphi_u$ is $\rho_u$-strongly convex with respect to $\nbr{\cdot}_\Ucal$; $\varphi_v$ is $\rho_v$-strongly convex with respect to $\nbr{\cdot}_\Vcal$.
    \item $\lim_{u \to \partial \Dcal_u} \nbr{\nabla \varphi_u(u)}_{\Ucal,*} = \lim_{v \to \partial \Dcal_v} \nbr{\nabla \varphi_v(v)}_{\Vcal,*} = +\infty$.
\end{enumerate}
\end{assumption}
Given the learning rates $\eta_u, \eta_v$, in each iteration $t=1,\dots,T$, the online mirror descent samples $i_t \sim \Ical$ and updates 
\begin{align}\label{eq:apx_omd_update}
    &v_{t+1} = \argmin_{v \in \Vcal}\ -\eta_v \cdot G_v\rbr{u_t,v_t;i_t}^\top v + D_{\varphi_v}\rbr{v, v_t},
    \nonumber \\
    &u_{t+1} = \argmin_{u \in \Ucal}\ \eta_u \cdot G_u\rbr{u_t,v_{t+1};i_t}^\top u + D_{\varphi_u}\rbr{u, u_t},
\end{align}
where $D_{\varphi}\rbr{w,w'} = \varphi(w)-\varphi(w')-\nabla\varphi(w')^\top(w-w')$ denotes the Bregman divergence.

We measure the convergence of the saddle point problem in the duality gap:
\begin{align*}
    \Ecal\rbr{\overline{u}_T, \overline{v}_T} \dfeq \max_{v \in \Vcal} f\rbr{\overline{u}_T, v} - \min_{u \in \Ucal} f\rbr{u, \overline{v}_T}
\end{align*}
such that, with
\begin{align*}
    R_\Ucal \dfeq \sqrt{\max_{u \in \Ucal \cap \Dcal_u} \varphi_u(u) - \min_{u \in \Ucal \cap \Dcal_u} \varphi_u(u)}
    \quad\t{and}\quad
    R_\Vcal \dfeq \sqrt{\max_{v \in \Vcal \cap \Dcal_v} \varphi_v(v) - \min_{v \in \Vcal \cap \Dcal_v} \varphi_v(v)},
\end{align*}
the online mirror descent converges as follows.
\begin{lemma}[\citep{nemirovski2009omd} (3.11)]\label{lemma:mirror_descent_convergence_saddle_point}
Under \Cref{ass:apx_omd_objective} and \Cref{ass:apx_omd_mirror_map}, when taking constant learning rates $\eta_u = \eta_v = 2 \Big/ \sqrt{5T \rbr{\frac{2 R_\Ucal^2}{\rho_u} C_{u,*}^2 + \frac{2 R_\Vcal^2}{\rho_v} C_{v,*}^2}}$, with $\overline{u}_T = \frac{1}{T} \sum_{t=1}^T u_t$ and $\overline{v}_T = \frac{1}{T} \sum_{t=1}^T v_t$,
\begin{align*}
    \E\sbr{\Ecal\rbr{\overline{u}_T, \overline{v}_T} }
    \le 2 \sqrt{\frac{10 \rbr{\rho_v R_\Ucal^2 C_{u,*}^2 + \rho_u R_\Vcal^2 C_{v,*}^2}}{\rho_u \rho_v \cdot T}}.
\end{align*}
\end{lemma}

\begin{example}[Binary linear pixel-wise classifiers with convex and continuous objectives]\label{ex:convex_continuous}
    We consider a pixel-wise binary classification problem with $\Xcal=[0,1]^d$, augmentations $A:\Xcal \to \Xcal$ for all $A \sim \Acal$, and a class of linear ``UNets'',
    \begin{align*}
        \Fcal = \csepp{f_\theta:\Xcal \to [0,1]^d}{f_\theta\rbr{\xb} = \sigma\rbr{\thetab_d \thetab_e^\top \xb} = \psi_\theta\rbr{\phi_\theta\rbr{\xb}},~ \phi_\theta\rbr{\xb} = \frac{1}{\sqrt{d}}\thetab_e^\top \xb},
    \end{align*}
    where the parameter space $\theta = \rbr{\thetab_e, \thetab_d} \in \Fcal_\theta = \SSS^{d-1} \times \SSS^{d-1}$ is equipped with the $\ell_2$ norm $\nbr{\theta}_{\Fcal} = \rbr{\nbr{\thetab_e}_2^2 + \nbr{\thetab_d}_2^2}^{1/2}$; 
    $\sigma:\R^d \to [0,1]^d$ denotes entry-wise application of the sigmoid function $\sigma(z) = (1+e^{-z})^{-1}$; and 
    the latent space of encoder outputs $\rbr{\Zcal,\varrho}$ is simply the real line. 
    Given the data distribution $P_\xi$, we recall that $\theta^* = \argmin_{\theta \in \Fcal_\theta} \E_{\rbr{\xb,\yb}\sim P_\xi}\sbr{\lossce\rbr{\theta;(\xb,\yb)} }$ for all $\xi \in [0,1]$ and let $\Fnb{\gamma} = \csepp{\theta \in \Fcal_\theta}{\nbr{\theta - \theta^*}_{\Fcal} \le \gamma}$ for some $\gamma = O\rbr{1/\sqrt{d}}$.
    We assume that $\abbr{\xb^\top \thetab_e^*} = O(1)$ for all $\xb \in \Xcal$. 
    Then, $\lossce\rbr{\theta;\rbr{\xb,\yb}}$ and $\regdac\rbr{\theta;\xb,A_1,A_2}$ are convex and continuous in $\theta$ for all $(\xb,\yb) \in \Xcal \times [K]^d$, $A_1,A_2 \sim \Acal^2$; while $C_{\theta,*} \le \max\rbr{2\sqrt{2}, 2\lambdac}$ and $C_{\betab,*} \le \max\rbr{O(1), 2\lambdac}$.
\end{example}

\begin{proof}[Rationale for \Cref{ex:convex_continuous}]
    Let $\yb_k = \iffun{\yb=k}$ entry-wise for $k=0,1$. We would like to show that, for any given $(\xb,\yb) \in \Xcal \times [K]^d$, $A_1,A_2 \sim \Acal^2$,
    \begin{align*}
        &\lossce\rbr{\theta} 
        = -\frac{1}{d} \rbr{\yb_1^\top \log \sigma\rbr{\thetab_d \thetab_e^\top \xb} + \yb_0^\top \log \sigma\rbr{-\thetab_d \thetab_e^\top \xb}},
        \\
        &\regdac\rbr{\theta} = \frac{\lambdac}{\sqrt{d}} \cdot \rbr{A_1(\xb)-A_2(\xb)}^\top \thetab_e
    \end{align*}
    are convex and continuous in $\theta = \rbr{\thetab_e,\thetab_d}$.

    First, we observe that $\regdac\rbr{\theta}$ is linear (and therefore convex and continuous) in $\theta$ for all $\xb \in \Xcal$, $A_1,A_2 \sim \Acal^2$, with
    \begin{align*}
        \nabla_{\thetab_e} \regdac\rbr{\theta} = \frac{\lambdac}{\sqrt{d}} \cdot \rbr{A_1(\xb)-A_2(\xb)},
        \quad
        \nabla_{\thetab_d} \regdac\rbr{\theta} = \b{0}
    \end{align*}
    such that $\nbr{\nabla_{\theta} \regdac\rbr{\theta}}_{\Fcal,*} \le 2 \lambdac$.
    
    Meanwhile, with $\zb\rbr{\theta} = \thetab_d \thetab_e^\top \xb$, we have $\lossce\rbr{\theta} = -\frac{1}{d} \rbr{\yb_1^\top \log \sigma\rbr{\zb\rbr{\theta}} + \yb_0^\top \log \sigma\rbr{-\zb\rbr{\theta}}}$ being convex and continuous in $\zb\rbr{\theta}$:
    \begin{align*}
        \nabla_{\zb}^2 \lossce\rbr{\theta} = \frac{1}{d} \diag\rbr{\sigma\rbr
        {\zb\rbr{\theta}}} \diag\rbr{1-\sigma\rbr
        {\zb\rbr{\theta}}} \ageq 0.
    \end{align*}
    Therefore, $\lossce\rbr{\theta}$ is convex and continuous in $\theta$ for all $\rbr{\xb,\yb} \in \Xcal \times [K]^d$:
    \begin{align*}
        \underbrace{\nabla_{\theta}^2 \lossce\rbr{\theta}}_{2d \times 2d} = 
        \bmat{\xb \thetab_d^\top \\ \rbr{\thetab_e^\top \xb} \Ib_{d}}
        \rbr{\frac{1}{d} \diag\rbr{\sigma\rbr
        {\zb\rbr{\theta}}} \diag\rbr{1-\sigma\rbr
        {\zb\rbr{\theta}}}}
        \bmat{\xb \thetab_d^\top & \rbr{\thetab_e^\top \xb} \Ib_{d}}
        \ageq 0,
    \end{align*}
    where $\Ib_d$ denotes the $d \times d$ identity matrix.
    Further, from the derivation, we have
    \begin{align*}
        &\nabla_{\thetab_e} \lossce\rbr{\theta} = \frac{1}{d} \thetab_d^\top \rbr{\sigma\rbr{\thetab_d \thetab_e^\top \xb}-\yb} \xb,
        \quad
        \nabla_{\thetab_d} \lossce\rbr{\theta} = \frac{\thetab_e^\top \xb}{d} \rbr{\sigma\rbr{\thetab_d \thetab_e^\top \xb}-\yb}
    \end{align*}
    such that $\nbr{\nabla_{\theta} \lossce\rbr{\theta}}_{\Fcal,*} = \sqrt{\nbr{\nabla_{\thetab_e} \lossce\rbr{\theta}}_2^2 + \nbr{\nabla_{\thetab_d} \lossce\rbr{\theta}}_2^2} \le 2 \sqrt{2}$.

    Finally, knowing $\nbr{\nabla_{\theta} \lossce\rbr{\theta}}_{\Fcal,*}  \le 2 \sqrt{2}$ and $\nbr{\nabla_{\theta} \regdac\rbr{\theta}}_{\Fcal,*} \le 2 \lambdac$, we have
    \begin{align*}
        \nbr{\nabla_\theta \wh L_i^\wdac\rbr{\theta,\betab}}_{\Fcal,*} 
        \le \iloc{\betab}{i} \nbr{\nabla_{\theta} \lossce\rbr{\theta}}_{\Fcal,*} + (1-\iloc{\betab}{i}) \nbr{\nabla_{\theta} \regdac\rbr{\theta}}_{\Fcal,*} \le \max\rbr{2\sqrt{2}, 2\lambdac}
    \end{align*}
    for all $i \in [n]$, and therefore,
    \begin{align*}
        C_{\theta,*} \le \max\rbr{2\sqrt{2}, 2\lambdac}.
    \end{align*}
    Besides, with
    \begin{align*}
        \regdac\rbr{\theta} \le \frac{\lambdac}{\sqrt{d}} \nbr{A_1(\xb)-A_2(\xb)}_2 \nbr{\thetab_e}_2 \le 2 \lambdac,
    \end{align*} 
    and since 
    \begin{align*}
        \iloc{\rbr{\thetab_d \thetab_e^\top \xb}}{j} 
        \le &\abbr{\xb^\top \thetab_e} \le \abbr{\xb^\top \rbr{\thetab_e - \thetab_e^*}} + \abbr{\xb^\top \thetab_e^*} \le \nbr{\xb}_2 \nbr{\thetab_e - \thetab_e^*}_2 + O(1)
        \\
        \le & \gamma \sqrt{d} + O(1) = O(1)
    \end{align*}
    for all $j \in [d]$, $\lossce\rbr{\theta} \le \log\rbr{1+e^{O(1)}} = O(1)$, we have 
    \begin{align*}
        C_{\betab,*} \le \max\rbr{O(1), 2\lambdac}.
    \end{align*}
\end{proof}

\section{Dice Loss for Pixel-wise Class Imbalance}\label{apx:dice}

With finite samples in practice, since the averaged cross-entropy loss (\Cref{eq:cross_entropy_loss}) weights each pixel in the image label equally, the pixel-wise class imbalance can become a problem. For example, the background pixels can be dominant in most of the segmentation labels, making the classifier prone to predict pixels as background. 

To cope with such vulnerability, \cite{chen2021transunet,cao2021swin,wong2018segmentation,taghanaki2019combo,yeung2022unified} propose to combine the cross-entropy loss with the \emph{dice loss}---a popular segmentation loss based on the overlap between true labels and their corresponding predictions in each class:
\begin{align}\label{eq:dice_loss}
    \lossdice\rbr{\theta;\rbr{\xb,\yb}} = 1 - \frac{1}{K} \sum_{k=1}^K \dsc\rbr{\iloc{f_\theta\rbr{\xb}}{:,k}, \iffun{\yb=k}},
\end{align}
where for any $\pb \in [0,1]^d$, $\qb \in \cbr{0,1}^d$, $\dsc\rbr{\pb,\qb} = \frac{2 \pb^\top \qb}{\nbr{\pb}_1 + \nbr{\qb}_1} \in [0,1]$ denotes the dice coefficient~\citep{milletari2016vnet,taghanaki2019deep}. Notice that by measuring the bounded dice coefficient for each of the $K$ classes individually, the dice loss tends to be robust to class imbalance. 

\cite{taghanaki2019combo} merges both dice and averaged cross-entropy losses via a convex combination. It is also a common practice to add a smoothing term in both the nominator and denominator of the DSC~\citep{russell2016artificial}.

Combining the dice loss (\Cref{eq:dice_loss}) with the weighted augmentation consistency regularization formulation (\Cref{eq:objective_dac_encoder_weighted_adawac}), in practice, we solve
\begin{align}\label{eq:objective_wac_dice}
    &\wh\theta^\wdac, \wh\betab \in 
    \argmin_{\theta \in \Fnb{\gamma}}~
    \argmax_{\betab \in [0,1]^n}~ 
    \cbr{\wh L^\wdac\rbr{\theta,\betab} \dfeq \frac{1}{n} \sum_{i=1}^n \wh L^\wdac_i\rbr{\theta,\betab}}
    \\
    &\wh L^\wdac_i \rbr{\theta,\betab} \dfeq 
    \lossdice\rbr{\theta;\rbr{\xb_i,\yb_i}} + 
    \iloc{\betab}{i} \cdot \lossce\rbr{\theta;\rbr{\xb_i,\yb_i}} + 
    (1-\iloc{\betab}{i}) \cdot \regdac\rbr{\theta;\xb_i,A_{i,1},A_{i,2}} \nonumber
\end{align}
with a slight modification in \Cref{alg:srw_dac} line 9:
\begin{align*}
    \theta_t \gets \theta_{t-1} - \eta_\theta \cdot \Big(
    \nabla_\theta\lossdice\rbr{\theta_{t-1};\rbr{\xb_{i_t},\yb_{i_t}}} + 
    \iloc{\rbr{\betab_t}}{i_t} \cdot \nabla_\theta\lossce\rbr{\theta_{t-1};\rbr{\xb_{i_t},\yb_{i_t}}}&
    \\
    + \rbr{1-\iloc{\rbr{\betab_t}}{i_t}} \cdot \nabla_\theta\regdac\rbr{\theta_{t-1};\xb_{i_t},A_{i_t,1},A_{i_t,2}}& \Big).
\end{align*}

\paragraph{On the influence of incorporating dice loss in experiments.} 
We note that, in the experiments, the dice loss $\lossdice$ is treated independently of \ours in \Cref{alg:srw_dac} via standard stochastic gradient descent. In particular for the comparison with hard-thresholding algorithms in \Cref{tab:trim}, we keep the updating on $\lossdice$ of the original untrimmed batch intact for both \b{trim-train} and \b{trim-ratio} to exclude the potential effect of $\lossdice$ that is not involved in reweighting.

\section{Implementation Details and Datasets}\label{app:imp}
We follow the official implementation of TransUNet\footnote{\href{https://github.com/Beckschen/TransUNet}{https://github.com/Beckschen/TransUNet}} for model training. We use the same optimizer (SGD with learning rate $0.01$, momentum $0.9$, and weight decay 1e-4). 
For the Synapse dataset, we train TransUNet for 150 epochs on the training dataset and evaluate the last-iteration model on the test dataset. For the ACDC dataset, we train TransUNet for 360 epochs in total, while validating models on the ACDC validation dataset for every 10 epochs and testing on the best model selected by the validation. 
The total number of training iterations (\ie, total number of batches) is set to be the same as that in the vanilla TransUNet~\citep{chen2021transunet} experiments. In particular, the results in \Cref{tab:synapse_sample_eff} are averages (and standard deviations) over $3$ arbitrary random seeds. The results in \Cref{tab:trim}, \Cref{tab:acdc}, and \Cref{tab:ablation} are given by the original random seed used in the TransUNet experiments.

\paragraph{Synapse multi-organ segmentation dataset (Synapse).}The Synapse dataset\footnote{See detailed description at \href{https://www.synapse.org/\#!Synapse:syn3193805/wiki/217789}{{https://www.synapse.org/\#!Synapse:syn3193805/wiki/217789}}} is multi-organ abdominal CT scans for medical image segmentation in the MICCAI 2015 Multi-Atlas Abdomen Labelling Challenge~\citep{chen2021transunet}. There are 30 cases of CT scans with variable sizes $(512 \times 512 \times 85 - 512 \times 512 \times 198)$, and slice thickness ranges from $2.5$mm to $5.0$mm. We use the pre-processed data provided by \cite{chen2021transunet} and follow their train/test split to use 18 cases for training and 12 cases for testing on 8 abdominal organs---aorta, gallbladder, left kidney (L), right kidney (R), liver, pancreas, spleen, and stomach. The abdominal organs were labeled by experience undergraduates and verified by a radiologist using MIPAV software according to the information from Synapse wiki page.

\paragraph{Automated cardiac diagnosis challenge dataset (ACDC).}The ACDC dataset\footnote{See detailed description at \href{https://www.creatis.insa-lyon.fr/Challenge/acdc/}{https://www.creatis.insa-lyon.fr/Challenge/acdc/}} is cine-MRI scans in the MICCAI 2017 Automated Cardiac Diagnosis Challenge. There are $200$ scans from $100$ patients, and each patient has two frames with slice thickness from $5$mm to $8$mm. We use the pre-processed data also provided by \cite{chen2021transunet} and follow their train/validate/test split to use $70$ patients' scans for training, $10$ patients' scans for validation, and 20 patients' scans for testing on three cardiac structures---left ventricle (LV), myocardium (MYO), and right ventricle (RV). The data were labeled by one clinical expert according to the description on ACDC dataset website.

\section{Additional Experimental Results}\label{apx:additional_experiment}

\subsection{Sample Efficiency and Robustness of \ours with UNet}\label{subsec:exp_unet}

In addition to the empirical evidence on TransUNet presented in \Cref{tab:synapse_sample_eff}, here, we demonstrate that the sample efficiency and distributional robustness of \ours extend to the more widely used UNet architecture. 
In \Cref{tab:synapse_sample_eff_unet}, analogous to \Cref{tab:synapse_sample_eff}, the experiments on the \b{full} and \b{half-slice} datasets provide evidence for the \emph{sample efficiency} of \ours compared to the baseline (ERM+SGD) on UNet. Meanwhile, the \emph{distributional robustness} of \ours with UNet is well illustrated by the \b{half-vol} and \b{half-sparse} experiments.

\begin{table}[ht]
    \caption{\ours with UNet trained on the full Synapse and its subsets
    }
    \label{tab:synapse_sample_eff_unet}
    \centering
    \begin{adjustbox}{width=1\textwidth}  
    \begin{tabular}{ll|cc|cccccccc}
    \toprule
    Training & Method &  DSC $\uparrow$ & HD95 $\downarrow$ & Aorta & Gallbladder & Kidney (L) & Kidney (R) & Liver & Pancreas & Spleen & Stomach 
    \\
    \midrule
    \multirow{2}{*}{full} 
    & baseline & 74.04 $\pm$ 1.52 & 36.65 $\pm$ 0.33 & 84.93 & 55.59 & 77.59 & 70.92 & 92.21 & 55.01 & 82.87 & 73.21
    \\
    & \ours & \best{76.71 $\pm$ 0.62} & \best{30.67 $\pm$ 2.85}
    & 85.68 & 55.19 & 80.15 & 75.45 & 94.11 & 56.19 & 87.54 & 81.39
    \\
    \midrule
    \multirow{2}{*}{half-slice} 
    & baseline & 73.09 $\pm$ 0.10 & 40.05 $\pm$ 4.99 & 83.23 & 53.18 & 74.69 & 71.51 & 92.74 & 52.81 & 83.85 & 72.71 
    \\
    & \ours & \best{75.12 $\pm$ 0.78} & \best{29.26 $\pm$ 2.16} & 85.15 & 55.77 & 79.29 & 72.47 & 93.71 & 54.93 & 86.09 & 73.53
    \\
    \midrule
    \multirow{2}{*}{half-vol} 
    & baseline & 63.21 $\pm$ 2.53 & 64.20 $\pm$ 4.46 & 79.46 & 45.79 & 55.79 & 54.91 & 88.65 & 41.61 & 71.68 & 67.77
    \\
    & \ours & \best{71.09 $\pm$ 1.14} & \best{39.95 $\pm$ 7.76} & 83.15 & 49.14 & 75.74 & 70.33 & 90.47 & 44.81 & 82.34 & 72.75
    \\
    \midrule
    \multirow{2}{*}{half-sparse} 
    & baseline & 37.30 $\pm$ 1.32 & 69.67 $\pm$ 2.89 & 61.57 & 8.33 & 57.45 & 50.44 & 60.28 & 23.51 & 17.83 & 18.99
    \\
    & \ours & \best{44.85 $\pm$ 1.03} & \best{62.40 $\pm$ 5.17} & 71.56 & 8.40 & 65.42 & 62.73 & 74.02 & 24.16 & 36.65 & 15.88
    \\
    \bottomrule
    \end{tabular}
    \end{adjustbox}
\end{table}

\paragraph{Implementation details of UNet experiments.}
For the backbone architecture of experiments in \Cref{tab:synapse_sample_eff_unet}, we use a UNet with a ResNet-34 encoder initialized with ImageNet pre-trained weights. We leverage the implementation of UNet and load the pre-trained model via the PyTorch API for segmentation models~\citep{Iakubovskii2019}. For training, we use the same optimizer (SGD with learning rate $0.01$, momentum $0.9$, and weight decay 1e-4) and the total number of epochs (150 epochs on Synapse training set) as the TransUNet experiments, evaluating the last-iteration model on the test dataset. As before, the results in \Cref{tab:synapse_sample_eff_unet} are averages (and standard deviations) over $3$ arbitrary random seeds.

\subsection{Visualization of Segmentation on ACDC dataset} 
As shown in Figure \ref{fig:acdc}, the model trained by \ours segments cardiac structures with more accurate shapes (column 1), identifies organs missed by baseline TransUNet (column 2-3) and circumvents the false-positive pixel classifications (\ie, fake predictions of background pixels as organs) suffered by the TransUNet baseline (column 4-6).

\begin{figure}[ht]
    \centering
    \includegraphics[width=.7\linewidth]{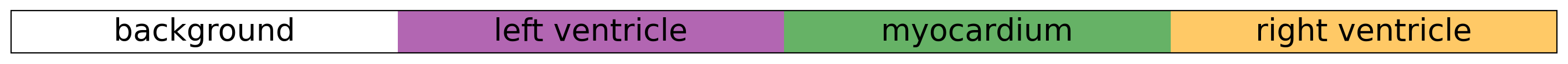}
    \includegraphics[width=.75\linewidth]{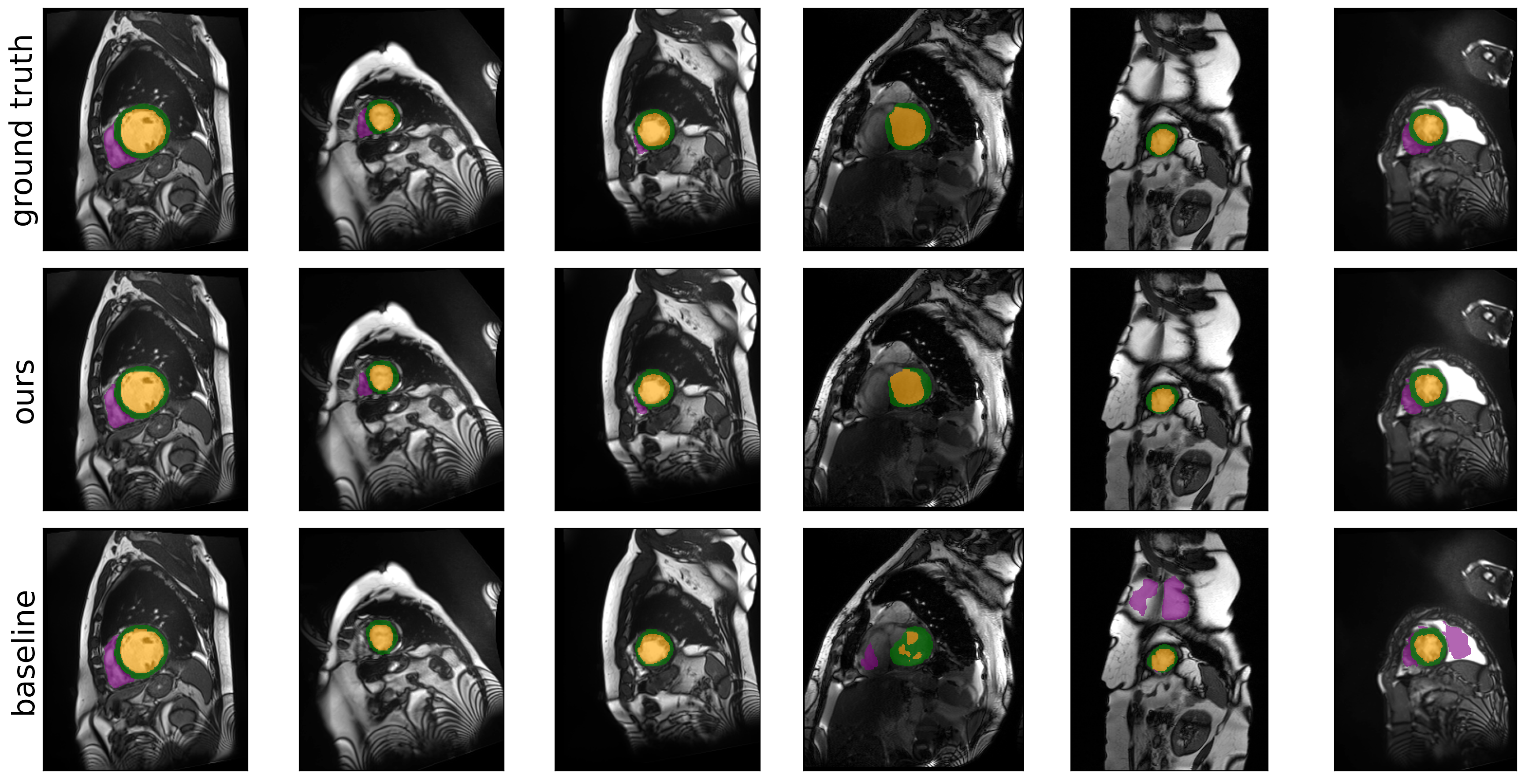}
    \caption{Visualization of segmentation results on ACDC dataset. From top to bottom: ground truth, ours, and baseline method.}
    \label{fig:acdc}
\end{figure}

\subsection{Visualization of Segmentation on Synapse with Distributional Shift} 
\Cref{fig:vis_synapse_half_slice_sparse} visualizes the segmentation predictions on $6$ Synapse test slices made by models trained via \ours(ours) and via the baseline (ERM+SGD) with TransUNet~\citep{chen2021transunet} on the \b{half-sparse} subset of the Synapse training set. We observe that, although the segmentation performances of both the baseline and \ours are compromised by the extreme scarcity of label-dense samples and the severe distributional shift, \ours provides more accurate predictions on the relative positions of organs, as well as less misclassification of organs (\eg, the baseline tends to misclassify other organs and the background as the left kidney). Nevertheless, due to the scarcity of labels, both the model trained with \ours and that trained with the baseline fail to make good predictions on the segmentation boundaries.

\begin{figure}[ht]
    \centering
    \includegraphics[width=.7\linewidth]{fig_adawac/synpase_legend.png}
    \includegraphics[width=.7\linewidth]{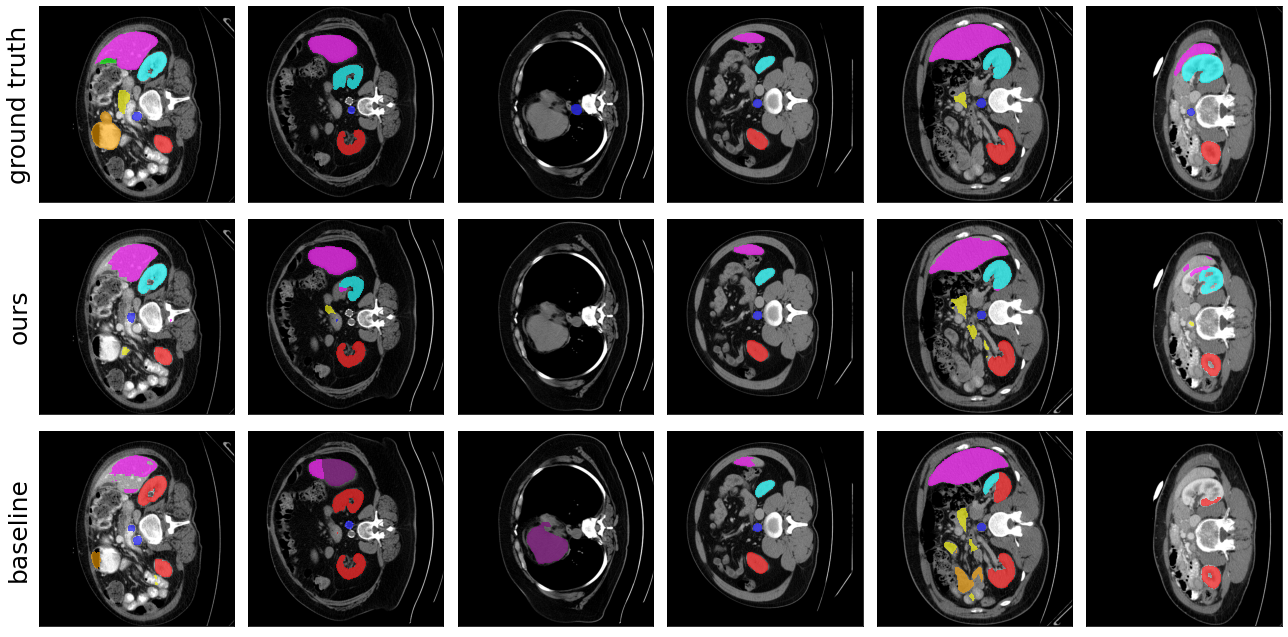}
    \caption{Visualization of segmentation predictions made by models trained via \ours(ours) and via the baseline (ERM+SGD) with TransUNet~\citep{chen2021transunet} on the \b{half-sparse} subset of the Synapse training set. Top to bottom: ground truth, ours (\ours), baseline.}
    \label{fig:vis_synapse_half_slice_sparse}
\end{figure}

\subsection{Experimental Results on Previous Metrics}
In this section, we include the results of experiments on Synapse\footnote{Note that the numbers of correct metrics and metrics in TransUNet~\citep{chen2021transunet} on ACDC dataset are the same.} dataset with metrics defined in TransUNet~\citep{chen2021transunet} for reference. In TransUNet~\citep{chen2021transunet}, DSC is 1 when the sum of ground truth labels is zero (i.e., gt.sum() == 0) while the sum of predicted labels is nonzero (i.e., pred.sum() $>$ 0). However, according to the definition of dice scores, $DSC=2|A \cap B| / (|A|+|B|), \forall A, B$, the DSC for the above case should be 0 since the intersection is 0 and the denominator is non-zero. In our evaluation, we change the special condition for DSC as 1 to pred.sum == 0 and gt.sum() == 0 instead, in which case the denominator is 0.

\begin{table}[ht]
    \caption{\ours with TransUNet trained on the full Synapse and its subsets, measured by metrics in TransUNet~\citep{chen2021transunet}.}
    \centering
    \begin{adjustbox}{width=1\textwidth}  
    \begin{tabular}{ll|cc|cccccccc}
    \toprule
    Training & Method &  DSC $\uparrow$ & HD95 $\downarrow$ & Aorta & Gallbladder & Kidney (L) & Kidney (R) & Liver & Pancreas & Spleen & Stomach 
    \\
    \midrule
    \multirow{2}{*}{full} & baseline &	77.32 &	29.23 &	87.46 &	63.54 &	82.06 &	77.76 &	94.10 &	54.06 &	85.07 &	74.54 
    \\
    & \ours &	80.16 &	25.79 &	87.23 &	63.27 &	84.58 &	81.69 &	94.62 &	58.29 &	90.63 &	81.01 
    \\
    \midrule
    \multirow{2}{*}{half-slice} & baseline & 76.24 &	24.66 &	86.26 &	57.61 &	79.32 &	76.55 &	94.34 &	54.04 &	86.20 &	75.57 
    \\
    & \ours &	78.14 &	29.75 &	86.66 &	62.28 &	81.36 &	78.84 &	94.60 &	57.95 &	85.38 &	78.01 
    \\
    \midrule
    \multirow{2}{*}{half-vol} & baseline &	72.65 &	35.86 &	83.29 &	43.70 &	78.25 &	77.25 &	92.92 &	51.32 &	83.80 &	70.66 
    \\
    & \ours &	75.93 &	34.95 &	84.45 &	60.40 &	79.59 &	76.06 &	93.19 &	54.46 &	84.91 &	74.37 
    \\
    \midrule
    \multirow{2}{*}{half-sparse} & baseline &	0.00 &	0.00 &	0.00 &	0.00 &	0.00 &	0.00 &	0.00 &	0.00 &	0.00 &	0.00 \\
    & \ours &	39.68 &	80.93 &	76.59 &	0.00 &	66.53 &	62.11 &	49.69 &	31.09 &	12.30 &	19.11 \\
    \bottomrule
    \end{tabular}
    \end{adjustbox}
    \label{tab:synapse_sample_eff_prev}
\end{table}

\begin{table}[ht]
    \centering
        \caption{AdaWAC versus hard-thresholding algorithms with TransUNet on Synapse, measured by metrics in TransUNet~\citep{chen2021transunet}.}
    \begin{adjustbox}{width=1\textwidth}  
    \begin{tabular}{l|cc|cccccccc}
    \toprule
    Method &  DSC $\uparrow$ & HD95$\downarrow$ & Aorta & Gallbladder & Kidney (L) & Kidney (R) & Liver & Pancreas & Spleen & Stomach \\
    \midrule
    baseline &	77.32 &	29.23 &	87.46 &	63.54 &	82.06 &	77.76 &	94.10 &	54.06 &	85.07 &	74.54 \\
    trim-train & 77.05 & 26.94 &  86.70 &	60.65 &	80.02 &	76.64 &	94.25 &	54.20 &	86.44 &	77.49 \\
    trim-ratio &	75.30 &	28.59 &	87.35 &	57.29 &	78.70 &	72.22 &	94.18 &	52.32 &	86.31 &	74.03 \\
    \midrule
    trim-train+ACR &	76.70 &	35.06 &	87.11 &	62.22 &	74.19 &	75.25 &	92.19 &	57.16 &	88.21 &	77.30 \\
    trim-ratio+ACR &	79.02 &	33.59 &	86.82 &	61.67 &	83.52 &	81.22 &	94.07 &	59.06 &	88.08 &	77.71 \\
    \ours (ours) &	80.16 &	25.79 &	87.23 &	63.27 &	84.58 &	81.69 &	94.62 &	58.29 &	90.63 &	81.01 \\
    \bottomrule
    \end{tabular}
    \end{adjustbox}
    \label{tab:trim_prev}
\end{table}

\begin{table}[ht]
    \caption{Ablation study of AdaWAC with TransUNet trained on Synapse, measured by metrics in TransUNet~\citep{chen2021transunet}.}
    \centering
    \begin{adjustbox}{width=1\textwidth}  
    \begin{tabular}{l|cc|cccccccc}
    \toprule
    Method &  DSC $\uparrow$ & HD95$\downarrow$ & Aorta & Gallbladder & Kidney (L) & Kidney (R) & Liver & Pancreas & Spleen & Stomach \\
    \midrule
    baseline &	77.32 &	29.23 &	87.46 &	63.54 &	82.06 &	77.76 &	94.10 &	54.06 &	85.07 &	74.54 \\
    reweight-only & 77.72 &	29.24 &	86.15 &	62.31 &	82.96 &	80.28 &	93.42 &	55.86 &	85.29 &	75.49 \\
    ACR-only &	78.93 &	31.65 &	87.96 &	62.67 &	81.79 &	80.21 &	94.52 &	60.41 &	88.07 &	75.83 \\
    \ours-0.01 &	78.98 &	27.81 &	87.58 &	61.09 &	82.29 &	80.22 &	94.90 &	55.92 &	91.63 &	78.23 \\
    \ours-1.0 &	80.16 &	25.79 &	87.23 &	63.27 &	84.58 &	81.69 &	94.62 &	58.29 &	90.63 &	81.01 \\
    \bottomrule
    \end{tabular}
    \end{adjustbox}
    \label{tab:synapse_prev}
\end{table}


\cleardoublepage
\addcontentsline{toc}{chapter}{Bibliography}
\bibliographystyle{plain}
\bibliography{ref}

\printindex

\begin{vita} 
\noindent
Yijun Dong was born in Shanghai, China in October 1995 and grew up there before college. She obtained a Bachelor of Science in Applied Mathematics and Engineering Science from Emory University in May 2018 where she got her first taste of research in biophysics and soft matter physics. In August 2018, she began her doctoral study in Computational Science, Engineering, and Mathematics at the Oden Institute of the University of Texas at Austin. Her primary research interests are randomized numerical linear algebra and statistical learning theory. In 2021 and 2022, she spent two summers at Dell Technologies working on edge computing and semi-supervised tabular learning. Upon completing her graduate study in May 2023, she expects to continue her postdoctoral research as a Courant Instructor at the Courant Institute of New York University.
\end{vita}

\end{document}